\documentclass[11pt,a4paper,onecolumn,oneside]{report}
\usepackage{tgtermes}
\usepackage{mathptmx}
\usepackage[T1]{fontenc}
\usepackage[utf8]{inputenc}

\RequirePackage[top=3cm, bottom=1in, left=1in, right=1in]{geometry}
\usepackage{titlesec}
\usepackage{amsmath}
\usepackage{amssymb}
\usepackage{mathtools}
\usepackage{amsthm}
\usepackage{enumerate}
\usepackage{bbm}
\usepackage{algorithm}
\usepackage{algorithmic}
\usepackage{epsfig}
\usepackage{color}
\usepackage{graphicx}
\usepackage{caption}
\usepackage{subcaption}
\usepackage{cases}
\usepackage{url}
\usepackage{cite}
\usepackage{float}
\usepackage{fancyhdr}
\usepackage{tocloft}
\usepackage{pdfpages}
\usepackage{sectsty,lmodern}
\usepackage{fontsize}
\usepackage{booktabs}
\usepackage{multirow}
\usepackage{rotating}
\usepackage[hidelinks]{hyperref}
\usepackage[table,xcdraw,dvipsnames]{xcolor}

\renewcommand{\thefigure}{\arabic{figure}}

\theoremstyle{plain}
\newtheorem{theorem}{Theorem}

\newtheorem{lemma}{Lemma}
\newtheorem{corollary}{Corollary}

\theoremstyle{definition}
\newtheorem{definition}{Definition}

\newtheorem{remark}{Remark}

\DeclareMathOperator*{\argmin}{arg\,min}

\begin{document}

\begin{center}
\vspace{3cm}
\LARGE Doctoral Thesis

\vspace{3cm}
\huge Algorithms for Collaborative Machine Learning under Statistical Heterogeneity
\vfill

\LARGE Seok-Ju Hahn
\vspace{2cm}

\LARGE Department of Industrial Engineering
\vspace{2cm}

\LARGE Ulsan National Institute of Science and Technology
\vspace{2cm}

\LARGE 2024
\vspace{3cm}

\end{center}
\thispagestyle{empty}
\clearpage

\begin{center}
\hbox{ }
\hbox{ }
\huge Algorithms for Collaborative Machine Learning under Statistical Heterogeneity
\vspace{5cm}

\LARGE Seok-Ju Hahn
\vspace{6cm}

\LARGE Department of Industrial Engineering
\vspace{2cm}

\LARGE Ulsan National Institute of Science and Technology

\end{center}
\thispagestyle{empty}
\clearpage

\includepdf[fitpaper=true, pages=-]{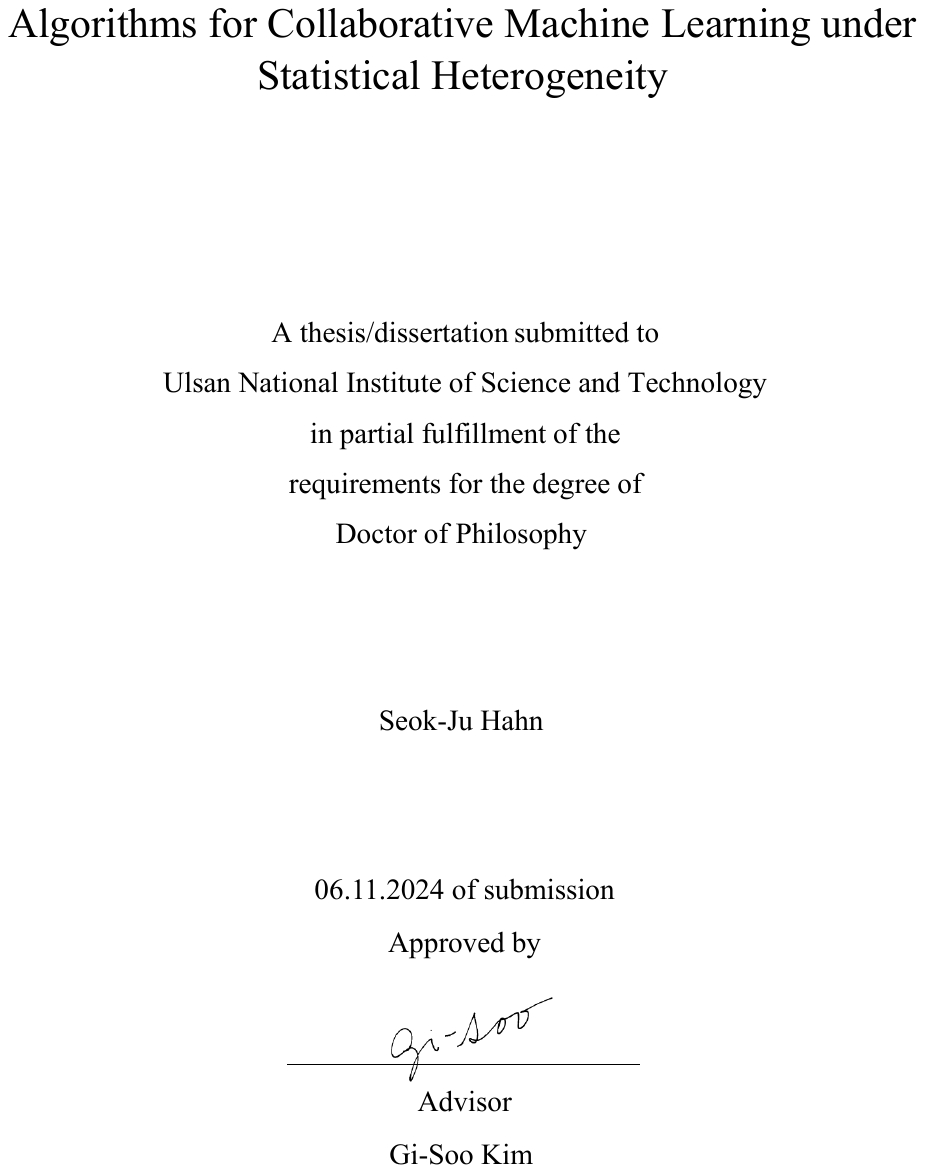}

\includepdf[fitpaper=true, pages=-]{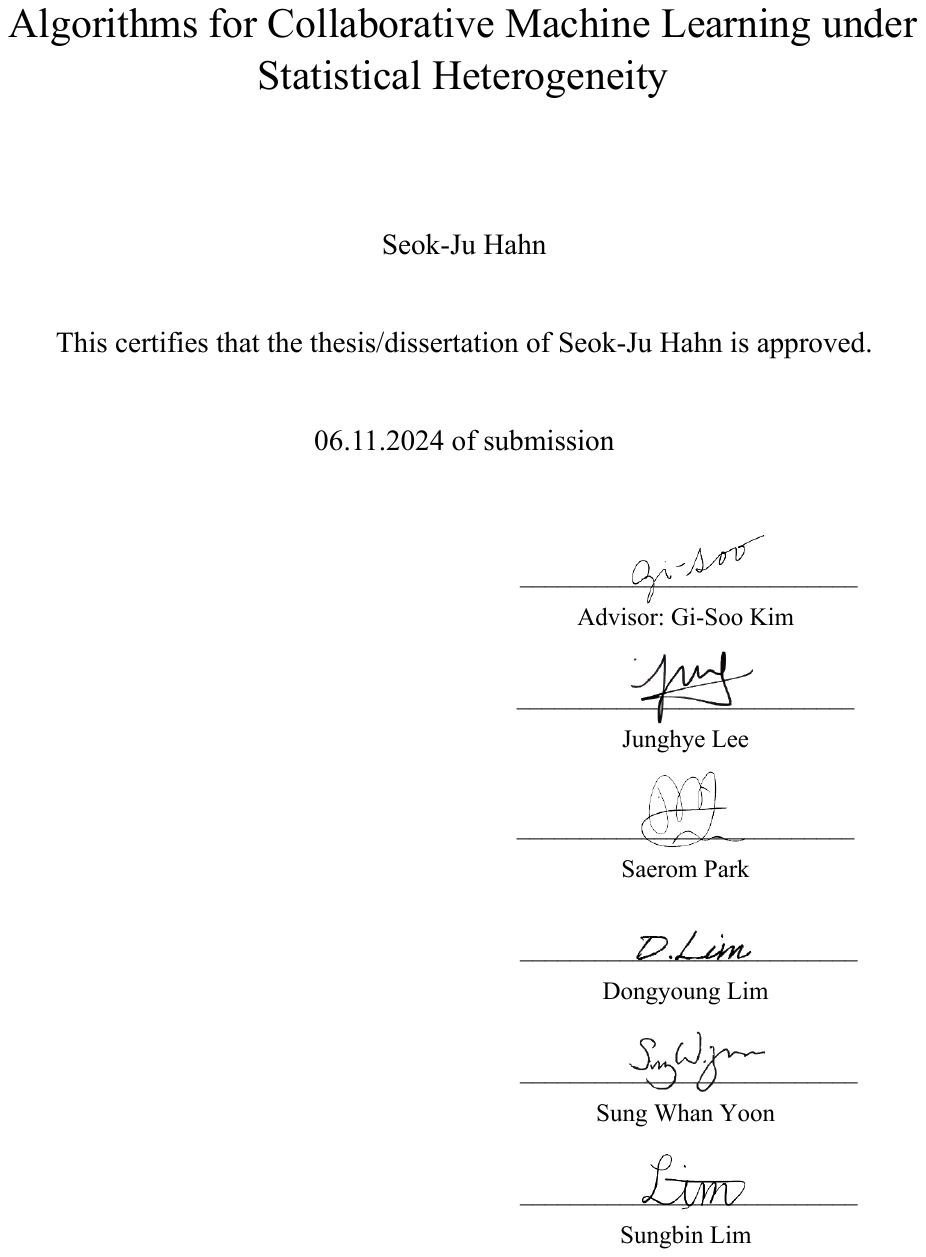}

\begin{abstract}
  Learning from distributed data without accessing them is undoubtedly a challenging and non-trivial task.
  Nevertheless, the necessity for distributed training of a statistical model has been increasing, 
  due to the privacy concerns of local data owners and the cost in centralizing the massively distributed data. 
  Federated learning (FL) is currently the de facto standard of training a machine learning model across heterogeneous data owners,   without leaving the raw data out of local silos.
  Nevertheless, several challenges must be addressed in order for FL to be more practical in reality. 
  Among these challenges, the statistical heterogeneity problem is the most significant and requires immediate attention.
  From the main objective of FL, three major factors can be considered as starting points --- \textit{parameter}, \textit{mixing coefficient}, and \textit{local data distributions}.
  
  In alignment with the components, this dissertation is organized into three parts.
  In Chapter~\ref{ch:superfed}, a novel personalization method, \texttt{SuPerFed}, inspired by the mode-connectivity is introduced. 
  This method aims to find better \textit{parameters} that are suitable for achieving enhanced generalization ability in all local data distributions. 
  In Chapter~\ref{ch:aaggff}, an adaptive decision-making algorithm, \texttt{AAggFF}, is introduced for inducing uniform performance distributions in participating clients,  which is realized by online convex optimization framework. 
  This method explicitly learns fairness-inducing \textit{mixing coefficients} sequentially, and is also specialized for two practical FL settings.
  Finally, in Chapter~\ref{ch:fedevg}, a collaborative synthetic data generation method, \texttt{FedEvg}, is introduced, leveraging the flexibility and compositionality of an energy-based modeling approach. 
  The objective of this method is to emulate the joint density of disparate local data distributions without accessing them, which enables to generate server-side synthetic dataset.
  
  Taken together, all of these approaches provide practical solutions to mitigate the statistical heterogeneity problem in data-decentralized settings, paving the way for distributed systems and applications using collaborative machine learning methods.
\vfill
\end{abstract}

\clearpage

\hbox{ }
\thispagestyle{empty}
\clearpage

\addtocontents{toc}{\protect\pagestyle{empty}}
\addtocontents{toc}{\protect\thispagestyle{empty}}
\tableofcontents
\thispagestyle{empty}
\clearpage

\listoffigures{}
\thispagestyle{empty}
\clearpage

\listoftables{}
\thispagestyle{empty}
\clearpage

\setcounter{page}{1}
\renewcommand{\baselinestretch}{1.5} 
\setlength{\parskip}{12pt} 
\chapternumberfont{\fontsize{36pt}{32pt}\selectfont} 
\chaptertitlefont{\fontsize{24pt}{20pt}\selectfont} 
\sectionfont{\fontsize{18pt}{16pt}\selectfont} 
\subsectionfont{\fontsize{16pt}{14pt}\selectfont} 
\subsubsectionfont{\fontsize{14pt}{12pt}\selectfont} 
\fontsize{11pt}{13pt}\selectfont

\chapter{Introduction} 
\numberwithin{equation}{chapter}
\renewcommand{\theequation}{1.\arabic{equation}}  

    Data is being collected every time, almost everywhere, including e.g., wearable devices, electronic transactions, and satellites orbiting the earth.
    Undoubtedly, the abundance of data is a critical component of the current success of machine learning.
    Once a sufficient amount of data is secured, one can train an accurate statistical model with a good generalization performance on the unseen data distribution.
    However, preparing a satisfactory amount of data is not always feasible in reality.
    This is due to practical reasons: some data may contain sensitive personal information, such as electronic health records, 
    while other data is massively distributed to a large population, such as smartphone users. 
    
    Federated Learning (FL)\cite{fedavg} is a tailor-made method for this problematic situation, 
    since it aims to train a machine learning model \textit{without} accessing or centralizing distributed data in one place.  
    Instead, an orchestrator (which is usually a central server) requests each data owner (i.e., client) to commit a locally updated statistical model trained on each client's data distribution.
    Subsequently, the server aggregates local updates into a single global model, which may contain knowledge from heterogeneous clients.
    With enough repetition of these local updates \& global aggregation procedures, the central server can ultimately obtain a global model with satisfactory performance.

\newpage
\section{Federated Learning: Concepts and Challenges}
\subsection{Overview}
    \begin{figure}
        \centering
        \includegraphics[width=\linewidth]{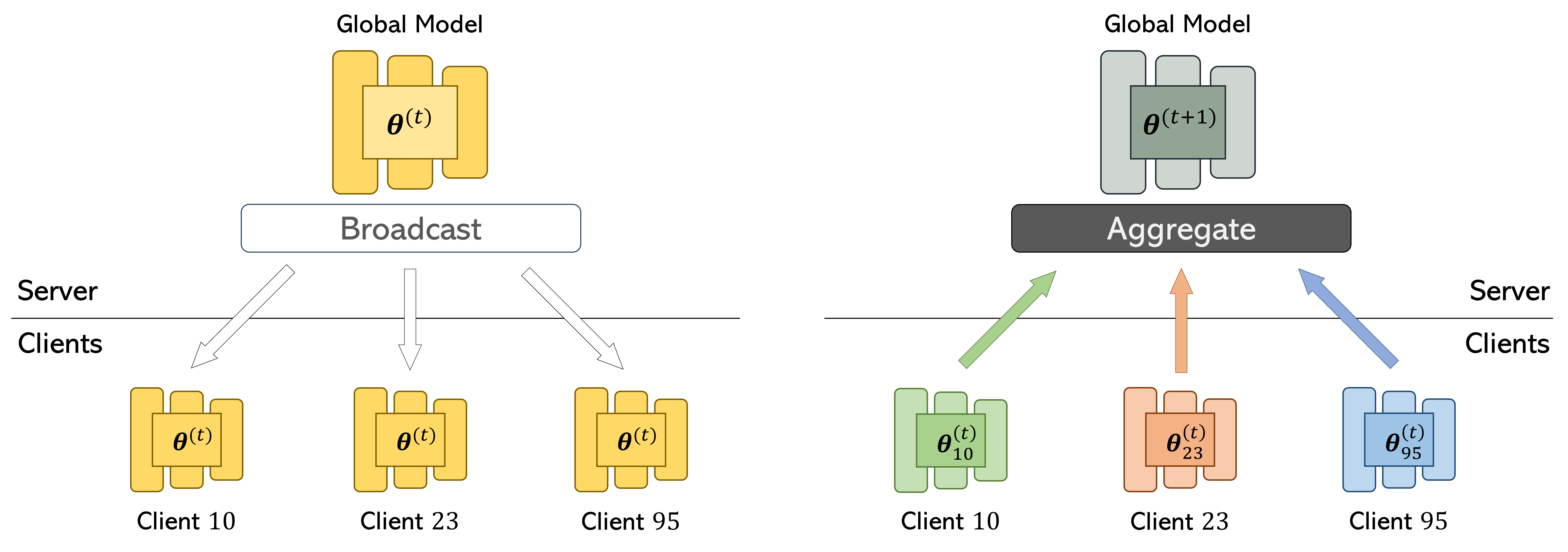}
        \caption[Overview of a one step of federated learning procedure] 
        {Overview of a one step of federated learning procedure.
        The central server broadcasts a global model to a random subset of clients (left). 
        Each client updates the global model using its own local dataset and computing power, 
        and the local updates are then uploaded to the server to be aggregated into a new global model (right).
        This process is repeated until a global model converges, orchestrated by the server.}
        \label{fig:main_fl}
    \end{figure}
    Federated learning (FL) is a distributed machine learning technique in a network, where a central server orchestrates the whole procedure of FL with participating clients.
    Each client has its own local data set and computing power, but it is typically assumed that the data set may be limited in quality or quantity.
    Thus, clients are motivated to participate in the federated system to benefit from the collaborative intelligence induced by the FL process.
    One of the privileged advantages of FL is a mild privacy protection, since data centralization is missing in the whole procedure.
    Instead of collecting data in one place as in traditional machine learning practices, FL requires each client to provide other signals, such as locally updated model parameters or gradients.
    By aggregating these proxy signals, the server can update a global statistical model without accessing distributed data silos.
    Moreover, since local computation can be performed parallelly on each client, the central server can avoid massive computation costs.
    This appealing feature has led to the practical application of FL, such as mobile applications~\cite{app1,app2}, autonomous driving~\cite{app3}, unmanned aerial vehicles~\cite{app4}, clinical risk predictions~\cite{app5,app6,app7,app8}, and financial domains~\cite{app9,app10,app11}.

\newpage
\subsection{Objective of Federated Learning}
    The main objective of FL is defined as follows:
    \begin{equation}
    \label{eq:fl_obj}
    \begin{gathered}
        \min_{\boldsymbol{\theta}\in\Theta\subseteq\mathbb{R}^d}
        F\left(\boldsymbol{\theta}\right)
        \triangleq
        \sum\nolimits_{i=1}^K w_i F_i\left({\boldsymbol{\theta}}\right), \\
        \text{ where }
        F_i\left({\boldsymbol{\theta}}\right) = \frac{1}{n_i} \sum\nolimits_{j=1}^{n_i} \mathcal{L}\left(\xi_j;\boldsymbol{\theta}\right),
        \xi_j\sim p(\mathcal{D}_i).
    \end{gathered}    
    \end{equation}
    In the FL system, one can assume that there are $K$ participating clients.
    The FL objective aims to minimize the composite local objectives of $K$ clients, where client $i$'s local objective is $F_i\left({\boldsymbol{\theta}}\right)$.
    Each local objective is to find an appropriate \textbf{\color{ForestGreen}\textit{parameter}}, $\boldsymbol{\theta}$, using \textbf{\color{Maroon} \textit{local data distribution}} $\mathcal{D}_i, i\in[K]$.
    The local objective is weighted by a corresponding \textbf{\color{YellowOrange} \textit{mixing coefficient}} $w_i \geq 0$ ($\sum_{i=1}^K w_i=1$),
    which is usually set to be a \textit{static} value proportional to the local sample size, $n_i$: e.g., $w_i = \frac{n_i}{n}$, where $ n=\sum_{j=1}^K n_j$.
    Note that the local objective is defined as the average of per-sample training loss $\mathcal{L}\left(\cdot;\boldsymbol{\theta}\right)$ w.r.t. the parameter $\boldsymbol{\theta}$ (which is an element of a parameter space, $\Theta$). 
    It is calculated from local samples, $\xi_k, k\in[n_i]$, following the convention of the empirical risk minimization (ERM) principle~\cite{erm}.

    One of the basic solvers for the main objective in eq.~\eqref{eq:fl_obj} is Federated Averaging (i.e., \texttt{FedAvg})~\cite{fedavg}, which iteratively updates a global model by aggregating locally updated signals from clients. 
    Each client downloads a copy of a global model and locally updates the model using their own dataset.
    This local update is again uploaded to the central server, and it is aggregated into a new global model.
    The aggregation scheme is no more than a simple weighted averaging of local updates proportional to the local sample sizes used for local updates.
    See Algorithm~\ref{alg:fedavg} in detail.
    \begin{algorithm}[H]
    \caption{\texttt{FedAvg} (Federated Averaging~\cite{fedavg})}
    \label{alg:fedavg}
    \begin{algorithmic}
        \STATE{\textbf{Inputs}: 
        number of clients ($K$),
        cilent sampling rate ($C$),
        number of communication rounds ($T$)} 
        \STATE{\textbf{Procedure:}} 
        \STATE Server initializes a global model $\boldsymbol{\theta}^{(0)}$.
        \FOR{each round $t=0,...,T-1$}
        \STATE $S^{(t)}\leftarrow$ randomly selected $\max(C \cdot K, 1)$ clients.
            \FOR{each client $i \in S^{(t)}$ \textbf{in parallel}}
                \STATE{$\boldsymbol{\theta}^{(t+1)}_i\leftarrow$\texttt{ClientUpdate}$(i,\boldsymbol{\theta}^{(t)})$}.
            \ENDFOR
        \STATE $\boldsymbol{\theta}^{(t+1)}\leftarrow\frac{1}{\sum_{j \in S^{(t)}}^{K} n_j}\sum_{i \in S^{(t)}}^{K} n_i \boldsymbol{\theta}^{(t+1)}_i$.
        \ENDFOR
        \STATE{\textbf{Return:} $\boldsymbol{\theta}^{(T)}$.}
    \end{algorithmic}
\end{algorithm}

\begin{algorithm}[H]
    \caption{\texttt{ClientUpdate}}
    \begin{algorithmic}
    \STATE{\textbf{Inputs}: 
        local batch size ($B$), 
        number of local epochs ($E$), 
        and local learning rate ($\eta$)} 
    \STATE{\textbf{Procedure:} }
    \STATE{Receive a global model $\boldsymbol{\theta}$ from the server.}
    \STATE{$\mathcal{B}\leftarrow$ split local dataset into batches of size $B$.}
    \FOR{each local epoch $e=0,...,E-1$}
        \FOR{mini-batch $\xi\in\mathcal{B}$}
            \STATE $\boldsymbol{\theta}\leftarrow\boldsymbol{\theta}-\eta\nabla_{\boldsymbol{\theta}}\mathcal{L}(\xi;\boldsymbol{\theta})$
        \ENDFOR
    \ENDFOR
    \STATE{\textbf{Return:} $\boldsymbol{\theta}$ to the server.}
    \end{algorithmic}
\end{algorithm}
    
\newpage
\subsection{Challenges in Federated Learning}
    We briefly introduce major challenges in FL that should be addressed to achieve a practical and faithful federated system.
    While this dissertation mainly focuses on the most critical and widespread challenge, a \underline{statistical heterogeneity}, other challenges should also be addressed in practical FL services.
    Interested readers may find more details about each challenge in the seminal paper by Kairouz et al.~\cite{prob_fl}.
    
\paragraph{Statistical Heterogeneity}
    In the federated system, each participating client corresponds to, for example, a smartphone user, a car owner, a clinical institution, or a financial organization.
    These clients generate and store data in a heterogeneous manner, which is also known as a non-IID problem (i.e., not independent and not identically distributed).
    In other words, each client has its own unique local distribution that is different from the others.
    Therefore, the resulting local updates computed from different local distributions will inevitably differ from each other.
    This negatively affects the update direction of a global model; therefore, careful tuning is required for FL configurations, such as the number of local updates and the global synchronization frequency.

\paragraph{Systematic Heterogeneity}
    Orthogonal to the aforementioned challenge, there is another line of heterogeneity in terms of system constraints, e.g., the communication or computation capacity of each client.
    Each client is both a model trainer and a data owner.
    However, the capabilities of all participating clients cannot be uniformly adjusted.
    Therefore, there should be inevitable differences between them, such as operating systems, computing capacity, and communication bandwidth.
    These sometimes affect the availability of certain clients in the federated system, where some clients may be dropped in a given communication round due to an unstable connection status or a response timeout due to slow computation speed.

\paragraph{Privacy}
    While FL guarantees a minimal level of privacy protection by only sharing local updates, the local update itself still has some risk of compromising privacy if a malicious attacker (including the client and server) intends to do so.
    Hence, by introducing privacy-preserving techniques such as Differential Privacy (DP)~\cite{dp}, Secure Multi-Party Computation (MPC)~\cite{smpc}, Homomorphic Encryption (HE)~\cite{he1,he2} into FL, the federated system can be more trustworthy and robust.
    Although these methods provide security and robustness to the federated system, they also involve trade-offs: performance degradation and computation/communication overhead.
    Thus, striking a good balance between performance and privacy is a key challenge in practice.

\newpage
\section{Overview of Dissertation Structure}
    This dissertation addresses the statistical heterogeneity from each of the three major components of the FL objective from an \textbf{independent perspective}: {\color{ForestGreen}\textit{parameter}}, 
    {\color{YellowOrange} \textit{mixing coefficient}} and {\color{Maroon} \textit{local data distribution}}.
    The statistical heterogeneity is the main problem to be addressed by modifying each component.
    Meanwhile, as each component is associated with several characteristics of the federated system, 
    it is still possible to address some of the other challenges, too.

    In \textbf{Chapter~\ref{ch:superfed}}, we introduce a new method, \texttt{SuPerFed} to find a better parameter in terms of personalization scheme.
    We first examine that existing model mixture-based personalized FL algorithms are actually an an implicit ensemble of different (sub-)models.
    We then adopt one of the explicit schemes of a model ensemble method, named as \textit{mode connectivity}~\cite{fort2019deep, fort2019large, modeconnect, garipov+18, draxler18}, 
    to boost the performance of the model mixture-based personalized FL algorithm.
    This approach not only enhances the overall accuracy but also renders the system more resilient to label noise and improves calibration error, 
    thereby making it scalable and practical.

    In \textbf{Chapter~\ref{ch:aaggff}}, we introduce the first sequential decision-making framework for fair FL, \texttt{AAggFF}. 
    It aims to guide a single global model to have a uniform performance distribution across heterogeneous clients, 
    which is also known as client-level fairness~\cite{clientlevelfairness}. 
    It is inspired by the observation that existing fair FL algorithms could be unified into a specific instance of an \textit{online convex optimization} framework~\cite{oco1,oco2,oco3}.
    We then further improve our method using an advanced framework, and subdivide it into two practical FL settings: cross-silo FL and cross-device FL.

    In \textbf{Chapter~\ref{ch:fedevg}}, we introduce \texttt{FedEvg}, 
    a communication-efficient synthetic dataset generation method in a heterogeneous network by leveraging the flexibility of an \textit{energy-based modeling}~\cite{ebm,your,ebmcomp1}.
    This approach integrates both discriminative and generative modeling, with the former being a common task in FL literature while the latter has not been sufficiently explored.
    We demonstrate that a central server can obtain a faithful synthetic dataset with communication efficiency, 
    which can also be exploited for a variety of downstream tasks.  
    Compared to existing works, \texttt{FedEvg} can construct synthetic datasets with better quality and downstream performances. 

\newpage
\section{Software}
    All implementations of the methods in this thesis are available in the links below.
    \begin{itemize}
        \item \textbf{Chapter~\ref{ch:superfed}} (\texttt{SuPerFed}): \href{https://github.com/vaseline555/SuPerFed}{\color{blue}https://github.com/vaseline555/SuPerFed}
        \item \textbf{Chapter~\ref{ch:aaggff}} (\texttt{AAggFF}): \href{https://github.com/vaseline555/AAggFF}{\color{blue}https://github.com/vaseline555/AAggFF}
        \item \textbf{Chapter~\ref{ch:fedevg}} (\texttt{FedEvg}): \href{https://github.com/vaseline555/FedEvg}{\color{blue}https://github.com/vaseline555/FedEvg}
    \end{itemize}

    All code is implemented in PyTorch \cite{pytorch}, simulating a central parameter server that orchestrates a whole FL procedure and operates each method (\texttt{SuPerFed}, \texttt{AAggFF}, and \texttt{FedEvg}). 
    It is further simulated that $K$ participating clients have their own local samples, computation power.
    The communication structure is a star topology between a single central server and multiple clients. 
    All experiments are conducted on a single computing server with two Intel\textsuperscript{\textregistered} Xeon\textsuperscript{\textregistered} Gold 6226R CPUs (@ 2.90GHz) and one NVIDIA\textsuperscript{\textregistered}  Tesla\textsuperscript{\textregistered} V100-PCIE-32GB GPU.  

\newpage 
\chapter{Perspective on Parameter, $\boldsymbol{\theta}$}
\label{ch:superfed}
\numberwithin{equation}{chapter}
\numberwithin{figure}{chapter}
\numberwithin{table}{chapter}
\numberwithin{algorithm}{chapter}
\renewcommand{\theequation}{2.\arabic{equation}}
\renewcommand{\thefigure}{2.\arabic{figure}}
\renewcommand{\thetable}{2.\arabic{table}}
\renewcommand{\thealgorithm}{2.\arabic{algorithm}}

\section*{Connecting Low-Loss Subspace for Personalized Federated Learning}
\addcontentsline{toc}{section}{\protect\numberline{}\textbf{Connecting Low-Loss Subspace for Personalized Federated Learning}}
    In this section, we examine potential improvements in the perspective of the \textbf{parameter}, denoted as $\boldsymbol{\theta}$ in eq.~\eqref{eq:fl_obj}.
    One intuitive modification is to learn multiple parameters in each client rather than a single one.
    It can be advantageous to learn multiple parameters simultaneously, particularly when one parameter acquires a global knowledge while the other becomes an expert in local knowledge.
    By simultaneously learning multiple parameters, the federated system can possibly mitigate statistical heterogeneity across disparate clients.
    
    This is precisely aligned with the model-based personalized federated learning (PFL) approaches. However, existing methodologies either exchange a partial parameter or necessitate additional local updates for personalization, which places a burden on computational resources and overlooks novel clients.
    Furthermore, the mixture between a parameter of a personalized model and a parameter of a federated model is \textit{not} designed to create an explicit synergy between them.
    
    Here comes the main research question:
    \begin{center} 
    \textit{How can we explicitly induce the synergy between the \textbf{parameters} of\\ the model mixture-based personalization methods to mitigate statistical heterogeneity?}
    \end{center}
    
\newpage
\section{Introduction}
    Individuals and institutions are nowadays actively engaged in the production and management of data, as a consequence of the advent of sophisticated communication and computational technologies. 
    Consequently, the training of statistical models in a data-centralized setting is not always a viable option due to a number of practical constraints. 
    These include the existence of massive clients having their own data, and the data privacy concerns that restrict the access and collection of data. 
    FL has been served as a promising solution to this problematic situation, as it enables the parallel training of a machine learning model across clients or silos without the need to share their private data. 
    This is usually proceeded under the orchestration of the central server, which is, for example, a service provider. 
    In the most common FL setting, such as \texttt{FedAvg}~\cite{fedavg}, when the server prepares and broadcasts a model appropriate for a target task, each participant trains the model with its own data and transmits the resulting model parameters to the server. 
    The server then aggregates (e.g., weighted averaging) the locally updated parameters into a new global model, which is broadcast again to some fraction of participating clients. 
    This collaborative learning process is repeated until convergence of the model is achieved. 
    However, obtaining a single global model through FL is not sufficient to provide satisfactory experiences to all clients, due to the \textit{statistical heterogeneity}. 
    That is, there exists an inherent difference among local distributions across clients, 
    or the data is not independent and identically distributed (non-IID). 
    Hence, the application of a straightforward FL approach, such as \texttt{FedAvg}, 
    is insufficient to circumvent the inherent limitations of high generalization error or model divergence~\cite{ka+19}. 
    In such cases, the performance of the global model may be inferior to that of the locally trained model, 
    prompting clients to question the value of their participation in the FL training.
    
    As a consequence, it is natural to reconsider the fundamental principles of FL. 
    A personalized FL (PFL) has become a promising alternative to single model training scheme, 
    which has led to a plethora of related studies.
    Many approaches are proposed for PFL, based on multi-task learning \cite{mocha, mar+21, fedu}, regularization technique \cite{l2sgd, pfedme, ditto}; meta-learning  \cite{Fallah, jiang+19}, clustering \cite{ghosh+20, clustered, mansour+20}, and a model mixture method \cite{fedrep, FedPer, lgfedavg, ditto, apfl, mansour+20}.
    In this chapter, we will focus on the \textit{model mixture-based PFL} method, 
    which has demonstrated decent performance and allows each client to have their own model. 
    The model mixture-based PFL method assumes that each client has two distinct parameters: 
    a \textit{local (sub-)model} and a \textit{federated (sub-)model}
    \footnote{We intentionally coined this term to avoid confusion with the global model at the server.} 
    (i.e., federated (sub-)model is originally a \textit{global model} before it is downloaded to a client).
    In the model mixture-based PFL method, it is expected that the local model will capture the information pertaining to the heterogeneous distribution of client data by remaining solely at the client-side.
    In contrast, the federated model will focus on learning common information across clients by being communicated with the server.
    Each of the federated model trained at the client is then uploaded to the server, where it is aggregated into a single global model (Figure~\ref{fig:fig1_1}).
    Note that this can be viewed as an implicit collaboration between the personalized model and the federated model.
    Upon completion of FL training, a participating client is eventually equipped with a personalized parameter (i.e., a local model), 
    which may mitigate the problem associated with non-IID nature.
    The central server, in turn, obtains a global model that can be exploited for other downstream tasks or for novel clients entering in the system.
    \begin{figure}
    \centering
        \includegraphics[width=0.85\linewidth]{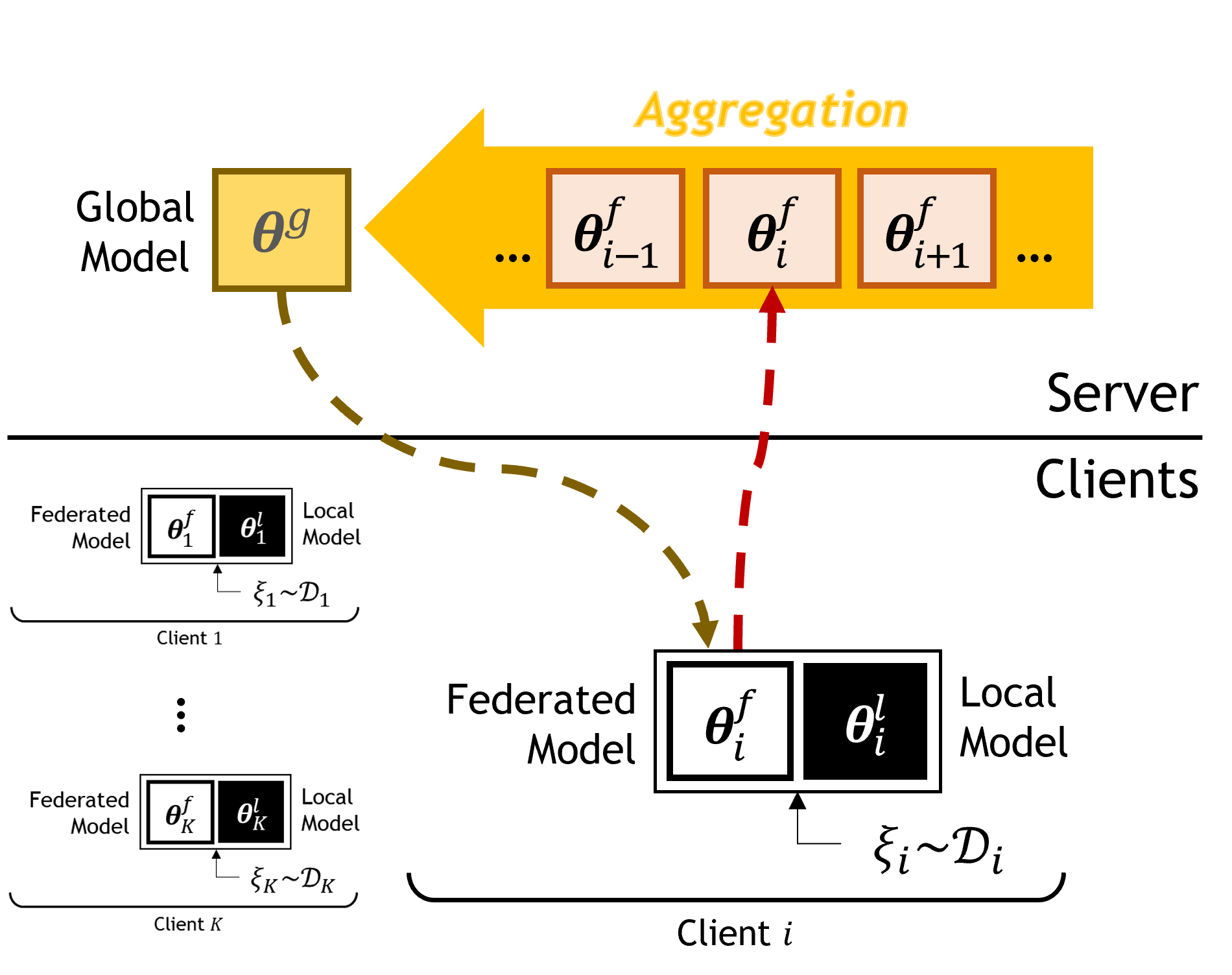}
        \caption{Overview of the model mixture-based personalized federated learning method}
        \label{fig:fig1_1}
    \end{figure}
    The model mixture-based PFL method comprises two sub-branches. 
    One involves sharing only a partial set of parameters of a single model as a federated model~\cite{fedrep, FedPer, lgfedavg} (e.g., only exchanging the weights of a penultimate layer). 
    The other involves having two identically structured models as a local and a federated model~\cite{ditto, apfl, l2sgd, pfedme, mansour+20}. 
    However, each of these approaches has certain limitations. 
    In the former case, novel clients entering the federated system cannot immediately exploit the model trained by FL, as the server only has a partial set of model parameters. 
    In the latter case, all of them require separate (or sequential) updates for the local and the federated model, each of which increases the local training time. This is sometimes burdensome to resource-poor clients and slows down the whole FL procedure.
    Moreover, both approaches \textbf{do NOT explicitly seek to achieve mutual benefits} between the local and federated models.

    Here comes our main motivation: the development of a novel model mixture-based PFL method that jointly trains both federated and local models by inducing explicit synergy between them. 
    We attempted to achieve this goal through the lens of \textit{mode connectivity}~\cite{garipov+18, fort2019deep, fort2019large, draxler18, modeconnect, nnsubspaces}. 
    As numerous studies on the loss landscape of the deep network have been conducted, one intriguing phenomenon is being actively discussed: 
    the connected subspace between two or more different deep networks. 
    Despite the well-known fact of the existence of numerous local minima in deep networks, recent findings have demonstrated that the performance of these minima is, in fact, quite similar~\cite{choro}.
    Additionally, it has recently been discovered that two different local minima obtained by two independent deep networks can be connected through the linear path~\cite{frankle20} or even through the non-linear path~\cite{garipov+18, draxler18} in the parameter space, where \textit{all} parameters along the path have a low loss. 
    Note that the two endpoints of such a path are two different local minima reached from the two different deep networks. 
    These findings have been extended more to a multi-dimensional subspace~\cite{nnsubspaces, benton21} with only a few gradient steps on two or more \textit{independently initialized} deep networks. 
    The resulting subspace, including the linear connected path, contains functionally diverse models having high accuracy. 
    It can also be viewed as an \textit{ensemble of deep networks} in the parameter space, thereby sharing similar properties such as good calibration performance and robustness to the label noise. 
    By introducing the \textit{mode connectivity} to a model mixture-based PFL method, such good properties can also be reproduced accordingly.

\paragraph{Contributions}
    We propose \texttt{SuPerFed}, a connected low-loss \textbf{su}bspace construction method for the \textbf{per}sonalized \textbf{fed}erated learning, adapting the concept of \textit{connectivity} to the model mixture-based PFL method. 
    This method aims to find a low-loss subspace between a single global model and many different local models at clients in a mutually beneficial way. Consequently, our objective is to enhance the personalization performance of model mixture-based PFL methods, while overcoming the aforementioned limitations.
    Adopting the \textit{mode connectivity} to FL is non-trivial as it obviously differs from the setting where the \textit{mode connectivity} is first considered. 
    In FL, the main reason is that each local model observes distinct data distributions. 
    Thus, only the global model (i.e., the federated model when transmitted to the client) should be connected to numerous local models that learn from on disparate data distributions of clients.
    
    The main contributions of this work are listed as follows:
    \begin{enumerate}
        \item[$\bullet$] We propose \texttt{SuPerFed}, a model mixture-based PFL method that aims to achieve superior personalization performance by connecting between each distinct local model and the global model in the parameter space. This connection serves to induce an explicit ensemble between the two models.
        \item[$\bullet$] Our proposed method extends the limitations of existing model mixture-based PFL methods by jointly updating both federated and local models in clients and exchanging the entire federated model with the server, allowing novel clients to benefit from it.
        \item[$\bullet$] Personalized models trained by adopting our method are well-calibrated and also robust to potential label noise that is a common problem in the realistic federated system. 
    \end{enumerate}

\newpage
\section{Prior Arts}
\subsection{Federated Learning with Non-IID Data} 
    After \texttt{FedAvg}~\cite{fedavg} proposed the basic FL algorithm, handling non-IID data across clients is one of the major points to be resolved in the FL field. 
    While some methods are proposed such as sharing a subset of client's local data at the server~\cite{zhao+18}, accumulating previous model updates at the server ~\cite{fedavgm}. These are either unrealistic assumptions for FL or not enough to handle a realistic level of statistical heterogeneity. Other branches to aid the stable convergence of a single global model include modifying a model aggregation method at the server~\cite{pfnm, fedma, pillu+19, fedbn} and adding a regularization to the optimization~\cite{fedprox, ka+19, feddyn}. 
    However, a single global model may still not be sufficient to provide a satisfactory experience to clients using FL-driven services in practice.
    
\subsection{Personalized Federated Learning Methods} 
    As an extension of the above, PFL methods shed light on the new perspective of FL. PFL aims to learn a client-specific personalized model, and many methodologies for PFL have been proposed.
    Multi-task learning-based PFL~\cite{mocha, fedu, mar+21} treats each client as a different task and learns a personalized model for each client. 
    Local fine-tuning-based PFL methods~\cite{jiang+19, Fallah} adopt a meta-learning approach for a global model to be adapted promptly to a personalized model for each client, and clustering-based PFL methods~\cite{ghosh+20, clustered, mansour+20} mainly assume that similar clients may reach a similar optimal global model. 
    Model mixture-based PFL methods~\cite{apfl, FedPer, lgfedavg, mansour+20, pfedme, fedrep, ditto, l2sgd} divide a model into two parts: 
    one (i.e., a local model) for capturing local knowledge, and the other (i.e., a federated model) for learning common knowledge across clients. 
    In these methods, only the federated model is shared with the server while the local model resides locally on each client. ~\cite{FedPer} keeps weights of the last layer as a local model (i.e., the personalization layers), and~\cite{fedrep} is similar to this except it requires a separate update on a local model before the update of a federated model. 
    In contrast, \texttt{LG-FedAvg}~\cite{lgfedavg} retains lower layer weights in clients and only exchanges higher layer weights with the server. 
    Due to the partial exchange, new clients in some PFL schemes~\cite{FedPer, fedrep, lgfedavg} should train own models from scratch or need at least some steps of fine-tuning. 
    In \cite{mansour+20, apfl, l2sgd, ditto, pfedme}, each client holds at least two separate models with the same structure:
    one for a local model and the other for a federated model. 
    While in \cite{ditto, pfedme}, the federated model affects the update of the local model in the form of proximity regularization.
    Close to our proposed method, in \cite{mansour+20, apfl,l2sgd}, they are explicitly interpolating the two different models in the form of a convex combination after the independent update of two models. 
    However, these methods did not seek into an explicit mutual ensemble of two models, and one should find the specific convex combination for exploiting a model with acceptable performances.

\newpage
\subsection{Mode Connectivity of Deep Networks} 
    The existence of either linear path~\cite{modeconnect} or non-linear path~\cite{garipov+18, draxler18} between two different minima derived by two different deep networks has been discovered through extensive studies on the loss landscape of deep networks~\cite{fort2019deep, fort2019large, modeconnect, garipov+18, draxler18}.
    This observation is named as \textit{mode connectivity}.
    In detail, there exists a low-loss subspace (e.g., line, simplex) connecting two or more deep networks independently trained on the same data. 
    Though there exist some studies on constructing such a low loss subspace~\cite{blundell+15, snapshot, swa, swa-gaussian, izmailov2020subspace, deepensembles}, they require multiple updates of different deep networks. 
    After it is observed that independently trained deep networks have a low cosine similarity with each other in the parameter space, as well as functionally dissimilar to each other~\cite{fort2019deep}, 
    a recent study proposes a straightforward and efficient method for explicitly inducing linear connectivity in a single training run~\cite{nnsubspaces}. 
    Our method is inspired by this technique and adapts it to construct a low-loss subspace between each different local model and a federated model 
    (which will later be aggregated into a global model at the server) suited for an effective PFL in heterogeneously distributed data across clients.
    
    In detail, our method resorts to the \textit{linear} mode connectivity between two non-convex deep networks,
    which does not typically hold for two independently trained deep networks~\cite{modeconnect,lmcdifficult}.
    However, it is observed either i) when the two deep networks shares partial or complete optimization trajectories~\cite{modeconnect,lmcdifficult},
    or ii) by permuting neurons of one independently trained networks to be aligned with those of the other~\cite{lmcperm1,lmcperm2,lmcperm3}.
    Our method is the former, since both federated and local models share training from start to finish.
    Intriguingly, the latter approach may have some connections to existing FL methods, e.g.,~\cite{pfnm,fedma} in that both of works aim to find suitable matching strategies when aggregating heterogeneous local updates by permuting neurons. 
    Concurrently, \texttt{FedSAM}~\cite{fedsam} seeks to find flat minima of a model in heterogeneous federated setting by using sharpness-aware minimization~\cite{sam} in local update and stochastic weight averaging (SWA)~\cite{swa} in server aggregation.
    While the mode connectivity is induced by the server (through SWA) in \texttt{FedSAM},
    it is induced by each client through random linear interpolation and orthogonality regularization in \texttt{SuPerFed}.  
    Though \texttt{FedSAM} is not devised for the personalization, it shares the same philosophy with ours --- flatness is favorable for the better generalization.

    \noindent Here, we provide the informal definition of the \textit{mode connectivity} as follows.
    
    \begin{definition} (Linear Mode Connectivity~\cite{frankle20, modeconnect, lmcdifficult, nnsubspaces}). 
    The mode connectivity (i.e., flat minima) is defined when there exists a connected path (i.e., a subspace of the parameter space) 
    between two distinct optima in the parameter space of different deep networks, $\boldsymbol{\boldsymbol{\theta}}_1 \text{ and } \boldsymbol{\boldsymbol{\theta}}_2$, such that $\text{Acc}(\lambda\boldsymbol{\theta}_1+(1-\lambda)\boldsymbol{\theta}_2)\geq\text{Acc}(\boldsymbol{\theta}_1), \text{Acc}(\boldsymbol{\theta}_2)$, where $\text{Acc}(\cdot)$ is the accuracy and $\lambda\in[0,1]$.
    \end{definition}

\newpage
\section{Proposed Method: \texttt{SuPerFed}}
\subsection{Outline} 
    In the standard FL scheme, the server orchestrates the entire learning process across participating clients through iterative communication of model (or parameters).\footnote{Note that the terms "\textit{model}" and "\textit{parameter}" is used interchangeably throughout this chapter.} 
    Since our method is essentially a model mixture-based PFL method, we require two models with the same structure per client (one for a federated model, the other for a local model). 
    We provide notation table in Table~\ref{tab:2_table1}.
    \begin{table}[!h]
\centering
\caption{Notations in Chapter~\ref{ch:superfed}}
\label{tab:2_table1}
\begin{tabular}{@{}ll@{}}
\toprule
\textbf{Notation} & \textbf{Description} \\
\midrule
$T$ & Total number of communication rounds \\
$L$ & Start round of a local model training for personalization \\
$B$ & Local batch size \\
$E$ & Number of local epochs \\
$K$ & Total number of participating clients, indexed by $i\in[K]$ \\
$C$ & Fraction of clients selected at each round \\
$\eta$ & Local learning rate \\
\midrule
$\boldsymbol{\theta}^{g}$ & Parameter/Model of a global model in the central server \\ 
    \begin{tabular}{@{}l} $\boldsymbol{\theta}^{f}_i$ \\ $\text{}$ \end{tabular} & 
    \begin{tabular}{@{}l} Parameter/Model of a federated model \\ (i.e., a global model downloaded and updated in each client) \end{tabular} \\
$\boldsymbol{\theta}^{l}_i$ & Parameter/Model of a local model in each client \\
$\lambda$ & Mixing variable sampled from the uniform distribution, $\text{U}(0,1)$ \\
\midrule
$\mu$ & Proximity regularization constant \\
$\nu$ & Orthogonality regularization constant \\
\bottomrule
\end{tabular}
\end{table}
    We introduce two hyperparameters for the optimization of \texttt{SuPerFed}:
    a constant $\mu$ for proximity regularization on the federated model so that it is not too distant from the global model, 
    and a constant $\nu$ for inducing \textit{mode connectivity} along parameter space between local and federated models.
    See Figure~\ref{fig:fig1_2} for the overview of \texttt{SuPerFed}.
    
    \begin{sidewaysfigure}[p]  
    \includegraphics[width=\textwidth,keepaspectratio]{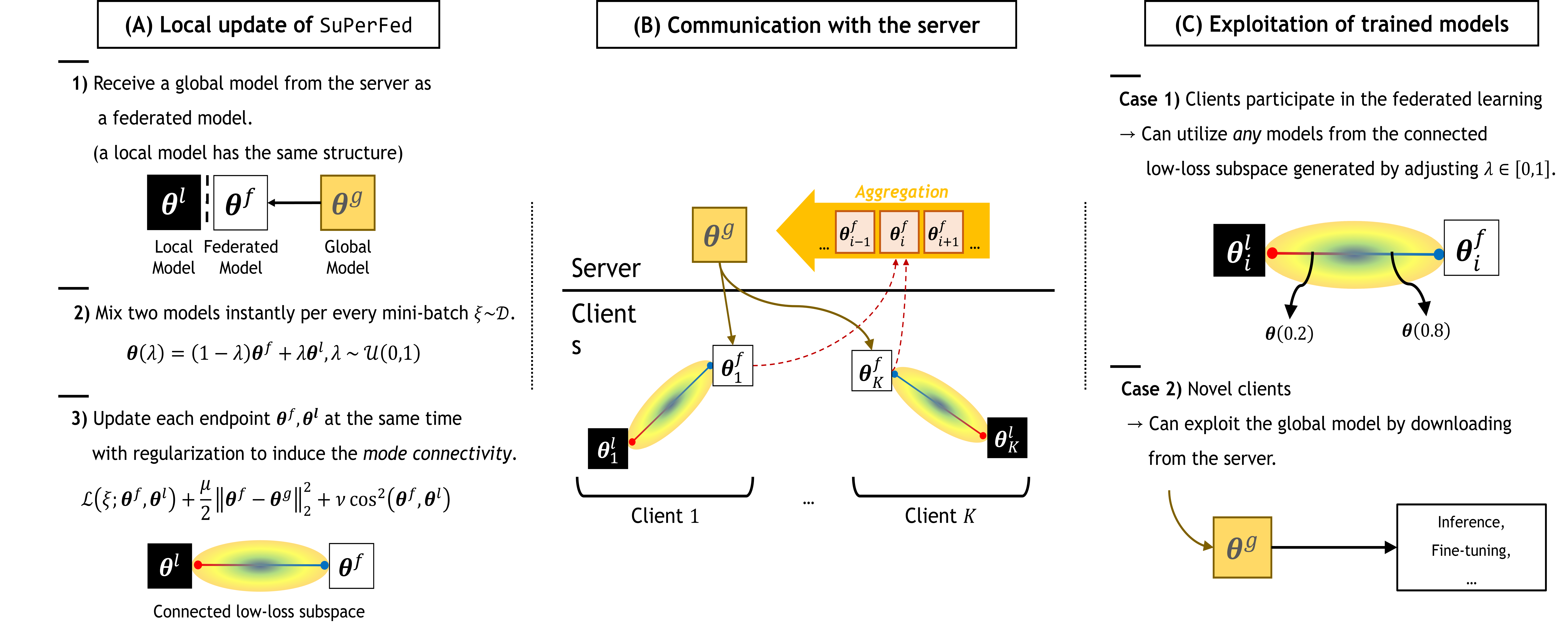}
    \caption[An illustration of the proposed method \texttt{SuPerFed}] 
        {An illustration of the proposed method \texttt{SuPerFed}.
        (A) Local update of \texttt{SuPerFed}: in each round, selected clients receive a global model from the server and sets it as a federated model. 
        After combining the federated model with a local model using a randomly sampled $\lambda~\mathcal{U}(0,1)$, 
        two models are jointly updated with regularization terms. 
        (B) Communication with the server: only the updated federated model is uploaded to the server (dotted arrow in crimson color) to be aggregated (e.g., weighted averaging) as a new global model, and it is broadcast again to other pool of sampled clients in the next round (arrow in gray color). 
        (C) Exploitation of trained models: 
        (\textbf{Case 1}) Clients already in the federated system can sample and exploit any parameter from the connected subspace (e.g., $\boldsymbol{\theta}(0.2)$ and $\boldsymbol{\theta}(0.8)$) because the subspace only contains low-loss solutions.
        (\textbf{Case 2}) Novel clients can download and use the trained global model $\boldsymbol{\theta}^g$ in the central server.
        }
     
    \label{fig:fig1_2}
    \end{sidewaysfigure}
    
\newpage
\subsection{Problem Statement} 
    Consider that each client $i\in[{K}]$ has its own local dataset $\left\{\xi_k\right\}_{k=1}^{n_i}\sim\mathcal{D}_i$, 
    as well as parameters $\boldsymbol{\theta}^l_i, \boldsymbol{\theta}^g\subseteq\Theta$. 
    A model mixture-based PFL assumes each client has a set of paired parameters $\boldsymbol{\boldsymbol{\theta}}_i^{l}$ (a local model) and $\boldsymbol{\theta}^g$ (a global model, which becomes a federated model $\boldsymbol{\theta}^f_i$ when it is transmitted to each client). 
    In previous works, the two parameters may be considered either as a unified model or as distinct parameters.
    To clarify this, we informally define a grouping function, $\mathrm{G}(\cdot,\cdot)$, to summarize existing model mixture-based PFL methods as follows.
    \begin{enumerate}
        \item[$\bullet$] Concatenation. It refers to stacking semantically separated layers (e.g., feature extractor and classifier)~\cite{fedrep, FedPer, lgfedavg}) into a single model: $\mathrm{G}(\boldsymbol{\theta}^g, \boldsymbol{\theta}^l_i) = \text{concat}(\boldsymbol{\theta}^g, \boldsymbol{\theta}^l_i)$, 
        \item[$\bullet$] Simple enumeration. It refers to using two models (having the same architecture) at once~\cite{ditto, pfedme}: 
        $\mathrm{G}(\boldsymbol{\theta}^g, \boldsymbol{\theta}^l_i) = (\boldsymbol{\theta}^g, \boldsymbol{\theta}^l_i)$,
        \item[$\bullet$] Convex combination. It refers to adding parameters given a constant $\lambda\in[0,1]$~\cite{apfl,mansour+20,l2sgd}: 
        $\mathrm{G}(\boldsymbol{\theta}^g, \boldsymbol{\theta}^l_i) = (1-\lambda)\boldsymbol{\boldsymbol{\theta}}^g + \lambda\boldsymbol{\boldsymbol{\theta}}_i^{l}$. 
    \end{enumerate}
     
    Next, the local loss function (e.g., cross-entropy loss, mean-squared loss) w.r.t. the parameters is defined similarly as in eq.~\eqref{eq:fl_obj}, 
    $\mathcal{L}:\mathcal{D}\times\Theta \rightarrow \mathbb{R}_{\geq0}$. 
    Denote further that each client aims to minimize the empirical loss defined as: 
    $F_i(\boldsymbol{\theta}^g, \boldsymbol{\theta}^l_i) 
    \triangleq \frac{1}{n_i}\sum_{k=1}^{n_i} 
    \mathcal{L}(\xi_k;\boldsymbol{\theta}^g, \boldsymbol{\theta}^l_i)$.
    While minimizing the empirical loss, one can additionally consider a regularization term w.r.t. the parameter, $\Omega(\boldsymbol{\theta}^g, \boldsymbol{\theta}^l_i)$.
    To sum up, the local objective is defines as follows.
    \begin{equation}
    \label{eq:pfl_obj}
    \begin{gathered}
        \min_{(\boldsymbol{\theta}^g, \boldsymbol{\theta}^l_i)\in\Theta\subseteq\mathbb{R}^d}
        F_i(\boldsymbol{\theta}^g, \boldsymbol{\theta}^l_i) 
        \triangleq 
        \frac{1}{n_i}\sum_{k=1}^{n_i} 
        \mathcal{L}(\xi_k;\boldsymbol{\theta}^g, \boldsymbol{\theta}^l_i) 
        + 
        \Omega(\boldsymbol{\theta}^g, \boldsymbol{\theta}^l_i).
    \end{gathered}
    \end{equation}
    Note that this is also a typical formulation of the structural risk minimization (SRM) principle~\cite{srm}.
    As there exist $K$ clients in total, the global objective is reduced to minimize:
    \begin{equation}
    \begin{gathered}
        \min_{\boldsymbol{\theta}^g; \boldsymbol{\theta}^l_i, i\in[K]} 
        \frac{1}{\sum_{j=1}^K n_j} \sum_{i=1}^{K} n_i F_i(\boldsymbol{\theta}^g, \boldsymbol{\theta}^l_i).
    \end{gathered}
    \end{equation}
    Note again that the federated model $\boldsymbol{\theta}^f_i$ will be presented when it comes to explaining a \textit{client-side} optimization procedure in order for preventing potential confusion in notations. 
    (See Algorithm~\ref{alg:SuPerFed} and~\ref{alg:LocalUpdate})
    
    Now, we have a problem at hand: how can we induce \textbf{explicit synergies} between $\boldsymbol{\theta}^g$ and $\{\boldsymbol{\theta}^l_i\}_{i=1}^K$?

\newpage
\subsection{\texttt{SuPerFed}: Subspace Learning for the Personalized Federated System} 
    \begin{algorithm}[h]
\caption{\texttt{SuPerFed}}
\label{alg:SuPerFed}
\begin{algorithmic}
\STATE{\textbf{Inputs}:
number of clients ($K$),
client sampling ratio (${C}\in(0,1)$),
total rounds ($T$),
start round of local model training ($L$),
local batch size ($B$), 
number of local epochs ($E$), 
local learning rate ($\eta$), 
regularization constants ($\mu$ and $\nu$)
}
\STATE{\textbf{Procedure}:} 
\STATE Server initializes a global model $\boldsymbol{\theta}^{g}$.
\FOR{$t=0,..,T-1$}
\STATE Server randomly selects $\max(C\cdot K, 1)$ clients as $S^{(t)}$.
\STATE Server broadcasts the current global model $\boldsymbol{\theta}^g$ to selected clients in $S^{(t)}$.
\FOR{each client $i\in S^{(t)}$ \textbf{in parallel}} 
    \STATE $\boldsymbol{\theta}^{f,(t)}_i\leftarrow$ \texttt{SuPerFedClientUpdate}($t \geq L, \boldsymbol{\theta}^{g,(t)}$)
\ENDFOR
\STATE Update a global model: $\boldsymbol{\theta}^{g,(t+1)}
\leftarrow
\boldsymbol{\theta}^{g,(t)} - \frac{1}{\sum_{j\in S^{(t)}}{n_j}} \sum_{i \in S^{(t)}}{n_i}\left( \boldsymbol{\theta}^{g,(t)} - \boldsymbol{\theta}_i^{f,(t)} \right)$.
\ENDFOR
\STATE{\textbf{Return:}} $\boldsymbol{\theta}^{g,(T)}$ (in the server) and $\{ \boldsymbol{\theta}^l_1, ..., \boldsymbol{\theta}^l_K \}$ (in each client)
\end{algorithmic}
\end{algorithm}

    In \texttt{SuPerFed}, we suppose both of a global model and a local model shares the same structure,
    thus $\mathrm{G}$ is a function for constructing the random linear interpolation with the constant $\lambda\in[0, 1]$. 
    With a slight abuse of notation, we explicitly combine $\lambda$ in notation as: 
    $\boldsymbol{\theta}_i(\lambda) 
    \triangleq 
    (1-\lambda)\boldsymbol{\boldsymbol{\theta}}^g + \lambda\boldsymbol{\boldsymbol{\theta}}_i^{l}$
    (or equivalently, $\boldsymbol{\theta}_i(\lambda) 
    \triangleq 
    (1-\lambda)\boldsymbol{\boldsymbol{\theta}}^f_i + \lambda\boldsymbol{\boldsymbol{\theta}}_i^{l}$).
    
    The main goal of the local update in \texttt{SuPerFed} is to find a flat minima between the local model and the federated model in the parameter space,
    as well as faithfully devoting to the training of a single global model as same in the single model-based FL scheme.
    This is naturally proceeded in the client side, since the local data can only be accessed by each client.
    In the initial round, each client receives a global model transmitted from the server, sets it as a federated model, $\boldsymbol{\theta}^f_i$.
    Simultaneously, the local model \textit{should also be initialized differently} from each other, while the architecture should be equivalent to the federated model.
    It is noteworthy that the disparate initialization scheme is not a common occurrence in FL~\cite{fedavg}. 
    However, this is an \textit{intended} behavior in our scheme, 
    which aims to facilitate the construction of a connected low-loss subspace between a federated model and a local model, following the scheme of~\cite{nnsubspaces}.
    
    In each step of update using random batch of samples, a client should at first mix parameters of the two models into a single mixed model, $\boldsymbol{\theta}_i(\lambda)$, using $\lambda$ sampled from $\mathcal{U}(0, 1)$.
    From this ad-hoc mixing scheme, the client optimizes $\mathbb{E}_{\lambda\sim\mathcal{U}(0,1)}\left[ \boldsymbol{\theta}_i(\lambda) \right]$ in effective per each update. 
    Since the local model and the federated model are dynamically interpolated during training, we can expect that both models should be synergistic.
    In addition, because two models act as a single model, the client does not have to update each model separately, reducing the training time for local updates.
    In detail, two endpoints of the mixed model 
    (i.e., $\boldsymbol{\theta}_i(0)=\boldsymbol{\boldsymbol{\theta}}^{f}_{i}$; federated model
    and $\boldsymbol{\theta}_i(1)=\boldsymbol{\boldsymbol{\theta}}^{l}_{i}$; local model) 
    can be jointly updated at the same time using a sample as follows.
    \begin{equation}
    \label{eq:pfl_joint_update}
    \begin{gathered}
        \frac{\partial F_i}{\partial \boldsymbol{\boldsymbol{\theta}}^{f}_{i}}
        =
        \frac{\partial F_i}{\partial \boldsymbol{\theta}_i(\lambda)}
        \frac{\partial \boldsymbol{\theta}_i(\lambda)}{\partial \boldsymbol{\boldsymbol{\theta}}^{f}_{i}}
        =
        (1-\lambda)\frac{\partial \mathcal{L}}{\partial \boldsymbol{\theta}^{f}_{i}} 
        + \frac{\partial \Omega}{\partial \boldsymbol{\theta}^{f}_{i}}, \\
        \frac{\partial F_i}{\partial \boldsymbol{\boldsymbol{\theta}}^{l}_{i}}
        =
        \frac{\partial F_i}{\partial \boldsymbol{\theta}_i(\lambda)}
        \frac{\partial \boldsymbol{\theta}_i(\lambda)}{\partial \boldsymbol{\boldsymbol{\theta}}^{l}_{i}}
        =
        \lambda\frac{\partial \mathcal{L}}{\partial \boldsymbol{\theta}^{l}_{i}} 
        + \frac{\partial \Omega}{\partial \boldsymbol{\theta}^{l}_{i}}.
    \end{gathered}
    \end{equation}  
    After the local update is finished, the server and clients only communicates federated models, thus the communication cost remains the same as the single model-based FL methods, e.g.,~\cite{fedavg, fedprox}.
    See Algorithm~\ref{alg:LocalUpdate} for the overall procedure of a local update.
    
    Note that the instant ensemble of both models, $\boldsymbol{\theta}_i(\lambda)$, should be optimized with two regularization terms hidden in the term of $\Omega(\cdot,\cdot)$, 
    which will be introduced in the subsequent section.
    
\newpage
\subsection{Regularization}
    \begin{algorithm}[!h]
\caption{\texttt{SuPerFedClientUpdate}}
\label{alg:LocalUpdate}
\begin{algorithmic}
    \STATE{\textbf{Inputs}: 
    local batch size ($B$), 
    number of local epochs ($E$), 
    local learning rate ($\eta$), 
    regularization constants ($\mu$ and $\nu$), 
    a local model ($\boldsymbol{\theta}^l$)}
    \STATE{\textbf{Procedure:}} 
    \STATE Receive a global model $\boldsymbol{\theta}^{g}$ from the server and set it as the federated model: 
    $\boldsymbol{\theta}^{f}\leftarrow\boldsymbol{\theta}^{g}$.
    \STATE{Receive a personalization flag, $t \geq L$.}
    \FOR{each local epoch $e=0,...,E-1$}
        \STATE{$\mathcal{B}\leftarrow$ split local dataset into batches of size $B$.}
        \FOR{mini-batch $\xi\in\mathcal{B}$}
            \IF{\textbf{not} $t \geq L$}
            \STATE Set $\lambda=0$
            \ELSE
            \STATE Set $\lambda \sim \mathcal{U}(0, 1)$.
            \ENDIF
            \STATE Generate a mixed model: $\boldsymbol{\theta}(\lambda)=(1-\lambda)\boldsymbol{\theta}^{f}+\lambda\boldsymbol{\theta}^{l}$.        
            \STATE Update $\boldsymbol{\theta}^{f}$ and $\boldsymbol{\theta}^{l}$ by optimizing eq.~\eqref{eq:pfl_obj} using the estimate in eq.~\eqref{eq:pfl_joint_update} with learning rate $\eta$.
        \ENDFOR
    \ENDFOR
\STATE{\textbf{Return:}} $\boldsymbol{\theta}^{f}$
\end{algorithmic}
\end{algorithm}
    In \texttt{SuPerFed}, two regularization terms are required for i) stable training of a global model from locally updated federated models, 
    as well as ii) inducing the \textit{functional diversities} between a global model and multiple local models.
    For i), the proximity regularization term is introduced, which is applied between the federated model and the previous round's global model as follows.
    \begin{equation}
    \begin{gathered}
        \mathcal{R}_{\text{prox}}\left( \boldsymbol{\theta}^{f}_i, \boldsymbol{\theta}^{g} \right)
        \triangleq 
        \frac{\mu}{2}\left\Vert \boldsymbol{\theta}^{f}_i - \boldsymbol{\theta}^{g} \right\Vert_2^2, 
    \end{gathered}
    \end{equation}
    where $\Vert\cdot\Vert_2$ is $L_2$-norm.
    This is a simple and classic $L_2$-regularization term adjusted by a constant $\mu\in\mathbb{R}_{\geq 0}$, which has also been used in previous works~\cite{fedprox, pfedme}
    The proximity regularization term constrains the update of the federated model to not deviate far from the current optimization trajectory of a global model, $\boldsymbol{\boldsymbol{\theta}}^{g}$. 
    By doing so, we expect to prevent possible divergence of a new global model when it is aggregated from federated models. 
    Note that a large degree of proximity regularization can rather hinder the local model's update, thus $\mu$ should be carefully tuned.  
    For ii), the orthogonality regularization term is additionally introduced, which is applied between the federated model and the local model as follows.
    \begin{equation}
    \begin{gathered}
        \mathcal{R}_{\text{ortho}} \left( \boldsymbol{\theta}^{f}_i, \boldsymbol{\theta}^{l}_i \right)
        \triangleq 
        \nu\operatorname{cos}^2 \left( \boldsymbol{\boldsymbol{\theta}}^{f}_i, \boldsymbol{\boldsymbol{\theta}}^{l}_i \right),
    \end{gathered}
    \end{equation}
    where $\operatorname{cos}(\boldsymbol{\theta}_1, \boldsymbol{\theta}_2) 
    = {\langle \boldsymbol{\theta}_1, \boldsymbol{\theta}_2 \rangle}/\left(\Vert \boldsymbol{\theta}_1 \Vert_2 \Vert \boldsymbol{\theta}_1 \Vert_2\right)$ is a cosine similarity between two vectors (note that $\langle \cdot,\cdot \rangle$ is an inner product).
    This aims to let each network learn functionally diverse information from the data, thereby helping to induce the \textit{mode connectivity} between the local model and the federated model in the parameter space.
    Together with on-the-fly mixing of two models using randomly sampled $\lambda\sim\mathcal{U}(0,1)$, this regularization allows the connected subspace between optima of the two deep networks, where all parameters in the subspace are low-loss solutions. 
    The magnitude of this regularization term is controlled by the constant $\nu\in\mathbb{R}_{\geq 0}$. 
    Since simply applying the mixing strategy between the local and the federated models has little benefit~\cite{nnsubspaces},
    and the mixing should be manually tuned~\cite{apfl}, it is required to set $\nu>0$. 
    The rationale behind the orthogonality regularization is based on the observation that functionally similar deep networks show a non-zero cosine similarity~\cite{fort2019deep}. 
    As a contrapostivie statement, one can induce functionally dissimilar (i.e., diverse) deep networks when if zero cosine similarity (i.e., orthogonality) is forced between the weights of two deep networks~\cite{nnsubspaces}. 
    Moreover, inducing orthogonality (i.e., forcing cosine similarity to zero) between different parameters can prevent deep networks from learning redundant features given their learning capacity~\cite{ortho2,ortho3}. 
    In this context, we expect that the local model of each client $\boldsymbol{\boldsymbol{\theta}}^{l}_{i}$ and the federated model $\boldsymbol{\boldsymbol{\theta}}^{f}_{i}$ will learn mutually beneficial knowledge in a complementary manner and to be harmoniously combined to improve the personalization performance.
    Taken together, the regularization term is summarized as follows.
    \begin{equation}
    \begin{gathered}
        \Omega(\boldsymbol{\theta}^g, \boldsymbol{\theta}^l_i)
        \triangleq
        \mathcal{R}_{\text{prox}}\left(\boldsymbol{\theta}^{f}_i, \boldsymbol{\theta}^{g} \right) 
        +
        \mathcal{R}_{\text{ortho}}\left( \boldsymbol{\theta}^{f}_i, \boldsymbol{\theta}^{l}_i \right).
    \end{gathered}
    \end{equation}

    Interestingly, \texttt{SuPerFed} can reduced to the existing FL methods by adjusting aforementioned hyperparameters. 
    In detail, when fixing $\lambda=0, \nu=0, \mu=0$, the objective of \texttt{SuPerFed} is reduced to \texttt{FedAvg}~\cite{fedavg}, when $\lambda=0, \nu=0, \mu>0$, it is equivalent to \texttt{FedProx}~\cite{fedprox}.
    See Algorithm~\ref{alg:SuPerFed} for finding a complete procedure of \texttt{SuPerFed}.

\paragraph{Note on the Orthogonality Regularization}
    The hypothesis supporting the effectiveness of the orthogonality regularization is based on the seminal work of~\cite{fort2019deep}.
    This paper presents empirical evidence that can be translated into the following proposition: ``\textit{functionally similar deep networks share similar optimization trajectories, represented by a non-zero cosine similarity between their weights}''.
    For an effective esnemble learning, however, functionally diverse networks are required to leverage the power of cooperation between different models.
    Thus, a contrapositive statement can be used as a mechanism to boost an ensemble perspective:
    ``\textit{deep networks having a zero cosine similarity in their weight space are functionally dissimilar with each other}''.
    This has been proved to work empirically well in a data-centralized setting~\cite{nnsubspaces}. 

\newpage
\section{Experiments}
\subsection{Experimental Setup} 
    \begin{table}[b]
\caption{Detailed experimental configurations of main experiments of Chapter~\ref{ch:superfed}}
\label{tab:2_table2}
\centering 
\begin{center}
\resizebox{\textwidth}{!}{%
\begin{tabular}{l|cc|cc|cc|cc}
\hline
\multicolumn{1}{c|}{\textbf{Setting}} & \multicolumn{2}{c|}{\begin{tabular}[c]{@{}c@{}}Pathological\\ Non-IID\end{tabular}} & \multicolumn{2}{c|}{\begin{tabular}[c]{@{}c@{}}Dirichlet distribution-based\\ Non-IID\end{tabular}} & \multicolumn{2}{c|}{\begin{tabular}[c]{@{}c@{}}Realistic\\ Non-IID\end{tabular}} & \multicolumn{2}{c}{\begin{tabular}[c]{@{}c@{}}Label Noise\\ (Pair/Symmetric)\end{tabular}} \\ \hline
\multicolumn{1}{c|}{\textbf{Dataset}} & \textbf{MNIST} & \textbf{CIFAR-10} & \textbf{CIFAR-100} & \textbf{TinyImageNet} & \textbf{FEMNIST} & \textbf{Shakespeare} & \textbf{MNIST} & \textbf{CIFAR-10} \\ \hline
$T$ & 500 & 500 & 500 & 500 & 500 & 500 & 500 & 500 \\
$E$ & 10 & 10 & 5 & 5 & 10 & 5 & 10 & 10 \\
$B$ & 10 & 10 & 20 & 20 & 10 & 50 & 10 & 10 \\
$K$ & 50,100,500 & 50,100,500 & 100 & 200 & 730 & 660 & 100 & 100 \\
$\eta$ & 0.01 & 0.01 & 0.01 & 0.02 & 0.01 & 0.8 & 0.01 & 0.01 \\
\hline
Model & TwoNN & TwoCNN & ResNet & MobileNet & VGG & StackedLSTM & TwoNN & TwoCNN \\ \hline
\end{tabular}%
}
\end{center}
\end{table}
    To verify the superiority of \texttt{SuPerFed}, we focus on two points in our experiments: 
    (i) personalization performance in various statistical heterogeneity scenarios, 
    and (ii) merits of explicit introduction of the \textit{mode connectivity}, e.g., improved calibration and robustness to possible label noise. 
    Until now, we only considered the setting of applying $\lambda$ identically to the whole parameters of $\boldsymbol{\theta}^f$ and $\boldsymbol{\theta}^l$. 
    We name this mixing scheme as \textit{model-mixing}, in short, \texttt{SuPerFed-MM}. 
    On the one hand, it is also possible to sample different $\lambda$ for each parameter of different layers, i.e., mixing two models in a layer-wise manner. 
    We also adopt this setting and name it \textit{layer-mixing}, in short, \texttt{SuPerFed-LM}.
    Accordingly, we compare these subdivided methods with other baselines in all experiments.
    For baseline methods, we select model mixture-based PFL methods: 
    \texttt{FedPer}~\cite{FedPer}, \texttt{LG-FedAvg}~\cite{lgfedavg}, \texttt{APFL}~\cite{apfl}, \texttt{pFedMe}~\cite{pfedme}, \texttt{Ditto}~\cite{ditto}, \texttt{FedRep}~\cite{fedrep}, along with basic single-model based FL methods, \texttt{FedAvg}~\cite{fedavg} \texttt{FedProx}~\cite{fedprox} and \texttt{SCAFFOLD}~\cite{ka+19}.
    
    Throughout all experiments, if not specified, we set 5 clients to be sampled at every round and used stochastic gradient descent (SGD) optimizer with a momentum of 0.9 and a weight decay factor of 0.0001. 
    We also let the client randomly split their data into a training set and a test set with a fraction of 0.2 to estimate the performance of each FL algorithm on each client's test set with task-specific metrics: (i) Top-1 \& Top-5 accuracy for measuring PFL performances and (ii) expected calibration error (ECE~\cite{ece}) and maximum calibration error (MCE~\cite{mce}) for measuring calibration of a model. 
    For a stable convergence of a global model, we applied applied learning rate decay by 1\% per each round, which is necessary in theory as well for the convergence of a FL algorithm~\cite{convergence}. 
    Each model architecture is borrowed from the original paper that proposed it: 
    {TwoNN}, {TwoCNN}, {StackedLSTM}~\cite{fedavg}, 
    {ResNet}~\cite{resnet}, 
    {MobileNet}~\cite{mobilenet}, 
    and {VGG}~\cite{vgg}.
    For more details, we summarized all information in Table~\ref{tab:2_table2}, including the number of clients $K$, total rounds $T$, local batch size $B$, number of local epochs $E$, learning rate $\eta$, and the model architecture. 
    
\newpage   
\subsection{Personalization Performance}
    For the estimation of the performance of \texttt{SuPerFed} as PFL methods, we simulate three different non-IID scenarios. 
    i) a \textit{pathological non-IID} setting proposed by~\cite{fedavg}, which assumes most clients have samples from two classes for a multi-class classification task. i) \textit{Dirichlet distribution-based non-IID} setting proposed by~\cite{diri}, in which the Dirichlet distribution with its concentration parameter $\alpha$ determines the skewness of local class distribution of each client. 
    All clients have samples from only one class when using $\alpha\rightarrow0$, whereas $\alpha\rightarrow\infty$ divides samples into an identical distribution. 
    iii) \textit{Realistic non-IID} setting proposed in~\cite{leaf}, which provides several benchmark datasets tailored to PFL.
    
    \paragraph{Pathological Non-IID Setting}
    In this setting, we used two multi-class classification datasets, MNIST~\cite{mnist} and CIFAR-10~\cite{cifar} and both have 10 classes. We used the two-layered fully-connected network for the MNIST dataset and the two-layered convolutional neural network (CNN) for the CIFAR-10 dataset as proposed in~\cite{fedavg}. For each dataset, we set the number of total clients ($K$=50, 100, 500) to check the scalability of the PFL methods. 
    The results of the \textit{pathological non-IID setting} are shown in Table~\ref{tab:2_table3}. 
    It is notable that our proposed method beats most of the existing model mixture-based PFL methods with a small standard deviation.
    
    \paragraph{Dirichlet Distribution-based Non-IID Setting} 
    In this setting, we used other multi-class classification datasets to simulate a more challenging setting than the \textit{pathological non-IID} setting. 
    We use CIFAR-100\cite{cifar} and TinyImageNet\cite{tinyimagenet} datasets, having 100 classes and 200 classes each. 
    We also selected deeper architectures, ResNet\cite{resnet} for CIFAR-100 and MobileNet\cite{mobilenet} for TinyImageNet. 
    For each dataset, we adjust the concentration parameter $\alpha={1, 10, 100}$ to control the degree of statistical heterogeneity across clients. 
    Note that the smaller the $\alpha$, the more heterogeneous each client's data distribution is. 
    Table~\ref{tab:2_table4} presents the results of the \textit{Dirichlet distribution-based non-IID} setting. 
    Both of our methods are less affected by the degree of statistical heterogeneity (i.e., non-IIDness) determined by $\alpha$.
    
    \paragraph{Realistic Non-IID Setting} 
    We used FEMNIST and Shakespeare datasets in the LEAF benchmark~\cite{leaf}. 
    As the purpose of these datasets is to simulate a realistic FL scenario, each dataset is naturally split for a image recognition task sequence modeling task, each. 
    The selected two datasets are for multi-class classification (62 classes) and next character prediction (80 characters) given a sentence, respectively. 
    We used VGG~\cite{vgg} for the FEMNIST dataset and networks with two stacked LSTM layers were used in~\cite{fedavg} for the Shakespeare dataset. 
    The results of the \textit{realistic non-IID} setting are shown in Table~\ref{tab:2_table5}. 
    In this massively distributed setting, our methods showed a consistent gain in terms of personalization performance.
    \begin{table}[H]
\centering
\caption[Experimental results on the \textit{pathological non-IID} setting of Chapter~\ref{ch:superfed}]
{Experimental results on the \textit{pathological non-IID} setting (MNIST and CIFAR-10 datasets) compared with other FL and PFL methods. 
The Top-1 accuracy is reported with a standard deviation.}
\label{tab:2_table3}
\resizebox{0.8\linewidth}{!}{%
\begin{tabular}{!{}lccc|ccc!{}}
\toprule
\textbf{Dataset} & \multicolumn{3}{c|}{\textbf{MNIST}} & \multicolumn{3}{c}{\textbf{CIFAR-10}} \\
\textbf{Method} & \multicolumn{3}{c|}{(Acc. 1)} & \multicolumn{3}{c}{(Acc. 1)} \\ \cmidrule(l){2-7} 
\multicolumn{1}{c}{\# clients} & 50 & 100 & 500 & 50 & 100 & 500 \\
\multicolumn{1}{c}{\# samples} & 960 & 480 & 96 & 800 & 400 & 80 \\ \midrule
\texttt{FedAvg}~\cite{fedavg} & \begin{tabular}[c]{@{}c@{}}95.69\\ \footnotesize \color[HTML]{9B9B9B}(2.39)\end{tabular} & \begin{tabular}[c]{@{}c@{}}89.78\\ \footnotesize \color[HTML]{9B9B9B}(11.30)\end{tabular} & \begin{tabular}[c]{@{}c@{}}96.04\\ \footnotesize\color[HTML]{9B9B9B}(4.74)\end{tabular} & \begin{tabular}[c]{@{}c@{}}43.09\\ \footnotesize\color[HTML]{9B9B9B}(24.56)\end{tabular} & \begin{tabular}[c]{@{}c@{}}36.19\\ \footnotesize\color[HTML]{9B9B9B}(29.54)\end{tabular} & \begin{tabular}[c]{@{}c@{}}47.90\\ \footnotesize\color[HTML]{9B9B9B}(25.05)\end{tabular} \\
\texttt{FedProx}~\cite{fedprox} & \begin{tabular}[c]{@{}c@{}}95.13\\ \footnotesize \color[HTML]{9B9B9B}(2.67)\end{tabular} & \begin{tabular}[c]{@{}c@{}}93.25\\ \footnotesize \color[HTML]{9B9B9B}(6.12)\end{tabular} & \begin{tabular}[c]{@{}c@{}}96.50\\ \footnotesize \color[HTML]{9B9B9B}(4.52)\end{tabular} & \begin{tabular}[c]{@{}c@{}}49.01\\ \footnotesize \color[HTML]{9B9B9B}(19.87)\end{tabular} & \begin{tabular}[c]{@{}c@{}}38.56\\ \footnotesize \color[HTML]{9B9B9B}(28.11)\end{tabular} & \begin{tabular}[c]{@{}c@{}}48.60\\ \footnotesize \color[HTML]{9B9B9B}(25.71)\end{tabular} \\
\texttt{SCAFFOLD}~\cite{ka+19} & \begin{tabular}[c]{@{}c@{}}95.50\\ \footnotesize \color[HTML]{9B9B9B}(2.71)\end{tabular} & \begin{tabular}[c]{@{}c@{}}90.58\\ \footnotesize \color[HTML]{9B9B9B}(10.13)\end{tabular} & \begin{tabular}[c]{@{}c@{}}96.60\\ \footnotesize \color[HTML]{9B9B9B}(4.26)\end{tabular} & \begin{tabular}[c]{@{}c@{}}43.81\\ \footnotesize \color[HTML]{9B9B9B}(24.30)\end{tabular} & \begin{tabular}[c]{@{}c@{}}36.31\\ \footnotesize \color[HTML]{9B9B9B}(29.42)\end{tabular} & \begin{tabular}[c]{@{}c@{}}40.27\\ \footnotesize \color[HTML]{9B9B9B}(26.90)\end{tabular} \\ \midrule
\texttt{LG-FedAvg}~\cite{lgfedavg} & \begin{tabular}[c]{@{}c@{}}98.21\\ \footnotesize \color[HTML]{9B9B9B}(1.28)\end{tabular} & \begin{tabular}[c]{@{}c@{}}97.52\\ \footnotesize \color[HTML]{9B9B9B}(2.11)\end{tabular} & \begin{tabular}[c]{@{}c@{}}96.05\\ \footnotesize \color[HTML]{9B9B9B}(5.02)\end{tabular} & \begin{tabular}[c]{@{}c@{}}89.03\\ \footnotesize \color[HTML]{9B9B9B}(4.53)\end{tabular} & \begin{tabular}[c]{@{}c@{}}70.25\\ \footnotesize \color[HTML]{9B9B9B}(35.66)\end{tabular} & \begin{tabular}[c]{@{}c@{}}78.52\\ \footnotesize \color[HTML]{9B9B9B}(11.22)\end{tabular} \\
\texttt{FedPer}~\cite{FedPer} & \begin{tabular}[c]{@{}c@{}}99.23\\ \footnotesize \color[HTML]{9B9B9B}(0.66)\end{tabular} & \begin{tabular}[c]{@{}c@{}}99.14\\ \footnotesize \color[HTML]{9B9B9B}(0.93)\end{tabular} & \begin{tabular}[c]{@{}c@{}}98.67\\ \footnotesize \color[HTML]{9B9B9B}(2.61)\end{tabular} & \begin{tabular}[c]{@{}c@{}}89.10\\ \footnotesize \color[HTML]{9B9B9B}(5.41)\end{tabular} & \begin{tabular}[c]{@{}c@{}}87.99\\ \footnotesize \color[HTML]{9B9B9B}(5.70)\end{tabular} & \begin{tabular}[c]{@{}c@{}}82.35\\ \footnotesize \color[HTML]{9B9B9B}(9.85)\end{tabular} \\
\texttt{APFL}~\cite{apfl} & \begin{tabular}[c]{@{}c@{}}99.40\\ \footnotesize \color[HTML]{9B9B9B}(0.58)\end{tabular} & \begin{tabular}[c]{@{}c@{}}99.19\\ \footnotesize \color[HTML]{9B9B9B}(0.92)\end{tabular} & \begin{tabular}[c]{@{}c@{}}98.98\\ \footnotesize \color[HTML]{9B9B9B}(2.22)\end{tabular} & \begin{tabular}[c]{@{}c@{}}92.83\\ \footnotesize \color[HTML]{9B9B9B}(3.47)\end{tabular} & \begin{tabular}[c]{@{}c@{}}91.73\\ \footnotesize \color[HTML]{9B9B9B}(4.61)\end{tabular} & \begin{tabular}[c]{@{}c@{}}87.38\\ \footnotesize \color[HTML]{9B9B9B}(9.39)\end{tabular} \\
\texttt{pFedMe}~\cite{pfedme} & \begin{tabular}[c]{@{}c@{}}81.10\\ \footnotesize \color[HTML]{9B9B9B}(8.52)\end{tabular} & \begin{tabular}[c]{@{}c@{}}82.48\\ \footnotesize \color[HTML]{9B9B9B}(7.62)\end{tabular} & \begin{tabular}[c]{@{}c@{}}81.96\\ \footnotesize \color[HTML]{9B9B9B}(12.28)\end{tabular} & \begin{tabular}[c]{@{}c@{}}92.97\\ \footnotesize \color[HTML]{9B9B9B}(3.07)\end{tabular} & \begin{tabular}[c]{@{}c@{}}92.07\\ \footnotesize \color[HTML]{9B9B9B}(5.05)\end{tabular} & \begin{tabular}[c]{@{}c@{}}88.30\\ \footnotesize \color[HTML]{9B9B9B}(8.53)\end{tabular} \\
\texttt{Ditto}~\cite{ditto} & \begin{tabular}[c]{@{}c@{}}97.07\\ \footnotesize \color[HTML]{9B9B9B}(1.38)\end{tabular} & \begin{tabular}[c]{@{}c@{}}97.13\\ \footnotesize \color[HTML]{9B9B9B}(2.06)\end{tabular} & \begin{tabular}[c]{@{}c@{}}97.20\\ \footnotesize \color[HTML]{9B9B9B}(3.72)\end{tabular} & \begin{tabular}[c]{@{}c@{}}85.53\\ \footnotesize \color[HTML]{9B9B9B}(6.22)\end{tabular} & \begin{tabular}[c]{@{}c@{}}83.01\\ \footnotesize \color[HTML]{9B9B9B}(5.62)\end{tabular} & \begin{tabular}[c]{@{}c@{}}84.45\\ \footnotesize \color[HTML]{9B9B9B}(10.67)\end{tabular} \\
\texttt{FedRep}~\cite{fedrep} & \begin{tabular}[c]{@{}c@{}}99.11\\ \footnotesize \color[HTML]{9B9B9B}(0.63)\end{tabular} & \begin{tabular}[c]{@{}c@{}}99.04\\ \footnotesize \color[HTML]{9B9B9B}(1.02)\end{tabular} & \begin{tabular}[c]{@{}c@{}}97.94\\ \footnotesize \color[HTML]{9B9B9B}(3.37)\end{tabular} & \begin{tabular}[c]{@{}c@{}}82.00\\ \footnotesize \color[HTML]{9B9B9B}(5.41)\end{tabular} & \begin{tabular}[c]{@{}c@{}}81.27\\ \footnotesize \color[HTML]{9B9B9B}(7.90)\end{tabular} & \begin{tabular}[c]{@{}c@{}}80.66\\ \footnotesize \color[HTML]{9B9B9B}(11.00)\end{tabular} \\ \midrule
\rowcolor[HTML]{FFF5E6} 
\texttt{SuPerFed-MM} & \begin{tabular}[c]{@{}c@{}}\underline{99.45}\\ \footnotesize\color[HTML]{9B9B9B}(0.46)\end{tabular} & \begin{tabular}[c]{@{}c@{}}\textbf{99.38}\\ \footnotesize \color[HTML]{9B9B9B}(0.93)\end{tabular} & \begin{tabular}[c]{@{}c@{}}\textbf{99.24}\\ \footnotesize \color[HTML]{9B9B9B}(2.12)\end{tabular} & \begin{tabular}[c]{@{}c@{}}\textbf{94.05}\\ \footnotesize \color[HTML]{9B9B9B}(3.18)\end{tabular} & \begin{tabular}[c]{@{}c@{}}\textbf{93.25}\\ \footnotesize \color[HTML]{9B9B9B}(3.80)\end{tabular} & \begin{tabular}[c]{@{}c@{}}\textbf{90.81}\\ \footnotesize\color[HTML]{9B9B9B}(9.35)\end{tabular} \\
\rowcolor[HTML]{FFF5E6} 
\texttt{SuPerFed-LM} & \begin{tabular}[c]{@{}c@{}}\textbf{99.48}\\ \footnotesize\color[HTML]{9B9B9B}(0.54)\end{tabular} & \begin{tabular}[c]{@{}c@{}}\underline{99.31}\\ \footnotesize \color[HTML]{9B9B9B}(1.09)\end{tabular} & \begin{tabular}[c]{@{}c@{}}\underline{98.83}\\ \footnotesize \color[HTML]{9B9B9B}(3.02)\end{tabular} & \begin{tabular}[c]{@{}c@{}}\underline{93.88}\\ \footnotesize \color[HTML]{9B9B9B}(3.55)\end{tabular} & \begin{tabular}[c]{@{}c@{}}\underline{93.20}\\ \footnotesize \color[HTML]{9B9B9B}(4.19)\end{tabular} & \begin{tabular}[c]{@{}c@{}}\underline{89.63}\\ \footnotesize \color[HTML]{9B9B9B}(11.11)\end{tabular} \\ \bottomrule
\end{tabular}%
}
\end{table}
    \begin{table}[H]
\centering
\caption[Experimental results on the \textit{Dirichlet distribution-based non-IID} setting of Chapter~\ref{ch:superfed}]
{Experimental results on the \textit{Dirichlet distribution-based non-IID} setting (CIFAR-100 and TinyImageNet datasets) compared with other FL and PFL methods.
The Top-5 accuracy is reported with a standard deviation.}
\label{tab:2_table4}
\resizebox{0.8\linewidth}{!}{%
\begin{tabular}{!{}lccc|ccc!{}}
\toprule
\textbf{Dataset} & \multicolumn{3}{c|}{\textbf{CIFAR-100}} & \multicolumn{3}{c}{\textbf{TinyImageNet}} \\
\textbf{Method} & \multicolumn{3}{c|}{(Acc. 5)} & \multicolumn{3}{c}{(Acc. 5)} \\ \cmidrule(l){2-7}
\multicolumn{1}{c}{\# clients} & 100 & 100 & 100 & 200 & 200 & 200 \\
\multicolumn{1}{c}{concentration ($\alpha$)} & 1 & 10 & 100 & 1 & 10 & 100 \\ \midrule
\texttt{FedAvg}~\cite{fedavg} & \begin{tabular}[c]{@{}c@{}}58.12\\ \footnotesize \color[HTML]{9B9B9B}(7.06)\end{tabular} & \begin{tabular}[c]{@{}c@{}}\underline{59.04}\\ \footnotesize \color[HTML]{9B9B9B}(7.19)\end{tabular} & \begin{tabular}[c]{@{}c@{}}58.49\\ \footnotesize\color[HTML]{9B9B9B}(5.27)\end{tabular} & \begin{tabular}[c]{@{}c@{}}46.61\\ \footnotesize\color[HTML]{9B9B9B}(5.64)\end{tabular} & \begin{tabular}[c]{@{}c@{}}48.90\\ \footnotesize\color[HTML]{9B9B9B}(5.50)\end{tabular} & \begin{tabular}[c]{@{}c@{}}48.90\\ \footnotesize\color[HTML]{9B9B9B}(5.40)\end{tabular} \\
\texttt{FedProx}~\cite{fedprox} & \begin{tabular}[c]{@{}c@{}}57.71\\ \footnotesize \color[HTML]{9B9B9B}(6.79)\end{tabular} & \begin{tabular}[c]{@{}c@{}}58.24\\ \footnotesize \color[HTML]{9B9B9B}(5.94)\end{tabular} & \begin{tabular}[c]{@{}c@{}}58.75\\ \footnotesize \color[HTML]{9B9B9B}(5.56)\end{tabular} & \begin{tabular}[c]{@{}c@{}}\underline{47.37}\\ \footnotesize \color[HTML]{9B9B9B}(5.94)\end{tabular} & \begin{tabular}[c]{@{}c@{}}47.73\\ \footnotesize \color[HTML]{9B9B9B}(5.94)\end{tabular} & \begin{tabular}[c]{@{}c@{}}48.97\\ \footnotesize \color[HTML]{9B9B9B}(5.02)\end{tabular} \\
\texttt{SCAFFOLD}~\cite{ka+19} & \begin{tabular}[c]{@{}c@{}}51.16\\ \footnotesize \color[HTML]{9B9B9B}(6.79)\end{tabular} & \begin{tabular}[c]{@{}c@{}}51.40\\ \footnotesize \color[HTML]{9B9B9B}(5.22)\end{tabular} & \begin{tabular}[c]{@{}c@{}}52.90\\ \footnotesize \color[HTML]{9B9B9B}(4.89)\end{tabular} & \begin{tabular}[c]{@{}c@{}}46.54\\ \footnotesize \color[HTML]{9B9B9B}(5.49)\end{tabular} & \begin{tabular}[c]{@{}c@{}}48.77\\ \footnotesize \color[HTML]{9B9B9B}(5.49)\end{tabular} & \begin{tabular}[c]{@{}c@{}}48.27\\ \footnotesize \color[HTML]{9B9B9B}(5.32)\end{tabular} \\ \midrule
\texttt{LG-FedAvg}~\cite{lgfedavg} & \begin{tabular}[c]{@{}c@{}}28.88\\ \footnotesize \color[HTML]{9B9B9B}(5.64)\end{tabular} & \begin{tabular}[c]{@{}c@{}}21.25\\ \footnotesize \color[HTML]{9B9B9B}(4.64)\end{tabular} & \begin{tabular}[c]{@{}c@{}}20.05\\ \footnotesize \color[HTML]{9B9B9B}(4.61)\end{tabular} & \begin{tabular}[c]{@{}c@{}}14.70\\ \footnotesize \color[HTML]{9B9B9B}(3.84)\end{tabular} & \begin{tabular}[c]{@{}c@{}}9.86\\ \footnotesize \color[HTML]{9B9B9B}(3.13)\end{tabular} & \begin{tabular}[c]{@{}c@{}}9.25\\ \footnotesize \color[HTML]{9B9B9B}(2.89)\end{tabular} \\
\texttt{FedPer}~\cite{FedPer} & \begin{tabular}[c]{@{}c@{}}46.78\\ \footnotesize \color[HTML]{9B9B9B}(7.63)\end{tabular} & \begin{tabular}[c]{@{}c@{}}35.73\\ \footnotesize \color[HTML]{9B9B9B}(6.80)\end{tabular} & \begin{tabular}[c]{@{}c@{}}35.52\\ \footnotesize \color[HTML]{9B9B9B}(6.58)\end{tabular} & \begin{tabular}[c]{@{}c@{}}21.90\\ \footnotesize \color[HTML]{9B9B9B}(4.71)\end{tabular} & \begin{tabular}[c]{@{}c@{}}11.10\\ \footnotesize \color[HTML]{9B9B9B}(3.19)\end{tabular} & \begin{tabular}[c]{@{}c@{}}9.63\\ \footnotesize \color[HTML]{9B9B9B}(3.12)\end{tabular} \\
\texttt{APFL}~\cite{apfl} & \begin{tabular}[c]{@{}c@{}}\underline{61.13}\\ \footnotesize \color[HTML]{9B9B9B}(6.86)\end{tabular} & \begin{tabular}[c]{@{}c@{}}56.90\\ \footnotesize \color[HTML]{9B9B9B}(7.05)\end{tabular} & \begin{tabular}[c]{@{}c@{}}55.43\\ \footnotesize \color[HTML]{9B9B9B}(5.45)\end{tabular} & \begin{tabular}[c]{@{}c@{}}41.98\\ \footnotesize\color[HTML]{9B9B9B}(5.94)\end{tabular} & \begin{tabular}[c]{@{}c@{}}34.74\\ \footnotesize\color[HTML]{9B9B9B}(5.14)\end{tabular} & \begin{tabular}[c]{@{}c@{}}34.23\\ \footnotesize\color[HTML]{9B9B9B}(5.07)\end{tabular} \\
\texttt{pFedMe}~\cite{pfedme} & \begin{tabular}[c]{@{}c@{}}19.00\\ \footnotesize \color[HTML]{9B9B9B}(5.37)\end{tabular} & \begin{tabular}[c]{@{}c@{}}17.94\\ \footnotesize \color[HTML]{9B9B9B}(4.72)\end{tabular} & \begin{tabular}[c]{@{}c@{}}18.28\\ \footnotesize \color[HTML]{9B9B9B}(3.41)\end{tabular} & \begin{tabular}[c]{@{}c@{}}6.05\\ \footnotesize \color[HTML]{9B9B9B}(2.84)\end{tabular} & \begin{tabular}[c]{@{}c@{}}8.01\\ \footnotesize \color[HTML]{9B9B9B}(2.92)\end{tabular} & \begin{tabular}[c]{@{}c@{}}7.69\\ \footnotesize \color[HTML]{9B9B9B}(2.41)\end{tabular} \\
\texttt{Ditto}~\cite{ditto} & \begin{tabular}[c]{@{}c@{}}60.04\\ \footnotesize \color[HTML]{9B9B9B}(6.82)\end{tabular} & \begin{tabular}[c]{@{}c@{}}58.55\\ \footnotesize \color[HTML]{9B9B9B}(7.12)\end{tabular} & \begin{tabular}[c]{@{}c@{}}58.73\\ \footnotesize \color[HTML]{9B9B9B}(5.39)\end{tabular} & \begin{tabular}[c]{@{}c@{}}46.36\\ \footnotesize \color[HTML]{9B9B9B}(5.44)\end{tabular} & \begin{tabular}[c]{@{}c@{}}43.84\\ \footnotesize \color[HTML]{9B9B9B}(5.44)\end{tabular} & \begin{tabular}[c]{@{}c@{}}43.11\\ \footnotesize \color[HTML]{9B9B9B}(5.35)\end{tabular} \\
\texttt{FedRep}~\cite{fedrep} & \begin{tabular}[c]{@{}c@{}}38.49\\ \footnotesize \color[HTML]{9B9B9B}(6.65)\end{tabular} & \begin{tabular}[c]{@{}c@{}}26.61\\ \footnotesize \color[HTML]{9B9B9B}(5.20)\end{tabular} & \begin{tabular}[c]{@{}c@{}}24.50\\ \footnotesize \color[HTML]{9B9B9B}(4.21)\end{tabular} & \begin{tabular}[c]{@{}c@{}}18.67\\ \footnotesize \color[HTML]{9B9B9B}(4.66)\end{tabular} & \begin{tabular}[c]{@{}c@{}}9.23\\ \footnotesize \color[HTML]{9B9B9B}(2.84)\end{tabular} & \begin{tabular}[c]{@{}c@{}}8.09\\ \footnotesize \color[HTML]{9B9B9B}(2.83)\end{tabular} \\ \midrule
\rowcolor[HTML]{FFF5E6} 
\texttt{SuPerFed-MM} & \begin{tabular}[c]{@{}c@{}}60.14\\ \footnotesize\color[HTML]{9B9B9B}(6.24)\end{tabular} & \begin{tabular}[c]{@{}c@{}}58.32\\ \footnotesize\color[HTML]{9B9B9B}(6.25)\end{tabular} & \begin{tabular}[c]{@{}c@{}}\textbf{59.08}\\ \footnotesize\color[HTML]{9B9B9B}(5.12)\end{tabular} & \begin{tabular}[c]{@{}c@{}}\textbf{50.07}\\ \footnotesize\color[HTML]{9B9B9B}(5.73)\end{tabular} & \begin{tabular}[c]{@{}c@{}}\textbf{49.86}\\ \footnotesize\color[HTML]{9B9B9B}(5.03)\end{tabular} & \begin{tabular}[c]{@{}c@{}}\textbf{49.73}\\ \footnotesize\color[HTML]{9B9B9B}(4.84)\end{tabular} \\
\rowcolor[HTML]{FFF5E6} 
\texttt{SuPerFed-LM} & \begin{tabular}[c]{@{}c@{}}\textbf{62.50}\\ \footnotesize\color[HTML]{9B9B9B}(6.34)\end{tabular} & \begin{tabular}[c]{@{}c@{}}\textbf{61.64}\\ \footnotesize\color[HTML]{9B9B9B}(6.23)\end{tabular} & \begin{tabular}[c]{@{}c@{}}\underline{59.05}\\ \footnotesize\color[HTML]{9B9B9B}(5.59)\end{tabular} & \begin{tabular}[c]{@{}c@{}}47.28\\ \footnotesize\color[HTML]{9B9B9B}(5.19)\end{tabular} & \begin{tabular}[c]{@{}c@{}}\underline{48.98}\\ \footnotesize\color[HTML]{9B9B9B}(4.79)\end{tabular} & \begin{tabular}[c]{@{}c@{}}\underline{49.29}\\ \footnotesize\color[HTML]{9B9B9B}(4.82)\end{tabular} \\ \bottomrule
\end{tabular}%
}
\end{table}
    \begin{table}[H]
\centering
\caption[Experimental results on the \textit{realistic non-IID} setting of Chapter~\ref{ch:superfed}]
{Experimental results on the \textit{realistic non-IID} setting (FEMNIST and Shakespeare datasets) compared with other FL and PFL methods. 
The Top-1 and Top-5 accuracy is reported with a standard deviation.}
\label{tab:2_table5}
\resizebox{0.6\textwidth}{!}{%
\begin{tabular}{!{}lcc|cc!{}}
\toprule
{\textbf{Dataset}} & \multicolumn{2}{c|}{\textbf{FEMNIST}} & \multicolumn{2}{c}{\textbf{Shakespeare}} \\
{\textbf{Method}} & \multicolumn{2}{c|}{(Acc.)} & \multicolumn{2}{c}{(Acc.)} \\ \cmidrule(l){2-5}
\multicolumn{1}{c}{\# clients} & \multicolumn{1}{c}{730} & \multicolumn{1}{c|}{730} & \multicolumn{1}{c}{660} & \multicolumn{1}{c}{660} \\
\multicolumn{1}{c}{Metric} & \multicolumn{1}{c}{Top-1} & \multicolumn{1}{c|}{Top-5} & \multicolumn{1}{c}{Top-1} & \multicolumn{1}{c}{Top-5} \\ \midrule
\texttt{FedAvg}~\cite{fedavg} & \begin{tabular}[c]{@{}c@{}}80.12\\ \footnotesize \color[HTML]{9B9B9B}(12.01)\end{tabular} & \begin{tabular}[c]{@{}c@{}}98.74\\ \footnotesize \color[HTML]{9B9B9B}(2.97)\end{tabular} & \begin{tabular}[c]{@{}c@{}}50.90\\ \footnotesize \color[HTML]{9B9B9B}(7.85)\end{tabular} & \begin{tabular}[c]{@{}c@{}}80.15\\ \footnotesize \color[HTML]{9B9B9B}(7.87)\end{tabular} \\
\texttt{FedProx}~\cite{fedprox} & \begin{tabular}[c]{@{}c@{}}80.23\\ \footnotesize \color[HTML]{9B9B9B}(11.88)\end{tabular} & \begin{tabular}[c]{@{}c@{}}98.73\\ \footnotesize \color[HTML]{9B9B9B}(2.94)\end{tabular} & \begin{tabular}[c]{@{}c@{}}51.33\\ \footnotesize \color[HTML]{9B9B9B}(7.54)\end{tabular} & \begin{tabular}[c]{@{}c@{}}80.31\\ \footnotesize \color[HTML]{9B9B9B}(6.95)\end{tabular} \\
\texttt{SCAFFOLD}~\cite{ka+19} & \begin{tabular}[c]{@{}c@{}}80.03\\ \footnotesize \color[HTML]{9B9B9B}(11.78)\end{tabular} & \begin{tabular}[c]{@{}c@{}}\underline{98.85}\\ \footnotesize \color[HTML]{9B9B9B}(2.77)\end{tabular} & \begin{tabular}[c]{@{}c@{}}50.76\\ \footnotesize \color[HTML]{9B9B9B}(8.01)\end{tabular} & \begin{tabular}[c]{@{}c@{}}80.43\\ \footnotesize \color[HTML]{9B9B9B}(7.09)\end{tabular} \\ \midrule
\texttt{LG-FedAvg}~\cite{lgfedavg} & \begin{tabular}[c]{@{}c@{}}50.84\\ \footnotesize \color[HTML]{9B9B9B}(20.97)\end{tabular} & \begin{tabular}[c]{@{}c@{}}75.11\\ \footnotesize \color[HTML]{9B9B9B}(21.49)\end{tabular} & \begin{tabular}[c]{@{}c@{}}33.88\\ \footnotesize \color[HTML]{9B9B9B}(10.28)\end{tabular} & \begin{tabular}[c]{@{}c@{}}62.84\\ \footnotesize \color[HTML]{9B9B9B}(13.16)\end{tabular} \\
\texttt{FedPer}~\cite{FedPer} & \begin{tabular}[c]{@{}c@{}}73.79\\ \footnotesize \color[HTML]{9B9B9B}(14.10)\end{tabular} & \begin{tabular}[c]{@{}c@{}}86.39\\ \footnotesize \color[HTML]{9B9B9B}(14.70)\end{tabular} & \begin{tabular}[c]{@{}c@{}}45.82\\ \footnotesize \color[HTML]{9B9B9B}(8.10)\end{tabular} & \begin{tabular}[c]{@{}c@{}}75.68\\ \footnotesize \color[HTML]{9B9B9B}(9.25)\end{tabular} \\
\texttt{APFL}~\cite{apfl} & \begin{tabular}[c]{@{}c@{}}\underline{84.85}\\ \footnotesize \color[HTML]{9B9B9B}(8.83)\end{tabular} & \begin{tabular}[c]{@{}c@{}}98.83\\ \footnotesize \color[HTML]{9B9B9B}(2.73)\end{tabular} & \begin{tabular}[c]{@{}c@{}}\underline{54.08}\\ \footnotesize \color[HTML]{9B9B9B}(8.31)\end{tabular} & \begin{tabular}[c]{@{}c@{}}83.32\\ \footnotesize \color[HTML]{9B9B9B}(6.22)\end{tabular} \\
\texttt{pFedMe}~\cite{pfedme} & \begin{tabular}[c]{@{}c@{}}5.98\\ \footnotesize \color[HTML]{9B9B9B}(4.55)\end{tabular} & \begin{tabular}[c]{@{}c@{}}24.64\\ \footnotesize \color[HTML]{9B9B9B}(9.43)\end{tabular} & \begin{tabular}[c]{@{}c@{}}32.29\\ \footnotesize \color[HTML]{9B9B9B}(6.64)\end{tabular} & \begin{tabular}[c]{@{}c@{}}63.12\\ \footnotesize \color[HTML]{9B9B9B}(8.00)\end{tabular} \\
\texttt{Ditto}~\cite{ditto} & \begin{tabular}[c]{@{}c@{}}64.61\\ \footnotesize \color[HTML]{9B9B9B}(31.49)\end{tabular} & \begin{tabular}[c]{@{}c@{}}81.14\\ \footnotesize \color[HTML]{9B9B9B}(28.56)\end{tabular} & \begin{tabular}[c]{@{}c@{}}49.04\\ \footnotesize \color[HTML]{9B9B9B}(10.22)\end{tabular} & \begin{tabular}[c]{@{}c@{}}78.14\\ \footnotesize \color[HTML]{9B9B9B}(12.61)\end{tabular} \\
\texttt{FedRep}~\cite{fedrep} & \begin{tabular}[c]{@{}c@{}}59.27\\ \footnotesize \color[HTML]{9B9B9B}(15.72)\end{tabular} & \begin{tabular}[c]{@{}c@{}}70.42\\ \footnotesize \color[HTML]{9B9B9B}(15.82)\end{tabular} & \begin{tabular}[c]{@{}c@{}}38.15\\ \footnotesize \color[HTML]{9B9B9B}(9.54)\end{tabular} & \begin{tabular}[c]{@{}c@{}}68.65\\ \footnotesize \color[HTML]{9B9B9B}(12.50)\end{tabular} \\ \midrule
\rowcolor[HTML]{FFF5E6} 
\texttt{SuPerFed-MM} & \begin{tabular}[c]{@{}c@{}}\textbf{85.20}\\ \footnotesize\color[HTML]{9B9B9B}(8.40)\end{tabular} & \begin{tabular}[c]{@{}c@{}}\textbf{99.16}\\ \footnotesize\color[HTML]{9B9B9B}(2.13)\end{tabular} & \begin{tabular}[c]{@{}c@{}}\textbf{54.52}\\ \footnotesize\color[HTML]{9B9B9B}(7.54)\end{tabular} & \begin{tabular}[c]{@{}c@{}}\textbf{84.27}\\ \footnotesize\color[HTML]{9B9B9B}(6.00)\end{tabular} \\
\rowcolor[HTML]{FFF5E6} 
\texttt{SuPerFed-LM} & \begin{tabular}[c]{@{}c@{}}83.36\\ \footnotesize\color[HTML]{9B9B9B}(9.61)\end{tabular} & \begin{tabular}[c]{@{}c@{}}98.81\\ \footnotesize\color[HTML]{9B9B9B}(2.58)\end{tabular} & \begin{tabular}[c]{@{}c@{}}\textbf{54.52}\\ \footnotesize\color[HTML]{9B9B9B}(7.54)\end{tabular} & \begin{tabular}[c]{@{}c@{}}\underline{83.97}\\ \footnotesize\color[HTML]{9B9B9B}(5.72)\end{tabular} \\ \bottomrule
\end{tabular}%
}
\end{table}
    
\newpage
\subsection{Benefits from Mode Connectivity}
    \paragraph{Robustness and Calibration} 
    In the practical federated system, it is possible that each client suffers from noisy labels. 
    We anticipate that \texttt{SuPerFed} will benefit from the strength of ensemble models from induced \textit{mode connectivity} since our method can be viewed as an explicit ensemble of two models, $\boldsymbol{\theta}^{f}$ and $\boldsymbol{\theta}^{l}$. 
    One of the strengths is robustness to the label noise, as each of the model components of the ensemble method can view diverse data distributions~\cite{ensemblenoise,ensemblenoise2}. 
    In \texttt{SuPerFed}, a federated model and a local model are induced to be orthogonal to each other in the parameter space 
    (i.e., having low cosine similarity between parameters). 
    Thus, we assume each of them should inevitably learn from a different view of data distribution, and naturally  

    \paragraph{Simulation of Label Noise}
    Since we only have clean-labeled datasets (MNIST and CIFAR-10), we manually obfuscate the labels of training samples following two widely-used schemes: 
    i) pairwise flipping~\cite{pair} and i) symmetric flipping~\cite{symm}. 
    Both are based on the defining label transition matrix $\mathbf{T}$, 
    where each element of $\mathbf{T}_{ij}$ indicates the probability of transitioning from an original label $y=i$ to an obfuscated label $\Tilde{y}=j$, 
    i.e., $\mathbf{T}_{ij}=p(\Tilde{y}=j|y=i)$. 
    
    \paragraph{Simulation i): pairwise flipping}
    The pairwise flipping emulates a situation when labelers make mistakes only within similar classes.
    In this case, the transition matrix $\mathbf{T}$ for a pairwise flip given the noise ratio $\epsilon$ is defined as follows~\cite{pair}.
    \begin{equation}
    \begin{gathered}
        \mathbf{T}_{\text{pair}} = \left[\begin{array}{ccccc}
        1-\epsilon & \epsilon & 0 & \cdots & 0 \\
        0 & 1-\epsilon & \epsilon & \cdots & 0 \\
        \vdots & \vdots & \ddots & \ddots & \vdots \\
        0 & 0 & \cdots & 1-\epsilon & \epsilon \\
        \epsilon & 0 & \cdots & 0 & 1-\epsilon
        \end{array}\right]
    \end{gathered}
    \end{equation}
    We select the noise ratio from {0.1, 0.4} for the pair flipping label noise.
    
    \paragraph{Simulation i): symmetric flipping}
    The symmetric flipping assumes a probability of mislabeling of a clean label as other labels are uniformly distributed. 
    For $N$ classes, the transition matrix for symmetric flipping is defined as follows~\cite{symm}.
    \begin{equation}
    \begin{gathered}
        \mathbf{T}_{\text{sym}}=\left[\begin{array}{ccccc}
        1-\epsilon & \frac{\epsilon}{N-1} & \cdots & \frac{\epsilon}{N-1} & \frac{\epsilon}{N-1} \\
        \frac{\epsilon}{N-1} & 1-\epsilon & \frac{\epsilon}{N-1} & \cdots & \frac{\epsilon}{N-1} \\
        \vdots & \vdots & \ddots & \vdots & \vdots \\
        \frac{\epsilon}{N-1} & \cdots & \frac{\epsilon}{N-1} & 1-\epsilon & \frac{\epsilon}{N-1} \\
        \frac{\epsilon}{N-1} & \frac{\epsilon}{N-1} & \cdots & \frac{\epsilon}{N-1} & 1-\epsilon
        \end{array}\right]
    \end{gathered}
    \end{equation}
    We select the noise ratio from {0.2, 0.6}  for the symmetric flipping label noise.
    
    Using these noise schemes, we obfuscate each client's labels of the training set while keeping the labels of the test set in the same.
    We used MNIST and CIFAR-10 datasets with \textit{Dirichlet distribution-based non-IID} setting in which the concentration parameter $\alpha=10.0$. 
    Since the robustness of the label noise can be quantified through evaluating ECE~\cite{ece} and MCE~\cite{mce}, we introduced these two metrics additionally. 
    Both ECE and MCE measure the consistency between the prediction accuracy and the prediction confidence (i.e., calibration), thus lower values are preferred.
    From all the extensive experiments, our method shows stable performance in terms of calibration (i.e., low ECE and MCE), and therefore overall performance is not degraded much.
    The results are presented in Table~\ref{tab:2_table6}.
    \begin{table}[tp]
\centering
\caption[Experimental results on the label noise simulation setting of Chapter~\ref{ch:superfed}]{Experimental results on the label noise simulation with \textit{Dirichlet distirubiotn-based non-IID} setting with $\alpha=10$ on MNIST and CIFAR-10 compared with other FL and PFL methods. 
Per each cell, expected calibration error (ECE), maximum calibration error (MCE), and the best averaged Top-1 accuracy (in parentheses) are enumerated from top to bottom. Note that the lower ECE and MCE, the better the model calibration is.}
\label{tab:2_table6}
\resizebox{\textwidth}{!}{%
\begin{tabular}{l|cccc|cccc}
\hline
\multicolumn{1}{l|}{\textbf{Dataset}} & \multicolumn{4}{c|}{\textbf{MNIST}} & \multicolumn{4}{c}{\textbf{CIFAR-10}} \\ \hline
\multicolumn{1}{c|}{Noise type} & \multicolumn{2}{c|}{pair} & \multicolumn{2}{c|}{symmetric} & \multicolumn{2}{c|}{pair} & \multicolumn{2}{c}{symmetric} \\ \hline
\multicolumn{1}{c|}{Noise ratio} & 0.1 & \multicolumn{1}{c|}{0.4} & 0.2 & 0.6 & 0.1 & \multicolumn{1}{c|}{0.4} & 0.2 & 0.6 \\ \hline
\texttt{FedAvg}~\cite{fedavg} & \begin{tabular}[c]{@{}c@{}}$0.17\pm0.03$\\ $0.58\pm0.08$\\ $(82.40\pm3.31)$\end{tabular} & \multicolumn{1}{c|}{\begin{tabular}[c]{@{}c@{}}$0.38\pm0.03$\\ $\mathbf{0.67\pm0.04}$\\ $(41.01\pm4.64)$\end{tabular}} & \begin{tabular}[c]{@{}c@{}}$0.29\pm0.04$\\ $0.66\pm0.07$\\ $(66.94\pm4.27)$\end{tabular} & \begin{tabular}[c]{@{}c@{}}$0.42\pm0.04$\\ $0.75\pm0.05$\\ $(49.52\pm5.42)$\end{tabular} & \begin{tabular}[c]{@{}c@{}}$0.46\pm0.04$\\ $0.80\pm0.06$\\ $(45.08\pm5.61)$\end{tabular} & \multicolumn{1}{c|}{\begin{tabular}[c]{@{}c@{}}$0.57\pm0.04$\\ $0.87\pm0.04$\\ $(20.90\pm4.66)$\end{tabular}} & \begin{tabular}[c]{@{}c@{}}$0.52\pm0.05$\\ $0.81\pm0.04$\\ $(38.62\pm5.28)$\end{tabular} & \begin{tabular}[c]{@{}c@{}}$0.59\pm0.05$\\ $0.84\pm0.04$\\ $(30.18\pm5.53)$\end{tabular} \\ \hline
\texttt{FedProx}~\cite{fedprox} & \begin{tabular}[c]{@{}c@{}}$0.17\pm0.03$\\ $0.58\pm0.07$\\ $(82.05\pm3.98)$\end{tabular} & \multicolumn{1}{c|}{\begin{tabular}[c]{@{}c@{}}$0.38\pm0.03$\\ $0.78\pm0.05$\\ $(41.39\pm4.61)$\end{tabular}} & \begin{tabular}[c]{@{}c@{}}$0.29\pm0.03$\\ $0.66\pm0.07$\\ $(67.15\pm4.60)$\end{tabular} & \begin{tabular}[c]{@{}c@{}}$0.42\pm0.05$\\ $0.74\pm0.05$\\ $(49.98\pm5.57)$\end{tabular} & \begin{tabular}[c]{@{}c@{}}$0.47\pm0.05$\\ $0.80\pm0.05$\\ $(44.31\pm6.20)$\end{tabular} & \multicolumn{1}{c|}{\begin{tabular}[c]{@{}c@{}}$0.70\pm0.05$\\ $0.87\pm0.04$\\ $(21.58\pm4.62)$\end{tabular}} & \begin{tabular}[c]{@{}c@{}}$0.53\pm0.05$\\ $0.81\pm0.05$\\ $(36.91\pm5.68)$\end{tabular} & \begin{tabular}[c]{@{}c@{}}$0.59\pm0.05$\\ $0.84\pm0.05$\\ $(29.50\pm6.11)$\end{tabular} \\ \hline
\texttt{SCAFFOLD}~\cite{ka+19} & \begin{tabular}[c]{@{}c@{}}$0.16\pm0.03$\\ $0.58\pm0.07$\\ $(60.86\pm4.09)$\end{tabular} & \multicolumn{1}{c|}{\begin{tabular}[c]{@{}c@{}}$0.45\pm0.04$\\ $0.77\pm0.04$\\ $(44.92\pm5.07)$\end{tabular}} & \begin{tabular}[c]{@{}c@{}}$0.29\pm0.04$\\ $\mathbf{0.59\pm0.08}$\\ $(70.51\pm4.25)$\end{tabular} & \begin{tabular}[c]{@{}c@{}}$0.44\pm0.04$\\ $0.73\pm0.05$\\ $(51.46\pm5.12)$\end{tabular} & \begin{tabular}[c]{@{}c@{}}$0.46\pm0.05$\\ $0.76\pm0.05$\\ $(47.54\pm5.64)$\end{tabular} & \multicolumn{1}{c|}{\begin{tabular}[c]{@{}c@{}}$0.65\pm0.04$\\ $0.86\pm0.03$\\ $(22.72\pm4.01)$\end{tabular}} & \begin{tabular}[c]{@{}c@{}}$0.53\pm0.04$\\ $0.79\pm0.04$\\ $(38.85\pm5.45)$\end{tabular} & \begin{tabular}[c]{@{}c@{}}$0.61\pm0.04$\\ $0.83\pm0.04$\\ $(30.18\pm4.63)$\end{tabular} \\ \hline
\texttt{LG-FedAvg}~\cite{lgfedavg} & \begin{tabular}[c]{@{}c@{}}$0.23\pm0.04$\\ $0.66\pm0.08$\\ $(73.65\pm5.32)$\end{tabular} & \multicolumn{1}{c|}{\begin{tabular}[c]{@{}c@{}}$0.50\pm0.05$\\ $0.81\pm0.04$\\ $(37.69\pm5.41)$\end{tabular}} & \begin{tabular}[c]{@{}c@{}}$0.34\pm0.04$\\ $0.71\pm0.07$\\ $(61.54\pm4.96)$\end{tabular} & \begin{tabular}[c]{@{}c@{}}$0.45\pm0.05$\\ $0.75\pm0.06$\\ $(47.79\pm5.02)$\end{tabular} & \begin{tabular}[c]{@{}c@{}}$0.59\pm0.06$\\ $0.83\pm0.05$\\ $(30.22\pm6.49)$\end{tabular} & \multicolumn{1}{c|}{\begin{tabular}[c]{@{}c@{}}$0.69\pm0.05$\\ $0.89\pm0.03$\\ $(17.44\pm4.66)$\end{tabular}} & \begin{tabular}[c]{@{}c@{}}$0.63\pm0.05$\\ $0.85\pm0.04$\\ $(25.68\pm5.18)$\end{tabular} & \begin{tabular}[c]{@{}c@{}}$0.66\pm0.05$\\ $0.87\pm0.03$\\ $(22.04\pm5.56)$\end{tabular} \\ \hline
\texttt{FedPer}~\cite{FedPer} & \begin{tabular}[c]{@{}c@{}}$0.17\pm0.03$\\ $0.57\pm0.08$\\ $(82.43\pm4.18)$\end{tabular} & \multicolumn{1}{c|}{\begin{tabular}[c]{@{}c@{}}$0.40\pm0.04$\\ $0.78\pm0.05$\\ $(40.75\pm5.49)$\end{tabular}} & \begin{tabular}[c]{@{}c@{}}$0.28\pm0.04$\\ $0.66\pm0.08$\\ $(68.44\pm5.63)$\end{tabular} & \begin{tabular}[c]{@{}c@{}}$0.40\pm0.04$\\ $0.73\pm0.07$\\ $(52.41\pm5.56)$\end{tabular} & \begin{tabular}[c]{@{}c@{}}$0.54\pm0.05$\\ $0.81\pm0.06$\\ $(39.81\pm6.02)$\end{tabular} & \multicolumn{1}{c|}{\begin{tabular}[c]{@{}c@{}}$0.70\pm0.05$\\ $0.87\pm0.04$\\ $(20.37\pm5.43)$\end{tabular}} & \begin{tabular}[c]{@{}c@{}}$0.60\pm0.06$\\ $0.82\pm0.05$\\ $(32.82\pm5.90)$\end{tabular} & \begin{tabular}[c]{@{}c@{}}$0.65\pm0.05$\\ $0.85\pm0.04$\\ $(26.82\pm5.38)$\end{tabular} \\ \hline
\texttt{APFL}~\cite{apfl} & \begin{tabular}[c]{@{}c@{}}$0.18\pm0.03$\\ $0.60\pm0.07$\\ $(80.18\pm4.57)$\end{tabular} & \multicolumn{1}{c|}{\begin{tabular}[c]{@{}c@{}}$0.42\pm.0.05$\\ $0.78\pm0.06$\\ $(40.43\pm6.06)$\end{tabular}} & \begin{tabular}[c]{@{}c@{}}$0.28\pm0.05$\\ $0.67\pm0.06$\\ $(67.83\pm5.53)$\end{tabular} & \begin{tabular}[c]{@{}c@{}}$0.40\pm0.05$\\ $0.79\pm0.07$\\ $(52.15\pm5.63)$\end{tabular} & \begin{tabular}[c]{@{}c@{}}$0.45\pm0.06$\\ $0.78\pm0.06$\\ $(45.27\pm6.21)$\end{tabular} & \multicolumn{1}{c|}{\begin{tabular}[c]{@{}c@{}}$0.60\pm0.06$\\ $0.86\pm0.04$\\ $(23.42\pm5.49)$\end{tabular}} & \begin{tabular}[c]{@{}c@{}}$0.51\pm0.06$\\ $0.81\pm0.05$\\ $(37.85\pm5.84)$\end{tabular} & \begin{tabular}[c]{@{}c@{}}$0.56\pm0.06$\\ $0.83\pm0.05$\\ $(31.46\pm5.73)$\end{tabular} \\ \hline
\texttt{pFedMe}~\cite{pfedme} & \begin{tabular}[c]{@{}c@{}}$0.23\pm0.04$\\ $0.66\pm0.07$\\ $(72.05\pm0.05)$\end{tabular} & \multicolumn{1}{c|}{\begin{tabular}[c]{@{}c@{}}$0.46\pm0.04$\\ $0.80\pm.0.05$\\ $(37.89\pm5.77)$\end{tabular}} & \begin{tabular}[c]{@{}c@{}}$0.33\pm0.05$\\ $0.72\pm0.06$\\ $(59.80\pm5.94)$\end{tabular} & \begin{tabular}[c]{@{}c@{}}$0.44\pm0.04$\\ $0.76\pm0.05$\\ $(45.79\pm5.26)$\end{tabular} & \begin{tabular}[c]{@{}c@{}}$0.58\pm0.06$\\ $0.83\pm0.52$\\ $(29.62\pm6.37)$\end{tabular} & \multicolumn{1}{c|}{\begin{tabular}[c]{@{}c@{}}$0.66\pm0.04$\\ $0.88\pm0.03$\\ $(17.94\pm3.97)$\end{tabular}} & \begin{tabular}[c]{@{}c@{}}$0.60\pm0.05$\\ $0.85\pm0.04$\\ $(26.01\pm0.05)$\end{tabular} & \begin{tabular}[c]{@{}c@{}}$0.64\pm0.05$\\ $0.87\pm0.03$\\ $(21.36\pm4.42)$\end{tabular} \\ \hline
\texttt{Ditto}~\cite{ditto} & \begin{tabular}[c]{@{}c@{}}$0.22\pm0.04$\\ $0.66\pm0.07$\\ $(72.39\pm5.25)$\end{tabular} & \multicolumn{1}{c|}{\begin{tabular}[c]{@{}c@{}}$0.42\pm0.05$\\ $0.79\pm0.05$\\ $(38.20\pm6.13)$\end{tabular}} & \begin{tabular}[c]{@{}c@{}}$0.31\pm0.04$\\ $0.69\pm0.07$\\ $(60.62\pm5.67)$\end{tabular} & \begin{tabular}[c]{@{}c@{}}$0.41\pm0.04$\\ $0.77\pm0.05$\\ $(45.11\pm4.95)$\end{tabular} & \begin{tabular}[c]{@{}c@{}}$0.54\pm0.05$\\ $0.84\pm0.05$\\ $(29.41\pm5.15)$\end{tabular} & \multicolumn{1}{c|}{\begin{tabular}[c]{@{}c@{}}$0.60\pm0.06$\\ $0.88\pm0.04$\\ $(18.17\pm4.43)$\end{tabular}} & \begin{tabular}[c]{@{}c@{}}$0.56\pm0.07$\\ $0.85\pm0.04$\\ $(26.39\pm5.83)$\end{tabular} & \begin{tabular}[c]{@{}c@{}}$0.58\pm0.06$\\ $0.87\pm0.04$\\ $(21.72\pm4.90)$\end{tabular} \\ \hline
\texttt{FedRep}~\cite{fedrep} & \begin{tabular}[c]{@{}c@{}}$0.20\pm0.04$\\ $0.62\pm0.08$\\ $(77.95\pm4.65)$\end{tabular} & \multicolumn{1}{c|}{\begin{tabular}[c]{@{}c@{}}$0.53\pm0.05$\\ $0.81\pm0.05$\\ $(35.88\pm5.15)$\end{tabular}} & \begin{tabular}[c]{@{}c@{}}$0.30\pm0.05$\\ $0.68\pm0.08$\\ $(66.58\pm5.60)$\end{tabular} & \begin{tabular}[c]{@{}c@{}}$0.43\pm0.05$\\ $0.74\pm0.06$\\ $(51.30\pm5.27)$\end{tabular} & \begin{tabular}[c]{@{}c@{}}$0.62\pm0.05$\\ $0.82\pm0.05$\\ $(33.10\pm5.23)$\end{tabular} & \multicolumn{1}{c|}{\begin{tabular}[c]{@{}c@{}}$0.76\pm0.04$\\ $0.87\pm0.03$\\ $(17.80\pm4.43)$\end{tabular}} & \begin{tabular}[c]{@{}c@{}}$0.66\pm0.05$\\ $0.83\pm0.05$\\ $(29.22\pm5.67)$\end{tabular} & \begin{tabular}[c]{@{}c@{}}$0.71\pm0.05$\\ $0.85\pm0.04$\\ $(22.94\pm5.20)$\end{tabular} \\ \hline
\rowcolor[HTML]{FFF5E6} 
\texttt{SuPerFed-MM} & \begin{tabular}[c]{@{}c@{}}$0.16\pm0.03$\\ $0.53\pm0.07$\\ $(83.67\pm3.51)$\end{tabular} & \multicolumn{1}{c|}{\begin{tabular}[c]{@{}c@{}}$\mathbf{0.28\pm0.04}$\\$0.69\pm0.06$\\ $(46.41\pm5.14)$\end{tabular}} & \begin{tabular}[c]{@{}c@{}}$\mathbf{0.27\pm0.03}$\\ $0.64\pm0.07$\\ $(\mathbf{71.99\pm5.01})$\end{tabular} & \begin{tabular}[c]{@{}c@{}}$\mathbf{0.30\pm0.04}$\\ $\mathbf{0.69\pm0.08}$\\ $(\mathbf{56.07\pm4.32})$\end{tabular} & \begin{tabular}[c]{@{}c@{}}$\mathbf{0.28\pm0.06}$\\ $0.69\pm0.07$\\ $(48.79\pm5.73)$\end{tabular} & \multicolumn{1}{c|}{\begin{tabular}[c]{@{}c@{}}$\mathbf{0.27\pm0.05}$\\ $\mathbf{0.63\pm0.09}$\\ $(26.64\pm4.34)$\end{tabular}} & \begin{tabular}[c]{@{}c@{}}$\mathbf{0.31\pm0.06}$\\ $0.71\pm0.10$\\ $42.66\pm5.61$\end{tabular} & \begin{tabular}[c]{@{}c@{}}$\mathbf{0.28\pm0.06}$\\ $\mathbf{0.68\pm0.10}$\\ $(\mathbf{35.74\pm5.21})$\end{tabular} \\ \hline
\rowcolor[HTML]{FFF5E6} 
\texttt{SuPerFed-LM} & \begin{tabular}[c]{@{}c@{}}$\mathbf{0.14\pm0.03}$\\ $\mathbf{0.49\pm0.08}$\\ $(\mathbf{84.78\pm3.63})$\end{tabular} & \multicolumn{1}{c|}{\begin{tabular}[c]{@{}c@{}}$0.35\pm0.04$\\ $0.77\pm0.04$\\ $(\mathbf{46.82\pm5.39})$\end{tabular}} & \begin{tabular}[c]{@{}c@{}}$0.27\pm0.04$\\ $0.66\pm0.07$\\ $(69.23\pm5.25)$\end{tabular} & \begin{tabular}[c]{@{}c@{}}$0.32\pm0.04$\\ $0.72\pm0.06$\\ $(54.69\pm4.20)$\end{tabular} & \begin{tabular}[c]{@{}c@{}}$0.29\pm0.04$\\ $\mathbf{0.68\pm0.06}$\\ $(\mathbf{51.44\pm5.67})$\end{tabular} & \multicolumn{1}{c|}{\begin{tabular}[c]{@{}c@{}}$0.40\pm0.04$\\ $0.80\pm0.04$\\ $(\mathbf{28.51\pm4.66})$\end{tabular}} & \begin{tabular}[c]{@{}c@{}}$0.36\pm0.05$\\ $\mathbf{0.70\pm0.07}$\\ $(\mathbf{42.81\pm5.50})$\end{tabular} & \begin{tabular}[c]{@{}c@{}}$0.37\pm0.06$\\ $0.76\pm0.05$\\ $(34.10\pm5.18)$\end{tabular} \\ \hline
\end{tabular}%
}
\end{table}
    
\newpage
\subsection{Ablation Studies}
    \paragraph{Start round of a local model training for personalization} 
    We conducted experiments on selecting $L$, which determines the start of mixing, in other words, sampling $\lambda\sim\mathcal{U}(0,1)$. 
    As our proposed method can choose a mixing scheme from either \textit{model-mixing} or \textit{layer-mixing}, 
    we conduct separate experiments per scheme. 
    By adjusting $L$ by 0.1 increments from 0 to 0.9 of $T$, we evaluated the Top-1 accuracy of a mixed model. 
    We used the CIFAR-10 dataset with $T=200$, $\nu=2,\text{ and } \mu=0.01$, and corresponding results are summarized in Table~\ref{tab:2_table7}. 
    \begin{table}[h]
\centering
\caption[Results on the effect of setting the personalization training round in \texttt{SuPerFed}]{Grid search results on selecting $L$, the starting round of training a local model. 
The best averaged Top-1 accuracy (top) and loss (bottom) among models that can be realized by $\lambda\in[0,1]$ is reported with a standard deviation.}
\label{tab:2_table7}
\resizebox{0.8\textwidth}{!}{%
\begin{tabular}{c|ccccc}
\hline
$L/T$ & 0.0 & 0.2 & 0.4 & 0.6 & 0.8 \\ \hline
\texttt{SuPerFed-MM} & \begin{tabular}[c]{@{}c@{}}$89.87\pm4.73$\\ $1.36\pm0.58$\end{tabular} & \begin{tabular}[c]{@{}c@{}}$92.67\pm4.04$\\ $\mathbf{0.72\pm0.40}$\end{tabular} & \begin{tabular}[c]{@{}c@{}}$\mathbf{92.69\pm3.86}$\\ $0.78\pm0.31$\end{tabular} & \begin{tabular}[c]{@{}c@{}}$92.50\pm4.14$\\ $0.74\pm0.88$\end{tabular} & \begin{tabular}[c]{@{}c@{}}$91.29\pm6.28$\\ $2.53\pm0.72$\end{tabular} \\ \hline
\texttt{SuPerFed-LM} & \begin{tabular}[c]{@{}c@{}}$91.55\pm4.26$\\ $1.43\pm0.81$\end{tabular} & \begin{tabular}[c]{@{}c@{}}$92.41\pm3.93$\\ $1.18\pm0.60$\end{tabular} & \begin{tabular}[c]{@{}c@{}}$\mathbf{92.69\pm3.66}$\\ $0.94\pm0.45$\end{tabular} & \begin{tabular}[c]{@{}c@{}}$92.60\pm3.86$\\ $\mathbf{0.63\pm0.33}$\end{tabular} & \begin{tabular}[c]{@{}c@{}}$90.59\pm9.93$\\ $6.28\pm1.76$\end{tabular} \\ \hline
\end{tabular}%
}
\vspace{-2mm}
\end{table}
    
    From the result, one can notice that there are some differences according to the mixing scheme. 
    In terms of the personalization performance, setting $L/T=0.4$ shows the best performance with a small standard deviation. 
    This implies that allowing some rounds to a federated model to focus on learning global knowledge (by setting non-zero $L$) is a valid strategy. 
    While $L$ is a tunable hyperparameter, it can be adjusted for the purpose of FL.
    To derive a good global model, choosing a large $L$ close to $T$ may be a valid choice, while in terms of personalization, it may be close to 0.
    
    \paragraph{Hyperparameters for regularization terms} 
    Our proposed method has two hyperparameters, $\mu$ and $\nu$. Each hyperparameter is supposed to affect in two-folds: 
    i) penalizes a local update of a federated model from going in the wrong direction, away from the direction of a global model update,
    and i) creates a connected low-loss subspace between the federated model and a local model in each client by rendering parameters of the two models orthogonal to each other.
    We conducted experiments with $\nu=\{0.0, 1.0, 2.0, 5.0\}$ and $\mu=\{0.0, 0.01, 0.1, 1.0\}$ (a total of 16 combinations) under the \textit{model-mixing} and \textit{layer-mixing} schemes, with $T=100$ and $L=\lfloor0.4 T\rfloor=40$, which yielded the best personalization performance on the CIFAR-10 dataset under \textit{pathological non-IID} setting. Experimental results on these combinations are displayed in Figure~\ref{fig:fig1_3}.
    
    There are noticeable observations supporting the efficacy of both of two regularization terms. 
    First, a learning strategy to find a connected low loss subspace between local and federated models yields high personalization performance. 
    In most cases with non-zero $\nu$ and $\mu$, a combination of federated and local models (i.e., models realized by $\lambda\in(0,1)$) outperforms two single models from both endpoints (i.e., a federated model ($\lambda=0$) and a local model ($\lambda=1$). 
    Second, the proximity regularization restricts learning too much when $\mu$ is large, but conistently helps a federated model not to diverge when it is properly tuned. This observation is also supported by previous results~\cite{fedprox, pfedme}.
    Third, a moderate degree of orthogonality regularization (adjusted by $\nu$) boosts the overall personalization performance.
    These empirically certifies that both regularization terms are necessary to achieve decent personalization performances.    

    \begin{sidewaysfigure}[p]  
    \includegraphics[width=\textwidth,keepaspectratio]{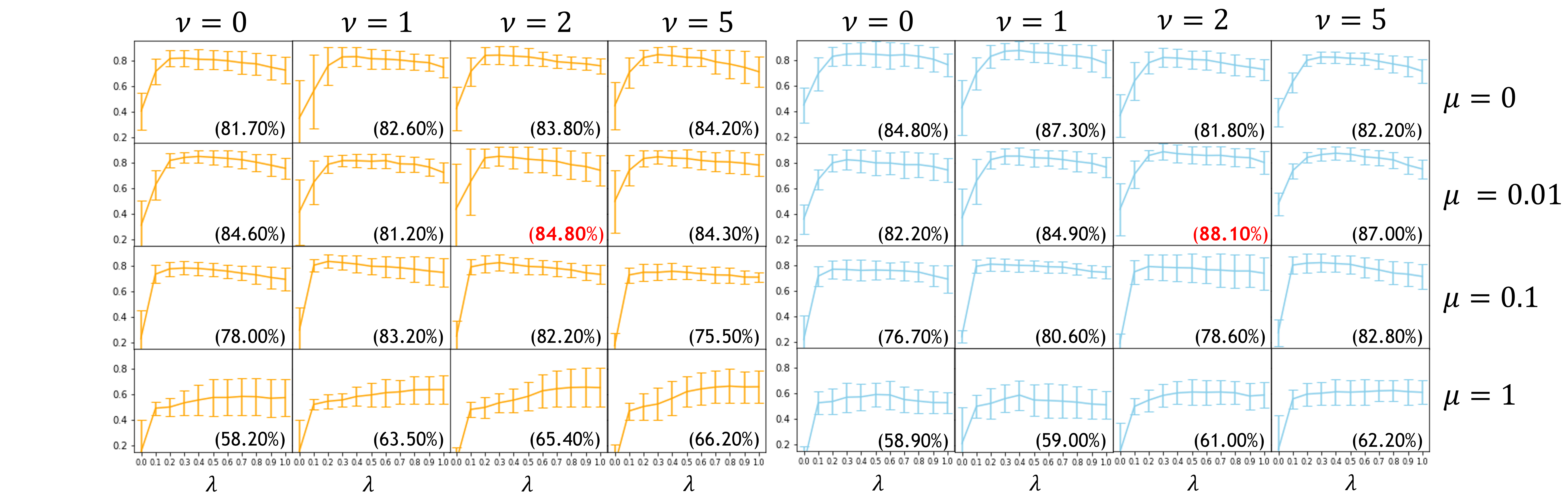}
    \caption[The effects of hyperparameters $\mu$ and $\nu$ in \texttt{SuPerFed}]
    {Left group of plots colored in orange is the performance of \texttt{SuPerFed-MM}, right group of plots colored in sky-blue is the performance of \texttt{SuPerFed-LM}. 
    Each subplot's vertical axis represents Top-1 accuracy, and the horizontal axis represents the range of possible $\lambda$ values from 0 to 1, with an interval of 0.1. 
    Numbers on subplots indicate the performance of the best averaged Top-1 accuracy evaluated by local test set of each client. 
    The error bar indicates a standard deviation of each local model realized by different values of $\lambda\in[0,1]$. 
    Note that each endpoint means a federated model ($\boldsymbol{\theta}^{f}$ when $\lambda=0$) or a local model ($\boldsymbol{\theta}^{l}$ when $\lambda=1$).}
    \label{fig:fig1_3}
    \end{sidewaysfigure}

\newpage
\section{Analysis}
\subsection{Loss Landscape of a Global Model} 
    \begin{sidewaysfigure}[p]  
    \includegraphics[width=\textwidth,keepaspectratio]{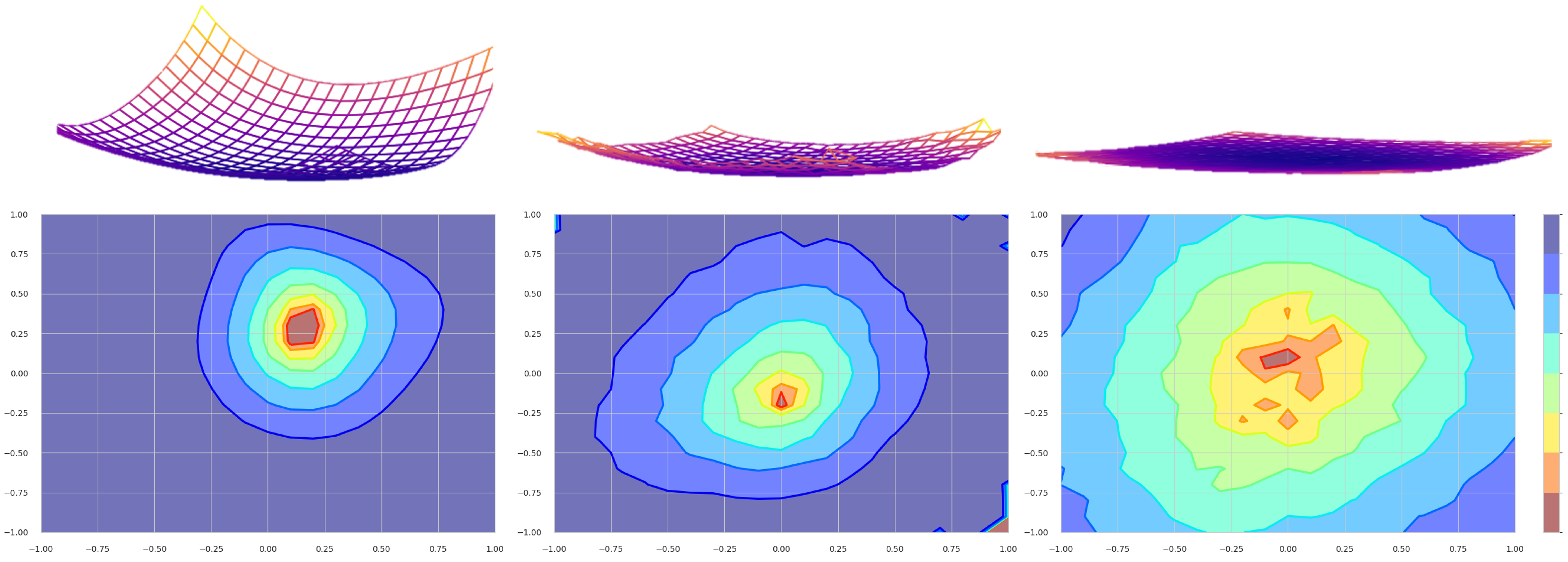}
    \caption[The loss landscapes of a global model trained by \texttt{SuPerFed}]
    {The loss landscapes of a deep network (TwoCNN) trained on CIFAR-10. 
    The top row indicates the curvature of loss landscapes and the bottom row indicates contours of a derived loss landscapes 
    (A) \textbf{First column:} loss landscape obtained from \texttt{FedAvg}. 
    (B) \textbf{Second column:} loss landscape obtained from \texttt{SuPerFed-MM}. 
    (C) \textbf{Third column:} loss landscape obtained from \texttt{SuPerFed-LM}.}
    \label{fig:fig1_4}
    \end{sidewaysfigure}
    
    The loss landscape of global models obtained from the proposed methods \texttt{SuPerFed-MM} and \texttt{SuPerFed-LM} are visualized in Figure~\ref{fig:fig1_4} along with one from \texttt{FedAvg}, following~\cite{garipov+18}. 
    As the proposed regularization term adjusted by $\nu$ induces a connected low-loss subspace, it is consequently observed that the wider minima is acquired in the parameter space of a global model from both \texttt{SuPerFed-MM} and \texttt{SuPerFed-LM} compared to that of \texttt{FedAvg}. 
    As the global model learns to be connected with multiple local models in the parameter space in each FL training round, flat minima of the global model are naturally induced, i.e., the \textit{mode connectivity} is observed. 
    One thing to note here is that the minima is far wider in layer-mixing scheme (i.e., \texttt{SuPerFed-LM}) than that of model-mixing scheme (i.e., \texttt{SuPerFed-MM}). It is presumably due to the fact that the number of distinct combinations of a federated model and a local model is much larger in the layer-mixing scheme than in the model-mixing scheme. 
    In other words, the resulting global model should be connected to multiple local models in more diverse ways. 
    Meanwhile, since the global model has flat minima, we can expect it is more generalized to unseen clients, even if their data distribution is heterogeneous from the existing clients' data distributions. 
    Since the global model contains many good candidates of low-loss solutions thanks to the induced \textit{mode connectivity}, 
    the odds of finding a suitable model are higher than that of the global model trained using other FL algorithms.

\newpage
\subsection{Discussion on the Convex Combination-based PFL Methods} 
    \begin{figure}[h]
    \centering
    \includegraphics[width=\textwidth]{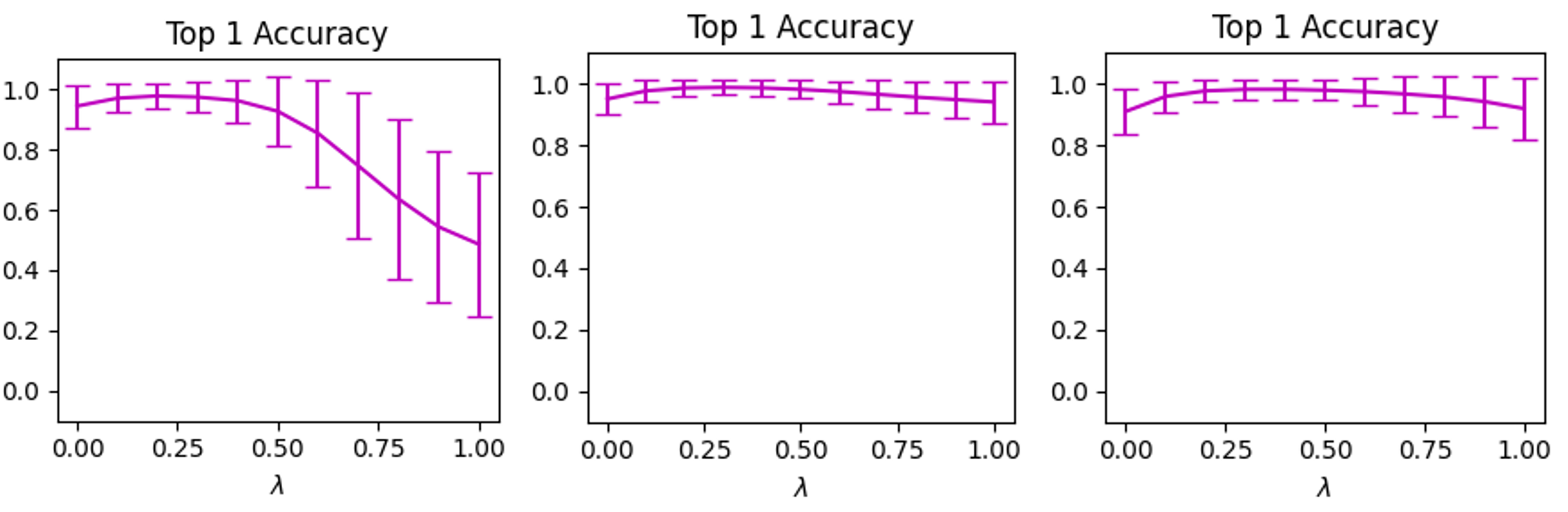}
    \caption[Comparison of \texttt{SuPerFed} with other convex combination-based personalization methods]{
        Comparison of the personalization performance between \texttt{APFL} (left), \texttt{SuPerFed-MM} (middle), and \texttt{SuPerFed-LM} (right) by varying $\lambda\in[0,1]$,
        trained using MNIST with TwoNN under \textit{pathological non-IID} setting with $K=500$ and $T=500$
    }
    \label{fig:fig1_5}
    \end{figure}
    
    In previous model mixture-based PFL methods, there exist similar approaches that interpolate a local model and a federated model in the form of convex combination: 
    i) \texttt{APFL}~\cite{apfl}, ii) \texttt{mapper}~\cite{mansour+20} and iii) \texttt{L2GD}~\cite{l2sgd}.
    
    \noindent In i) \texttt{APFL}, the federated model and the local model are updated \textit{sequentially} with a fixed $\lambda$, 
    which is not true in \texttt{SuPerFed} as it jointly updates both federated and local models at once. 
    While authors of ~\cite{apfl} proposed the method to dynamically tune $\lambda$ through another step of gradient descent, 
    \texttt{SuPerFed} \textbf{does not} resort to a single optimal $\lambda$. 
    Instead, it aims to find a good set of parameters $\boldsymbol{\theta}^{l}$ and $\boldsymbol{\theta}^{f}$ that can be mixed to yield good performance \textbf{regardless of} $\lambda$, thereby connected low-loss subspace is induced in the weight space between two models.
    
    \noindent In ii) \texttt{mapper}~\cite{mansour+20}, on the other hand, it iteratively selects \textit{only one} client (which is an unrealistic assumption) and requests to find the best local model $\boldsymbol{\theta}^{l}$ and $\lambda$ based on the transmitted global model $\boldsymbol{\theta}^{g}$. 
    The federated model is then updated with the the local model and the optimal $\lambda$ is tuned again. 
    In \texttt{SuPerFed}, it is not required to find a fixed value of $\lambda$ and updates of both local and federated models can be conducted in \textbf{a single run of back-propagation}. In other words, no separate training on each model is required as in other methods. 
    
    \noindent Lastly, in iii) \texttt{L2GD} proposed by~\cite{l2sgd}, the interpolation of the global model and the local model is intermittently occurred in the \textit{server} as other methods. 
    However, it requires \textit{all} clients to participate every round, which is a fairly unrealistic assumption as FL setting.
    
    Therefore, we only compared with a scalable method applicable for realistic FL settings, \texttt{APFL}. 
    For the comparison, we observed the change of performances by varying $\lambda\in[0,1]$ in both \texttt{SuPerFed} (\texttt{SuPerFed-MM} \& \texttt{SuPerFed-LM}) and \texttt{APFL}.
    In Figure~\ref{fig:fig1_5}, we can find that \texttt{SuPerFed-MM} and \texttt{SuPerFed-LM} consistently perform well in terms of personalization performance
    \textit{regardless of} $\lambda$, 
    while \texttt{APFL} performs best only when $\lambda$ is close to 0.25. 
    This empirical evidence supports the superiortiy of \texttt{SuPerFed}, realized by the explicit introduction of the \textit{mode connectivity}. 
    As two endpoints (i.e., $\lambda=0$ indicates a federated model and $\lambda=1$ indicates a local model) are explicitly ensembled to form a connected low-loss subspace, both are trained accordingly in a mutually beneficial manner.

\newpage
\section{Conclusion}
    In this chapter, we propose a simple and scalable personalized federated learning method \texttt{SuPerFed} under the scheme of the model mixture-based personalization method. 
    Its goal is to improve personalization performance while entailing unexpected benefits, such as robustness to label noise and better calibration. 
    In each participating client, \texttt{SuPerFed} aims to form a connected low-loss subspace between a local and a federated model in the parameter space,
    which is inspired by the intriguing phenomenon between multiple deep networks, i.e., \textit{mode connectivity}. 
    As a consequence, two models are connected in a mutually cooperative way, so that each client has a personalized model with good performance and calibration that is robust to possible label noise in practice. 
    Through extensive experiments, we empirically demonstrate that \texttt{SuPerFed}, with its two variants in the mixing scheme, outperforms existing model mixture-based PFL methods as well as single model-based basic FL methods.

\newpage 
\chapter{Perspective on Mixing Coefficients, $\boldsymbol{w}$}
\label{ch:aaggff}
\numberwithin{equation}{chapter}
\numberwithin{figure}{chapter}
\numberwithin{table}{chapter}
\numberwithin{algorithm}{chapter}
\renewcommand{\theequation}{3.\arabic{equation}}
\renewcommand{\thefigure}{3.\arabic{figure}}
\renewcommand{\thetable}{3.\arabic{table}}
\renewcommand{\thealgorithm}{3.\arabic{algorithm}}
\renewcommand{\thedefinition}{3.\arabic{definition}} 
\renewcommand{\thetheorem}{3.\arabic{theorem}} 
\renewcommand{\thelemma}{3.\arabic{lemma}} 
\renewcommand{\thecorollary}{3.\arabic{corollary}}
\renewcommand{\theremark}{3.\arabic{remark}} 

\section*{Pursuing Overall Welfare in Federated Learning via Sequential Decision Making} 
\addcontentsline{toc}{section}{\protect\numberline{}\textbf{Pursuing Overall Welfare in Federated Learning via Sequential Decision Making}}
    In this section, we introduce a new perspective in terms of the \textbf{mixing coefficient}, denoted as $\boldsymbol{w}$ in eq.~\eqref{eq:fl_obj}.
    Most existing approaches stick to a static mixing coefficient proportional to the local sample size.
    This reduces the federated learning scheme to the empirical risk minimization (ERM) principle in expectation~\cite{fieldguide}, but this sometimes leads to unattractive situations where the performance distributions across clients are not uniform, also known as a violation of the \textit{client-level fairness}~\cite{clientlevelfairness}.
    This is undoubtedly an undesirable situation caused by the statistical heterogeneity across heterogeneous clients.
    
    Instead of adhering to using the static mixing coefficient, the federated system can adopt an \textit{adaptive mixing scheme} for inducing uniform performance distributions across clients, by deciding the mixing coefficient adaptively.
    The decision-maker is usually the central server, who updates the mixing coefficient in response to \textit{feedbacks} transmitted from clients, e.g., local losses, local gradients.
    
    In this context, many fair FL algorithms have been proposed thus far, however, 
    most of them overlooked the plight of the central server.  
    While clients can use some local samples to update their local model parameters, 
    the central server can only use signals whose size is at most equal to the number of clients to update its adaptive decision, which is far smaller in size.
    In other words, existing approaches do not take into account the central server's sample deficiency problem, so a federated system equipped with existing algorithms may lead to an unsatisfactory degree of client-level fairness.
    In addition, current methods do not take into account past decisions when updating a new decision, making the quality of the decision questionable. 
    
    We start with the following research question:
    \begin{center} 
    \textit{How can we improve the scheme of deciding \textbf{mixing coefficients}\\ 
    so that it is truly adaptive even under the sample-deficient condition?}
    \end{center}
    
\newpage
\section{Introduction}
    FL has been posed as an effective strategy to acquire a global model without centralizing data, therefore with no compromise in privacy \cite{fedavg}.
    It is commonly assumed that the central server coordinates the whole FL procedure by repeatedly \textit{aggregating} local updates from participating $K$ clients during $T$ rounds.
    
    Since each client updates the copy of a global model with its own data, variability across clients' data distributions causes many problems \cite{prob_fl, challenge_fl}.
    The \textit{client-level fairness} \cite{clientlevelfairness} is one of the main problems affected by such a statistical heterogeneity \cite{prob_fl, challenge_fl, afl, qffl}. 
    Although the performance of a global model is high in average, some clients may be more benefited than others, resulting in violation of the client-level fairness. 
    In this situation, there inevitably exists a group of clients who cannot utilize the trained global model due to its poor performance.
    This is a critical problem in practice since the underperformed groups may lose motivation to participate in the federated system.
    To remedy this problem, previous works \cite{afl, qffl, term, fedmgda, propfair} proposed to modify the static aggregation scheme into an \textit{adaptive aggregation} strategy, according to given local signals (e.g., losses or gradients). 
    In detail, the server \textit{re-weights local updates} by assigning larger \textit{mixing coefficients} to higher local losses.
    
    When updating the mixing coefficients, however, \textit{only a few bits are provided to the server}, compared to the update of a model parameter.
    For example, suppose that there exist $K$ clients in the federated system, each of which has $N$ local samples. 
    When all clients participate in each round, $KN$ samples are used effectively for updating a new global model $\boldsymbol{\theta}$.
    On the contrary, only $K$ bits (e.g., local losses: $F_1(\boldsymbol{\theta}), ..., F_K(\boldsymbol{\theta})$) are provided to the server for an update of mixing coefficients.
    This is aggravated in cases where $K$ is too large, thus client sampling is inevitably required. 
    In this case, the server is provided with far less than $K$ signals, which hinders faithful update of the mixing coefficients.
    
    For sequentially updating a status in this \textit{sample-deficient} situation, the online convex optimization (OCO) framework is undoubtedly the best solution.
    Interestingly, we discovered that most existing adaptive aggregation strategies can be readily unified into the OCO framework. 
    Starting from this unification result, we propose an improved design for a fair FL algorithm in the view of sequential decision making.
    Since there exist OCO algorithms specialized for the setting where the decision space is a simplex (i.e., same as the domain of the mixing coefficient), these may be adopted to FL setting with some modifications for practical constraints. 
    
    In practice, FL is subdivided into two settings: \textit{cross-silo} setting and \textit{cross-device} setting \cite{prob_fl}.
    For $K$ clients and $T$ training rounds, each setting requires a different dependency on $K$ and $T$.
    In the cross-silo setting, the number of clients (e.g., institutions) is small and usually less than the number of rounds (i.e., $K < T$).
    e.g., $K=20$ institutions with $T=200$ rounds~\cite{real_silo}.
    On the other hand, in the cross-device setting, the number of clients (e.g., mobile devices) is larger than the number of rounds ($K > T$). 
    e.g., $K=1.5\times10^6$ with $T=3,000$ rounds~\cite{real_device}.
    In designing an FL algorithm, these conditions should be reflected for the sake of practicality.

\paragraph{Contributions}
    We propose \texttt{AAggFF}, a sequential decision making framework for the central server in FL 
    tailored for inducing client-level fairness in the system.
    The contributions of our work are summarized as follows.
    \begin{enumerate}
      \item[$\bullet$] We unify existing fairness-aware adaptive aggregation methods into an OCO framework 
      and propose better online decision making designs for pursuing client-level fairness by the central server. 
      \item[$\bullet$] We propose \texttt{AAggFF}, which is designed to enhance the client-level fairness, 
      and further specialize our method into two practical settings:
      \texttt{AAggFF-S} for cross-silo FL and \texttt{AAggFF-D} for cross-device FL. 
      \item[$\bullet$] We provide regret analyses on the behavior of two algorithms, 
      \texttt{AAggFF-S} and \texttt{AAggFF-D}, presenting sublinear regrets. 
      \item[$\bullet$] We evaluate \texttt{AAggFF} on extensive benchmark datasets for realistic FL scenarios with other baselines.
      \texttt{AAggFF} not only boosts the worst-performing clients but also maintains overall performance. 
    \end{enumerate}

\newpage
\section{Related Works}
\subsection{Client-Level Fairness in Federated Learning} 
    The statistical heterogeneity across clients often causes non-uniform performances of a single global model on participating clients, which is also known as the violation of client-level fairness.
    Fairness-aware FL algorithms aim to eliminate such inequality to achieve uniform performance across all clients.
    There are mainly two approaches to address the problem \cite{clientlevelfairness}: 
    a single model approach, and a personalization approach\nocite{superfed}.
    This paper mainly focuses on the former, which is usually realized by modifying the FL objective,
such as a minimax fairness objective~\cite{afl} (which is also solved by multi-objective optimization \cite{fedmgda} and is also modified to save a communication cost \cite{drfa}), 
    alpha-fairness \cite{alphafair} objective \cite{qffl}, 
    suppressing outliers (i.e., clients having high losses) by tilted objective \cite{term},
    and adopting the concept of proportional fairness to reach Nash equilibrium in performance distributions \cite{propfair}.
    While the objective can be directly aligned with existing notions of fairness, 
    it is not always a standard for the design of a fair FL algorithm.
    Notably, most of works share a common underlying principle: \textit{assigning more weights to a local update having larger losses}. 

    \begin{definition} (Client-Level Fairness; Definition 1 of~\cite{qffl}, Section 4.2 of~\cite{clientlevelfairness})
    We informally define the notion of client-level fairness in FL as the status where a trained global model yields uniformly good performance across all participating clients.
    Note that uniformity can be measured by the spread of performances.
    \end{definition}

\newpage
\subsection{Online Decision Making}
    The OCO framework is designed for making \textit{sequential decisions} with the best utilities, having solid theoretical backgrounds.
    It aims to minimize the cumulative mistakes of a decision maker (e.g., central sever), 
    given a response of the environment (e.g., losses from clients) for finite rounds $t\in[T]$.
    The cumulative mistakes of the learner are usually denoted as the cumulative regret (see (\ref{eq:regret})), 
    and the learner can achieve sublinear regret in finite rounds using well-designed OCO algorithms \cite{oco1, oco2, oco3}.
    In designing an OCO algorithm, two main frameworks are mainly considered: 
    online mirror descent (OMD) \cite{md, omd, bregmanmd} 
    and follow-the-regularized-leader (FTRL) \cite{ftrl1, ftrl2, ftrl3, ftrl4}.
    One of popular instantiations of both frameworks is the Online Portfolio Selection (OPS) algorithm, of which decision space is restricted to a probability simplex.
    The universal portfolio algorithm is the first that yields an optimal theoretical regret, $\mathcal{O}(K \log{T})$) despite its heavy computation ($\mathcal{O}(K^4 T^{14})$) \cite{up}, 
    the Online Gradient Descent \cite{ogd} and the Exponentiated Gradient (EG) \cite{eg} show slightly worse regrets (both are $\mathcal{O}(\sqrt{T})$), but can be executed in linear runtime in practice ($\mathcal{O}(K)$).
    Plus, the online Newton step (ONS) \cite{ons1, ons2} presents logarithmic regret with quadratic runtime in $K$.
    Since these OPS algorithms are proven to perform well when the decision is a probability vector, we adopt them for finding adaptive mixing coefficients to achieve performance fairness in FL.
    To the best of our knowledge, we are the first to consider fair FL algorithms under the OCO framework.

\newpage
\section{Backgrounds}
    In traditional FL scheme, not all clients can be equally benefited from a trained global model. 
    Therefore, the need to achieve the \textit{client-level fairness} in a federated learning system has been emphasized, which can be realized by modifying the static aggregation scheme for updating the global model to an adaptive one, in response to the local signals of the participating clients.
    
    We first reveal that existing fairness-aware aggregation strategies can be unified into an online convex optimization framework, in other words, a central server's \textit{sequential decision making} process.
    Then, we propose \texttt{AAggFF}, from simple and intuitive improvements for suboptimal designs within existing methods to enhance the decision making capability. 
    
    Considering practical requirements, we further subdivide our method tailored for the \textit{cross-device} and the \textit{cross-silo} settings, respectively. 
    Theoretical analyses guarantee sublinear regret upper bounds for both settings: $\mathcal{O}(\sqrt{T \log{K}})$ for the cross-device setting, and $\mathcal{O}(K \log{T})$ for the cross-silo setting, with $K$ clients and $T$ federation rounds. 
    We provide notation table in Table~\ref{tab:3_table1}.
    \begin{table}[!ht]
\centering
\caption{Notations in Chapter~\ref{ch:aaggff}}
\label{tab:3_table1}
\resizebox{0.9\textwidth}{!}{%
\begin{tabular}{@{}ll@{}}
\toprule
Notation                       & Description                                                                                         \\ \midrule
$T$                            & Total number of communication rounds                                                                \\
$B$                            & Local batch size                                                                                    \\
$E$                            & Number of local epochs                                                                              \\
$K$                            & Total number of participating clients, indexed by $i\in[K]$                                         \\
$C$                            & Fraction of clients selected at each round                                                          \\
$\zeta$                         & Step size for updating adaptive mixing coefficients                                                                               \\ \midrule
$\boldsymbol{p}$ &
  \begin{tabular}[c]{@{}l@{}}Decision variable, or an adaptive mixing coefficient vector,\\ of which entry $p_i$ corresponds to a mixing coefficient of client $i$.\end{tabular} \\
$\boldsymbol{p}^\star$         & Optimal decision in hindsight                                                                       \\
$\Delta_{K-1}$                 & Probability simplex in $\mathbb{R}^K$                                                               \\ \midrule
$\ell^{(t)}$                   & Decision loss function at round $t$                                                                 \\
$\boldsymbol{r}^{(t)}$         & Response vector at round $t$                                                                        \\
$\boldsymbol{g}^{(t)}$         & Gradient of a decision loss function w.r.t. decision variable, $\boldsymbol{p}$ at round $t$        \\
$\rho^{(t)}$                   & Transformation function for changing a local loss value into a response at round $t$                \\
$L_\infty$                     & Lipschitz constant of a decision loss on a response w.r.t. $\Vert \cdot \Vert_\infty$               \\ \midrule
$\tilde{\ell}^{(t)}$           & Linearized loss, defined as $\langle\boldsymbol{p},\boldsymbol{g}^{(t)}\rangle$ at round $t$        \\
$\breve{\boldsymbol{r}}^{(t)}$ & Doubly-Robust (DR) estimator of a response vector at round $t$                                      \\
$\breve{\boldsymbol{g}}^{(t)}$ & Gradient estimate of a DR estimator at round $t$                                                    \\
$\tilde{\boldsymbol{g}}^{(t)}$ & Linearly approximated gradient of a true gradient of decision loss at round $t$                     \\
\begin{tabular}[l]{@{}l@{}}$\breve{L}_\infty$ \\ \end{tabular}             & \begin{tabular}[l]{@{}l@{}}Lipschitz constant of a decision loss on a DR estimator\\ of a partially observed response w.r.t. $\Vert \cdot \Vert_\infty$\end{tabular}   \\ \bottomrule
\end{tabular}%
}
\end{table}

\newpage
\subsection{Mixing Coefficients in Federated Learning}
    For $K$ clients, the FL objective (eq.~\eqref{eq:fl_obj}) aims to minimize the composite objectives, where client $i$'s local objective is $F_i\left({\boldsymbol{\theta}}\right)$, 
    weighted by a corresponding \textit{mixing coefficient} $w_i \geq 0$ ($\sum_{i=1}^K w_i=1$),
    which is usually set to be a \textit{static} value proportional to the sample size $n_i$: 
    e.g., $w_i= \frac{n_i}{n}, n=\sum_{j=1}^K n_j$.
    Denote $\Vert\cdot\Vert_p$ as an $L_p$-norm
    and $\Delta_{K-1}$ as a probability simplex where $\Delta_{K-1}=\left\{\boldsymbol{q} \in \mathbb{R}^K: q_i \geq 0,\Vert \boldsymbol{q} \Vert_1 = 1\right\}$. 
    Note that the mixing coefficient is a member of $\Delta_{K-1}$.
    
    In vanilla FL, the role of the server to solve (\ref{eq:fl_obj}) is to naively add up local updates into a new global model by weighting each update with the \textit{static} mixing coefficient proportional to $n_i$.
    As the fixed scheme often violates the client-level fairness, 
    the server should use \textit{adaptive} mixing coefficients to pursue overall welfare across clients.
    This can be modeled as an optimization w.r.t. $\boldsymbol{p}\in\Delta_{K-1}$.

\newpage
\subsection{Online Convex Optimization as a Unified Language}
\label{sec:oco_lang}
    \begin{table*}[!ht]
\centering
\caption[Summary of unification results of existing fair FL Methods]{Summary of unification results of existing fair FL Methods into an OCO framework (eq.~\eqref{eq:omd_objective}), viz. \textit{Remark}~\ref{remark:unification}}.
\label{tab:3_table2}
\resizebox{\textwidth}{!}{%
\begin{tabular}{l|l|llll}
\toprule
\textbf{Method} & 
\textbf{Original Objective} (w.r.t. $\boldsymbol{\theta}$) & 
\multicolumn{1}{l}{\begin{tabular}[l]{@{}l@{}}\textbf{Response}\\ $({r}_i^{(t)})$ \end{tabular}} &
\multicolumn{1}{l}{\begin{tabular}[l]{@{}l@{}}\textbf{Last Decision}\\ $(p_i^{(t)})$ \end{tabular}} &
\multicolumn{1}{l}{\begin{tabular}[l]{@{}l@{}}\textbf{Step Size}\\ $\left(\zeta\right)$ \end{tabular}} &
\multicolumn{1}{l}{\begin{tabular}[l]{@{}l@{}}\textbf{New Decision}\\ $(p_i^{(t+1)})$ \end{tabular}} \\ \midrule
\begin{tabular}[l]{@{}l@{}}\texttt{FedAvg} \cite{fedavg}\end{tabular} & 
    $\min_{\boldsymbol{\theta}\in\mathbb{R}^d}
    \sum_{i=1}^K \frac{n_i}{n} F_i\left({\boldsymbol{\theta}}\right)$ & 
    $0$ &
    $n_i/n$ & 
    $1$ & 
    $\propto n_i$ \\
\begin{tabular}[l]{@{}l@{}}\texttt{q-FedAvg} \cite{qffl}\\ (\texttt{AFL} \cite{afl}) \end{tabular} & 
    \begin{tabular}[l]{@{}l@{}} $\min_{\boldsymbol{\theta}\in\mathbb{R}^d}  \sum_{i=1}^K \frac{n_i/n}{q+1} F_i^{q+1}\left(\boldsymbol{\theta}\right)$\\(\texttt{AFL} if $q\rightarrow \infty$) \end{tabular} & 
    $q\log {F}_i(\boldsymbol{\theta}^{(t)})$ & 
    $n_i/n$ & 
    $1$ & 
    $\propto n_i {F}_i^q\left(\boldsymbol{\theta}^{(t)}\right)$ \\
\begin{tabular}[l]{@{}l@{}}\texttt{TERM} \cite{term}\end{tabular} & 
    $\min_{\boldsymbol{\theta}\in\mathbb{R}^d}
    \frac{1}{\lambda} \log (\sum_{i=1}^K \frac{n_i}{n} \exp({\lambda F_i({\boldsymbol{\theta}})}))$ & 
    $F_i({\boldsymbol{\theta}}^{(t)})$ & 
    $n_i/n$ & 
    $\frac{1}{\lambda}$ & 
    $\propto n_i \exp\left(\lambda {F}_i\left(\boldsymbol{\theta}^{(t)}\right)\right)$ \\
\begin{tabular}[l]{@{}l@{}}\texttt{PropFair} \cite{propfair}\end{tabular} & 
    $\min_{\boldsymbol{\theta}\in\mathbb{R}^d} -\sum_{i=1}^K \frac{n_i}{n} \log(M-F_i(\boldsymbol{\theta}))$ & 
    $-\log (M - F_i(\boldsymbol{\theta}^{(t)}))$ & 
    $n_i/n$ & 
    $1$ & 
    $\propto \frac{n_i}{M - F_i\left(\boldsymbol{\theta}^{(t)}\right)}$ \\ \bottomrule
\end{tabular}%
}
\end{table*}

    To mitigate the performance inequalities across clients, 
    adaptive mixing coefficients can be estimated in response to local signals 
    (e.g., local losses of a global model).
    Intriguingly, the adaptive aggregation strategies in existing fair FL methods \cite{fedavg, afl, qffl, term, propfair} can be readily \textit{unified into one framework}, borrowing the language of OCO.
    \begin{remark}
    \label{remark:unification}  
    Suppose we want to solve a minimization problem defined in (\ref{eq:omd_objective}).
    For all $t\in[T]$, it aims to minimize a \textit{decision loss} $\ell^{(t)}\left(\boldsymbol{p}\right) = -\left\langle \boldsymbol{p}, \boldsymbol{r}^{(t)} \right\rangle$ 
    (where $\left\langle \cdot, \cdot \right\rangle$ is an inner product)
    defined by a \textit{response} $\boldsymbol{r}^{(t)}\in\mathbb{R}^K$ 
    and a \textit{decision}  $\boldsymbol{p}\in\Delta_{K-1}$, 
    with a \textit{regularizer} $R\left(\boldsymbol{p}\right)$ having a constant \textit{step size} $\zeta\in\mathbb{R}_{\geq0}$.
    \begin{equation}
    \begin{aligned}
    \label{eq:omd_objective}
        \boldsymbol{p}^{(t+1)}
        =   
        \argmin_{\boldsymbol{p} \in \Delta_{K-1}} \ell^{(t)}\left(\boldsymbol{p}\right)
        + \zeta R\left(\boldsymbol{p}\right)
    \end{aligned}    
    \end{equation} 
    \end{remark}
    
    As long as the regularizer $R\left(\boldsymbol{p}\right)$ in the Remark~\ref{remark:unification} is fixed as the negative entropy, i.e., $R\left(\boldsymbol{p}\right)=\sum_{i=1}^K p_i \log p_i$, this subsumes aggregation strategies proposed in \texttt{FedAvg} \cite{fedavg}, \texttt{AFL} \cite{afl}, \texttt{q-FFL} \cite{qffl}, \texttt{TERM} \cite{term}, and \texttt{PropFair} \cite{propfair}.
    It has an update as follows.
    \begin{equation}
    \begin{aligned}
    \label{eq:eg_update}
        p^{(t+1)}_i \propto {p^{(t)}_i \exp\left( r^{(t)}_i / \zeta\right)} 
    \end{aligned}    
    \end{equation}
    This is widely known as EG~\cite{eg}, a special realization of OMD~\cite{md, omd, bregmanmd}.
    We summarized how existing methods can be unified under this OCO framework in Table~\ref{tab:2_table1}, and derivations of mixing coefficients (i.e.,, $p_i^{(t+1)}$) of each method are in section~\ref{app:unification}.
    
    To sum up, we can interpret the aggregation mechanism in FL is secretly a result of \textit{the server’s sequential decision making} behind the scene.
    Since the sequential learning scheme is \textit{well-behaved in a sample-deficient setting}, 
    adopting OCO is surely a suitable tactic for the server in that \textit{only a few bits are provided to update the mixing coefficients} in each FL round, e.g, the number of local responses collected from the clients is at most $K$.
    However, existing methods have not been devised with sequential decision making in mind.
    Therefore, one can easily find suboptimal designs inherent in existing methods from an OCO perspective.

\newpage
\subsection{Sequential Probability Assignment}
    To address the client-level fairness, the server should make an adaptive mixing coefficient vector, 
    $\boldsymbol{p}^{(t)}\in\Delta_{K-1}$, for each round $t\in[T]$.
    In other words, the server needs to assign appropriate probabilities sequentially to local updates in every FL communication round.
    
    Notably, this fairly resembles OPS, which seeks to maximize an investor's cumulative profits on a set of $K$ assets during $T$ periods, 
    by assigning his/her wealth $\boldsymbol{p}\in\Delta_{K-1}$ to each asset every time.
    In the OPS, the investor observes a price of all assets, $\boldsymbol{r}^{(t)}\in\mathbb{R}^K$ for each time $t\in[T]$ 
    and accumulates corresponding wealth according to the portfolio $\boldsymbol{p}^{(t)}\in\Delta_{K-1}$. 
    After $T$ periods, achieved cumulative profits is represented as $\prod_{t=1}^T \left(1 + \left\langle\boldsymbol{p}^{(t)},\boldsymbol{r}^{(t)}\right\rangle\right)$, 
    or in the form of logarithmic growth ratio, 
    $\sum_{t=1}^T \log\left(1+\left\langle\boldsymbol{p}^{(t)},\boldsymbol{r}^{(t)}\right\rangle\right)$.
    In other words, one can view that OPS algorithms adopt negative logarithmic growth as a decision loss.
    
    \begin{definition} (Negative Logarithmic Growth as a Decision Loss) 
    For all $t\in[T]$, define a decision loss $\ell^{(t)}:\Delta_{K-1} \times \mathbb{R}^K \rightarrow \mathbb{R}$ as follows.
    \begin{equation}
    \begin{gathered}
    \label{eq:decision_loss}
        \ell^{(t)}(\boldsymbol{p}) = -\log(1+\langle\boldsymbol{p}, \boldsymbol{r}^{(t)}\rangle),
    \end{gathered}
    \end{equation}    
    where $\boldsymbol{p}$ is a decision vector in $\Delta_{K-1}$ and $\boldsymbol{r}^{(t)}$ is a response vector given at time $t$.
    \end{definition}
    
    Again, the OPS algorithm can serve as a metaphor for the central server's fairness-aware online decision making in FL. 
    For example, one can regard a response (i.e., local losses) of $K$ clients at a specific round $t$ the same as returns of assets on the day $t$. 
    Similarly, by considering cumulative losses (i.e., cumulative wealth) achieved until $t$, 
    the server can determine the next mixing coefficients (i.e., portfolio ratios) in $\Delta_{K-1}$. 
    In the same context, the negative logarithmic growth can also be adopted as the decision loss.
    Accordingly, we can adopt well-established OPS strategies for achieving client-level fairness in FL.
    By adopting the negative logarithmic growth, we can find a loose connection between the well-known fairness notion, i.e., max-min fairness.
    \begin{remark} When the decision loss is set to the negative logarithmic growth, we can regard the global objective can be viewed as:
    \begin{equation}
    \begin{gathered}
        \min_{\boldsymbol{\theta}}\max_{\boldsymbol{p}} \log\left(1+\sum_{i=1}^K p_i F_i(\boldsymbol{\theta})\right),
    \end{gathered}
    \end{equation}
    by ignoring the regularizer and further assuming that each entry of the response $r_i$ is a monotonic increasing transformation of $F_i(\boldsymbol{\theta})$.
    Since the logarithm is also a monotonically increasing function, this objective share the same optimum as min-max fairness objective~\cite{afl}, i.e., $\min_{\boldsymbol{\theta}}\max_{\boldsymbol{p}} \left(\sum_{i=1}^K p_i F_i(\boldsymbol{\theta})\right)$.
    Different from the min-max objective, this reduced form can be regarded as a \textit{rectified} version of the min-max objective in that it is bounded below.
    \end{remark}
    
    Including OPS, a de facto standard objective for OCO is to minimize the regret defined in (\ref{eq:regret}), 
    with regard to the best decision in hindsight, $\boldsymbol{p}^\star\triangleq\argmin_{\boldsymbol{p}\in\Delta_{K-1}}\sum_{t=1}^T\ell^{(t)}(\boldsymbol{p})$,
    given all decisions $\{ \boldsymbol{p}^{(1)}, ..., \boldsymbol{p}^{(T)} \}$. \cite{oco1, oco2, oco3}
    \begin{equation}
    \begin{gathered}
    \label{eq:regret}
        \text{Regret}^{(T)}(\boldsymbol{p}^{\star}) 
        = 
        \sum\nolimits_{t=1}^T 
        \ell^{(t)}(\boldsymbol{p}^{(t)}) 
        - 
        \sum\nolimits_{t=1}^T \ell^{(t)}\left(\boldsymbol{p}^{\star}\right)
    \end{gathered}
    \end{equation}    
    In finite time $T$, an online decision making strategy should guarantee that the regret grows sublinearly.
    Therefore, when OPS strategies are modified for the fair FL, 
    we should check if the strategy can guarantee vanishing regret upper bound in $T$.
    Besides, we should also consider \textit{the dependency on $K$} due to practical constraints of the federated system.

\newpage
\section{Proposed Methods}
\subsection{Improved Design for Better Decision Making}
    From the Remark~\ref{remark:unification} and Table~\ref{tab:3_table2}, one can easily notice suboptimal designs of existing methods in terms of OCO, as follows.
    \begin{enumerate}[i)] 
    \label{list:existing_problems}
        \item \label{item:problem1} Existing methods are \textit{stateless} in making a new decision, $p^{(t+1)}_i$. 
        The previous decision is ignored as a fixed value ($p^{(t)}_i=n_i/n$) in the subsequent decision making.
        This naive reliance on static coefficients still runs the risk of violating client-level fairness.  
        \item \label{item:problem2} The decision maker sticks to a \textit{fixed} and \textit{arbitrary} step size $\eta$, or a \textit{fixed} regularizer $R(\boldsymbol{p})$ across $t\in[T]$, which can significantly affect the performance of OCO algorithms and should be manually selected. 
        \item \label{item:problem3} The decision loss is neither \textit{Lipschitz continuous} nor \textit{strictly convex}, which is related to achieving a sublinear regret.
    \end{enumerate}

    As a remedy for handling \ref{item:problem1} and \ref{item:problem2}, 
    the OMD objective for the server (i.e., eq.~\eqref{eq:omd_objective}) can be replaced as follows. 
    \begin{equation}
    \begin{aligned}
    \label{eq:ftrl_objective}
        \boldsymbol{p}^{(t+1)} 
        = \argmin_{\boldsymbol{p} \in \Delta_{K-1}} \sum\nolimits_{\tau=1}^{t} \ell^{(\tau)}\left(\boldsymbol{p}\right) 
        + R^{(t+1)}\left(\boldsymbol{p}\right) \\
    \end{aligned}
    \end{equation}
    This is also known as FTRL objective \cite{ftrl1, ftrl2, ftrl3, ftrl4}, 
    which is inherently a \textit{stateful} sequential decision making algorithm 
    that adapts to histories of decision losses, $\sum_{\tau=1}^{t} \ell^{(\tau)}\left(\boldsymbol{p}\right)$,
    where $\ell^{(t)}:\Delta_{K-1} \times \mathbb{R}^K \rightarrow \mathbb{R}$, 
    with the \textit{time-varying} regularizer $R^{(t)}:\Delta_{K-1} \rightarrow \mathbb{R}$. 
    Note that the time-varying regularizer is sometimes represented as, $\eta^{(t+1)}R(\boldsymbol{p})$, 
    a fixed regularizer $R\left(\boldsymbol{p}\right)$ multiplied by a \textit{time-varying} step size, $\eta^{(t+1)}\in\mathbb{R}_{\geq0}$, which can later be automatically determined from the regret analysis
    (see e.g., Remark~\ref{remark:closed_form}).
    
    Additionally, when equipped with the negative logarithmic growth as a decision loss (i.e., eq.~\eqref{eq:decision_loss}), 
    the problem \ref{item:problem3} can be addressed due to its strict convexity and Lipscthiz continuity. 
    (See Lemma \ref{lemma:lipschitz})
    Note that when the loss function is convex, we can run the FTRL with a linearized loss 
    (i.e., $\tilde\ell^{(t)}\left(\boldsymbol{p}\right) 
    =\left\langle \boldsymbol{p}, \boldsymbol{g}^{(t)} \right\rangle$ 
    where $\boldsymbol{g}^{(t)}=\nabla\ell^{(t)}\left(\boldsymbol{p}^{(t)}\right)$). 
    This is useful in that a closed-form update can be obtained thanks to the properties of the Fenchel conjugate (see Remark \ref{remark:closed_form}).

\newpage
\subsection{\texttt{AAggFF}: Adaptive Aggregation for Fair Federated Learning}
    Based on the improved objective design derived from the FTRL, we now introduce our methods, \texttt{AAggFF},
    an acronym of \textbf{\underline{A}}daptive \textbf{\underline{Agg}}regation for \textbf{\underline{F}}air \textbf{\underline{F}}ederated Learning).
    Mirroring the practical requirements of FL, we further subdivide into two algorithms: 
    \texttt{AAggFF-S} for the cross-silo setting and \texttt{AAggFF-D} for the cross-device setting.

\subsubsection{\texttt{AAggFF-S}: Algorithm for the Cross-Silo Federated Learning}
\label{subsubsec:aaggff_s}
    \begin{algorithm}[h]
   \caption{\texttt{AAggFF-S}}
   \label{alg:aaggff-s}
    \begin{algorithmic}
       \STATE {\bfseries Input:} number of clients $K$, total rounds $T$, transformation $\rho$, minimum and maximum of a response $[C_1, C_2]$.
       \STATE {\bfseries Initialize:} mixing coefficients $\boldsymbol{p}^{(1)}=\frac{1}{K}\boldsymbol{1}_K$, global model $\boldsymbol{\theta}^{(1)}\in\mathbb{R}^d$
       \STATE {\bfseries Procedure:} 
       \FOR{$t=0$ {\bfseries to} $T-1$}
        \FOR{each client $i=1,...,K$ {\bfseries in parallel}} 
            \STATE $F_i\left(\boldsymbol{\theta}^{(t)}\right), \boldsymbol{\theta}^{(t)}-\boldsymbol{\theta}^{(t+1)}_i \leftarrow \texttt{AAggFFClientUpdate}\left(\boldsymbol{\theta}^{(t)}\right)$
        \ENDFOR
        \STATE Return $\boldsymbol{r}^{(t)}$ according to eq.~\eqref{eq:resp_vec} and $C_1, C_2$.
        \STATE Suffer decision loss $\ell^{(t)}(\boldsymbol{p}^{(t)})$ according to eq.~\eqref{eq:decision_loss}.
        \STATE Return a gradient $\boldsymbol{g}^{(t)}=\nabla\ell^{(t)}\left(\boldsymbol{p}^{(t)}\right)$.
        \STATE Return a mixing coefficient $\boldsymbol{p}^{(t+1)}$ according to eq.~\eqref{eq:ons}.
        \STATE Update a global model $\boldsymbol{\theta}^{(t+1)} = \boldsymbol{\theta}^{(t)}-\sum\limits_{i=1}^K p^{(t+1)}_i \left(\boldsymbol{\theta}^{(t)}-\boldsymbol{\theta}^{(t+1)}_i\right)$.
       \ENDFOR
       \STATE{\bfseries Return:} $\boldsymbol{\theta}^{(T)}$
    \end{algorithmic}
\end{algorithm}
    In the cross-silo setting, it is typically assumed that \textit{all} $K$ clients participate in $T$ rounds,
    since there are a moderately small number of clients in the federated system.
    Therefore, the server's stateful decision making is beneficial for enhancing overall welfare across federation rounds.
    This is also favorable since existing OPS algorithms can be readily adopted.
    
\paragraph{Online Newton Step \cite{ons1, ons2}} 
    The ONS algorithm updates a new decision as follows ($\alpha$ and $\beta$ are constants to be determined).
    \begin{equation}
    \begin{gathered}
    \label{eq:ons}
        \boldsymbol{p}^{(t+1)}
        =
        \argmin_{\boldsymbol{p} \in \Delta_{K-1}} \sum\nolimits_{\tau=1}^{t} 
        \tilde\ell^{(\tau)}\left(\boldsymbol{p}\right)
        +
        \frac{\alpha}{2} \Vert \boldsymbol{p} \Vert_2^2 \\
        +
        \frac{\beta}{2} \sum\nolimits_{\tau=1}^{t} 
        (\langle \boldsymbol{g}^{(\tau)}, \boldsymbol{p} - \boldsymbol{p}^{(\tau)}\rangle)^2
    \end{gathered}
    \end{equation}
    The ONS can be reduced to the FTRL objective introduced in (\ref{eq:ftrl_objective}).
    It can be retrieved when we use a linearized loss, 
    $\tilde\ell^{(t)}\left(\boldsymbol{p}\right)
    =
    \left\langle \boldsymbol{p}, \boldsymbol{g}^{(t)} \right\rangle$,
    and the time-varying proximal regularizer, defined as 
    $R^{(t+1)}\left(\boldsymbol{p}\right)=\frac{\alpha}{2}\Vert\boldsymbol{p}\Vert^2_2 + 
    \frac{\beta}{2}\sum_{\tau=1}^{t} \left( \left\langle \boldsymbol{g}^{(\tau)}, \boldsymbol{p} - \boldsymbol{p}^{(\tau)} \right\rangle \right)^2$.
    
    We choose ONS in that its regret is optimal in $T$, which is also a dominating constant for the cross-silo FL setting:
    $\mathcal{O}(L_\infty K\log{T})$ regret upper bound, 
    where $L_\infty$ is the Lipschitz constant of decision loss w.r.t. $\Vert\cdot\Vert_\infty$.
    That is, $L_\infty$ should be \textit{finite} for a vanishing regret (see Theorem~\ref{thm:crosssilo}).
    
\paragraph{Necessity of Bounded Response}
\label{sec:resp_trans}
    Note that the Lipchitz continuity of the negative logarithmic growth as a decision loss is determined as follows.
    \begin{lemma}
    \label{lemma:lipschitz}
    For all $t\in[T]$, suppose each entry of a response vector $\boldsymbol{r}^{(t)}\in\mathbb{R}^K$ is bounded 
    as $r_i^{(t)}\in[C_1,C_2]$ for some constants $C_1$ and $C_2$ satisfying $0<C_1<C_2$.
    Then, the decision loss $\ell^{(t)}$ defined in (\ref{eq:decision_loss}) is $\frac{C_2}{1+C_1}$-Lipschitz continuous 
    in $\Delta_{K-1}$ w.r.t. $\Vert \cdot \Vert_\infty$.
    \end{lemma}
    
    From now on, all proofs are deferred to section~\ref{sec:proofs}. 
    According to the Lemma~\ref{lemma:lipschitz}, the Lipschitz constant of the decision loss, $L_\infty$,
    is dependent upon \textit{the range of a response vector's element}.
    While from the unification result in Table~\ref{tab:3_table2}, 
    one can easily notice that the response is constructed from local losses collected in round $t$, 
    $F_i(\boldsymbol{\theta}^{(t)})\in\mathbb{R}_{\geq0}, i\in[K]$.
    
    This is a scalar value calculated from a local training set of each client, 
    using the current model $\boldsymbol{\theta}^{(t)}$ \textit{before its local update}. 
    Since the local loss function is typically unbounded above (e.g., cross-entropy), 
    it should be transformed into bounded values to satisfy the Lipschitz continuity. 
    In existing fair FL methods, however, all responses are not bounded above, 
    thus we cannot guarantee the Lipschitz continuity.  
    
    To ensure a bounded response, we propose to use a transformation denoted as $\rho^{(t)}(\cdot)$, 
    inspired by the cumulative distribution function (CDF) as follows.
    \begin{definition} (CDF-driven Response Transformation)
    \label{def:resp_cdf}
        We define $r_i^{(t)} \equiv \rho^{(t)} \left( {F_i(\boldsymbol{\theta}^{(t)})} \right)$,
        each element of the response vector is defined from the corresponding entry of a local loss
        by an element-wise mapping $\rho^{(t)}:\mathbb{R}_{\geq0}\rightarrow[C_1,C_2]$,
        given a pre-defined \texttt{CDF} as: 
        \begin{equation}
        \begin{aligned}
        \label{eq:resp_vec}
            \rho^{(t)} \left( {F_i\left(\boldsymbol{\theta}^{(t)}\right)} \right)
            \triangleq C_1 + (C_2 - C_1)\texttt{CDF} \left( \frac{F_i\left(\boldsymbol{\theta}^{(t)}\right)}{\bar{\mathrm{F}}^{(t)}} \right),
        \end{aligned}
        \end{equation}
        where $\bar{\mathrm{F}}^{(t)} = \frac{1}{\lvert S^{(t)} \rvert}\sum_{i\in S^{(t)}} F_i\left(\boldsymbol{\theta}^{(t)}\right)$, 
        and $S^{(t)}$ is an index set of available clients in $t$.
    \end{definition}
    
    Note again that the larger mixing coefficient should be assigned for the larger local loss.
    In such a perspective, using the CDF for transforming a loss value is an acceptable approach 
    in that the CDF value is a good indicator for estimating ``\textit{how large a specific local loss is}'', 
    relative to local losses from other clients.
    To instill the comparative nature, local losses are divided by the average of observed losses in time $t$ 
    before applying the transformation.
    As a result, all local losses are centered on 1 in expectation.
    See section~\ref{subsec:exp_resp_transformation} and section~\ref{app:range} for detailed discussions.
    
    In summary, the whole procedure of \texttt{AAggFF-S} is illustrated as a pseudocode in Algorithm~\ref{alg:aaggff-s}.
    
\newpage
\subsubsection{\texttt{AAggFF-D}: Algorithm for the Cross-Device Federated Learning}
\label{subsubsec:aaggff_d}
    \begin{algorithm}[h]
   \caption{\texttt{AAggFF-D}}
   \label{alg:aaggff-d}
\begin{algorithmic}
   \STATE {\bfseries Input:} number of clients $K$, client sampling ratio $C\in(0,1)$, total rounds $T$, transformation $\rho$, range of a response $[C_1, C_2]$.
   \STATE {\bfseries Initialize:} mixing coefficients $\boldsymbol{p}^{(0)}=\frac{1}{K}\boldsymbol{1}_K$, global model $\boldsymbol{\theta}^{(0)}\in\mathbb{R}^d$
   \STATE {\bfseries Procedure:} \FOR{$t=0$ {\bfseries to} $T-1$}
    \STATE $S^{(t)} \leftarrow \text{Wait until } \min(1, \lfloor C \cdot K \rfloor) \text{ clients are active in a network}$.
    \FOR{each client $i \in S^{(t)}$ {\bfseries in parallel}} 
        \STATE $F_i\left(\boldsymbol{\theta}^{(t)}\right), \boldsymbol{\theta}^{(t)}-\boldsymbol{\theta}^{(t+1)}_i \leftarrow \texttt{AAggFFClientUpdate}\left(\boldsymbol{\theta}^{(t)}\right)$
    \ENDFOR
    \STATE Return $\breve{\boldsymbol{r}}^{(t)}$ according to eq.~\eqref{eq:resp_vec}, eq.~\eqref{eq:dr_response}, and  $C_1, C_2$.
    \STATE Suffer decision loss $\ell^{(t)}(\boldsymbol{p}^{(t)})$ according to eq.~\eqref{eq:decision_loss}.
    \STATE Return a gradient estimate $\breve{\boldsymbol{g}}^{(t)}$ according to eq.~\eqref{eq:linearized_gradient}.
    \STATE Return mixing coefficients $\boldsymbol{p}^{(t+1)}$ according to eq.~\eqref{eq:closed_form}).
    \STATE Acquire selected coefficients $\tilde{\boldsymbol{p}}^{(t+1)} = \frac{p_i^{(t+1)}}{ \sum_{j \in S^{(t)}} p_j^{(t+1)} }, i \in S^{(t)}$. 
    \STATE Update a global model $\boldsymbol{\theta}^{(t+1)} = \boldsymbol{\theta}^{(t)}-\sum\limits_{i \in S^{(t)}} \tilde{p}^{(t+1)}_i \left(\boldsymbol{\theta}^{(t)}-\boldsymbol{\theta}^{(t+1)}_i\right)$.
   \ENDFOR
   \STATE{\bfseries Return:} $\boldsymbol{\theta}^{(T)}$
\end{algorithmic}
\end{algorithm}
    Unlike the cross-silo setting, we cannot be na\"{i}vely edopt existing OCO algorithms for finding adaptive mixing coefficients in the cross-device setting.
    It is attributed to \textit{the large number of participating clients} in this special setting.
    Since the number of participating clients ($K$) is massive (e.g., Android users are over 3 billion \cite{android}), 
    the dependence on $K$ in terms of regret bound and algorithm runtime is as significant as a total communication round $T$. 
    
\paragraph{Linear Runtime OCO Algorithm}
    The ONS has regret proportional to $K$ and runs in $\mathcal{O}\left(K^2+K^3\right)$\footnote{A generalized projection required for the Online Newton Step can be solved in $\tilde{\mathcal{O}}\left(K^3\right)$ \cite{ons1, ons2}.} per round,
    which is acceptable in practice for the cross-silo setting thanks to the moderate size of $K$, 
    but is \textit{nearly impossible} to be adopted for the cross-device FL setting due to large $K$, 
    even though the logarithmic regret is guaranteed in $T$.
    Instead, we can exploit the variant of EG adapted to FTRL \cite{oco3}, 
    which can be run in $\mathcal{O}\left(K\right)$ time per round. 
    \begin{equation}
    \begin{gathered}
    \label{eq:lin_ftrl}
        \boldsymbol{p}^{(t+1)}
        =
        \argmin_{\boldsymbol{p} \in \Delta_{K-1}} \sum\nolimits_{\tau=1}^{t} 
        \tilde\ell^{(\tau)}\left(\boldsymbol{p}\right)
        +
        \eta^{(t+1)}R(\boldsymbol{p}),
    \end{gathered}
    \end{equation}
    where $\eta^{(t)}$ is non-decreasing step size across $t\in[T]$, 
    and $R(\boldsymbol{p})=\sum\nolimits_{i=1}^K p_i \log p_i$ is a negative entropy regularizer.
    Still, the regret bound gets worse than that of ONS, as $\mathcal{O}\left(L_\infty \sqrt{T\log{K}}\right)$
    (see Theorem \ref{thm:crossdevice_full}).
    
\paragraph{Partially Observed Response} 
    The large number of clients coerces the federated system to introduce the \textit{client sampling scheme} in each round.
    Therefore, the decision maker (i.e., the central server) cannot always observe all entries of a response vector per round.
    This is problematic in terms of OCO, since OCO algorithms assume that they can acquire intact response vector for every round $t\in[T]$.
    Instead, when the client sampling is introduced, the learner can only observe entries of sampled client indices in the round $t$, denoted as $S^{(t)}$.
    
    To make a new decision using a \textit{partially observed} response vector, the effect of unobserved entries should be appropriately estimated.
    We solve this problem by adopting a doubly robust (DR) estimator \cite{doublyrobust2, doublyrobust3} for the expectation of the response vector.
    The rationale behind the adoption of the DR estimator is the fact that the unobserved entries are \textit{missing data}.
    
    For handling the missingness problem, the DR estimator combines inverse probability weighting (IPW \cite{exp3}) estimator and imputation mechanism,
    where the former is to adjust the weight of observed entries by the inverse of its observation probability (i.e., client sampling probability), 
    and the latter is to fill unobserved entries with appropriate values specific to a given task.
    
    Similar to the IPW estimator, the DR estimator is an unbiased estimator when the true observation probability is known. 
    Since we sample clients uniformly at random without replacement, \textit{the observation probability is known} 
    (i.e., $C\in(0,1)$) to the algorithm.
    
    \begin{lemma}
    \label{lemma:unbiased_resp}
    Denote $C=P\left(i \in S^{(t)}\right)$ as a client sampling probability in a cross-device FL setting for every round $t\in[T]$. 
    The DR estimator $\breve{\boldsymbol{r}}^{(t)}$, of which element is defined in (\ref{eq:dr_response}) is an unbiased estimator of given partially observed response vector $\boldsymbol{r}^{(t)}$.
    i.e., $\mathbb{E}\left[\breve{\boldsymbol{r}}^{(t)}\right] = \boldsymbol{r}^{(t)}$.
    \begin{equation}
    \begin{gathered}
    \label{eq:dr_response}
        \breve{r}^{(t)}_i = \left(1 - \frac{\mathbb{I}(i \in S^{(t)})}{C}\right)\mathrm{\bar{r}}^{(t)} + \frac{\mathbb{I}\left(i \in S^{(t)}\right)}{C}{r}^{(t)}_i,
    \end{gathered}
    \end{equation}
    where $\bar{\mathrm{r}}^{(t)}=\frac{1}{\left\vert S^{(t)}\right\vert} \sum_{i \in S^{(t)}} r_i^{(t)}$.
    \end{lemma}
    
    Still, it is required to guarantee that the gradient vector from the DR estimator is also an unbiased estimator of a true gradient of a decision loss.
    Unfortunately, the gradient of a decision loss is \textit{not linear} in the response vector due to its fractional form: 
    ${\boldsymbol{g}}^{(t)}
    =\nabla \ell^{(t)}\left(\boldsymbol{p}^{(t)}\right)
    =-\frac{\boldsymbol{r}^{(t)}}{1 +\left\langle \boldsymbol{p}^{(t)}, \boldsymbol{r}^{(t)} \right\rangle}$.
    
    Therefore, we instead use linearly approximated gradient \textit{w.r.t. a response vector} as follows.
    \begin{lemma}
    \label{lemma:linearized_grad}
        Denote the gradient of a decision loss in terms of a response vector as $\boldsymbol{g}\equiv\mathrm{\mathbf{h}}(\boldsymbol{r}) = \left[h_1(\boldsymbol{r}),...,h_K(\boldsymbol{r})\right]^\top = -\frac{\boldsymbol{r}}{1 +\left\langle \boldsymbol{p}, \boldsymbol{r} \right\rangle}$. 
        It can be linearized for the response vector into ${\tilde{\boldsymbol{g}}}\equiv\tilde{\mathrm{\mathbf{h}}}(\boldsymbol{r})$, 
        given a reference $\boldsymbol{r}_0$ as follows. 
        (Note that the superscript ${(t)}$ is omitted for a brevity of notation)
        \begin{equation}
        \begin{gathered}
        \label{eq:linearized_gradient}
            {\boldsymbol{g}}
            \approx
            {\tilde{\boldsymbol{g}}}
            \equiv
            \tilde{\mathrm{\mathbf{h}}}(\boldsymbol{r})
            =
            -\frac{\boldsymbol{r}}{1 +\left\langle
            \boldsymbol{p},
            \boldsymbol{r}_0
            \right\rangle}
            +
            \frac{\boldsymbol{r}_0 \boldsymbol{p}^\top (\boldsymbol{r} - \boldsymbol{r}_0)}{(1 +\left\langle
            \boldsymbol{p},
            \boldsymbol{r}_0
            \right\rangle)^2}
        \end{gathered}
        \end{equation}
        Further denote $\breve{\boldsymbol{g}}$ as a gradient estimate from (\ref{eq:linearized_gradient}) 
        using the DR estimator of a response vector according to (\ref{eq:dr_response}),
        at an arbitrary reference $\boldsymbol{r}_0$. 
        Then, $\breve{\boldsymbol{g}}$ is an unbiased estimator of the linearized gradient of a decision loss at $\boldsymbol{r}_0$, 
        which is close to the true gradient: $\mathbb{E}\left[\breve{\boldsymbol{g}}\right] = \tilde{\boldsymbol{g}}\approx\boldsymbol{g}$.
    \end{lemma}
    
    As suggested in (\ref{eq:dr_response}), we similarly set the reference as an average of observed responses at round $t$, i.e., $\boldsymbol{r}_0^{(t)} = \mathrm{\bar{r}}^{(t)}\boldsymbol{1}_K$.
    It is a valid choice in that dominating unobserved entries are imputed by the average of observed responses as in (\ref{eq:dr_response}).
    
    To sum up, we can update a new decision using this unbiased and linearly approximated gradient estimator even if only a partially observed response vector is provided (i.e., mixing coefficients of unsampled clients can also be updated).
    Note that the linearized gradient calculated from the DR estimator, $\breve{\boldsymbol{g}}$, has finite norm w.r.t. $\Vert \cdot \Vert_\infty$ (see Lemma~\ref{lemma:lipschitz_lin_grad}).
    
\paragraph{Closed-Form Update} 
    Especially for the cross-device setting, we can obtain a closed-form update of the objective (\ref{eq:lin_ftrl}), 
    which is due to the property of Fenchel conjugate.
    
    \begin{remark}
    \label{remark:closed_form} 
    The objective of \texttt{AAggFF-D} stated in (\ref{eq:lin_ftrl}) has a closed-form update formula as follows. \cite{oco3}
    \begin{equation}
    \begin{aligned}
    \label{eq:closed_form}
        {p}^{(t+1)}_i 
        \propto
        { \exp\left(-\frac{\sqrt{\log{K}}\sum_{\tau=1}^{t} \breve{g}^{(\tau)}_i}{\breve{L}_\infty\sqrt{t+1}}\right) }
    \end{aligned}
    \end{equation}
    \end{remark}
    It is equivalent to setting the time-varying step size as $\eta^{(t)}=\frac{ \breve{L}_\infty\sqrt{t} }{ \sqrt{\log{K}} }$. 
    Note that $\breve{g}^{(t)}_i$ is an entry of gradient from DR estimator defined in Lemma~\ref{lemma:linearized_grad} 
    and $\breve{L}_\infty$ is a corresponding Lipschitz constant satisfying $\Vert \breve{\boldsymbol{g}} \Vert_\infty \leq \breve{L}_\infty$
    stated in Lemma~\ref{lemma:lipschitz_lin_grad}. 
    See section~\ref{app:deriv_closed} for the derivation.

    In summary, the whole procedure of \texttt{AAggFF-D} is illustrated in Algorithm~\ref{alg:aaggff-d}.
    \begin{algorithm}[h]
   \caption{\texttt{AAggFFClientUpdate}}
   \label{alg:clientupate}
\begin{algorithmic}
   \STATE {\bfseries Input:} number of local epochs $E$, local batch size $B$, local learning rate $\eta$, global model $\boldsymbol{\theta}$
   \STATE {\bfseries Procedure:}
   \STATE Evaluate the received global model on training set to yield feedback $F_i\left(\boldsymbol{\theta}\right)$.
   \STATE Set a local model $\boldsymbol{\theta}^{(0)}\leftarrow\boldsymbol{\theta}$.
   \FOR{$e=0$ {\bfseries to} $E-1$}
    \STATE $\mathcal{B}_e$ $\leftarrow$ Split the client training dataset into batches of size $B$.
    \FOR{mini-batch $b$ in $\mathcal{B}_e$}
     \STATE Update the model $\boldsymbol{\theta}^{(e)}\leftarrow\boldsymbol{\theta}^{(e)} - \frac{\eta}{B} \sum\limits_{k=1}^B \nabla_{\boldsymbol{\theta}} \mathcal{L}\left(b;\boldsymbol{\theta}^{(e)}\right)$.
    \ENDFOR
    \STATE Set $\boldsymbol{\theta}^{(e+1)}\leftarrow\boldsymbol{\theta}^{(e)}$
   \ENDFOR
   \STATE{\bfseries Return:} $F_i\left(\boldsymbol{\theta}\right)$, $\boldsymbol{\theta} - \boldsymbol{\theta}^{(E)}$.
\end{algorithmic}
\end{algorithm}

\newpage
\section{Analyses}
\subsection{Cumulative Distribution Function for Response Transformation}
\label{app:cdfs}
\paragraph{Choice of Distributions}
    We used the CDF to transform unbounded responses from clients (i.e., local losses of clients) into bounded values.
    Among diverse options, we used one of the following 6 CDFs in this work.
    (Note that  $\operatorname{erf}(x)=\frac{2}{\sqrt{\pi}}\int_{0}^x e^{-y^2} \mathrm{d}y$ is the Gauss error function)
    
    \begin{enumerate}
        \item Weibull \cite{weibull}: $\texttt{CDF}(x)=1 - e^{-(x/\alpha)^\beta}$; we set $\alpha=1$ (scale) and $\beta=2$ (shape).
        \item Frechet \cite{frechet}: $\texttt{CDF}(x)=e^{-(\frac{x}{\alpha})^{-\beta}}$; we set $\alpha=1$ (scale) and $\beta=1$ (shape).
        \item Gumbel \cite{gumbel}: $\texttt{CDF}(x)=e^{-e^{-(x - \alpha) / \beta}}$; we set $\alpha=1$ (scale) and $\beta=1$ (shape).
        \item Exponential: $\texttt{CDF}(x)=1 - e^{-\alpha x}$; we set $\alpha=1$ (scale).
        \item Logistic: $\texttt{CDF}(x)=\frac{1}{1+e^{(-(x-\alpha)/\beta)}}$; we set $\alpha=1$ (scale) and $\beta=1$ (shape).
        \item Normal \cite{gaussian}: $\texttt{CDF}(x)=\frac{1}{2}\left[1+\operatorname{erf}\left(\frac{x - \alpha}{\beta\sqrt{2}}\right)\right]$; we set $\alpha=1$ (scale) and $\beta=1$ (shape).
    \end{enumerate}

    Commonly, the scale parameter of all distributions is set to 1, since in (\ref{eq:resp_vec}) we centered inputs to 1 in expectation.
    Although we fixed the parameters of each CDF, they can be statistically estimated in practice, such as using maximum spacing estimation \cite{mse}. 
    
    For imposing larger mixing coefficients for larger losses, the transformation should (i) preserve the relative difference between responses, as well as (ii) not too sensitive for outliers.
    While other heuristics (e.g., clipping values, subtracting from arbitrary large constant \cite{propfair}) for the transformation are also viable options for (i),
    additional efforts are still required to address (ii).
    
    On the contrary, CDFs can address both conditions with ease.
    As CDFs are increasing functions, (i) can be easily satisfied.
    For (ii), it can be intrinsically addressed by the nature of CDF itself.
    Let us start with a simple example.
    
    Suppose we have $K=3$ local losses: $F_1(\boldsymbol{\theta}) = 0.01, F_2(\boldsymbol{\theta}) = 0.10, F_3(\boldsymbol{\theta}) = 0.02$.
    Since the average is $\bar{\mathrm{F}}=\frac{0.01+0.10+0.02}{3}\approx0.043$, we have inputs of CDF as follows: $F_1(\boldsymbol{\theta})/\bar{\mathrm{F}} = 0.23, F_2(\boldsymbol{\theta})/\bar{\mathrm{F}} = 2.31, F_3(\boldsymbol{\theta})/\bar{\mathrm{F}} = 0.46$.
    These centered inputs are finally transformed into bounded values as in Table~\ref{tab:cdf_example}.
    \begin{table}[h]
\footnotesize
\centering
\caption{Effects of CDF transformations in \texttt{AAggFF}}
\label{tab:cdf_example}
\begin{tabular}{@{}lc@{}}

\toprule                                                                                           & Transformed Responses \\ \midrule
\begin{tabular}[c]{@{}l@{}}Weibull CDF \\ $\texttt{CDF}(x)=1-e^{(-x^2)}$\end{tabular}              & 0.05 / \underline{1.00} / 0.19 \\
\begin{tabular}[c]{@{}l@{}}Frechet CDF \\ $\texttt{CDF}(x)=e^{(-1/x)}$\end{tabular}                & 0.01 / \underline{0.65} / 0.11 \\
\begin{tabular}[c]{@{}l@{}}Gumbel CDF \\ $\texttt{CDF}(x)=e^{\left(-e^{(-(x - 1))}\right)}$\end{tabular}      & 0.12 / \underline{0.76} / 0.18 \\
\begin{tabular}[c]{@{}l@{}}Exponential CDF \\ $\texttt{CDF}(x)=1 - e^{(-x)}$\end{tabular}          & 0.21 / \underline{0.90} / 0.37 \\
\begin{tabular}[c]{@{}l@{}}Logistic CDF \\ $\texttt{CDF}(x)=\frac{1}{1+e^{(-(x-1))}}$\end{tabular} & 0.32 / \underline{0.79} / 0.37 \\
\begin{tabular}[c]{@{}l@{}}Normal CDF \\ $\texttt{CDF}(x)=\frac{1}{2}\left[1+\operatorname{erf}\left(\frac{x - 1}{\sqrt{2}}\right)\right]$\end{tabular} &
  0.22 / \underline{0.90} / 0.29 \\ \bottomrule
\end{tabular}
\end{table}

    While all losses become bounded values in $[0,1]$, the maximum local loss (i.e., $F_2(\boldsymbol{\theta}) = 0.10$) is transformed into different values by each CDF (see underlined figures in the `Transformed Responses' column of Table~\ref{tab:cdf_example}).
    When using the Weibull CDF, the maximum local loss is translated into $1.00$, which means that there may be no value greater than 0.10 (i.e., $0.10$ is the largest one in $100\%$ probability) given current local losses.
    Meanwhile, when using the Frechet CDF, the maximum local loss is translated into $0.65$, which means that there still is a 35\% chance that some other local losses are greater than $0.10$ when provided with other losses similar to 0.01 and 0.02.
    This implies that each CDF \textit{treats a maximum value differently}.
    When a transformation is easily inclined to the maximum value, thereby returning 1 (i.e., maximum of CDF), it may yield a degenerate decision, e.g., $\boldsymbol{p}\approx[0, 1, 0]^\top$. 

    Fortunately, most of the listed CDFs are designed for \textit{modeling maximum values}.
    For example, the three distributions, Gumbel, Frechet, and Weibull, are grouped as the Extreme Value Distribution (EVD) \cite{evd}.
    As its name suggests, it models the behavior of extreme events, and it is well known that any density modeling a minimum or a maximum of independent and identically distributed (IID) samples follows the shape of one of these three distributions (by the Extreme Value Theorem \cite{evt}).
    In other words, EVDs can reasonably measure \textit{how a certain value is close to a maximum}.
    
    Thus, they can estimate whether a certain value is relatively large or small. 
    Otherwise, the Exponential distribution is a special case of Weibull distribution, and the logistic distribution is also related to the Gumbel distribution.
    Last but not least, although it is not a family of EVD, the Normal distribution is also considered due to the central limit theorem, since the local loss is the sum of errors from IID local samples.
    We expect the CDF transformation can appropriately measure a relative magnitude of local losses,
    and it should be helpful for decision making.

\paragraph{Effects of Response Transformation}
\label{subsec:exp_resp_transformation}
    We illustrated that the response should be bounded (i.e., Lipschitz continuous) in section \ref{sec:resp_trans}, to have non-vacuous regret upper bound.
    To acquire bounded response, we compare the cumulative values of a global objective in (\ref{eq:fl_obj}), 
    i.e., $\sum_{t=1}^T \sum_{i=1}^K p^{(t)}_i F_i (\boldsymbol{\theta}^{(t)})$ for the cross-silo setting, 
    and $\sum_{t=1}^T \sum_{i\in S^{(t)}} p^{(t)}_i F_i (\boldsymbol{\theta}^{(t)})$ for the cross-device setting.

    \begin{figure}[t]
        \centering
        \includegraphics[scale=0.6]{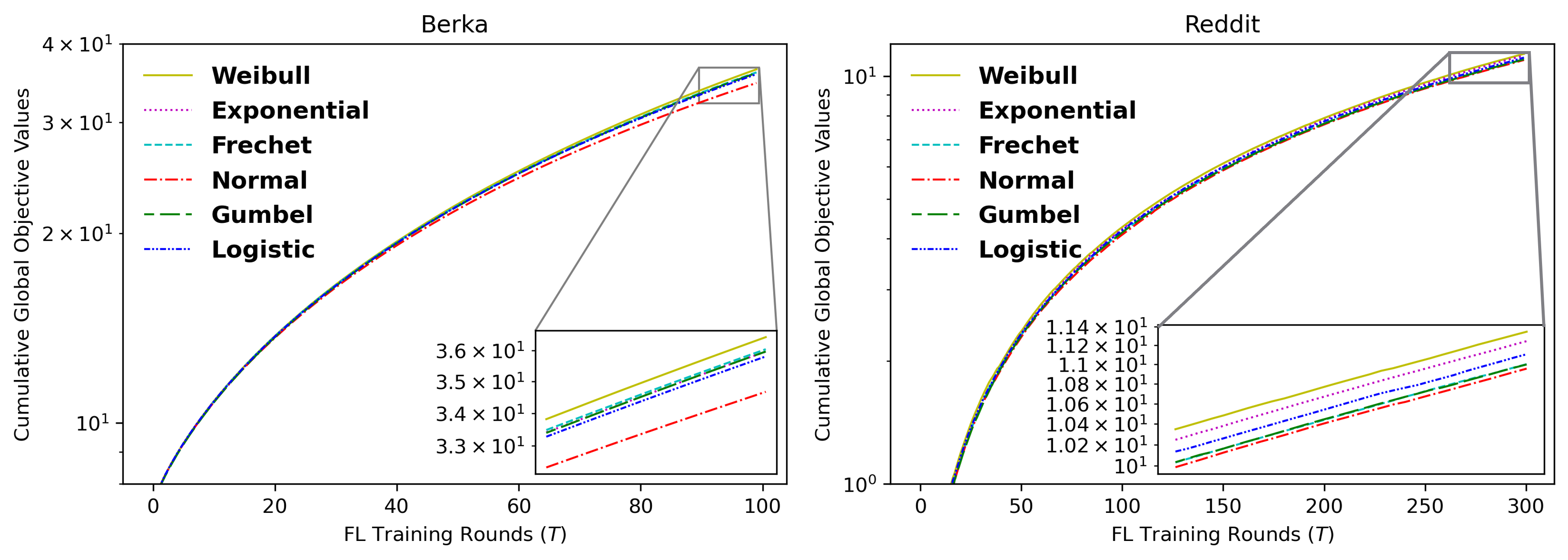}
        \caption[Effects of different CDF transformations on the performance of \texttt{AAggFF}]
        {Cumulative values of a global objective according to different CDFs (smaller is better): 
        (Left) Berka dataset (cross-silo setting; $K=7, T=100$). 
        (Right) Reddit dataset (cross-device setting; $K=817, T=300, C=0.00612$)}
        \label{fig:enter-label}
    \end{figure}
    
    In the cross-silo experiment with the Berka dataset, the Normal CDF shows the smallest cumulative values, 
    while in the cross-device experiment with the Reddit dataset, the Weibull CDF yields the smallest value.
    For the Berka dataset, the Normal CDF yields an average performance (AUROC) of 79.37 with the worst performance of 43.75, 
    but the Weibull CDF shows an average performance of 73.02 with the worst performance of 25.00.
    The same tendency is also observed in the Reddit dataset. 
    The Normal CDF presents an average performance (Acc. 1) of 14.05 and the worst performance of 4.26, 
    while the Weibull CDF shows an average of 12.62 with the worst of 3.35.
    From these observations, we can conclude that an appropriate choice of CDF is necessary for better sequential decision making, and suitable transformation helps minimize a global objective of FL. 
    Note also that these behaviors are also directly related to the global convergence of the algorithm w.r.t. $\boldsymbol{\theta}$.

\newpage
\subsection{Choice of a Response Range}
\label{app:range}
    In regard to determining the range of a response vector, i.e., $[C_1, C_2]$, 
    we can refer to the Lipschitz continuity in Lemma~\ref{lemma:lipschitz} and Lemma~\ref{lemma:lipschitz_lin_grad}.
    For the cross-silo setting, we can set arbitrary constants so that $L_\infty=\mathcal{O}(1)$ according to Lemma~\ref{lemma:lipschitz}.
    Thus, for all experiments of \texttt{AAggFF-S}, we set $C_1=0, C_2=\frac{1}{K}$.
    
    For the cross-device setting, the Lipschitz constant is changed into $\breve{L}_\infty$, since it is influenced by the client sampling probability ($C$).
    In detail, $C$ is located in the denominator of the Lipschitz constant, $\breve{L}_\infty$, which inflates the resulting gradient value as $C$ is a constant close to 0. 
    (e.g., when 10 among 100,000 clients are participating in each round, ${1}/{C}=10^{4}$)
    This is problematic and even causes an overflow problem empirically in updating a new decision.
    Thus, we propose a simple remedy --- setting $C_1$ and $C_2$ to be \textit{a multiple of $C$}, 
    so that the $C$ in the denominator is to be canceled out, according to Lemma~\ref{lemma:lipschitz_lin_grad}.
    For instance, when $C_1=0, C_2=C$, resulting Lipschitz constant simply becomes $\Vert \breve{\boldsymbol{g}}^{(t)} \Vert_\infty \leq \breve{L}_\infty = \frac{C}{1+0} + \frac{2(C-0)}{C(1+0)}=C+2\approx2$,
    which is a constant far smaller than $T$ and $K$, typically assumed in the practical cross-device FL setting.
    Therefore, for all experiments of \texttt{AAggFF-D}, we set $C_1=0, C_2=C$.

\newpage
\subsection{Regret Analysis}
\label{sec:analyses}
    In this section, we provide theoretical guarantees of our methods, \texttt{AAggFF-S} and \texttt{AAggFF-D} in terms of sequential decision making.
    The common objective for OCO algorithms is to \textit{minimize the regret} across a sequence of decision losses in eq.~\eqref{eq:regret}.
    We provide sublinear regret upper bounds in terms of $T$ as follows.

    \begin{theorem}
    \label{thm:crosssilo} (Regret Upper Bound for \texttt{AAggFF-S} (i.e., ONS \cite{ons1, ons2}))
        With the notation in eq.~\eqref{eq:ons},
        suppose for every $\boldsymbol{p}\in\Delta_{K-1}$, 
        and for every $t\in[T]$, 
        let the decisions $\{ \boldsymbol{p}^{(t)} : t\in[T] \}$ be derived by \texttt{AAggFF-S} for $K$ clients during $T$ rounds in Algorithm~\ref{alg:aaggff-s}.
        Then, the regret defined in eq.~\eqref{eq:regret} is bounded above as follows, where $\alpha$ and $\beta$ are determined as $\alpha=4KL_\infty, \beta=\frac{1}{4L_\infty}$.
        \begin{align*}
        \begin{gathered}
            \normalfont\text{Regret}^{(T)}\left(\boldsymbol{p}^\star\right) \leq 2 L_\infty K \left( 1 + \log \left( 1 + \frac{T}{16 K} \right) \right).
        \end{gathered}
        \end{align*}
    \end{theorem}
    From the Theorem~\ref{thm:crosssilo}, we can enjoy $\mathcal{O}\left(L_\infty K \log{T}\right)$ regret upper bound,
    which is an acceptable result, considering a typical assumption in the cross-silo setting (i.e., $K<T$).

    For the cross-device setting, we first present the full synchronization setting,
    which requires no extra adjustment.
    \begin{theorem}
    \label{thm:crossdevice_full} (Regret Upper Bound for \texttt{AAggFF-D} with Full-Client Participation)
        With the notation in (\ref{eq:lin_ftrl}),
        suppose for every $\boldsymbol{p}\in\Delta_{K-1}$, 
        and for every $t\in[T]$, 
        let the decisions $\{ \boldsymbol{p}^{(t)} : t\in[T] \}$ be derived by \texttt{AAggFF-D} for $K$ clients with client sampling probability $C=1$ during $T$ rounds in Algorithm~\ref{alg:aaggff-d}.
        Then, the regret defined in (\ref{eq:regret}) is bounded above as follows.
        \begin{align*}
        \begin{gathered}
            \normalfont\text{Regret}^{(T)}\left(\boldsymbol{p}^\star\right) \leq 2L_\infty\sqrt{T\log{K}}.
        \end{gathered}
        \end{align*}
    \end{theorem}

    When equipped with a client sampling, the randomness from the sampling should be considered.
    Due to local losses of selected clients can only be observed, 
    \texttt{AAggFF-D} should be equipped with the unbiased estimator of a response vector 
    (from Lemma~\ref{lemma:unbiased_resp}) 
    and a corresponding linearly approximated gradient vector 
    (from Lemma~\ref{lemma:linearized_grad}).
    Since they are unbiased estimators, the expected regret is also the same.
    
    \begin{corollary}
    \label{cor:crossdevice_partial} (Regret Upper Bound for \texttt{AAggFF-D} with Partial-Client Participation)
    With the client sampling probability $C\in(0,1)$,
    the DR estimator of a partially observed response $\breve{\boldsymbol{r}}^{(t)}$
    and corresponding linearized gradient $\breve{\boldsymbol{g}}^{(t)}$ for all $t\in[T]$,
    the regret defined in (\ref{eq:regret}) is bounded above in expectation as follows.
    \begin{align*}
    \begin{gathered}
        \mathbb{E}\left[\normalfont\text{Regret}^{(T)}(\boldsymbol{p}^{\star})\right] 
        \leq
        \mathcal{O}\left(L_\infty\sqrt{T\log{K}}\right).
    \end{gathered}
    \end{align*}
    \end{corollary}

\newpage
\section{Experimental Results}
\label{sec:experiments}
    We design experiments to evaluate empirical performances of our proposed framework \texttt{AAggFF}, 
    composed of sub-methods \texttt{AAggFF-S} and \texttt{AAggFF-D}, each of which is practical FL setting.
    
\subsection{Experimental Setup}
\label{subsec:exp_setup}
    \begin{table}[h]
\centering
\caption[Description of datasets for cross-silo and cross-device FL settings in Chapter~\ref{ch:aaggff}]
{Detailed description of federated benchmarks for cross-silo and cross-device settings}
\label{tab:dataset_desc}
\resizebox{0.6\textwidth}{!}{%
\begin{tabular}{lcc|lcc}
\hline
\multicolumn{3}{l|}{\textbf{Cross-Silo}} & \multicolumn{3}{l}{\textbf{Cross-Device}} \\
\textbf{Dataset}     & $K$     & $T$     & \textbf{Dataset}    & $K$      & $T$      \\ \hline
Berka                & 7       & 100     & CelebA              & 9,343    & 3,000    \\
MQP                  & 11      & 100     & Reddit              & 817      & 300      \\
ISIC             & 6       & 50      & SpeechCommands      & 2,005    & 500      \\ \hline
\end{tabular}%
}
\end{table}
    We conduct experiments on datasets mirroring \textit{realistic scenarios} in federated systems: 
    multiple modalities (vision, text, speech, and tabular form) and natural data partitioning. 
    We briefly summarize FL settings of each dataset in Table~\ref{tab:dataset_desc}.
    For the \textit{cross-silo} setting, 
    we used Berka tabular dataset \cite{berka},
    MQP clinical text dataset \cite{mqp}, 
    and ISIC oncological image dataset \cite{isic} (also a part of FLamby benchmark \cite{flamby}).
    For the \textit{cross-device} setting,
    we used CelebA vision dataset \cite{celeba}, 
    Reddit text dataset (both are parts of LEAF benchmark \cite{leaf})
    and SpeechCommands audio dataset \cite{speechcommands}. 
    
    Instead of manually partitioning data to simulate statistical heterogeneity, 
    we adopt natural client partitions inherent in each dataset.
    Each client dataset is split into an 80\% training set and a 20\% test set in a stratified manner where applicable.
    All experiments are run with 3 different random seeds after tuning hyperparameters.
    See Table~\ref{tab:dataset_desc} for descriptions of federated settings.

\paragraph{Datasets}
    \begin{table}[ht]
\centering
\caption[Statistics of federated benchmarks in Chapter~\ref{ch:aaggff}]{Statistics of federated benchmarks}
\label{tab:dataset_stat}
\begin{tabular}{@{}lcccc@{}}
\toprule
\textbf{Dataset} & \textbf{Clients} & \textbf{Samples} & \textbf{Avg.} & \textbf{Std.} \\ \midrule
Berka          & 7     & 621     & 88.71   & 24.78   \\
MQP            & 11    & 3,048   & 277.09  & 63.25   \\
ISIC       & 6     & 21,310  & 3551.67 & 3976.16 \\
CelebA         & 9,343 & 200,288 & 21.44   & 7.63    \\
Reddit         & 817   & 97,961  & 119.90  & 229.85  \\
SpeechCommands & 2,005 & 84,700  & 42.24   & 36.69   \\ \bottomrule
\end{tabular}
\end{table}
    We used 3 datasets for the cross-silo setting (Berka \cite{berka}, MQP \cite{mqp}, and ISIC \cite{isic, flamby}), 
    and 3 other datasets for the cross-device setting (CelebA, Reddit \cite{leaf}, and SpeechCommands \cite{speechcommands}).
    
    The statistics of all federated benchmarks are summarized in Table~\ref{tab:dataset_stat}.
    Note that \textbf{Avg.} and \textbf{Std.} in the table refer to the average and a standard deviation of a sample size of each client in the FL system
    
    First, we present details of the federated benchmark for the cross-silo setting.
    \begin{enumerate}
        \item[$\bullet$] \textbf{Berka} is a tabular dataset containing bank transaction records collected from Czech bank \cite{berka}.
        Berka accompanies the loan default prediction task (i.e., binary classification) of each bank's customers.
        It is fully anonymized and is originally composed of 8 relational tables: accounts, clients, disposition, loans, permanent orders, transactions, demographics, and credit cards.
        We merged all 8 tables into one dataset by joining the primary keys of each table, and finally have 15 input features.
        From the demographics table, we obtain information on the region: Prague, Central Bohemia, South Bohemia, West Bohemia, North Bohemia, East Bohemia, South Moravia, and North Moravia.
        We split each client according to the region and excluded all samples of North Bohemia since it has only one record of loan default, thus we finally have 7 clients (i.e., banks).
        Finally, we used the area under the receiver operating characteristic (ROC) curve for the evaluation metric.
        
        \item[$\bullet$] \textbf{MQP} is a clinical question pair dataset crawled from medical question answering dataset \cite{healthtap}, 
        and labeled by 11 doctors \cite{mqp}.
        All paired sentences are labeled as either similar or dissimilar, thereby it is suitable for the binary classification task.
        As a pre-processing, we merge two paired sentences into one sentence by adding special tokens: \texttt{[SEP]}, \texttt{[PAD]}, and \texttt{[UNK]}.
        We set the maximum token length to 200, thus merged sentences less than 200 are filled with \texttt{[PAD]} tokens, and otherwise are truncated. Then, merged sentences are tokenized using pre-trained DistilBERT tokenizer \cite{distilbert}.
        We regard each doctor as a separate client and thus have 11 clients.
        Finally, we used the area under the ROC curve for the evaluation metric.
        
        \item[$\bullet$] \textbf{ISIC} is a dermoscopic image dataset for a skin cancer classification, collected from 4 hospitals. \cite{isic, flamby}
        The task contains 8 distinct melanoma classes, thus designed for the multi-class classification task.
        Following \cite{flamby}, as one hospital has three different imaging apparatus, its samples are further divided into 3 clients, thus we finally have 6 clients in total.   
        Finally, we used top-5 accuracy for the evaluation metric.
    \end{enumerate}
    
    Next, we illustrate details of the federated benchmark for the cross-device setting.
    \begin{enumerate}
        \item[$\bullet$] \textbf{CelebA} is a vision dataset containing the facial images of celebrities \cite{celeba}.
        It is curated for federated setting in LEAF benchmark \cite{leaf}, and is targeted for the binary classification task (presence of smile).
        We follow the processing of \cite{leaf}, thereby each client corresponds to each celebrity, having 9,343 total clients in the FL system.
        Finally, we used top-1 accuracy for the evaluation metric following \cite{leaf}.
        
        \item[$\bullet$] \textbf{Reddit} is a text dataset containing the comments of community users of Reddit in December 2017, and a part of LEAF benchmark \cite{leaf}. 
        Following \cite{leaf}, we build a dictionary of vocabularies of size 10,000 from tokenized sentences and set the maximum sequence length to 10.
        The main task is tailored for language modeling, i.e., next token prediction, given word embeddings in each sentence of clients.
        Each client corresponds to one of the community users, thus 817 clients are presented in total.
        Finally, we used the top-1 accuracy for the evaluation metric following \cite{leaf}.
        
        \item[$\bullet$] \textbf{SpeechCommands} is designed for a short-length speech recognition task that includes one second 35 short-length words,
        such as `{Up}', `{Down}', `{Left}', and `{Right}'~\cite{speechcommands}. 
        It is accordingly a multi-class classification task for 35 different classes.
        While it is collected from 2,618 speakers, we rule out all samples of speakers having too few samples when splitting the dataset into training and test sets.
        (e.g., exclude speakers whose total sample counts are less than 3)
        As a result, we have 2,005 clients, and each client has 16,000-length time-domain waveform samples.
        Finally, we used top-5 accuracy for the evaluation metric.
    \end{enumerate}
    
\paragraph{Model Architectures}
    For each dataset, we used task-specific model architectures which are already used in previous works, or widely used in reality, to simulate the practical FL scenario as much as possible.
    
    For the experiment of the cross-silo setting, we used the following models.
    \begin{enumerate}
        \item[$\bullet$] \textit{Logistic Regression} is used for the Berka dataset. 
        We used a simple logistic regression model with a bias term, and the output (i.e., logit vector) is transformed into predicted class probabilities by the softmax function.
    
        \item[$\bullet$] \textit{DistilBERT \cite{distilbert}} is used for the MQP dataset.
        We used a pre-trained DistilBERT model, from BookCorpus and English Wikipedia \cite{distilbert}. 
        We also used the corresponding DistilBERT tokenizer for the pre-processing of raw clinical sentences.
        For a fine-tuning of the pre-trained DistilBERT model, we attach a classifier head next to the last layer of the DistilBERT's encoder, which outputs an embedding of $768$ dimension.
        The classifier is in detail processing the embedding as follows: ($768$-ReLU-Dropout-$2$),
        where each figure is an output dimension of a fully connected layer with a bias term, 
        ReLU is a rectified linear unit activation layer, 
        and Dropout \cite{dropout} is a dropout layer having probability of $0.1$.
        In the experiment, we trained all layers including pre-trained weights.
    
        \item[$\bullet$] \textit{EfficientNet \cite{efficientnet}} is used for the ISIC dataset.
        We also used the pre-trained EfficientNet-B0 model from ImageNet benchmark dataset \cite{imagenet}.
        For fine-tuning, we attach a classifier head after the convolution layers of EfficientNet.
        The classifier is composed of the following components: (AdaptiveAvgPool($7, 7$)-Dropout-$8$),
        where AdaptiveAvgPool($cdot, cdot$) is a 2D adaptive average pooling layer outputs a feature map of size $7 \times 7$ (which are flattened thereafter), 
        Dropout is a dropout layer with a probability of 0.1, 
        and the last linear layer outputs an 8-dimensional vector,
        which is the total number of classes.
    \end{enumerate}
    
    Next, for the cross-device setting, we used the following models.
    \begin{enumerate}
        \item[$\bullet$] \textit{ConvNet} model used in LEAF benchmark \cite{leaf} is used for the CelebA dataset.
        It is composed of four convolution layers, of which components are:
        2D convolution layer without bias term with 32 filters of size $3\times 3$ (stride=1, padding=1),
        group normalization layer (the number of groups is decreased from 32 by a factor of 2: 32, 16, 8, 4),
        2D max pooling layer with $2 \times 2$ filters,
        and a ReLU nonlinear activation layer.
        Plus, a classifier comes after the consecutive convolution layers, which are composed of:
        (AdaptiveAvgPool($5, 5$)-$1$),
        which are a 2D adaptive average pooling layer that outputs a feature map of size $5 \times 5$ (which are flattened thereafter), 
        and a linear layer with a bias term outputs a scalar value since it is a binary classification task.
    
        \item[$\bullet$] \textit{StackedLSTM} model used in LEAF benchmark \cite{leaf} is used for the Reddit dataset.
        It is composed of an embedding layer of which the number of embeddings is 200, and outputs an embedding vector of 256 dimensions.
        It is processed by consecutive 2 LSTM \cite{lstm} layers with the hidden size of 256, and enters the last linear layer with a bias term, which outputs a logit vector of 10,000 dimensions, which corresponds to the vocabulary size.
    
        \item[$\bullet$] \textit{M5 \cite{m5}} model is used for the SpeechCommands dataset.
        It is composed of four 1D convolution layers followed by a 1D batch normalization layer, ReLU nonlinear activation, and a 1D max pooling layer with a filter of size 4.
        All convolution layers EXCEPT the input layer have a filter of size $3$, and the numbers of filters are 64, 128, and 256 (all with stride=1 and padding=1).
        The input convolution layer has 64 filters with a filter of size $80$, and stride of $4$.
        Lastly, one linear layer outputs a logit vector of 35 dimensions. 
    \end{enumerate}
    
\paragraph{Hyperparameters}
    Before the main experiment, we first tuned hyperparameters of all baseline fair FL algorithms from a separate random seed.
    The hyperparameter of each fair algorithm is listed as follows.
    \begin{enumerate}
        \item[$\bullet$] \texttt{AFL} \cite{afl} --- a learning rate of a mixing coefficient $\in$\{0.01, 0.1, 1.0\}
        \item[$\bullet$] \texttt{q-FedAvg} \cite{qffl} --- a magnitude of fairness, $\in$\{0.1, 1.0, 5.0\}
        \item[$\bullet$] \texttt{TERM} \cite{term} --- a tilting constant, $\lambda$, $\in$\{0.1, 1.0, 10.0\}
        \item[$\bullet$] \texttt{FedMGDA} \cite{fedmgda} --- a deviation from static mixing coefficient $\in$\{0.1, 0.5, 1.0\}
        \item[$\bullet$] \texttt{PropFair} \cite{propfair} --- a baseline constant $\in$\{2, 3, 5\}
    \end{enumerate}
    
    Each candidate value is selected according to the original paper, and we fix the number of local epochs, $E=1$  (following the set up in \cite{qffl}), along with the number of local batch size $B=20$ in all experiments.
    For each dataset, a weight decay (L2 penalty) factor ($\psi$; following a theoretical result~\cite{convergence}), a local learning rate ($\eta$), and variables related to a learning rate scheduling (i.e., learning rate decay factor ($\phi$), and a decay step ($s$)) are tuned first with \texttt{FedAvg} \cite{fedavg} as follows.
    \begin{enumerate}
        \item[$\bullet$] \textbf{Berka}: $\psi=10^{-3}, \eta=10^0, \phi=0.99, s=10$
        \item[$\bullet$] \textbf{MQP}: $\psi=10^{-2}, \eta=10^{-\frac{5}{2}}, \phi=0.99, s=15$
        \item[$\bullet$] \textbf{ISIC}: $\psi=10^{-2}, \eta=10^{-4}, \phi=0.95, s=5$
        \item[$\bullet$] \textbf{CelebA}: $\psi=10^{-4}, \eta=10^{-1}, \phi=0.96, s=300$
        \item[$\bullet$] \textbf{Reddit}: $\psi=10^{-6}, \eta=10^{\frac{7}{8}}, \phi=0.95, s=20$
        \item[$\bullet$] \textbf{SpeechCommands}: $\psi=0, \eta=10^{-1}, \phi=0.999, s=10$
    \end{enumerate}
    
    This is intended under the assumption that all fair FL algorithms should at least be effective in the same setting of the FL algorithm with the static aggregation scheme (i.e., \texttt{FedAvg}).
    Note that client-side optimization in all experiments is done by the Stochastic Gradient Descent (SGD) optimizer.
        
\newpage
\subsection{Improvement in the Client-Level Fairness}
\label{subsec:exp_improved_fairness}
    We compare our methods with existing fair FL methods including \texttt{FedAvg} \cite{fedavg}, \texttt{AFL} \cite{afl}, \texttt{q-FedAvg} \cite{qffl}, \texttt{TERM} \cite{term}, \texttt{FedMGDA} \cite{fedmgda}, and \texttt{PropFair} \cite{propfair}.
    Since \texttt{AFL} requires full synchronization of clients every round, it is only compared in the cross-silo setting.
    
    In the \textit{cross-silo} setting, we assume all $K$ clients participate in $T$ federation rounds (i.e., $C=1$), 
    and in the \textit{cross-device} setting, the client participation rate $C\in(0, 1)$ is set to ensure $5$ among $K$ clients are participating in each round.
    We evaluate each dataset using appropriate metrics for tasks as indicated under the dataset name in Table~\ref{tab:result_silo} and~\ref{tab:result_device}, where the average, the best (10\%), the worst (10\%), and Gini coefficient\footnote{The Gini coefficient is inflated by ($\times 10^2$) for readability.} of clients' performance distributions are reported with the standard deviation inside parentheses in gray color below the averaged metric. 
    
    From the results, we verify that \texttt{AAggFF} can lead to enhanced worst-case metric and Gini coefficient in both settings while retaining competitive average performance to other baselines.
    Remarkably, along with the improved worst-case performance, the small Gini coefficient indicates that performances of clients are close to each other,
    which is directly translated into the improved client-level fairness.
    \begin{sidewaystable}
\centering
\caption[Experimental results on the cross-silo setting of Chapter~\ref{ch:aaggff}]{Comparison results of \texttt{AAggFF-S} in the cross-silo setting}
\label{tab:result_silo}
\resizebox{\textwidth}{!}{%
\renewcommand{\arraystretch}{0.9}

\begin{tabular}{lcccccccccccc}
\toprule
\textbf{Dataset} &
  \multicolumn{4}{c}{\textbf{Berka}} &
  \multicolumn{4}{c}{\textbf{MQP}} &
  \multicolumn{4}{c}{\textbf{ISIC}} \\
 &
  \multicolumn{4}{c}{\footnotesize(AUROC)} &
  \multicolumn{4}{c}{\footnotesize(AUROC)} &
  \multicolumn{4}{c}{\footnotesize(Acc. 5)} \\ \cmidrule(l){2-13}
\multirow{-2}{*}{\textbf{Method}} &
  \begin{tabular}[c]{@{}c@{}}Avg.\\ ($\uparrow$)\end{tabular} &
  \begin{tabular}[c]{@{}c@{}}Worst\\ ($\uparrow$)\end{tabular} &
  \begin{tabular}[c]{@{}c@{}}Best\\ ($\uparrow$)\end{tabular} &
  \begin{tabular}[c]{@{}c@{}}Gini\\ ($\downarrow$)\end{tabular} &
  \begin{tabular}[c]{@{}c@{}}Avg.\\ ($\uparrow$)\end{tabular} &
  \begin{tabular}[c]{@{}c@{}}Worst\\ ($\uparrow$)\end{tabular} &
  \begin{tabular}[c]{@{}c@{}}Best\\ ($\uparrow$)\end{tabular} &
  \begin{tabular}[c]{@{}c@{}}Gini\\ ($\downarrow$)\end{tabular} &
  \begin{tabular}[c]{@{}c@{}}Avg.\\ ($\uparrow$)\end{tabular} &
  \begin{tabular}[c]{@{}c@{}}Worst\\ ($\uparrow$)\end{tabular} &
  \begin{tabular}[c]{@{}c@{}}Best\\ ($\uparrow$)\end{tabular} &
  \begin{tabular}[c]{@{}c@{}}Gini\\ ($\downarrow$)\end{tabular} \\ \midrule
\begin{tabular}[c]{@{}l@{}}\texttt{FedAvg} \footnotesize\cite{fedavg}\end{tabular} &
  \begin{tabular}[c]{@{}c@{}}80.09\\ \footnotesize \color[HTML]{9B9B9B} (2.45)\end{tabular} &
  \begin{tabular}[c]{@{}c@{}}48.06\\ \footnotesize \color[HTML]{9B9B9B} (25.15)\end{tabular} &
  \begin{tabular}[c]{@{}c@{}}\textbf{99.03}\\ \footnotesize \color[HTML]{9B9B9B} (1.37)\end{tabular} &
  \multicolumn{1}{c|}{\begin{tabular}[c]{@{}c@{}}10.87\\ \footnotesize \color[HTML]{9B9B9B} (4.11)\end{tabular}} &
  \begin{tabular}[c]{@{}c@{}}56.06\\ \footnotesize \color[HTML]{9B9B9B} (0.06)\end{tabular} &
  \begin{tabular}[c]{@{}c@{}}41.03\\ \footnotesize \color[HTML]{9B9B9B} (4.33)\end{tabular} &
  \begin{tabular}[c]{@{}c@{}}76.31\\ \footnotesize \color[HTML]{9B9B9B} (8.42)\end{tabular} &
  \multicolumn{1}{c|}{\begin{tabular}[c]{@{}c@{}}8.63\\ \footnotesize \color[HTML]{9B9B9B} (0.91)\end{tabular}} &
  \begin{tabular}[c]{@{}c@{}}87.42\\ \footnotesize \color[HTML]{9B9B9B} (2.11)\end{tabular} &
  \begin{tabular}[c]{@{}c@{}}69.92\\ \footnotesize \color[HTML]{9B9B9B} (6.78)\end{tabular} &
  \begin{tabular}[c]{@{}c@{}}92.57\\ \footnotesize \color[HTML]{9B9B9B} (2.56)\end{tabular} &
  \begin{tabular}[c]{@{}c@{}}4.84\\ \footnotesize \color[HTML]{9B9B9B} (1.17)\end{tabular} \\
\begin{tabular}[c]{@{}l@{}}\texttt{AFL} \footnotesize\cite{afl}\end{tabular} &
  \begin{tabular}[c]{@{}c@{}}79.70\\ \footnotesize \color[HTML]{9B9B9B} (4.14)\end{tabular} &
  \begin{tabular}[c]{@{}c@{}}49.02\\ \footnotesize \color[HTML]{9B9B9B} (25.89)\end{tabular} &
  \begin{tabular}[c]{@{}c@{}}\underline{98.55}\\ \footnotesize \color[HTML]{9B9B9B} (2.05)\end{tabular} &
  \multicolumn{1}{c|}{\begin{tabular}[c]{@{}c@{}}10.58\\ \footnotesize \color[HTML]{9B9B9B} (5.03)\end{tabular}} &
  \begin{tabular}[c]{@{}c@{}}56.01\\ \footnotesize \color[HTML]{9B9B9B} (0.30)\end{tabular} &
  \begin{tabular}[c]{@{}c@{}}41.28\\ \footnotesize \color[HTML]{9B9B9B} (3.92)\end{tabular} &
  \begin{tabular}[c]{@{}c@{}}75.54\\ \footnotesize \color[HTML]{9B9B9B} (6.77)\end{tabular} &
  \multicolumn{1}{c|}{\begin{tabular}[c]{@{}c@{}}\underline{8.56}\\ \footnotesize \color[HTML]{9B9B9B} (1.24)\end{tabular}} &
  \begin{tabular}[c]{@{}c@{}}87.39\\ \footnotesize \color[HTML]{9B9B9B} (2.31)\end{tabular} &
  \begin{tabular}[c]{@{}c@{}}68.17\\ \footnotesize \color[HTML]{9B9B9B} (10.09)\end{tabular} &
  \begin{tabular}[c]{@{}c@{}}93.33\\ \footnotesize \color[HTML]{9B9B9B} (2.18)\end{tabular} &
  \begin{tabular}[c]{@{}c@{}}4.80\\ \footnotesize \color[HTML]{9B9B9B} (1.74)\end{tabular} \\
\begin{tabular}[c]{@{}l@{}}\texttt{q-FedAvg} \footnotesize\cite{qffl}\end{tabular} &
  \begin{tabular}[c]{@{}c@{}}79.98\\ \footnotesize \color[HTML]{9B9B9B} (3.89)\end{tabular} &
  \begin{tabular}[c]{@{}c@{}}\underline{49.44}\\ \footnotesize \color[HTML]{9B9B9B} (26.15)\end{tabular} &
  \begin{tabular}[c]{@{}c@{}}98.07\\ \footnotesize \color[HTML]{9B9B9B} (2.73)\end{tabular} &
  \multicolumn{1}{c|}{\begin{tabular}[c]{@{}c@{}}10.62\\ \footnotesize \color[HTML]{9B9B9B} (5.22)\end{tabular}} &
  \begin{tabular}[c]{@{}c@{}}\textbf{56.89}\\ \footnotesize \color[HTML]{9B9B9B} (0.42)\end{tabular} &
  \begin{tabular}[c]{@{}c@{}}40.22\\ \footnotesize \color[HTML]{9B9B9B} (3.06)\end{tabular} &
  \begin{tabular}[c]{@{}c@{}}\textbf{79.38}\\ \footnotesize \color[HTML]{9B9B9B} (9.09)\end{tabular} &
  \multicolumn{1}{c|}{\begin{tabular}[c]{@{}c@{}}8.68\\ \footnotesize \color[HTML]{9B9B9B} (0.57)\end{tabular}} &
  \begin{tabular}[c]{@{}c@{}}41.59\\ \footnotesize \color[HTML]{9B9B9B} (16.22)\end{tabular} &
  \begin{tabular}[c]{@{}c@{}}20.38\\ \footnotesize \color[HTML]{9B9B9B} (23.24)\end{tabular} &
  \begin{tabular}[c]{@{}c@{}}58.08\\ \footnotesize \color[HTML]{9B9B9B} (28.52)\end{tabular} &
  \begin{tabular}[c]{@{}c@{}}22.25\\ \footnotesize \color[HTML]{9B9B9B} (10.02)\end{tabular} \\
\begin{tabular}[c]{@{}l@{}}\texttt{TERM} \footnotesize\cite{term}\end{tabular} &
  \begin{tabular}[c]{@{}c@{}}\underline{80.11}\\ \footnotesize \color[HTML]{9B9B9B} (3.08)\end{tabular} &
  \begin{tabular}[c]{@{}c@{}}48.96\\ \footnotesize \color[HTML]{9B9B9B} (25.79)\end{tabular} &
  \begin{tabular}[c]{@{}c@{}}\textbf{99.03}\\ \footnotesize \color[HTML]{9B9B9B} (1.37)\end{tabular} &
  \multicolumn{1}{c|}{\begin{tabular}[c]{@{}c@{}}10.86\\ \footnotesize \color[HTML]{9B9B9B} (4.73)\end{tabular}} &
  \begin{tabular}[c]{@{}c@{}}56.47\\ \footnotesize \color[HTML]{9B9B9B} (0.19)\end{tabular} &
  \begin{tabular}[c]{@{}c@{}}40.73\\ \footnotesize \color[HTML]{9B9B9B} (4.36)\end{tabular} &
  \begin{tabular}[c]{@{}c@{}}76.80\\ \footnotesize \color[HTML]{9B9B9B} (8.30)\end{tabular} &
  \multicolumn{1}{c|}{\begin{tabular}[c]{@{}c@{}}8.67\\ \footnotesize \color[HTML]{9B9B9B} (1.43)\end{tabular}} &
  \begin{tabular}[c]{@{}c@{}}\underline{87.89}\\ \footnotesize \color[HTML]{9B9B9B} (1.69)\end{tabular} &
  \begin{tabular}[c]{@{}c@{}}\underline{77.32}\\ \footnotesize \color[HTML]{9B9B9B} (5.84)\end{tabular} &
  \begin{tabular}[c]{@{}c@{}}\underline{96.00}\\ \footnotesize \color[HTML]{9B9B9B} (3.27)\end{tabular} &
  \begin{tabular}[c]{@{}c@{}}\underline{3.77}\\ \footnotesize \color[HTML]{9B9B9B} (0.94)\end{tabular} \\
\begin{tabular}[c]{@{}l@{}}\texttt{FedMGDA} \footnotesize\cite{fedmgda}\end{tabular} &
  \begin{tabular}[c]{@{}c@{}}79.24\\ \footnotesize \color[HTML]{9B9B9B} (2.96)\end{tabular} &
  \begin{tabular}[c]{@{}c@{}}46.38\\ \footnotesize \color[HTML]{9B9B9B} (24.11)\end{tabular} &
  \begin{tabular}[c]{@{}c@{}}\textbf{99.03}\\ \footnotesize \color[HTML]{9B9B9B} (1.37)\end{tabular} &
  \multicolumn{1}{c|}{\begin{tabular}[c]{@{}c@{}}11.64\\ \footnotesize \color[HTML]{9B9B9B} (4.84)\end{tabular}} &
  \begin{tabular}[c]{@{}c@{}}53.02\\ \footnotesize \color[HTML]{9B9B9B} (1.67)\end{tabular} &
  \begin{tabular}[c]{@{}c@{}}34.91\\ \footnotesize \color[HTML]{9B9B9B} (2.22)\end{tabular} &
  \begin{tabular}[c]{@{}c@{}}69.65\\ \footnotesize \color[HTML]{9B9B9B} (3.89)\end{tabular} &
  \multicolumn{1}{c|}{\begin{tabular}[c]{@{}c@{}}10.33\\ \footnotesize \color[HTML]{9B9B9B} (0.44)\end{tabular}} &
  \begin{tabular}[c]{@{}c@{}}42.36\\ \footnotesize \color[HTML]{9B9B9B} (14.94)\end{tabular} &
  \begin{tabular}[c]{@{}c@{}}21.44\\ \footnotesize \color[HTML]{9B9B9B} (21.30)\end{tabular} &
  \begin{tabular}[c]{@{}c@{}}59.21\\ \footnotesize \color[HTML]{9B9B9B} (28.52)\end{tabular} &
  \begin{tabular}[c]{@{}c@{}}22.25\\ \footnotesize \color[HTML]{9B9B9B} (10.02)\end{tabular} \\
\begin{tabular}[c]{@{}l@{}}\texttt{PropFair} \footnotesize\cite{propfair}\end{tabular} &
  \begin{tabular}[c]{@{}c@{}}79.61\\ \footnotesize \color[HTML]{9B9B9B} (4.49)\end{tabular} &
  \begin{tabular}[c]{@{}c@{}}\underline{49.44}\\ \footnotesize \color[HTML]{9B9B9B} (26.15)\end{tabular} &
  \begin{tabular}[c]{@{}c@{}}98.07\\ \footnotesize \color[HTML]{9B9B9B} (2.73)\end{tabular} &
  \multicolumn{1}{c|}{\begin{tabular}[c]{@{}c@{}}\underline{10.47}\\ \footnotesize \color[HTML]{9B9B9B} (5.04)\end{tabular}} &
  \begin{tabular}[c]{@{}c@{}}56.60\\ \footnotesize \color[HTML]{9B9B9B} (0.39)\end{tabular} &
  \begin{tabular}[c]{@{}c@{}}\underline{41.71}\\ \footnotesize \color[HTML]{9B9B9B} (3.80)\end{tabular} &
  \begin{tabular}[c]{@{}c@{}}\underline{79.09}\\ \footnotesize \color[HTML]{9B9B9B} (7.40)\end{tabular} &
  \multicolumn{1}{c|}{\begin{tabular}[c]{@{}c@{}}8.74\\ \footnotesize \color[HTML]{9B9B9B} (0.87)\end{tabular}} &
  \begin{tabular}[c]{@{}c@{}}83.88\\ \footnotesize \color[HTML]{9B9B9B} (2.50)\end{tabular} &
  \begin{tabular}[c]{@{}c@{}}58.36\\ \footnotesize \color[HTML]{9B9B9B} (11.63)\end{tabular} &
  \begin{tabular}[c]{@{}c@{}}91.35\\ \footnotesize \color[HTML]{9B9B9B} (2.48)\end{tabular} &
  \begin{tabular}[c]{@{}c@{}}7.91\\ \footnotesize \color[HTML]{9B9B9B} (2.10)\end{tabular} \\
\rowcolor[HTML]{FFF5E6} 
\begin{tabular}[c]{@{}l@{}}\texttt{AAggFF-S}\\ (Proposed)\end{tabular} &
  \begin{tabular}[c]{@{}c@{}}\textbf{80.93}\\ \footnotesize \color[HTML]{9B9B9B} (2.96)\end{tabular} &
  \begin{tabular}[c]{@{}c@{}}\textbf{52.08}\\ \footnotesize \color[HTML]{9B9B9B} (23.59)\end{tabular} &
  \cellcolor[HTML]{FFF5E6}\begin{tabular}[c]{@{}c@{}}\textbf{99.03}\\ \footnotesize \color[HTML]{9B9B9B} (1.37)\end{tabular} &
  \multicolumn{1}{c|}{\cellcolor[HTML]{FFF5E6}\begin{tabular}[c]{@{}c@{}}\textbf{10.16}\\ \footnotesize \color[HTML]{9B9B9B} (3.80)\end{tabular}} &
  \begin{tabular}[c]{@{}c@{}}\underline{56.63}\\ \footnotesize \color[HTML]{9B9B9B} (0.54)\end{tabular} &
  \begin{tabular}[c]{@{}c@{}}\textbf{41.79}\\ \footnotesize \color[HTML]{9B9B9B} (4.43)\end{tabular} &
  \begin{tabular}[c]{@{}c@{}}75.56\\ \footnotesize \color[HTML]{9B9B9B} (6.53)\end{tabular} &
  \multicolumn{1}{c|}{\cellcolor[HTML]{FFF5E6}\begin{tabular}[c]{@{}c@{}}\textbf{8.38}\\ \footnotesize \color[HTML]{9B9B9B} (0.77)\end{tabular}} &
  \begin{tabular}[c]{@{}c@{}}\textbf{89.76}\\ \footnotesize \color[HTML]{9B9B9B} (1.03)\end{tabular} &
  \begin{tabular}[c]{@{}c@{}}\textbf{85.17}\\ \footnotesize \color[HTML]{9B9B9B} (3.87)\end{tabular} &
  \begin{tabular}[c]{@{}c@{}}\textbf{98.22}\\ \footnotesize \color[HTML]{9B9B9B} (1.66)\end{tabular} &
  \begin{tabular}[c]{@{}c@{}}\textbf{2.52}\\ \footnotesize \color[HTML]{9B9B9B} (0.38)\end{tabular} \\ \bottomrule
\end{tabular}%
}
\end{sidewaystable}  
    \begin{sidewaystable}
\centering
\caption[Experimental results on the cross-device setting of Chapter~\ref{ch:aaggff}]{Comparison results of \texttt{AAggFF-D} in the cross-device setting}
\label{tab:result_device}
\resizebox{\textwidth}{!}{%
\renewcommand{\arraystretch}{0.9}
\begin{tabular}{!{}lcccccccccccc!{}}
\toprule
\textbf{Dataset} &
  \multicolumn{4}{c}{\textbf{CelebA}} &
  \multicolumn{4}{c}{\textbf{Reddit}} &
  \multicolumn{4}{c}{\textbf{SpeechCommands}} \\
 &
  \multicolumn{4}{c}{(Acc. 1)} &
  \multicolumn{4}{c}{(Acc. 1)} &
  \multicolumn{4}{c}{(Acc. 5)} \\ \cmidrule(l){2-13} 
\multirow{-2}{*}{\textbf{Method}} &
  \begin{tabular}[c]{@{}c@{}}Avg.\\ ($\uparrow$)\end{tabular} &
  \begin{tabular}[c]{@{}c@{}}Worst\\ 10\% ($\uparrow$)\end{tabular} &
  \begin{tabular}[c]{@{}c@{}}Best\\ 10\%($\uparrow$)\end{tabular} &
  \begin{tabular}[c]{@{}c@{}}Gini\\ ($\downarrow$)\end{tabular} &
  \begin{tabular}[c]{@{}c@{}}Avg.\\ ($\uparrow$)\end{tabular} &
  \begin{tabular}[c]{@{}c@{}}Worst\\ 10\%($\uparrow$)\end{tabular} &
  \begin{tabular}[c]{@{}c@{}}Best\\ 10\%($\uparrow$)\end{tabular} &
  \begin{tabular}[c]{@{}c@{}}Gini\\ ($\downarrow$)\end{tabular} &
  \begin{tabular}[c]{@{}c@{}}Avg.\\ ($\uparrow$)\end{tabular} &
  \begin{tabular}[c]{@{}c@{}}Worst\\ 10\%($\uparrow$)\end{tabular} &
  \begin{tabular}[c]{@{}c@{}}Best\\ 10\%($\uparrow$)\end{tabular} &
  \begin{tabular}[c]{@{}c@{}}Gini\\ ($\downarrow$)\end{tabular} \\ \midrule
\begin{tabular}[c]{@{}l@{}}\texttt{FedAvg} \footnotesize\cite{fedavg}\end{tabular} &
  \begin{tabular}[c]{@{}c@{}}90.79\\ \footnotesize \color[HTML]{9B9B9B} (0.53)\end{tabular} &
  \begin{tabular}[c]{@{}c@{}}\underline{55.76} \\ \footnotesize \color[HTML]{9B9B9B} (0.84)\end{tabular} &
  \begin{tabular}[c]{@{}c@{}}\underline{100.00}\\ \footnotesize \color[HTML]{9B9B9B} (0.00)\end{tabular} &
  \multicolumn{1}{c|}{\begin{tabular}[c]{@{}c@{}}7.86\\ \footnotesize \color[HTML]{9B9B9B} (0.30)\end{tabular}} &
  \begin{tabular}[c]{@{}c@{}}10.76\\ \footnotesize \color[HTML]{9B9B9B} (1.45)\end{tabular} &
  \begin{tabular}[c]{@{}c@{}}2.50\\ \footnotesize \color[HTML]{9B9B9B} (0.21)\end{tabular} &
  \begin{tabular}[c]{@{}c@{}}20.86\\ \footnotesize \color[HTML]{9B9B9B} (3.64)\end{tabular} &
  \multicolumn{1}{c|}{\begin{tabular}[c]{@{}c@{}}25.66\\ \footnotesize \color[HTML]{9B9B9B} (0.49)\end{tabular}} &
  \begin{tabular}[c]{@{}c@{}}\underline{75.51}\\ \footnotesize \color[HTML]{9B9B9B} (1.08)\end{tabular} &
  \begin{tabular}[c]{@{}c@{}}7.93\\ \footnotesize \color[HTML]{9B9B9B} (2.87)\end{tabular} &
  \begin{tabular}[c]{@{}c@{}}\underline{100.00}\\ \footnotesize \color[HTML]{9B9B9B} (0.00)\end{tabular} &
  \begin{tabular}[c]{@{}c@{}}24.58\\ \footnotesize \color[HTML]{9B9B9B} (1.34)\end{tabular} \\
\begin{tabular}[c]{@{}l@{}}\texttt{q-FedAvg} \footnotesize\cite{qffl}\end{tabular} &
  \begin{tabular}[c]{@{}c@{}}\underline{90.88}\\ \footnotesize \color[HTML]{9B9B9B} (0.19)\end{tabular} &
  \begin{tabular}[c]{@{}c@{}}55.73\\ \footnotesize \color[HTML]{9B9B9B} (0.85)\end{tabular} &
  \begin{tabular}[c]{@{}c@{}}\underline{100.00}\\ \footnotesize \color[HTML]{9B9B9B} (0.00)\end{tabular} &
  \multicolumn{1}{c|}{\begin{tabular}[c]{@{}c@{}}\underline{7.82}\\ \footnotesize \color[HTML]{9B9B9B} (0.21)\end{tabular}} &
  \begin{tabular}[c]{@{}c@{}}\underline{12.76}\\ \footnotesize \color[HTML]{9B9B9B} (0.32)\end{tabular} &
  \begin{tabular}[c]{@{}c@{}}\underline{3.38}\\ \footnotesize \color[HTML]{9B9B9B} (0.20)\end{tabular} &
  \begin{tabular}[c]{@{}c@{}}\underline{21.81}\\ \footnotesize \color[HTML]{9B9B9B} (0.19)\end{tabular} &
  \multicolumn{1}{c|}{\begin{tabular}[c]{@{}c@{}}\underline{23.34}\\ \footnotesize \color[HTML]{9B9B9B} (0.34)\end{tabular}} &
  \begin{tabular}[c]{@{}c@{}}73.34\\ \footnotesize \color[HTML]{9B9B9B} (0.47)\end{tabular} &
  \begin{tabular}[c]{@{}c@{}}\underline{11.19}\\ \footnotesize \color[HTML]{9B9B9B} (0.47)\end{tabular} &
  \begin{tabular}[c]{@{}c@{}}\underline{100.00}\\ \footnotesize \color[HTML]{9B9B9B} (0.00)\end{tabular} &
  \begin{tabular}[c]{@{}c@{}}\underline{23.16}\\ \footnotesize \color[HTML]{9B9B9B} (0.13)\end{tabular} \\
\begin{tabular}[c]{@{}l@{}}\texttt{TERM} \footnotesize\cite{term}\end{tabular} &
  \begin{tabular}[c]{@{}c@{}}90.71\\ \footnotesize \color[HTML]{9B9B9B} (0.65)\end{tabular} &
  \begin{tabular}[c]{@{}c@{}}55.66\\ \footnotesize \color[HTML]{9B9B9B} (0.93)\end{tabular} &
  \begin{tabular}[c]{@{}c@{}}\underline{100.00}\\ \footnotesize \color[HTML]{9B9B9B} (0.00)\end{tabular} &
  \multicolumn{1}{c|}{\begin{tabular}[c]{@{}c@{}}7.90\\ \footnotesize \color[HTML]{9B9B9B} (0.38)\end{tabular}} &
  \begin{tabular}[c]{@{}c@{}}12.02\\ \footnotesize \color[HTML]{9B9B9B} (0.16)\end{tabular} &
  \begin{tabular}[c]{@{}c@{}}2.85\\ \footnotesize \color[HTML]{9B9B9B} (0.41)\end{tabular} &
  \begin{tabular}[c]{@{}c@{}}20.74\\ \footnotesize \color[HTML]{9B9B9B} (0.65)\end{tabular} &
  \multicolumn{1}{c|}{\begin{tabular}[c]{@{}c@{}}24.15\\ \footnotesize \color[HTML]{9B9B9B} (1.05)\end{tabular}} &
  \begin{tabular}[c]{@{}c@{}}70.90\\ \footnotesize \color[HTML]{9B9B9B} (2.96)\end{tabular} &
  \begin{tabular}[c]{@{}c@{}}5.98\\ \footnotesize \color[HTML]{9B9B9B} (1.10)\end{tabular} &
  \begin{tabular}[c]{@{}c@{}}\underline{100.00}\\ \footnotesize \color[HTML]{9B9B9B} (0.00)\end{tabular} &
  \begin{tabular}[c]{@{}c@{}}26.37\\ \footnotesize \color[HTML]{9B9B9B} (1.32)\end{tabular} \\
\begin{tabular}[c]{@{}l@{}}\texttt{FedMGDA} \footnotesize\cite{fedmgda}\end{tabular} &
  \begin{tabular}[c]{@{}c@{}}88.33\\ \footnotesize \color[HTML]{9B9B9B} (0.63)\end{tabular} &
  \begin{tabular}[c]{@{}c@{}}48.60\\ \footnotesize \color[HTML]{9B9B9B} (25.85)\end{tabular} &
  \begin{tabular}[c]{@{}c@{}}\underline{100.00}\\ \footnotesize \color[HTML]{9B9B9B} (0.00)\end{tabular} &
  \multicolumn{1}{c|}{\begin{tabular}[c]{@{}c@{}}9.75\\ \footnotesize \color[HTML]{9B9B9B} (0.59)\end{tabular}} &
  \begin{tabular}[c]{@{}c@{}}10.58\\ \footnotesize \color[HTML]{9B9B9B} (0.18)\end{tabular} &
  \begin{tabular}[c]{@{}c@{}}2.35\\ \footnotesize \color[HTML]{9B9B9B} (0.20)\end{tabular} &
  \begin{tabular}[c]{@{}c@{}}19.09\\ \footnotesize \color[HTML]{9B9B9B} (0.62)\end{tabular} &
  \multicolumn{1}{c|}{\begin{tabular}[c]{@{}c@{}}25.20\\ \footnotesize \color[HTML]{9B9B9B} (0.22)\end{tabular}} &
  \begin{tabular}[c]{@{}c@{}}72.45\\ \footnotesize \color[HTML]{9B9B9B} (1.88)\end{tabular} &
  \begin{tabular}[c]{@{}c@{}}9.65\\ \footnotesize \color[HTML]{9B9B9B} (2.90)\end{tabular} &
  \begin{tabular}[c]{@{}c@{}}\underline{100.00}\\ \footnotesize \color[HTML]{9B9B9B} (0.00)\end{tabular} &
  \begin{tabular}[c]{@{}c@{}}23.68\\ \footnotesize \color[HTML]{9B9B9B} (1.27)\end{tabular} \\
\begin{tabular}[c]{@{}l@{}}\texttt{PropFair} \footnotesize\cite{propfair}\end{tabular} &
  \begin{tabular}[c]{@{}c@{}}87.25\\ \footnotesize \color[HTML]{9B9B9B} (5.01)\end{tabular} &
  \begin{tabular}[c]{@{}c@{}}48.11\\ \footnotesize \color[HTML]{9B9B9B} (10.03)\end{tabular} &
  \begin{tabular}[c]{@{}c@{}}\underline{100.00}\\ \footnotesize \color[HTML]{9B9B9B} (0.00)\end{tabular} &
  \multicolumn{1}{c|}{\begin{tabular}[c]{@{}c@{}}10.39\\ \footnotesize \color[HTML]{9B9B9B} (3.43)\end{tabular}} &
  \begin{tabular}[c]{@{}c@{}}11.26\\ \footnotesize \color[HTML]{9B9B9B} (0.71)\end{tabular} &
  \begin{tabular}[c]{@{}c@{}}1.95\\ \footnotesize \color[HTML]{9B9B9B} (0.32)\end{tabular} &
  \begin{tabular}[c]{@{}c@{}}21.33\\ \footnotesize \color[HTML]{9B9B9B} (0.92)\end{tabular} &
  \multicolumn{1}{c|}{\begin{tabular}[c]{@{}c@{}}25.97\\ \footnotesize \color[HTML]{9B9B9B} (1.02)\end{tabular}} &
  \begin{tabular}[c]{@{}c@{}}73.64\\ \footnotesize \color[HTML]{9B9B9B} (3.31)\end{tabular} &
  \begin{tabular}[c]{@{}c@{}}7.30\\ \footnotesize \color[HTML]{9B9B9B} (1.02)\end{tabular} &
  \begin{tabular}[c]{@{}c@{}}\underline{100.00}\\ \footnotesize \color[HTML]{9B9B9B} (0.00)\end{tabular} &
  \begin{tabular}[c]{@{}c@{}}24.97\\ \footnotesize \color[HTML]{9B9B9B} (1.09)\end{tabular} \\
\rowcolor[HTML]{FFF5E6} 
\begin{tabular}[c]{@{}l@{}}\texttt{AAggFF-D}\\ (Proposed)\end{tabular} &
  \begin{tabular}[c]{@{}c@{}}\textbf{91.27}\\ \footnotesize \color[HTML]{9B9B9B} (0.07)\end{tabular} &
  \begin{tabular}[c]{@{}c@{}}\textbf{56.71}\\ \footnotesize \color[HTML]{9B9B9B} (0.08)\end{tabular} &
  \cellcolor[HTML]{FFF5E6}\begin{tabular}[c]{@{}c@{}}\underline{100.00}\\ \footnotesize \color[HTML]{9B9B9B} (0.00)\end{tabular} &
  \multicolumn{1}{c|}{\cellcolor[HTML]{FFF5E6}\begin{tabular}[c]{@{}c@{}}\textbf{7.54}\\ \footnotesize \color[HTML]{9B9B9B} (0.04)\end{tabular}} &
  \begin{tabular}[c]{@{}c@{}}\textbf{12.95}\\ \footnotesize \color[HTML]{9B9B9B} (0.39)\end{tabular} &
  \begin{tabular}[c]{@{}c@{}}\textbf{4.75}\\ \footnotesize \color[HTML]{9B9B9B} (0.76)\end{tabular} &
  \begin{tabular}[c]{@{}c@{}}\textbf{22.81}\\ \footnotesize \color[HTML]{9B9B9B} (1.36)\end{tabular} &
  \multicolumn{1}{c|}{\cellcolor[HTML]{FFF5E6}\begin{tabular}[c]{@{}c@{}}\textbf{22.59}\\ \footnotesize \color[HTML]{9B9B9B} (0.28)\end{tabular}} &
  \begin{tabular}[c]{@{}c@{}}\textbf{76.68}\\ \footnotesize \color[HTML]{9B9B9B} (0.80)\end{tabular} &
  \begin{tabular}[c]{@{}c@{}}\textbf{14.54}\\ \footnotesize \color[HTML]{9B9B9B} (2.58)\end{tabular} &
  \cellcolor[HTML]{FFF5E6}\begin{tabular}[c]{@{}c@{}}\underline{100.00}\\ \footnotesize \color[HTML]{9B9B9B} (0.00)\end{tabular} &
  \begin{tabular}[c]{@{}c@{}}\textbf{21.42}\\ \footnotesize \color[HTML]{9B9B9B} (0.81)\end{tabular} \\ \bottomrule
\end{tabular}%
}
\end{sidewaystable}

\newpage
\subsection{Connections to Accuracy Parity}
    \begin{table}[H]
\centering
    \begin{minipage}{.4\linewidth}
        \centering
        \caption[Accuracy parity gap of fair FL methods in cross-silo FL setting in Chapter~\ref{ch:aaggff}]
        {Accuracy parity gap in the cross-silo setting}
        \label{tab:ag_silo}
        \resizebox{\textwidth}{!}{%
            \begin{tabular}{!{}lccc!{}}
            \toprule
            \textbf{Dataset} &
              \textbf{Berka} &
              \textbf{MQP} &
              \textbf{ISIC} \\ \cmidrule(l){2-4} 
            \textbf{Method} &
              \multicolumn{3}{c}{$\Delta\text{AG } (\downarrow)$} \\ \midrule
            \begin{tabular}[c]{@{}l@{}}\texttt{FedAvg} \cite{fedavg}\end{tabular} &
              \begin{tabular}[c]{@{}c@{}}50.84\\ \footnotesize \color[HTML]{9B9B9B} (23.98)\end{tabular} &
              \begin{tabular}[c]{@{}c@{}}35.30\\ \footnotesize \color[HTML]{9B9B9B} (5.39)\end{tabular} &
              \begin{tabular}[c]{@{}c@{}}22.64\\ \footnotesize \color[HTML]{9B9B9B} (4.50)\end{tabular} \\
            \begin{tabular}[c]{@{}l@{}}\texttt{AFL} \cite{afl}\end{tabular} &
              \begin{tabular}[c]{@{}c@{}}50.98\\ \footnotesize \color[HTML]{9B9B9B} (23.78)\end{tabular} &
              \begin{tabular}[c]{@{}c@{}}34.26\\ \footnotesize \color[HTML]{9B9B9B} (5.16)\end{tabular} &
              \begin{tabular}[c]{@{}c@{}}25.16\\ \footnotesize \color[HTML]{9B9B9B} (8.01)\end{tabular} \\
            \begin{tabular}[c]{@{}l@{}}\texttt{q-FedAvg} \cite{qffl}\end{tabular} &
              \begin{tabular}[c]{@{}c@{}}50.43\\ \footnotesize \color[HTML]{9B9B9B} (22.15)\end{tabular} &
              \begin{tabular}[c]{@{}c@{}}39.16\\ \footnotesize \color[HTML]{9B9B9B} (7.13)\end{tabular} &
              \begin{tabular}[c]{@{}c@{}}37.69\\ \footnotesize \color[HTML]{9B9B9B} (5.52)\end{tabular} \\
            \begin{tabular}[c]{@{}l@{}}\texttt{TERM} \cite{term}\end{tabular} &
              \begin{tabular}[c]{@{}c@{}}49.60\\ \footnotesize \color[HTML]{9B9B9B} (23.74)\end{tabular} &
              \begin{tabular}[c]{@{}c@{}}36.07\\ \footnotesize \color[HTML]{9B9B9B} (6.93)\end{tabular} &
              \begin{tabular}[c]{@{}c@{}}15.19\\ \footnotesize \color[HTML]{9B9B9B} (9.26)\end{tabular} \\
            \begin{tabular}[c]{@{}l@{}}\texttt{FedMGDA} \cite{fedmgda}\end{tabular} &
              \begin{tabular}[c]{@{}c@{}}44.46\\ \footnotesize \color[HTML]{9B9B9B} (17.49)\end{tabular} &
              \begin{tabular}[c]{@{}c@{}}34.74\\ \footnotesize \color[HTML]{9B9B9B} (1.74)\end{tabular} &
              \begin{tabular}[c]{@{}c@{}}37.69\\ \footnotesize \color[HTML]{9B9B9B} (5.52)\end{tabular} \\
            \begin{tabular}[c]{@{}l@{}}\texttt{PropFair} \cite{propfair}\end{tabular} &
              \begin{tabular}[c]{@{}c@{}}49.05\\ \footnotesize \color[HTML]{9B9B9B} (23.78)\end{tabular} &
              \begin{tabular}[c]{@{}c@{}}37.38\\ \footnotesize \color[HTML]{9B9B9B} (4.35)\end{tabular} &
              \begin{tabular}[c]{@{}c@{}}32.99\\ \footnotesize \color[HTML]{9B9B9B} (9.60)\end{tabular} \\
            \rowcolor[HTML]{FFF5E6} 
            \begin{tabular}[c]{@{}l@{}}\texttt{AAggFF-S} \end{tabular} &
              \begin{tabular}[c]{@{}c@{}}\textbf{44.03}\\ \footnotesize \color[HTML]{9B9B9B} (17.55)\end{tabular} &
              \begin{tabular}[c]{@{}c@{}}\textbf{33.77}\\ \footnotesize \color[HTML]{9B9B9B} (3.31)\end{tabular} &
              \begin{tabular}[c]{@{}c@{}}\textbf{13.05}\\ \footnotesize \color[HTML]{9B9B9B} (2.23)\end{tabular} \\ \bottomrule
            \end{tabular}%
        }
    \end{minipage}
    \quad
    \begin{minipage}{0.5\linewidth}
        \centering
        \caption[Accuracy parity gap of fair FL methods in cross-device FL setting in Chapter~\ref{ch:aaggff}]
        {Accuracy parity gap in the cross-device setting}
        \label{tab:ag_device}
        \resizebox{\textwidth}{!}{%
            \begin{tabular}{!{}lccc!{}}
            \toprule
            \textbf{Dataset} &
              \textbf{CelebA} &
              \textbf{Reddit} &
              \textbf{\begin{tabular}[c]{@{}c@{}}Speech\\ Commands\end{tabular}} \\ \cmidrule(l){2-4} 
            \textbf{Method} &
              \multicolumn{3}{c}{$\Delta\text{AG } (\downarrow)$} \\ \midrule
            \begin{tabular}[c]{@{}l@{}}\texttt{FedAvg} \cite{fedavg}\end{tabular} &
              \begin{tabular}[c]{@{}c@{}}44.25\\ \footnotesize \color[HTML]{9B9B9B} (0.84)\end{tabular} &
              \begin{tabular}[c]{@{}c@{}}18.36\\ \footnotesize \color[HTML]{9B9B9B} (3.52)\end{tabular} &
              \begin{tabular}[c]{@{}c@{}}92.07\\ \footnotesize \color[HTML]{9B9B9B} (2.87)\end{tabular} \\
            \begin{tabular}[c]{@{}l@{}}\texttt{q-FedAvg} \cite{qffl}\end{tabular} &
              \begin{tabular}[c]{@{}c@{}}44.27\\ \footnotesize \color[HTML]{9B9B9B} (0.85)\end{tabular} &
              \begin{tabular}[c]{@{}c@{}}18.43\\ \footnotesize \color[HTML]{9B9B9B} (0.09)\end{tabular} &
              \begin{tabular}[c]{@{}c@{}}88.81\\ \footnotesize \color[HTML]{9B9B9B} (0.47)\end{tabular} \\
            \begin{tabular}[c]{@{}l@{}}\texttt{TERM} \cite{term}\end{tabular} &
              \begin{tabular}[c]{@{}c@{}}44.34\\ \footnotesize \color[HTML]{9B9B9B} (0.93)\end{tabular} &
              \begin{tabular}[c]{@{}c@{}}17.89\\ \footnotesize \color[HTML]{9B9B9B} (0.75)\end{tabular} &
              \begin{tabular}[c]{@{}c@{}}94.02\\ \footnotesize \color[HTML]{9B9B9B} (1.10)\end{tabular} \\
            \begin{tabular}[c]{@{}l@{}}\texttt{FedMGDA} \cite{fedmgda}\end{tabular} &
              \begin{tabular}[c]{@{}c@{}}51.40\\ \footnotesize \color[HTML]{9B9B9B} (2.59)\end{tabular} &
              \begin{tabular}[c]{@{}c@{}}\textbf{16.74}\\ \footnotesize \color[HTML]{9B9B9B} (0.43)\end{tabular} &
              \begin{tabular}[c]{@{}c@{}}90.35\\ \footnotesize \color[HTML]{9B9B9B} (2.90)\end{tabular} \\
            \begin{tabular}[c]{@{}l@{}}\texttt{PropFair} \cite{propfair}\end{tabular} &
              \begin{tabular}[c]{@{}c@{}}51.90\\ \footnotesize \color[HTML]{9B9B9B} (10.03)\end{tabular} &
              \begin{tabular}[c]{@{}c@{}}19.39\\ \footnotesize \color[HTML]{9B9B9B} (0.64)\end{tabular} &
              \begin{tabular}[c]{@{}c@{}}92.70\\ \footnotesize \color[HTML]{9B9B9B} (1.02)\end{tabular} \\
            \rowcolor[HTML]{FFF5E6} 
            \begin{tabular}[c]{@{}l@{}}\texttt{AAggFF-D} \end{tabular} &
              \begin{tabular}[c]{@{}c@{}}\textbf{43.29}\\ \footnotesize \color[HTML]{9B9B9B} (0.08)\end{tabular} &
              \begin{tabular}[c]{@{}c@{}}18.07\\ \footnotesize \color[HTML]{9B9B9B} (0.70)\end{tabular} &
              \begin{tabular}[c]{@{}c@{}}\textbf{85.46}\\ \footnotesize \color[HTML]{9B9B9B} (2.58)\end{tabular} \\ \bottomrule
            \end{tabular}%
        }
    \end{minipage}
\end{table}

    As discussed in~\cite{qffl}, the client-level fairness can be loosely connected to existing fairness notion, the \textit{accuracy parity}~\cite{accparity}.
    It is guaranteed if the accuracies in protected groups are equal to each other.
    While the accuracy parity requires \textit{equal} performances among specific groups having protected attributes~\cite{accparity,accparity2}, 
    this is too restrictive to be directly applied to FL settings, since each client cannot always be exactly corresponded to the concept of `a group', and each client's local distribution may not be partitioned by protected attributes in the federated system.

    With a relaxation of the original concept, we adopt the notion of accuracy parity for measuring the degree of the client-level fairness in the federated system, i.e., we simply regard the group as each client.
    As a metric, we adopt the accuracy parity gap ($\Delta\text{AG}$) proposed by~\cite{apgap,apgap2}, which is simply defined as an absolute difference between the performance of the best and the worst performing groups (clients).  
    The results are in Table~\ref{tab:ag_silo} and Table~\ref{tab:ag_device}.
    It can be said that the smaller the $\Delta\text{AG}$, the more degree of the accuracy parity fairness (and therefore the client-level fairness) is achieved. 

    It should be noted that strictly achieving the accuracy parity can sometimes require sacrifice in the average performance. 
    This is aligned with the result of Reddit dataset in Table~\ref{tab:ag_device}, where \texttt{FedMGDA}~\cite{fedmgda} achieved the smallest $\Delta\text{AG}$, while its average performance is only 10.58 in Table~\ref{tab:result_device}. 
    This is far lower than our proposed method’s average performance, 12.95.
    Except this case, \texttt{AAggFF} consistently shows the smallest $\Delta\text{AG}$ than other baseline methods, 
    which is important in the perspective of striking a good balance between overall utility and the client-level fairness.
    
\newpage
\subsection{Plug-and-Play Boosting}
\label{subsec:pnp}
    We additionally check if \texttt{AAggFF} can also boost other FL algorithms than \texttt{FedAvg}, 
    such as \texttt{FedAdam, FedAdagrad, FedYogi} \cite{adaptivefl} and \texttt{FedProx} \cite{fedprox}.
    Since the sequential decision making procedure required in \texttt{AAggFF} is about finding a good mixing coefficient, $\boldsymbol{p}$, this is orthogonal to the minimization of $\boldsymbol{\theta}$.
    Thus, our method can be easily integrated into existing methods with no special modification, in a plug-and-play manner.
    
    For the verification, we test with two more datasets, Heart \cite{heart} and TinyImageNet \cite{tinyimagenet}, each of which is suited for binary and multi-class classification (i.e., 200 classes in total).
    Since the Heart dataset is a part of FLamby benchmark \cite{flamby}, it has pre-defined $K=4$ clients.
    For the TinyImageNet dataset, we simulate statistical heterogeneity for $K=1,000$ clients using Dirichlet distribution with a concentration of $0.01$, following \cite{diri}.
    The results are in Table~\ref{tab:pnp}, where the upper cell represents the performance of a naive FL algorithm, 
    and the lower cell contains a performance of the FL algorithm with \texttt{AAggFF}.
    While the average performance remains comparable, the worst performance is consistently boosted in both cross-silo and cross-device settings.
    This underpins the efficacy and flexibility of \texttt{AAggFF}, which can strengthen the fairness perspective of existing FL algorithms.
    \begin{table}[H]
\centering
\caption[Experimental results on the plug-and-play boosting of \texttt{AAggFF}]{Improved performance of FL algorithms after being equipped with \texttt{AAggFF}}
\label{tab:pnp}
\resizebox{0.7\textwidth}{!}{%
\begin{tabular}{!{}lcccc!{}}
\toprule
\textbf{Dataset} &
  \multicolumn{2}{c}{\begin{tabular}[c]{@{}c@{}}\textbf{Heart}\\ (AUROC)\end{tabular}} &
  \multicolumn{2}{c}{\begin{tabular}[c]{@{}c@{}}\textbf{TinyImageNet}\\ (Acc. 5)\end{tabular}} \\ \cmidrule(l){2-5} 
\textbf{Method} &
  \begin{tabular}[c]{@{}c@{}}Avg.\\ ($\uparrow$)\end{tabular} &
  \begin{tabular}[c]{@{}c@{}}Worst\\ ($\uparrow$)\end{tabular} &
  \begin{tabular}[c]{@{}c@{}}Avg.\\ ($\uparrow$)\end{tabular} &
  \begin{tabular}[c]{@{}c@{}}Worst\\ 10\%($\uparrow$)\end{tabular} \\ \midrule
 &
  84.42 \color[HTML]{9B9B9B}(2.45) &
  \multicolumn{1}{c|}{65.22 \color[HTML]{9B9B9B}(9.78)} &
  85.93 \color[HTML]{9B9B9B}(0.77) &
  50.95 \color[HTML]{9B9B9B}(0.15) \\
\multirow{-2}{*}{\begin{tabular}[c]{@{}l@{}}\texttt{FedAvg} \cite{fedavg}\end{tabular}} &
  \cellcolor[HTML]{FFF5E6}\textbf{85.04}\color[HTML]{9B9B9B}(2.86) &
  \multicolumn{1}{c|}{\cellcolor[HTML]{FFF5E6}\textbf{66.56} \color[HTML]{9B9B9B}(10.81)} &
  \cellcolor[HTML]{FFF5E6}\textbf{86.66} \color[HTML]{9B9B9B}(0.63) &
  \cellcolor[HTML]{FFF5E6}\textbf{51.50} \color[HTML]{9B9B9B}(2.32) \\ \cmidrule(l){2-5} 
 &
  84.48\color[HTML]{9B9B9B}(0.25) &
  \multicolumn{1}{c|}{65.44\color[HTML]{9B9B9B}(9.77)} &
  \textbf{86.49} \color[HTML]{9B9B9B}(0.72) &
  51.64 \color[HTML]{9B9B9B}(2.07) \\
\multirow{-2}{*}{\begin{tabular}[c]{@{}l@{}}\texttt{FedProx} \cite{fedprox}\end{tabular}} &
  \cellcolor[HTML]{FFF5E6}\textbf{85.04}\color[HTML]{9B9B9B}(2.81) &
  \multicolumn{1}{c|}{\cellcolor[HTML]{FFF5E6}\textbf{66.67}\color[HTML]{9B9B9B}(10.71)} &
  \cellcolor[HTML]{FFF5E6}86.11 \color[HTML]{9B9B9B}(0.72) &
  \cellcolor[HTML]{FFF5E6}\textbf{52.29} \color[HTML]{9B9B9B}(2.16) \\ \cmidrule(l){2-5} 
 &
  84.34\color[HTML]{9B9B9B}(2.78) &
  \multicolumn{1}{c|}{65.44\color[HTML]{9B9B9B}(10.12)} &
  87.04 \color[HTML]{9B9B9B}(1.05) &
  53.54 \color[HTML]{9B9B9B}(2.63) \\
\multirow{-2}{*}{\begin{tabular}[c]{@{}l@{}}\texttt{FedAdam} \cite{adaptivefl}\end{tabular}} &
  \cellcolor[HTML]{FFF5E6}\textbf{84.84}\color[HTML]{9B9B9B}(2.85) &
  \multicolumn{1}{c|}{\cellcolor[HTML]{FFF5E6}\textbf{67.00}\color[HTML]{9B9B9B}(10.61)} &
  \cellcolor[HTML]{FFF5E6}\textbf{87.89} \color[HTML]{9B9B9B}(0.90) &
  \cellcolor[HTML]{FFF5E6}\textbf{55.92} \color[HTML]{9B9B9B}(2.25) \\ \cmidrule(l){2-5} 
 &
  84.29\color[HTML]{9B9B9B}(2.62) &
  \multicolumn{1}{c|}{65.67\color[HTML]{9B9B9B}(10.68)} &
  86.70 \color[HTML]{9B9B9B}(1.40) &
  52.81 \color[HTML]{9B9B9B}(3.50) \\
\multirow{-2}{*}{\begin{tabular}[c]{@{}l@{}}\texttt{FedYogi} \cite{adaptivefl}\end{tabular}} &
  \cellcolor[HTML]{FFF5E6}\textbf{84.86}\color[HTML]{9B9B9B}(3.01) &
  \multicolumn{1}{c|}{\cellcolor[HTML]{FFF5E6}\textbf{67.00}\color[HTML]{9B9B9B}(11.09)} &
  \cellcolor[HTML]{FFF5E6}\textbf{87.42} \color[HTML]{9B9B9B}(0.94) &
  \cellcolor[HTML]{FFF5E6}\textbf{54.76} \color[HTML]{9B9B9B}(3.11) \\ \cmidrule(l){2-5} 
 &
  84.61\color[HTML]{9B9B9B}(2.96) &
  \multicolumn{1}{c|}{65.67\color[HTML]{9B9B9B}(10.68)} &
  83.52 \color[HTML]{9B9B9B}(0.63) &
  45.09 \color[HTML]{9B9B9B}(1.79) \\
\multirow{-2}{*}{\begin{tabular}[c]{@{}l@{}}\texttt{FedAdagrad} \cite{adaptivefl}\end{tabular}} &
  \cellcolor[HTML]{FFF5E6}\textbf{85.09}\color[HTML]{9B9B9B}(2.91) &
  \multicolumn{1}{c|}{\cellcolor[HTML]{FFF5E6}\textbf{66.67}\color[HTML]{9B9B9B}(10.37)} &
  \cellcolor[HTML]{FFF5E6}\textbf{84.62} \color[HTML]{9B9B9B}(0.51) &
  \cellcolor[HTML]{FFF5E6}\textbf{47.88} \color[HTML]{9B9B9B}(1.95) \\ \bottomrule
\end{tabular}%
}
\end{table}

\newpage
\section{Conclusion}
    For improving the degree of client-level fairness in FL, 
    we first reveal the connection between existing fair FL methods and the OCO.
    To emphasize the sequential decision-making perspective, we propose improved designs and further specialize them into two practical settings: cross-silo FL \& cross-device FL.
    Our framework not only efficiently enhances a low-performing group of clients compared to existing baselines,
    but also maintains an acceptable average performance with theoretically guaranteed behaviors.
    It should also be noted that \texttt{AAggFF} requires \textit{no extra communication} and \textit{no added local computation}, 
    which are significant constraints for serving FL-based services.
    With this scalability, our method can also improve the fairness of the performance distributions of existing FL algorithms without much modification to their original mechanisms.
    By explicitly bringing the sequential decision-making scheme to the front, 
    we expect our work to open up new designs to promote the practicality and scalability of FL. 

    Our work suggests interesting future directions for better federated systems, which may also be a limitation of the current work.
    First, we can exploit side information (e.g., parameters of local updates) to not preserve all clients' mixing coefficients, 
    and filter out malicious signals for robustness. 
    For example, the former can be realized by adopting other decision-making schemes such as contextual Bayesian optimization \cite{cbo},
    and the latter can be addressed by clustered FL \cite{cfl1, cfl2} for a group-wise estimation of mixing coefficients.
    Both directions are promising and may improve the practicality of federated systems. 
    Furthermore, the FTRL objective can be replaced by the Follow-The-Perturbed-Leader (FTPL) \cite{ftpl}, 
    of which random perturbation in decision-making process can be directly linked to the differential privacy (DP \cite{dp}) guarantee \cite{ftpl2}, which is frequently considered for the cross-silo setting.
    Last but not least, further convergence analysis is required w.r.t. the parameter perspective along with mixing coefficients, e.g., using a bi-level optimization formulation.

\newpage
\section{Derivations \& Deferred Proofs}
\label{sec:proofs}
\subsection{Derivation of Mixing Coefficients from Existing Methods}
\label{app:unification}
    In this section, we provide details of the unification of existing methods in the OCO framework, introduced in Section~\ref{sec:oco_lang}.
    We assume full-client participation for derivation, 
    and we denote $n = \sum_{i=1}^K n_i$ as a total sample size for the brevity of notation.
    
    Suppose any FL algorithms follow the update formula in (\ref{eq:generic_update}), 
    where we define $\boldsymbol{p}^{(t+1)}$ as a \textit{mixing coefficient} vector discussed in Section~\ref{sec:oco_lang}.
    \begin{equation}
    \label{eq:generic_update}
    \begin{gathered}
        \boldsymbol{\theta}^{(t+1)}
        \leftarrow
        \boldsymbol{\theta}^{(t)}
        -
        \left(
        \sum_{i=1}^K
        {p}^{(t+1)}_i
        \left(
        \boldsymbol{\theta}^{(t)}
        -
        \boldsymbol{\theta}^{(t+1)}_i
        \right)
        \right),
    \end{gathered}
    \end{equation}
    where $\boldsymbol{\theta}^{(t)}$ is a global model in a previous round $t$, 
    $\boldsymbol{\theta}^{(t+1)}_i$ is a local update from $i$-th client starting from $\boldsymbol{\theta}^{(t)}$, 
    and $\boldsymbol{\theta}^{(t+1)}$ is a new global model updated by averaging local updates with corresponding mixing coefficient $p^{(t+1)}_i$. 

\paragraph{\texttt{FedAvg}~\cite{fedavg}} 
    The update of a global model from \texttt{FedAvg} is defined as follows.
    \begin{equation}
    \begin{gathered}
        \boldsymbol{\theta}^{(t+1)}
        \leftarrow
        \boldsymbol{\theta}^{(t)}
        -
        \left(
        \sum_{i=1}^K
        \frac{n_i}{n}
        \left(
        \boldsymbol{\theta}^{(t)}
        -
        \boldsymbol{\theta}^{(t+1)}_i
        \right)
        \right),
    \end{gathered}
    \end{equation}
    where $n_i$ is the sample size of client $i$.
    Thus, we can regard ${p}^{(t+1)}_i \propto n_i$ in \texttt{FedAvg}.

\paragraph{\texttt{AFL \& q-FedAvg}~\cite{afl, qffl}} 
    The objective of \texttt{AFL} is a minimax objective defined as follows.
    \begin{equation}
    \begin{gathered}
        \min_{\boldsymbol{\theta}\in\mathbb{R}^d} \max_{\boldsymbol{u}\in\Delta_{K-1}}
        \sum_{i=1}^K u_i F_i\left( \boldsymbol{\theta} \right),
    \end{gathered}
    \end{equation}
    which is later subsumed by \texttt{q-FedAvg} as its special case for the algorithm-specific constant $q$, where $q\rightarrow0$.
    
    The objective of \texttt{q-FedAvg} is therefore defined with a nonnegative constant $q$ as follows.
    \begin{equation}
    \begin{gathered}
    \label{eq:q-ffl}
        \min_{\boldsymbol{\theta}\in\mathbb{R}^d}
        \sum_{i=1}^K\frac{1}{q+1} \frac{n_i}{n} F_i^{q+1} \left(\boldsymbol{\theta} \right)\\
        =
        \min_{\boldsymbol{\theta}\in\mathbb{R}^d}
        \sum_{i=1}^K \frac{n_i}{n} \tilde{F}_i\left(\boldsymbol{\theta}\right),
    \end{gathered}
    \end{equation}
    which is reduced to \texttt{FedAvg} when $q=0$.
    
    The update of a global model from (\ref{eq:q-ffl}) has been proposed in the form of a Newton style update 
    by assuming $L$-Lipschitz continuous gradient of each local objective (i.e., \texttt{q-FedSGD}) \cite{qffl}.
    \begin{equation}
    \begin{gathered}
        \boldsymbol{\theta}^{(t+1)}
        =
        \boldsymbol{\theta}^{(t)}
        -
        \left(
        \sum_{j=1}^K \frac{n_j}{n} \nabla^2 \tilde{F}_j \left(\boldsymbol{\theta}^{(t)}\right)
        \right)^{-1}
        \sum_{i=1}^K \frac{n_i}{n} \nabla\tilde{F}_i \left(\boldsymbol{\theta}^{(t)} \right)\\
        \preceq
        \boldsymbol{\theta}^{(t)}
        -
        \left(
        \sum_{j=1}^K
        \frac{n_j}{n} L_{q,j}\boldsymbol{I}
        \right)^{-1}
        \sum_{i=1}^K
        \frac{n_i}{n} {F}_i^q\left(\boldsymbol{\theta}^{(t)}\right)
        \nabla{F}_i\left(\boldsymbol{\theta}^{(t)}\right),
    \end{gathered}
    \end{equation}
    where $L_{q,i} = q F_i^{q-1}\left(\boldsymbol{\theta}^{(t)}\right) \Vert\nabla F_i\left(\boldsymbol{\theta}^{(t)}\right) \Vert^2 
    + L F_i^q\left(\boldsymbol{\theta}^{(t)}\right)$ is an upper bound of the local Lipschitz gradient of $\tilde{F}_i\left(\boldsymbol{\theta}^{(t)}\right)$ (see Lemma 3 of \cite{qffl}).
    
    This can be extended to \texttt{q-FedAvg} by replacing $\nabla{F}_i\left(\boldsymbol{\theta}^{(t)}\right)$ into $L\left(\boldsymbol{\theta}^{(t)} - \boldsymbol{\theta}^{(t+1)}_i\right)$.
    To sum up, the update formula of a global model from \texttt{q-FedAvg} (including \texttt{AFL} as a special case) is as follows.
    \begin{equation}
    \begin{gathered}
        \boldsymbol{\theta}^{(t+1)}
        \propto
        \boldsymbol{\theta}^{(t)}
        -
        \left(
        \sum_{i=1}^K
        { \frac{ \frac{n_i}{n} L {F}_i^q(\boldsymbol{\theta}^{(t)}) }
        { \sum_{j=1}^K \frac{n_j}{n} L_{q,j} } }
        \left(
        \boldsymbol{\theta}^{(t)}
        -
        \boldsymbol{\theta}^{(t+1)}_i\right)
        \right),
    \end{gathered}
    \end{equation}
    which implies $p_i^{(t+1)} \propto { n_i {F}_i^q\left(\boldsymbol{\theta}^{(t)}\right) }$.

\paragraph{\texttt{TERM}~\cite{term}} 
    The objective of \texttt{TERM} is dependent upon a hyperparameter, 
    a \textit{tilting} constant $\lambda\in\mathbb{R}$.
    \begin{equation}
    \begin{gathered}
        \min_{\boldsymbol{\theta}\in\mathbb{R}^d}
        \frac{1}{\lambda} \log \left( \sum_{i=1}^K \frac{n_i}{n} \exp\left({\lambda {F}_i(\boldsymbol{\theta}^{(t)})}\right) \right)
    \end{gathered}
    \end{equation}
    
    The corresponding update formula is given as follows.
    \begin{equation}
    \begin{gathered}
        \boldsymbol{\theta}^{(t+1)}
        =
        \boldsymbol{\theta}^{(t)}
        -
        \left(
        \sum_{i=1}^K
        \frac{
        \left(n_i / n\right) \exp\left(\lambda {F}_i\left(\boldsymbol{\theta}^{(t)}\right)\right)
        }
        {
        \sum_{j=1}^K 
        \left(n_j / n\right) \exp\left(\lambda {F}_j\left(\boldsymbol{\theta}^{(t)}\right)\right)
        }
        \left(
        \boldsymbol{\theta}^{(t)}
        -
        \boldsymbol{\theta}^{(t+1)}_i\right)
        \right)
    \end{gathered}
    \end{equation}
    
    From the update formula, we can conclude that $p_i^{(t+1)} \propto n_i \exp\left(\lambda {F}_i\left(\boldsymbol{\theta}^{(t)}\right)\right)$.

\paragraph{\texttt{PropFair}~\cite{propfair}} 
    The objective of \texttt{PropFair} is to maximize Nash social welfare 
    by regarding a negative local loss as an achieved utility as follows.
    \begin{equation}
    \begin{gathered}
        \min_{\boldsymbol{\theta}\in\mathbb{R}^d} -\sum_{i=1}^K p_i\log\left(M - F_i\left(\boldsymbol{\theta}\right)\right),
    \end{gathered}
    \end{equation}
    where $M\geq1$ is a problem-specific constant.
    
    The corresponding update formula is given as follows.
    \begin{equation}
    \begin{gathered}
        \boldsymbol{\theta}^{(t+1)}
        \propto
        \boldsymbol{\theta}^{(t)}
        +
        \left(
        \sum_{i=1}^K
        \frac{n_i}{n}\nabla\log\left(M - F_i\left(\boldsymbol{\theta}^{(t)}\right)\right)
        \right)
        =
        \boldsymbol{\theta}^{(t)}
        -
        \left(
        \sum_{i=1}^K
        \frac{n_i}{n}\frac{
        \nabla F_i\left(\boldsymbol{\theta}^{(t)}\right)
        }{
        M-F_i\left(\boldsymbol{\theta}^{(t)}\right)
        }
        \right).
    \end{gathered}
    \end{equation}
    
    Similar to \texttt{q-FedAvg}, by replacing the gradient $\nabla{F}_i\left(\boldsymbol{\theta}^{(t)}\right)$ 
    into $\left(\boldsymbol{\theta}^{(t)} - \boldsymbol{\theta}^{(t+1)}_i\right)$, the update formula finally becomes:
    \begin{equation}
    \begin{gathered}
        \boldsymbol{\theta}^{(t+1)}
        \propto
        \boldsymbol{\theta}^{(t)}
        -
        \left(
        \sum_{i=1}^K
        \frac{ n_i/n }{ M-F_i\left(\boldsymbol{\theta}^{(t)}\right) }
        \left(\boldsymbol{\theta}^{(t)} - \boldsymbol{\theta}^{(t+1)}_i\right)
        \right),
    \end{gathered}
    \end{equation}
    which implies $p_i^{(t+1)} \propto \frac{n_i}{M - F_i\left(\boldsymbol{\theta}^{(t)}\right)}$.

\newpage
\subsection{Technical Lemmas}
\label{app:proofs}
    In this section, we provide technical lemmas and proofs (including deferred ones in the main text) required for proving Theorem~\ref{thm:crosssilo}, Theorem~\ref{thm:crossdevice_full}, and Corollary~\ref{cor:crossdevice_partial}.
    
\subsection{Strict Convexity of Decision Loss}
    \begin{lemma}
    \label{lemma:convexity}
        For all $t\in[T]$, the decision loss $\ell^{(t)}$ defined in (\ref{eq:decision_loss}) satisfies following for $\gamma\in(0,1)$, 
        i.e., the decision loss is a strictly convex function of its first argument.
    \begin{equation}
    \begin{gathered}
        \ell^{(t)}\left(\gamma \boldsymbol{p} + (1-\gamma) \boldsymbol{q}\right)
        <
        \gamma \ell^{(t)}\left(\boldsymbol{p}\right) + (1-\gamma) \ell^{(t)}\left(\boldsymbol{q}\right),
        \forall \boldsymbol{p}, \boldsymbol{q} \in \Delta_{K-1}, \boldsymbol{p} \neq \boldsymbol{q}.
    \end{gathered}
    \end{equation}
    \end{lemma}
    
    \begin{proof}
    From the left-hand side, we have
    \begin{equation}
    \begin{split}
        &\ell^{(t)}\left(\gamma \boldsymbol{p} + (1-\gamma) \boldsymbol{q}\right)\\
        =
        &-\log\left(1+ \langle \gamma \boldsymbol{p} + (1-\gamma) \boldsymbol{q}, \boldsymbol{r}^{(t)} \rangle\right) \\
        =
        &-\log\left(
        1+ \langle \gamma \boldsymbol{p}, \boldsymbol{r}^{(t)} \rangle + \langle (1-\gamma) \boldsymbol{q}, \boldsymbol{r}^{(t)} \rangle 
        \right) \\
        =
        &-\log \left( 
        \gamma (1 + \langle \boldsymbol{p}, \boldsymbol{r}^{(t)} \rangle) 
        + (1-\gamma)(1 + \langle  \boldsymbol{q}, \boldsymbol{r}^{(t)} \rangle) 
        \right).
    \end{split}
    \end{equation}
    
    Since the negative of logarithm is strictly convex, the last term becomes
    \begin{equation}
    \begin{gathered}
        -\log \left( 
        \gamma (1 + \langle \boldsymbol{p}, \boldsymbol{r}^{(t)} \rangle) 
        + (1-\gamma)(1 + \langle  \boldsymbol{q}, \boldsymbol{r}^{(t)} \rangle) 
        \right)
        < \gamma \left( -\log (1 + \langle \boldsymbol{p}, \boldsymbol{r}^{(t)} \rangle) \right)
        + (1-\gamma) \left( -\log(1 + \langle  \boldsymbol{q}, \boldsymbol{r}^{(t)} \rangle) \right),
    \end{gathered}
    \end{equation}
    which satisfies the definition of the strict convexity, thereby concludes the proof.
    \end{proof}

\newpage
\subsection{Lipschitz Continuity of Decision Loss (Lemma~\ref{lemma:lipschitz})}
    From the definition of the Lipschitz continuity w.r.t. $\Vert\cdot\Vert$, 
    we need to check if the decision loss $\ell^{(t)}$ satisfies following inequality for the constant $L_\infty$.
    \begin{equation}
    \begin{gathered}
        \left\vert \ell^{(t)} \left(\boldsymbol{p}\right) - \ell^{(t)} \left(\boldsymbol{q}\right)\right\vert 
        \leq L_\infty \left\Vert \boldsymbol{p} - \boldsymbol{q} \right\Vert_\infty.
    \end{gathered}
    \end{equation}
    
    \begin{proof}
    From Lemma~\ref{lemma:convexity}, we have the following inequality from the convexity of the decision loss.
    \begin{equation}
    \begin{split}
        &\left\vert \ell^{(t)} \left(\boldsymbol{p}\right) - \ell^{(t)} \left(\boldsymbol{q}\right) \right\vert
        \leq
        \left\vert \langle \nabla \ell^{(t)} \left(\boldsymbol{p}\right), \boldsymbol{p} - \boldsymbol{q} \rangle \right\vert\\
        &=
        \left\vert
        - \frac{ \langle \boldsymbol{p}-\boldsymbol{q}, \boldsymbol{r}^{(t)} \rangle }
        { 1 +\langle \boldsymbol{p}, \boldsymbol{r}^{(t)} \rangle }    
        \right\vert\\
        &=
        \frac{1}{
        1 +\langle
        \boldsymbol{p},
        \boldsymbol{r}^{(t)}
        \rangle}
        \left\vert
        \left\langle \boldsymbol{q}, \boldsymbol{r}^{(t)} \right\rangle
        - \left\langle \boldsymbol{p}, \boldsymbol{r}^{(t)} \right\rangle
        \right\vert
    \end{split}
    \end{equation}
    
    Setting the denominator to be the minimum value, $\left\langle \boldsymbol{p}, \boldsymbol{r}^{(t)} \right\rangle$ is $C_1$,
    we have the upper bound as follows.
    \begin{equation}
    \begin{split}
        &\frac{1}{
        1 +\left\langle
        \boldsymbol{p},
        \boldsymbol{r}^{(t)}
        \right\rangle}
        \left\vert
        \left\langle \boldsymbol{q}, \boldsymbol{r}^{(t)} \right\rangle
        - \left\langle \boldsymbol{p}, \boldsymbol{r}^{(t)} \right\rangle
        \right\vert\\
        &\leq
        \frac{1}{
        1 +\left\langle
        \boldsymbol{p},
        \boldsymbol{r}^{(t)}
        \right\rangle}
        \max\left({
        \left\langle \boldsymbol{q}, \boldsymbol{r}^{(t)} \right\rangle, 
        \left\langle \boldsymbol{p}, \boldsymbol{r}^{(t)} \right\rangle
        }\right)\\
        &\leq
        \frac{1}{
        1 + C_1}
        \max\left({ 
        \left\langle \boldsymbol{q}, \boldsymbol{r}^{(t)} \right\rangle,
        C_1
        }\right),
    \end{split}
    \end{equation}
    where the first inequality is from the fact that both $\langle \boldsymbol{p}, \boldsymbol{r}^{(t)} \rangle$ and $\langle \boldsymbol{q}, \boldsymbol{r}^{(t)} \rangle$ are nonnegative,
    and the second inequality is due to the minimized denominator achieving the upper bound. 
    
    Since $\langle \boldsymbol{q}, \boldsymbol{r}^{(t)} \rangle$ can achieve its maximum as $C_2$,
    we can further bound as follows.
    \begin{equation}
    \begin{gathered}
        \frac{1}{
        1 + C_1}
        \max({
        \langle \boldsymbol{q}, \boldsymbol{r}^{(t)} \rangle,
        C_1
        })
        \leq
        \frac{1}{
        1 + C_1}
        \max({
        C_2, 
        C_1
        })
        =
        \frac{C_2}{
        1 + C_1}
    \end{gathered}
    \end{equation}
    Finally, using the fact that $\Vert \boldsymbol{p} - \boldsymbol{q} \Vert_\infty=\max_i \vert p_i - q_i \vert = 1$,
    we can conclude the statement by setting $L_\infty = \frac{C_2}{1 + C_1}$.
    \end{proof}

\newpage
\subsection{Unbiasedness of Doubly Robust Estimator (Lemma~\ref{lemma:unbiased_resp})}
    \begin{proof}
        Denote the client sampling probability $C\in[0,1]$ in time $t$ as $P(i\in S^{(t)})=C$.
        Taking expectation on the doubly robust estimator of partially observed response defined in (\ref{eq:dr_response}), we have
    \begin{equation}
    \begin{split}
        &\mathbb{E} \left[ \breve{r}^{(t)}_i \right]
        = 
        \mathbb{E} \left[ \bigg(1 - \frac{\mathbb{I}(i \in S^{(t)})}{C}\bigg)\mathrm{\bar{r}}^{(t)} \right] 
        + 
        \mathbb{E} \left[ \frac{\mathbb{I}(i \in S^{(t)})}{C}{r}^{(t)}_i \right] \\
        &=
        \bigg(1 - \frac{ \mathbb{E}[ \mathbb{I}(i \in S^{(t)}) ] }{C} \bigg)\mathrm{\bar{r}}^{(t)}  
        + 
        \frac{\mathbb{E} \left[ \mathbb{I}(i \in S^{(t)}) \right] }{ C }{r}^{(t)}_i\\
        &=
        \bigg(1 - \frac{ P(i\in S^{(t)}) }{C} \bigg)\mathrm{\bar{r}}^{(t)}  
        + 
        \frac{P(i\in S^{(t)}) }{ C }{r}^{(t)}_i\\
        &= {r}^{(t)}_i,
    \end{split}
    \end{equation}    
        where $\mathbb{I}(\cdot)$ is an indicator function.
    
        Note that the randomness of the doubly robust estimator comes from the random sampling of client indices $i\in S^{(t)}$ in round $t$,
        thus the expectation is with respect to $i\in S^{(t)}$.
        Thus, we can conclude that $\mathbb{E} \left[ \breve{\boldsymbol{r}}^{(t)} \right] = \boldsymbol{r}^{(t)}$.
        See also \cite{doublyrobust2, doublyrobust3}.
    \end{proof}

\newpage
\subsection{Unbiasedness of Linearly Approximated Gradient (Lemma~\ref{lemma:linearized_grad})}
    \begin{proof}
        The gradient of a decision loss in terms of a response, 
        $\boldsymbol{g}\equiv\mathrm{\mathbf{h}}(\boldsymbol{r}) 
        = 
        [h_1(\boldsymbol{r}),...,h_K(\boldsymbol{r})]^\top 
        = -\frac{\boldsymbol{r}}{1 +\langle \boldsymbol{p}, \boldsymbol{r} \rangle}$ can be linearly approximated at reference 
        $\boldsymbol{r}_0$ as follows.
    \begin{equation}
    \label{app:eq_lin_grad}
    \begin{gathered}
        \tilde{\mathrm{\mathbf{h}}}(\boldsymbol{r})
        =
        {\mathrm{\mathbf{h}}}(\boldsymbol{r}_0)
        +\mathrm{\mathbf{J}}_{\mathrm{\mathbf{h}}}(\boldsymbol{r}_0)(\boldsymbol{r}-\boldsymbol{r}_0)
    \end{gathered}
    \end{equation}
    
        The Jacobian $\mathrm{\mathbf{J}}_{\mathrm{\mathbf{h}}}(\boldsymbol{r})\in\mathbb{R}^{K\times K}$ is defined as follows.
    \begin{equation}
    \label{app:eq_lin_grad_jacobian}
    \begin{split}
        \mathrm{\mathbf{J}}_{\mathrm{\mathbf{h}}}(\boldsymbol{r})
        &=
        \left[
        \frac{\partial\mathrm{\mathbf{h}}}{\partial r_1},
        ...,
        \frac{\partial\mathrm{\mathbf{h}}}{\partial r_K}
        \right] \\
        &=\left[
        \begin{array}{ccc}
        \frac{\partial h_1}{\partial r_1} & \cdots & \frac{\partial h_1}{\partial r_K} \\
        \vdots & \ddots & \vdots \\
        \frac{\partial h_K}{\partial r_1} & \cdots & \frac{\partial h_K}{\partial r_K}
        \end{array}
        \right]\\
        &=
        \left[
        \begin{array}{cccc}
        -\frac{1}{1 +\langle
        \boldsymbol{p},
        \boldsymbol{r}
        \rangle} + \frac{p_1r_1}{(1 +\langle
        \boldsymbol{p},
        \boldsymbol{r}
        \rangle)^2} & 
        \frac{p_2 r_1}{(1 +\langle
        \boldsymbol{p},
        \boldsymbol{r}
        \rangle)^2} & 
        \cdots &
        \frac{p_K r_1}{(1 +\langle
        \boldsymbol{p},
        \boldsymbol{r}
        \rangle)^2} \\
        \frac{p_1r_2}{(1 +\langle
        \boldsymbol{p},
        \boldsymbol{r}
        \rangle)^2} & 
        -\frac{1}{1 +\langle
        \boldsymbol{p},
        \boldsymbol{r}
        \rangle} + \frac{p_2r_2}{(1 +\langle
        \boldsymbol{p},
        \boldsymbol{r}
        \rangle)^2} & 
        \cdots &
        \frac{p_K r_2}{(1 +\langle
        \boldsymbol{p},
        \boldsymbol{r}
        \rangle)^2} \\
        \vdots & 
        \vdots & 
        \ddots & 
        \vdots \\
        \frac{p_1r_K}{(1 +\langle
        \boldsymbol{p},
        \boldsymbol{r}
        \rangle)^2} & 
        \frac{p_2r_K}{(1 +\langle
        \boldsymbol{p},
        \boldsymbol{r}
        \rangle)^2} & 
        \cdots &
        -\frac{1}{1 +\langle
        \boldsymbol{p},
        \boldsymbol{r}
        \rangle} + \frac{p_K r_K}{(1 +\langle
        \boldsymbol{p},
        \boldsymbol{r}
        \rangle)^2} 
        \end{array}
        \right] \\
        &=
        -\frac{1}{1 +\langle
        \boldsymbol{p},
        \boldsymbol{r}
        \rangle} \boldsymbol{I}_K
        +
        \frac{1}{(1 +\langle
        \boldsymbol{p},
        \boldsymbol{r}
        \rangle)^2} \boldsymbol{r}
        \boldsymbol{p}^\top
    \end{split}
    \end{equation}
    
    Plugging (\ref{app:eq_lin_grad_jacobian}) into (\ref{app:eq_lin_grad}) with respect to arbitrary reference $\boldsymbol{r}_0$,
    we have a linearized gradient of a decision loss as follows.
    \begin{equation}
    \begin{gathered}
    \label{eq:lin_grad_formula}
        \tilde{\boldsymbol{g}}
        \triangleq
        \tilde{\mathrm{\mathbf{h}}}(\boldsymbol{r})
        =
        -\frac{\boldsymbol{r}_0}{1 +\langle
        \boldsymbol{p},
        \boldsymbol{r}_0
        \rangle}
        -\frac{\left(\boldsymbol{r} - \boldsymbol{r}_0\right)}
        {1 +\langle
        \boldsymbol{p},
        \boldsymbol{r}_0
        \rangle}
        +
        \frac{\boldsymbol{r}_0 \boldsymbol{p}^\top (\boldsymbol{r} - \boldsymbol{r}_0)}{(1 +\langle
        \boldsymbol{p},
        \boldsymbol{r}_0
        \rangle)^2}\\
        =
        -\frac{\boldsymbol{r}}{1 +\langle
        \boldsymbol{p},
        \boldsymbol{r}_0
        \rangle}
        +
        \frac{\boldsymbol{r}_0 \boldsymbol{p}^\top (\boldsymbol{r} - \boldsymbol{r}_0)}{(1 +\langle
        \boldsymbol{p},
        \boldsymbol{r}_0
        \rangle)^2}.
    \end{gathered}
    \end{equation}

    From the statement of Lemma~\ref{lemma:linearized_grad}, 
    plugging the doubly robust estimator of the partially observed response, $\breve{\boldsymbol{r}}$ from Lemma~\ref{lemma:unbiased_resp} into above, we have gradient estimate $\breve{\boldsymbol{g}}$ as follows.
    \begin{equation}
    \begin{gathered}
        \breve{\boldsymbol{g}}
        =
        \tilde{\mathrm{\mathbf{h}}}(\breve{\boldsymbol{r}})
        =
        -\frac{{\breve{\boldsymbol{r}}}}
        {1 +\langle
        \boldsymbol{p},
        \boldsymbol{r}_0
        \rangle}
        +
        \frac{\boldsymbol{r}_0 \boldsymbol{p}^\top (\breve{\boldsymbol{r}} - \boldsymbol{r}_0)}{(1 +\langle
        \boldsymbol{p},
        \boldsymbol{r}_0
        \rangle)^2}.
    \end{gathered}
    \end{equation}
    
    Taking an expectation, we have
    \begin{equation}
    \begin{gathered}
        \mathbb{E}\left[ \breve{\boldsymbol{g}} \right]
        =
        \mathbb{E}\left[ \tilde{\mathrm{\mathbf{h}}}(\breve{\boldsymbol{r}}) \right]
        =
        -\frac{ \mathbb{E}\left[ \breve{\boldsymbol{r}} \right] }
        {1 +\langle
        \boldsymbol{p},
        \boldsymbol{r}_0
        \rangle}
        +
        \frac{
        \boldsymbol{r}_0 \boldsymbol{p}^\top 
        ( \mathbb{E}\left[ \breve{\boldsymbol{r}} \right] - \boldsymbol{r}_0  ) 
        }
        { (1 +\langle \boldsymbol{p},  \boldsymbol{r}_0 \rangle)^2 }
        =
        -\frac{ \boldsymbol{r} }
        {1 +\langle
        \boldsymbol{p},
        \boldsymbol{r}_0
        \rangle}
        +
        \frac{
        \boldsymbol{r}_0 \boldsymbol{p}^\top 
        ( \boldsymbol{r} - \boldsymbol{r}_0 ) 
        }
        { (1 +\langle \boldsymbol{p},  \boldsymbol{r}_0 \rangle)^2 }\\
        = \tilde{\mathrm{\mathbf{h}}}({\boldsymbol{r}}) 
        = \tilde{\boldsymbol{g}}
        \approx 
        \boldsymbol{g}.
    \end{gathered}
    \end{equation}
    \end{proof}

\newpage
\subsection{Supportive Tools for the Proof of Main Theorems}
\subsubsection{Lipschitz Continuity of Linearly Approximated Gradient from Doubly Robust Estimator} 
    \begin{lemma}
    \label{lemma:lipschitz_lin_grad}
        Denote $\breve{\boldsymbol{g}}^{(t)}$ as the linearized gradient calculated from the doubly robust estimator of a response vector, $\boldsymbol{r}^{(t)}$, with reference $\boldsymbol{r}_0^{(t)} = \bar{\boldsymbol{r}} = \mathrm{\bar{r}}^{(t)}\boldsymbol{1}_K$
            where $\bar{\mathrm{r}}^{(t)}=\frac{1}{\vert S^{(t)} \vert} \sum_{i \in S^{(t)}} r_i^{(t)}$.
        When $S^{(t)}$ is a randomly selected client indices in round $t$ and $C=P\left(i \in S^{(t)}\right)$ is a client sampling probability, 
        then $\left\Vert \breve{\boldsymbol{g}}^{(t)} \right\Vert_\infty \leq \breve{L}_\infty 
        = \frac{C_2}{1+C_1} + \frac{2(C_2-C_1)}{C(1+C_1)}$ for $r_i^{(t)}\in[C_1, C_2], \forall i\in S^{(t)}$.
    \end{lemma}
    
    \begin{proof}
        Note that we intentionally omit superscript $^{(t)}$ from now on for the brevity of notation.
        The linearized gradient constructed from the doubly robust estimator of a response vector has a form as follows,
         according to (\ref{eq:lin_grad_formula}).
    \begin{equation}
    \begin{gathered}
        \breve{\boldsymbol{g}}
        =
        -\frac{ \breve{\boldsymbol{r}} }
        {1 +\langle
        \boldsymbol{p},
        \bar{\boldsymbol{r}}
        \rangle}
        +
        \frac{
        \bar{\boldsymbol{r}} \boldsymbol{p}^\top 
        ( \breve{\boldsymbol{r}} - \bar{\boldsymbol{r}} ) 
        }
        { (1 +\langle \boldsymbol{p},  \bar{\boldsymbol{r}} \rangle)^2 },
    \end{gathered}
    \end{equation}
        where we used $\bar{\boldsymbol{r}}=\bar{\mathrm{r}}\boldsymbol{1}_K$ as a reference $\boldsymbol{r}_0$,
        therefore $\left\Vert \bar{\boldsymbol{r}} \right\Vert_\infty=\bar{\mathrm{r}}\leq C_2$.
    
    Thus, we have
    {\allowdisplaybreaks
    \begin{align}
    \label{app:lin_dr_grad}
        &\left\Vert \breve{\boldsymbol{g}} \right\Vert_\infty \nonumber\\
        &=
        \left\Vert
        -\frac{ \breve{\boldsymbol{r}} }
        {1 +\langle
        \boldsymbol{p},
        \bar{\boldsymbol{r}}
        \rangle}
        +
        \frac{
        \bar{\boldsymbol{r}} \boldsymbol{p}^\top 
        ( \breve{\boldsymbol{r}} - \bar{\boldsymbol{r}} ) 
        }
        { (1 +\langle \boldsymbol{p},  \bar{\boldsymbol{r}} \rangle)^2 }
        \right\Vert_\infty \nonumber\\
        &\leq
        \left\Vert
        -\frac{ \breve{\boldsymbol{r}} }
        {1 +\langle
        \boldsymbol{p},
        \bar{\boldsymbol{r}}
        \rangle}
        \right\Vert_\infty
        +
        \left\Vert
        \frac{
        \bar{\boldsymbol{r}} \boldsymbol{p}^\top 
        ( \breve{\boldsymbol{r}} - \bar{\boldsymbol{r}} ) 
        }
        { (1 +\langle \boldsymbol{p},  \bar{\boldsymbol{r}} \rangle)^2 }
        \right\Vert_\infty \nonumber\\
        &=
        \frac{ \left\Vert \breve{\boldsymbol{r}} \right\Vert_\infty }
        {1 +\langle
        \boldsymbol{p},
        \bar{\boldsymbol{r}}
        \rangle}
        +
        \frac{ 1 }
        { (1 +\langle \boldsymbol{p},  \bar{\boldsymbol{r}} \rangle)^2 }
        \left\Vert
        \bar{\boldsymbol{r}} \boldsymbol{p}^\top 
        ( \breve{\boldsymbol{r}} - \bar{\boldsymbol{r}} ) 
        \right\Vert_\infty \nonumber\\
        &\leq
        \frac{ \left\Vert \breve{\boldsymbol{r}} \right\Vert_\infty }
        {1 +\langle
        \boldsymbol{p},
        \bar{\boldsymbol{r}}
        \rangle}
        +
        \frac{ 1 }
        { (1 +\langle \boldsymbol{p},  \bar{\boldsymbol{r}} \rangle)^2 }
        \left\Vert
        \bar{\boldsymbol{r}} \boldsymbol{p}^\top
        \right\Vert_\infty
        \left\Vert
        \breve{\boldsymbol{r}} - \bar{\boldsymbol{r}} 
        \right\Vert_\infty \nonumber\\
        &=
        \frac{ \left\Vert \breve{\boldsymbol{r}} \right\Vert_\infty }
        {1 +\langle
        \boldsymbol{p},
        \bar{\boldsymbol{r}}
        \rangle}
        +
        \frac{ 1 }
        { (1 +\langle \boldsymbol{p},  \bar{\boldsymbol{r}} \rangle)^2 }
        \left\Vert
        \bar{\boldsymbol{r}}
        \right\Vert_\infty
        \left\Vert
        \boldsymbol{p}
        \right\Vert_1
        \left\Vert
        \breve{\boldsymbol{r}} - \bar{\boldsymbol{r}} 
        \right\Vert_\infty \nonumber\\
        &=
        \frac{ \left\Vert \breve{\boldsymbol{r}} \right\Vert_\infty }
        {1 +\langle
        \boldsymbol{p},
        \bar{\boldsymbol{r}}
        \rangle}
        +
        \frac{ 1 }
        { (1 +\langle \boldsymbol{p},  \bar{\boldsymbol{r}} \rangle)^2 }
        \left\Vert
        \bar{\boldsymbol{r}}
        \right\Vert_\infty
        \left\Vert
        \breve{\boldsymbol{r}} - \bar{\boldsymbol{r}} 
        \right\Vert_\infty \nonumber
    \end{align}
    }
    , where the first inequality is due to triangle inequality, 
    the second inequality is due to the property that
    $\left\Vert \boldsymbol{Ax} \right\Vert_\infty \leq  \left\Vert \boldsymbol{A} \right\Vert_\infty  \left\Vert \boldsymbol{x} \right\Vert_\infty$ 
    for a matrix $\boldsymbol{A}\in\mathbb{R}^{K \times K}$ and a vector $\boldsymbol{x} \in\mathbb{R}^{K}, 
    \boldsymbol{x}\neq\boldsymbol{0}_K$,
    the very next equality is due to $\vert \boldsymbol{x} \boldsymbol{y}^\top\vert_\infty
    =\max_i \Vert {x}_i \boldsymbol{y}^\top\Vert_1 
    =\max_i \vert {x}_i \vert \Vert \boldsymbol{y} \Vert_1
    =\Vert \boldsymbol{x} \Vert_\infty \Vert \boldsymbol{y} \Vert_1$,
    and the last equality is trivial since $\boldsymbol{p}\in\Delta_{K-1}$.

    Since 
    $\langle \boldsymbol{p}, \bar{\boldsymbol{r}} \rangle 
    = \sum_{i=1}^K \left( p_i \bar{\mathrm{r}} \right) = \bar{\mathrm{r}}$,
    this can be further bounded as follows.
    \begin{equation}
    \begin{split}
        &=
        \frac{ 1 }{1 + \bar{\mathrm{r}}} \left\Vert \breve{\boldsymbol{r}} \right\Vert_\infty
        +
        \frac{ \bar{\mathrm{r}} }{ \left( 1 + \bar{\mathrm{r}} \right)^2 }
        \left\Vert
        \breve{\boldsymbol{r}} - \bar{\boldsymbol{r}} 
        \right\Vert_\infty \\
        &=
        \frac{ 1 + \bar{\mathrm{r}} }{ \left( 1 + \bar{\mathrm{r}} \right)^2 }
        \left\Vert \breve{\boldsymbol{r}} \right\Vert_\infty
        +
        \frac{ \bar{\mathrm{r}} }{ \left( 1 + \bar{\mathrm{r}} \right)^2 }
        \left\Vert
        \breve{\boldsymbol{r}} - \bar{\boldsymbol{r}} 
        \right\Vert_\infty \\
        &\leq
        \frac{ 1 }{ 1 + \bar{\mathrm{r}}}
        \left(
        \left\Vert \breve{\boldsymbol{r}} \right\Vert_\infty
        +
        \left\Vert
        \breve{\boldsymbol{r}} - \bar{\boldsymbol{r}} 
        \right\Vert_\infty
        \right),
    \end{split}
    \end{equation}
    
    Since $\frac{ 1 }{1 + \bar{\mathrm{r}}}\leq\frac{1}{1+C_1}$,
    we can further upper bound as follows.
    \begin{equation}
    \begin{gathered}
    \label{eq:lipschitz_two_terms}
        \frac{ 1 }{ 1 + \bar{\mathrm{r}}}
        \left(
        \left\Vert \breve{\boldsymbol{r}} \right\Vert_\infty
        +
        \left\Vert
        \breve{\boldsymbol{r}} - \bar{\boldsymbol{r}} 
        \right\Vert_\infty
        \right)
        \leq
        \frac{ 1 }{ 1 + C_1 }
        \left(
        \left\Vert \breve{\boldsymbol{r}} \right\Vert_\infty
        +
        \left\Vert
        \breve{\boldsymbol{r}} - \bar{\boldsymbol{r}} 
        \right\Vert_\infty
        \right)
    \end{gathered}
    \end{equation}

    To upper bound each term, let us look into $\breve{\boldsymbol{r}}$ first.
    By the definition in (\ref{eq:dr_response}), we have
    \begin{equation}
    \begin{split}
        &\breve{\boldsymbol{r}}
        =
        \begin{cases}
        \bar{\mathrm{r}} \boldsymbol{1}_K, & i\notin S^{(t)} \\
        \left( 1 - \frac{1}{C} \right) \bar{\mathrm{r}}\boldsymbol{1}_K + \frac{1}{C} \boldsymbol{r}, & i\in S^{(t)}
        \end{cases}.
    \end{split}
    \end{equation}
    
    For each case, $\left\Vert \breve{\boldsymbol{r}} \right\Vert_\infty$ becomes
    \begin{equation}
    \label{app:brev_inf_norm}
    \begin{split}
        &\left\Vert \breve{\boldsymbol{r}} \right\Vert_\infty
        =
        \begin{cases}
        \bar{\mathrm{r}}, & i\notin S^{(t)} \\
        \sup_i \left\vert \frac{1}{C} (r_i - \bar{\mathrm{r}}) + \bar{\mathrm{r}} \right\vert, & i\in S^{(t)}.
        \end{cases}
    \end{split}
    \end{equation}
    
    For the first case, the average is smaller than its maximum, thus $\bar{\mathrm{r}} \leq C_2$.
    For the second case, it can be upper bounded as 
    $\sup_i \left\vert \frac{1}{C} (r_i - \bar{\mathrm{r}}) + \bar{\mathrm{r}} \right\vert
    \leq 
    \frac{1}{C} \sup_i \left\vert r_i - \bar{\mathrm{r}} \right\vert
    + C_2$ by the triangle inequality.
    
    From the trivial fact that the deviation from the average is always smaller than its range,
    \begin{equation}
    \begin{gathered}
        \frac{1}{C} \sup_i \left\vert r_i - \bar{\mathrm{r}} \right\vert
        \leq
        \frac{1}{C} (C_2 - C_1).
    \end{gathered}
    \end{equation}
    
    Combined, we have the following upper bounds.
    \begin{equation}
    \begin{split}
    \label{eq:lipschitz_first}
        \left\Vert \breve{\boldsymbol{r}} \right\Vert_\infty
        \leq
        \begin{cases}
        C_2, & i\notin S^{(t)} \\
        \frac{C_2 - C_1}{C} + C_2, & i\in S^{(t)}
        \end{cases}
    \end{split}
    \end{equation}

    Similarly, for the second term inside in (\ref{eq:lipschitz_two_terms}), we have:
    \begin{equation}
    \label{app:breve_minus_bar_inf_norm}
    \begin{split}
        &\left\Vert \breve{\boldsymbol{r}} - \bar{\boldsymbol{r}} \right\Vert_\infty
        =
        \begin{cases}
        0, & i\notin S^{(t)} \\
        \frac{1}{C} \sup_i \left\vert r_i - \bar{\mathrm{r}} \right\vert, & i\in S^{(t)}.
        \end{cases}
    \end{split}
    \end{equation}
    
    Corresponding upper bounds are:
    \begin{equation}
    \begin{gathered}
    \label{eq:lipschitz_second}
        \left\Vert
        \breve{\boldsymbol{r}} - \bar{\boldsymbol{r}} 
        \right\Vert_\infty
        \leq
        \begin{cases}
        0, & i\notin S^{(t)} \\
        \frac{C_2 - C_1}{C}, & i\in S^{(t)}
        \end{cases}
    \end{gathered}
    \end{equation}
    
    Adding (\ref{eq:lipschitz_first}) and (\ref{eq:lipschitz_second}) to have (\ref{eq:lipschitz_two_terms}), we have:
    \begin{equation}
    \begin{gathered}
        \left\Vert \breve{\boldsymbol{g}}^{(t)} \right\Vert_\infty
        \leq
        \begin{cases}
        \frac{C_2}{1+C_1}, & i\notin S^{(t)} \\
        \frac{C_2}{1+C_1} + \frac{2(C_2-C_1)}{C(1+C_1)}, & i\in S^{(t)}.
        \end{cases}
    \end{gathered}    
    \end{equation}
    
    Finally, it suffices to set $\breve{L}_\infty = \frac{C_2}{1+C_1} + \frac{2(C_2-C_1)}{C(1+C_1)}$ to conclude the proof.
    \end{proof}

\newpage
\subsubsection{Regret from Linearized Loss}
    \begin{corollary}
    \label{corollary:lin_loss}
    From the convexity of a decision loss $\ell^{(t)}$ (Lemma~\ref{lemma:convexity}), 
    the regret defined in (\ref{eq:regret}) satisfies
    \begin{equation}
    \begin{gathered}
        \normalfont\text{Regret}^{(T)}\left(\boldsymbol{p}^{\star}\right) 
        = 
        \sum_{t=1}^T \ell^{(t)}\left(\boldsymbol{p}^{(t)}\right) 
        - 
        \sum_{t=1}^T \ell^{(t)}\left(\boldsymbol{p}^{\star}\right)
        \leq
        \sum_{t=1}^T \tilde\ell^{(t)}\left(\boldsymbol{p}^{(t)}\right) 
        - 
        \sum_{t=1}^T \tilde\ell^{(t)}\left(\boldsymbol{p}^{\star}\right),
    \end{gathered}
    \end{equation}
    where $\tilde\ell^{(t)}$ is a linearized loss defined as $\tilde\ell^{(t)}\left(\boldsymbol{p}\right) 
    =\left\langle \boldsymbol{p}, \boldsymbol{g}^{(t)} \right\rangle$
    and $\boldsymbol{g}^{(t)}=\nabla\ell^{(t)}\left(\boldsymbol{p}^{(t)}\right)$. 
    \end{corollary}
    
    \begin{proof}
    It is straightforward from the convexity of the decision loss.
    \begin{equation}
    \begin{gathered}
        \ell^{(t)}(\boldsymbol{p}^{(t)})
        -
        \ell^{(t)}(\boldsymbol{p}^{\star})
        \leq
        \left\langle
        \boldsymbol{g}^{(t)},
        \boldsymbol{p}^{(t)}
        -
        \boldsymbol{p}^{\star}
        \right\rangle.
    \end{gathered}
    \end{equation}
    Summing up both sides for $t\in[T]$, we proved the statement.
    \end{proof}

\newpage
\subsubsection{Equality for the Regret}
    \begin{lemma} (Lemma 7.1 of \cite{oco2}; Lemma 5 of \cite{oco3})
    \label{lemma:regret_bound_default_form}
    Let us define 
    $L^{(t)}\left(\boldsymbol{p}\right)\triangleq\sum_{\tau=1}^{t-1} \ell^{(\tau)} (\boldsymbol{p}) + R^{(t)}(\boldsymbol{p})$,
    where
    $\ell:\Delta_{K-1}\times\mathbb{R}^d\rightarrow\mathbb{R}$ is an arbitrary loss function and
    $R^{(t)}:\Delta_{K-1}\rightarrow\mathbb{R}$ is an arbitrary regularizer, non-decreasing across $t\in[T]$. 
    Assume further that $\boldsymbol{p}^{(t)} = \argmin_{\boldsymbol{p}\in\Delta_{K-1}} L^{(t)} \left( \boldsymbol{p} \right)$. 
    Then, for any $\boldsymbol{p}^\star\in\Delta_{K-1}$, we have
    \begin{equation}
    \begin{split}
        &\normalfont\text{Regret}^{(T)}(\boldsymbol{p}^{\star})
        =
        \sum_{t=1}^T 
        \ell^{(t)}\left(\boldsymbol{p}^{(t)}\right) 
        - 
        \sum_{t=1}^T \ell^{(t)}\left(\boldsymbol{p}^{\star}\right) \\
        &=
        R^{(T+1)}\left(\boldsymbol{p}^{\star}\right)
        -R^{(1)}\left(\boldsymbol{p}^{(1)}\right)
        +L^{(T+1)}\left(\boldsymbol{p}^{(T+1)}\right)
        -L^{(T+1)}\left({\boldsymbol{p}^{\star}}\right) \\
        &+
        \sum_{t=1}^T
        \left[
        L^{(t)}\left(\boldsymbol{p}^{(t)}\right)
        -
        L^{(t+1)}\left(\boldsymbol{p}^{(t+1)}\right)
        +
        \ell^{(t)}\left(\boldsymbol{p}^{(t)}\right) \right].
    \end{split}
    \end{equation}
    \end{lemma}
    
    \begin{proof}
    Since $\sum_{t=1}^T \ell^{(t)}\left(\boldsymbol{p}^{(t)}\right)$ appears in both sides, 
    we only need to show that
    \begin{equation}
    \begin{split}
        &-\sum_{t=1}^T \ell^{(t)}(\boldsymbol{p}^\star) \\
        &=
        R^{(T+1)}\left(\boldsymbol{p}^{\star}\right)
        -R^{(1)}\left(\boldsymbol{p}^{(1)}\right) \\
        &+L^{(T+1)}\left(\boldsymbol{p}^{(T+1)}\right)
        -L^{(T+1)}\left({\boldsymbol{p}^{\star}}\right) \\
        &+\sum_{t=1}^T
        \left[
        L^{(t)}\left(\boldsymbol{p}^{(t)}\right)
        -
        L^{(t+1)}\left(\boldsymbol{p}^{(t+1)}\right) \right]
    \end{split}
    \end{equation}
    
    First, note that
    \begin{equation}
    \begin{split}
        &\ell^{(t)}\left(\boldsymbol{p}^{\star}\right) \\ 
        &=
        \sum_{\tau=1}^t \ell^{(\tau)}\left(\boldsymbol{p}^{\star}\right)
        -
        \sum_{\tau=1}^{t-1} \ell^{(\tau)}\left(\boldsymbol{p}^{\star}\right) \\
        &=
        \left( 
        L^{(t+1)}\left({\boldsymbol{p}^{\star}}\right)
        -
        R^{(t+1)}\left({\boldsymbol{p}^{\star}}\right)
        \right)  \\
        &-
        \left(
        L^{(t)}\left({\boldsymbol{p}^{\star}}\right)
        -
        R^{(t)}\left({\boldsymbol{p}^{\star}}\right)
        \right).  
    \end{split}
    \end{equation}
    
    Summing up the right-hand side of the above from $t=1,...,T$, we have 
    \begin{equation}
    \begin{split}
        &\sum_{t=1}^T \ell^{(t)}\left(\boldsymbol{p}^{\star}\right) \\
        &=
        \left( 
        L^{(T+1)}\left({\boldsymbol{p}^{\star}}\right)
        -
        R^{(T+1)}\left({\boldsymbol{p}^{\star}}\right)
        \right) \\
        &-
        \left(
        L^{(1)}\left({\boldsymbol{p}^{\star}}\right)
        -
        R^{(1)}\left({\boldsymbol{p}^{\star}}\right)
        \right) \\
        &=
        L^{(T+1)}\left({\boldsymbol{p}^{\star}}\right)
        -
        R^{(T+1)}\left({\boldsymbol{p}^{\star}}\right),
    \end{split}
    \end{equation}
    by telescoping summation.
    
    Thus,
    \begin{equation}
    \begin{gathered}
    \label{app:first}
        -\sum_{t=1}^T \ell^{(t)}\left(\boldsymbol{p}^{\star}\right)
        =
        R^{(T+1)}(\boldsymbol{p}^{\star})
        -
        L^{(T+1)}(\boldsymbol{p}^{\star}).
    \end{gathered}
    \end{equation}
    
    Similarly, 
    \begin{equation}
    \begin{split}
        &\sum_{t=1}^T 
        \left[
        L^{(t)}\left(\boldsymbol{p}^{(t)}\right)
        -
        L^{(t+1)}\left(\boldsymbol{p}^{(t+1)}\right)
        \right] \\
        &=
        L^{(1)}\left(\boldsymbol{p}^{(1)}\right)
        -
        L^{(T+1)}\left(\boldsymbol{p}^{(T+1)}\right) \\
        &=
        R^{(1)}\left(\boldsymbol{p}^{(1)}\right)
        -
        L^{(T+1)}\left(\boldsymbol{p}^{(T+1)}\right).
    \end{split}
    \end{equation}
    
    Rearranging, 
    \begin{equation}
    \begin{gathered}
    \label{app:second}
        0
        =
        L^{(T+1)}\left(\boldsymbol{p}^{(T+1)}\right)
        -
        R^{(1)}\left(\boldsymbol{p}^{(1)}\right)
        +
        \sum_{t=1}^T \left[
        L^{(t)}\left(\boldsymbol{p}^{(t)}\right)
        -
        L^{(t+1)}\left(\boldsymbol{p}^{(t+1)}\right)
        \right].
    \end{gathered}
    \end{equation}
    
    Adding (\ref{app:first}) and (\ref{app:second}), we have
    \begin{equation}
    \begin{split}
        &- 
        \sum_{t=1}^T \ell^{(t)}\left(\boldsymbol{p}^{\star}\right) \\
        &=
        R^{(T+1)}\left(\boldsymbol{p}^{\star}\right)
        -R^{(1)}\left(\boldsymbol{p}^{(1)}\right) \\
        &+L^{(T+1)}\left(\boldsymbol{p}^{(T+1)}\right)
        -L^{(T+1)}\left({\boldsymbol{p}^{\star}}\right) \\
        &+ \sum_{t=1}^T
        \left[
        L^{(t)}\left(\boldsymbol{p}^{(t)}\right)
        -
        L^{(t+1)}\left(\boldsymbol{p}^{(t+1)}\right)
        \right].
    \end{split}
    \end{equation}
    
    Finally, by adding $\sum_{t=1}^T \ell^{(t)}\left(\boldsymbol{p}^{(t)}\right)$ to both sides, we prove the statement.
    Note that the left hand side of the main statement does not depend on $R^{(t)}$, thus we can replace $R^{(T+1)}\left(\boldsymbol{p}^{\star}\right)$ 
    into $R^{(T)}\left(\boldsymbol{p}^{\star}\right)$. (Remark 7.3 of \cite{oco3})
    \end{proof}

\newpage
\subsubsection{Upper Bound to the Suboptimality Gap}
    \begin{lemma} (Oracle Gap; adapted from Corollary 7.7 of \cite{oco3})
    \label{lemma:cor_orabona}
        Let $f:\mathbb{R}^K \rightarrow \mathbb{R}$ be a $\mu$-strongly convex w.r.t. $\Vert\cdot\Vert$ over its domain.
        Let $\boldsymbol{x}^\star=\argmin_{\boldsymbol{x}}f(\boldsymbol{x})$.
        Then, for all $\boldsymbol{x}\in\operatorname{dom}\partial f(\boldsymbol{x})$, and $\boldsymbol{g}\in\partial f(\boldsymbol{x})$, we have:
    \begin{equation}
    \begin{gathered}
        f(\boldsymbol{x}) - f(\boldsymbol{x}^\star) \leq \frac{1}{2 \mu} \Vert \boldsymbol{g} \Vert_\star^2,
    \end{gathered}
    \end{equation}
        where $\Vert\cdot\Vert_\star$ is a dual norm of $\Vert\cdot\Vert$.
    \end{lemma}

\newpage
\subsubsection{Progress Bound}
    \begin{lemma} 
    \label{lemma:one_step_bound_proximal} (Progress Bound of FTRL with Proximal Regularizer)
    With a slight abuse of notation, 
    assume $L^{(t)}$ is closed, subdifferentiable and $\mu^{(t)}$-strongly convex w.r.t. 
    $\Vert\cdot\Vert$ in $\Delta_{K-1}$.
    First assume that $\boldsymbol{p}^{(t+1)} = \argmin_{\boldsymbol{p}\in\Delta_{K-1}}L^{(t+1)}\left(\boldsymbol{p}\right)$.
    Assume further that the regularizer satisfies
    $\boldsymbol{p}^{(t)} 
    = \argmin_{\boldsymbol{p}\in\Delta_{K-1}}
    \left(R^{(t+1)}\left(\boldsymbol{p}\right) - R^{(t)}\left(\boldsymbol{p}\right)\right)$,
    and $\boldsymbol{g}^{(t)}\in\partial L^{(t+1)}(\boldsymbol{p}^{(t)})$.
    Then, we have the following inequality:
    \begin{equation}
    \begin{gathered}
        L^{(t)}\left(\boldsymbol{p}^{(t)}\right) 
        -
        L^{(t+1)}\left(\boldsymbol{p}^{(t+1)}\right)
        +
        \ell^{(t)}\left(\boldsymbol{p}^{(t)}\right)
        \leq
        \frac{ \left\Vert \boldsymbol{g}^{(t)} \right\Vert_\star^2 }
        { 2\mu^{(t+1)} }
        +
        \left(
        R^{(t)}\left(\boldsymbol{p}^{(t)}\right)
        -
        R^{(t+1)}\left(\boldsymbol{p}^{(t)}\right)
        \right),
    \end{gathered}
    \end{equation}
    where $\Vert\cdot\Vert_\star$ is a dual norm of $\Vert\cdot\Vert$.
    \end{lemma}

    \begin{proof}
    \begin{equation}
    \begin{split}
        &L^{(t)}\left(\boldsymbol{p}^{(t)}\right)
        -
        L^{(t+1)}\left(\boldsymbol{p}^{(t+1)}\right)
        +
        \ell^{(t)}\left(\boldsymbol{p}^{(t)}\right)\\
        &=
        \left(
        L^{(t)}\left(\boldsymbol{p}^{(t)}\right)
        +
        \ell^{(t)}\left(\boldsymbol{p}^{(t)}\right)
        +
        R^{(t+1)}\left(\boldsymbol{p}^{(t)}\right)
        -
        R^{(t)}\left(\boldsymbol{p}^{(t)}\right)
        \right) \\
        &-
        L^{(t+1)}\left(\boldsymbol{p}^{(t+1)}\right)
        -
        R^{(t+1)}\left(\boldsymbol{p}^{(t)}\right)
        +
        R^{(t)}\left(\boldsymbol{p}^{(t)}\right) \\
        &=
        L^{(t+1)}\left(\boldsymbol{p}^{(t)}\right)
        -
        L^{(t+1)}\left(\boldsymbol{p}^{(t+1)}\right)
        -
        R^{(t+1)}\left(\boldsymbol{p}^{(t)}\right)
        +
        R^{(t)}\left(\boldsymbol{p}^{(t)}\right) \\
        &\leq
        \frac{ \left\Vert \boldsymbol{g}^{(t)} \right\Vert_\star^2 }
        { 2\mu^{(t+1)} }
        -
        R^{(t+1)}\left(\boldsymbol{p}^{(t)}\right)
        +
        R^{(t)}\left(\boldsymbol{p}^{(t)}\right),
    \end{split}
    \end{equation}
    where the first inequality is due the assumption that $\boldsymbol{p}^{(t+1)} = \argmin_{\boldsymbol{p}\in\Delta_{K-1}}L^{(t+1)}\left(\boldsymbol{p}\right)$, 
    $\boldsymbol{g}^{(t)}\in\partial L^{(t+1)}(\boldsymbol{p}^{(t)})$, and the result from Lemma~\ref{lemma:cor_orabona}.
    See also Lemma 7.23 of \cite{oco3}.
    \end{proof}

    \begin{lemma} 
    \label{lemma:one_step_bound} (Progress Bound of FTRL with Non-Decreasing Regularizer)
    With a slight abuse of notation, 
    assume $L^{(t)}$ to be closed and sub-differentiable in $\Delta_K$, 
    and $\left(L^{(t)} + \ell^{(t)}\right)$ to be closed, differentiable and $\nu^{(t)}$-strongly convex w.r.t. 
    $\Vert\cdot\Vert_1$ in $\Delta_{K-1}$. 
    Further define with an abuse of notation again that $\boldsymbol{g}^{(t)}=\nabla \ell^{(t)}(\boldsymbol{p}^{(t)})\in\partial\left(L^{(t)}+\ell^{(t)}\right)\left(\boldsymbol{p}^{(t)}\right)$,
    and define further that $\boldsymbol{p}^{(t)} = \argmin_{\boldsymbol{p}\in\Delta_{K-1}} L^{(t)}\left(\boldsymbol{p}\right)$. 
    
    Then, we have the following inequality:
    \begin{equation}
    \begin{gathered}
        L^{(t)}\left(\boldsymbol{p}^{(t)}\right) 
        -
        L^{(t+1)}\left(\boldsymbol{p}^{(t+1)}\right)
        +
        \ell^{(t)}\left(\boldsymbol{p}^{(t)}\right)
        \leq
        \frac{ \left\Vert \boldsymbol{g}^{(t)} \right\Vert_\infty^2 }
        { 2\nu^{(t)} }
        +
        \left(
        R^{(t)}\left(\boldsymbol{p}^{(t+1)}\right)
        -
        R^{(t+1)}\left(\boldsymbol{p}^{(t+1)}\right)
        \right).
    \end{gathered}
    \end{equation}
    \end{lemma}
    
\begin{proof}
    Let us first assume that $\boldsymbol{p}^{(t)}_* = \argmin_{\boldsymbol{p}\in\Delta_{K-1}}\left(L^{(t)}\left(\boldsymbol{p}\right) + \ell^{(t)}\left(\boldsymbol{p}\right)\right)$.
    Observe that 
    \begin{equation}
    \begin{gathered}
        L^{(t+1)}\left(\boldsymbol{p}^{(t+1)}\right)
        =
        L^{(t)}\left(\boldsymbol{p}^{(t+1)}\right)
        +
        \ell^{(t)}\left(\boldsymbol{p}^{(t+1)}\right)
        -
        R^{(t)}\left(\boldsymbol{p}^{(t+1)}\right)
        +
        R^{(t+1)}\left(\boldsymbol{p}^{(t+1)}\right),
    \end{gathered}
    \end{equation}
    we have
    \begin{equation}
    \begin{split}
        &L^{(t)}\left(\boldsymbol{p}^{(t)}\right)
        -
        L^{(t+1)}\left(\boldsymbol{p}^{(t+1)}\right)
        +
        \ell^{(t)}\left(\boldsymbol{p}^{(t)}\right)\\
        &=
        L^{(t)}\left(\boldsymbol{p}^{(t)}\right)
        -
        \left(
        L^{(t)}\left(\boldsymbol{p}^{(t+1)}\right)
        +
        \ell^{(t)}\left(\boldsymbol{p}^{(t+1)}\right)
        -
        R^{(t)}\left(\boldsymbol{p}^{(t+1)}\right)
        +
        R^{(t+1)}\left(\boldsymbol{p}^{(t+1)}\right)
        \right)
        +
        \ell^{(t)}\left(\boldsymbol{p}^{(t)}\right)\\
        &=
        \left(
        L^{(t)}\left(\boldsymbol{p}^{(t)}\right)
        +
        \ell^{(t)}\left(\boldsymbol{p}^{(t)}\right)
        \right)
        -
        \left(
        L^{(t)}\left(\boldsymbol{p}^{(t+1)}\right)
        +
        \ell^{(t)}\left(\boldsymbol{p}^{(t+1)}\right)
        \right)
        +
        \left(
        R^{(t)}\left(\boldsymbol{p}^{(t+1)}\right)
        -
        R^{(t+1)}\left(\boldsymbol{p}^{(t+1)}\right)
        \right)\\
        &\leq
        \left(
        L^{(t)}\left(\boldsymbol{p}^{(t)}\right)
        +
        \ell^{(t)}\left(\boldsymbol{p}^{(t)}\right)
        \right)
        -
        \left(
         L^{(t)}\left(\boldsymbol{p}^{(t)}_*\right)
        +
        \ell^{(t)}\left(\boldsymbol{p}^{(t)}_*\right)
        \right)
        +
        \left(
        R^{(t)}\left(\boldsymbol{p}^{(t+1)}\right)
        -
        R^{(t+1)}\left(\boldsymbol{p}^{(t+1)}\right)
        \right) \\
        &\leq
        \frac{ \left\Vert \boldsymbol{g}^{(t)} \right\Vert_\infty^2 }{ 2\nu^{(t)} }
        +
        \left(
        R^{(t)}\left(\boldsymbol{p}^{(t+1)}\right)
        -
        R^{(t+1)}\left(\boldsymbol{p}^{(t+1)}\right)
        \right),
    \end{split}
    \end{equation}
    where the first inequality is due the assumption that $\boldsymbol{p}^{(t)}_* = \argmin_{\boldsymbol{p}\in\Delta_{K-1}}\left(L^{(t)}+\ell^{(t)}\right)\left(\boldsymbol{p}\right)$,
    $\boldsymbol{g}^{(t)}\in\partial\left(L^{(t)}+\ell^{(t)}\right)$.
    Lastly, the second inequality is the direct result from Lemma~\ref{lemma:cor_orabona}.
    See also Lemma 7.8 of \cite{oco3}.
    \end{proof}

\newpage
\subsubsection{Exp-concavity}
\begin{definition}
    A function $f:X\rightarrow\mathbb{R}$ is $\gamma$-exp-concave if $\exp\left(-\gamma f\left(\boldsymbol{x}\right)\right)$ is concave for $\boldsymbol{x}\in X$.
\end{definition}

\begin{remark}
    The decision loss defined in (\ref{eq:decision_loss}) is $1$-exp-concave.
\end{remark}

\begin{lemma}
\label{lemma:exp_concave}
    For an $\gamma$-exp-concave function $f:X\rightarrow\mathbb{R}$, let the domain $X$ is bounded, 
    and choose $\delta\leq\frac{\gamma}{2}$ such that $\left\vert 
    \delta 
    \left\langle \nabla f(\boldsymbol{x}), \boldsymbol{y} - \boldsymbol{x}\right\rangle  
    \right\vert \leq \frac{1}{2}$ and for all $\boldsymbol{x}, \boldsymbol{y}\in X$, the following statement holds.
\begin{equation}
\begin{gathered}
    f(\boldsymbol{y}) \geq f(\boldsymbol{x}) 
    + \left\langle \nabla f(\boldsymbol{x}), \boldsymbol{y} - \boldsymbol{x}\right\rangle
    + \frac{\delta}{2} \left( \left\langle \nabla f(\boldsymbol{y}), \boldsymbol{y} - \boldsymbol{x} \right\rangle \right)^2
\end{gathered}
\end{equation}
\end{lemma}

\begin{proof}
    For all such that $\delta\leq\frac{\gamma}{2}$, a function $g(\boldsymbol{x})=\exp\left(-2\delta f(\boldsymbol{x})\right)$ is concave.
    From the concavity of $g$, we have:
\begin{equation}
\begin{gathered}
    g(\boldsymbol{x}) \leq g(\boldsymbol{y}) + \left\langle \nabla g(\boldsymbol{y}), \boldsymbol{x} - \boldsymbol{y}\right\rangle.
\end{gathered}
\end{equation}
    By taking logarithm on both sides, we have:
\begin{equation}
\begin{gathered}
    f(\boldsymbol{x}) \geq f(\boldsymbol{y}) - \frac{1}{2\delta} 
    \log \left(
    1 - 2\delta \left\langle \nabla f(\boldsymbol{y}), \boldsymbol{x} - \boldsymbol{y} \right\rangle
    \right).
\end{gathered}
\end{equation}
    From the assumption, we have 
    $\left\vert 
    \delta 
    \left\langle \nabla f(\boldsymbol{x}), \boldsymbol{y} - \boldsymbol{x}\right\rangle  
    \right\vert \leq \frac{1}{2}$,
    and using the fact that $\log(1+x)\leq x - \frac{x^2}{4}$ for $\vert x \vert \leq 1$,
    we can conclude the proof.
\end{proof}

\begin{remark} (Remark 7.27 from \cite{oco3})
For a positive definite matrix $\boldsymbol{B}\in\mathbb{R}^{K \times K}$, $\Vert \boldsymbol{p} \Vert_{\boldsymbol{B}}$ is a norm induced by $\boldsymbol{B}$, defined as 
$\Vert \boldsymbol{p} \Vert_{\boldsymbol{B}} \triangleq \sqrt{\boldsymbol{p}^\top \boldsymbol{B} \boldsymbol{p}}$
for $\boldsymbol{p}\in\mathbb{R}^K$.
A function $f(\boldsymbol{p})=\frac{1}{2}\boldsymbol{p}^\top \boldsymbol{B} \boldsymbol{p}$ is therefore $1$-strongly convex w.r.t. $\Vert \cdot \Vert_{\boldsymbol{B}}$. 
Note that the dual norm of $\Vert \cdot \Vert_{\boldsymbol{B}}$ is $\Vert \cdot \Vert_{\boldsymbol{B}^{-1}}$.
\end{remark}

\newpage
\subsection{Regret Bound of \texttt{AAggFF-S}: Proof of Theorem~\ref{thm:crosssilo}}
    \begin{remark}
        The regularizer of \texttt{AAggFF-S} (i.e., ONS) is proximal since it has a quadratic form.
    \end{remark}
    
    \begin{proof}
    Since Lemma~\ref{lemma:regret_bound_default_form} holds for arbitrary loss function,
    let us set ${L}^{(t)}\left(\boldsymbol{p}\right)\triangleq\sum_{\tau=1}^{t-1} \tilde\ell^{(\tau)} (\boldsymbol{p}) + \frac{\alpha}{2} \Vert \boldsymbol{p} \Vert_2^2 
    + \frac{\beta}{2} \sum_{\tau=1}^{t-1} 
    \left(\left\langle \boldsymbol{g}^{(\tau)}, \boldsymbol{p} - \boldsymbol{p}^{(\tau)} \right\rangle\right)^2$ 
    as in (\ref{eq:ons}) with a slight abuse of notation.
    Note that we set $R^{(t)}\left(\boldsymbol{p}\right) = \frac{\alpha}{2} \Vert \boldsymbol{p} \Vert_2^2 + 
    \frac{\beta}{2} \sum_{\tau=1}^{t-1} 
    \left(\left\langle \boldsymbol{g}^{(\tau)}, \boldsymbol{p} - \boldsymbol{p}^{(\tau)} \right\rangle\right)^2$, 
    which is often called a proximal regularizer.
    
    From the regret, we have:
    \begin{equation}
    \begin{split}
        &\normalfont\text{Regret}^{(T)}\left(\boldsymbol{p}^{\star}\right) \\
        &= 
        \sum_{t=1}^T \ell^{(t)}\left(\boldsymbol{p}^{(t)}\right) 
        - 
        \sum_{t=1}^T \ell^{(t)}\left(\boldsymbol{p}^{\star}\right) \\
        &\leq
        \sum_{t=1}^T \tilde\ell^{(t)}\left(\boldsymbol{p}^{(t)}\right) 
        - 
        \sum_{t=1}^T \tilde\ell^{(t)}\left(\boldsymbol{p}^{\star}\right) \\
        &=
        \underbrace{
        R^{(T+1)}\left(\boldsymbol{p}^{\star}\right)
        -R^{(1)}\left(\boldsymbol{p}^{(1)}\right)}_\text{(i)}
        +
        \underbrace{
        {L}^{(T+1)}\left(\boldsymbol{p}^{(T+1)}\right)
        -{L}^{(T+1)}\left({\boldsymbol{p}^{\star}}\right)}_\text{(ii)} \\
        &+
        \underbrace{
        \sum_{t=1}^T
        \left[
        {L}^{(t)}\left(\boldsymbol{p}^{(t)}\right)
        -
        {L}^{(t+1)}\left(\boldsymbol{p}^{(t+1)}\right)
        +
        \tilde\ell^{(t)}\left(\boldsymbol{p}^{(t)}\right) \right]}_\text{(iii)}
    \end{split}
    \end{equation}
    
    Let us first denote $\boldsymbol{A}^{(t)}\triangleq \boldsymbol{g}^{(t)} {\boldsymbol{g}^{(t)}}^\top$ 
    for $\boldsymbol{g}^{(t)}=\nabla \ell^{(t)}(\boldsymbol{p}^{(t)})$ as in the main text,
    and further denote that $\boldsymbol{B}^{(t)}\triangleq \alpha \boldsymbol{I}_K + \beta\sum_{\tau=1}^t \boldsymbol{A}^{(\tau)}$.
    Then, we can rewrite the regularizer of \texttt{AAggFF-D} as follows.
    \begin{equation}
    \begin{gathered}
    \label{eq:rewriting_of_reg}
        R^{(t)}\left(\boldsymbol{p}\right)
        = \frac{\alpha}{2} \Vert \boldsymbol{p} \Vert_2^2 + 
        \frac{\beta}{2} \sum_{\tau=1}^{t-1} \left(\left\langle \boldsymbol{g}^{(\tau)}, \boldsymbol{p} - \boldsymbol{p}^{(\tau)}\right\rangle\right)^2
        = \frac{\alpha}{2} \Vert \boldsymbol{p} \Vert_2^2 + 
        \frac{\beta}{2} \sum_{\tau=1}^{t-1} 
        \left\Vert \boldsymbol{p}^{(\tau)} - \boldsymbol{p} \right\Vert^2_{\boldsymbol{A}^{(\tau)}}
    \end{gathered}
    \end{equation}
    
    That is, $R^{(t)}\left(\boldsymbol{p}\right)$, as well as $L^{(t)}\left(\boldsymbol{p}\right)$ is $\beta$-strongly convex function w.r.t. $\left\Vert\cdot\right\Vert_{\boldsymbol{B}^{(t-1)}}, t\in[T]$.

    For (i), since the regularizer $R^{(t)}\left(\boldsymbol{p}\right)$ is nonnegative for all $t\in[T]$,
    it can be upper bounded as $R^{(T+1)}\left(\boldsymbol{p}^{\star}\right)$.
    Using (\ref{eq:rewriting_of_reg}), we have:
    \begin{equation}
    \begin{gathered}
        R^{(T+1)}\left(\boldsymbol{p}^{\star}\right)
        =
        \frac{\alpha}{2} \Vert \boldsymbol{p} \Vert_2^2 + 
        \frac{\beta}{2} \sum_{t=1}^{T} 
        \left\Vert \boldsymbol{p}^{(t)} - \boldsymbol{p}^\star \right\Vert^2_{\boldsymbol{A}^{(t)}}.
    \end{gathered}
    \end{equation}
    
    For (ii), we use the assumption in Lemma~\ref{lemma:one_step_bound}, 
    where $\boldsymbol{p}^{(t)} = \argmin_{\boldsymbol{p}\in\Delta_{K-1}} L^{(t)}\left(\boldsymbol{p}\right)$.
    From the assumption, since $\boldsymbol{p}^{(T+1)}$ is the minimizer of $L^{(T+1)}$, (ii) becomes negative.
    Thus, we can exclude it in further upper bounding.
    
    For (iii), we can directly use the result of Lemma~\ref{lemma:one_step_bound_proximal}.
    {\allowdisplaybreaks
    \begin{align}
        &\sum_{t=1}^T
        \left[
        {L}^{(t)}\left(\boldsymbol{p}^{(t)}\right)
        -
        {L}^{(t+1)}\left(\boldsymbol{p}^{(t+1)}\right)
        +
        \tilde\ell^{(t)}\left(\boldsymbol{p}^{(t)}\right) \right] \\
        &\leq
        \frac{1}{2\beta} 
        \sum_{t=1}^T \left\Vert \boldsymbol{g}^{(t)} \right\Vert_{{\boldsymbol{B}^{(t)}}^{-1}}^2
        +
        \sum_{t=1}^T \left(
        R^{(t)}\left(\boldsymbol{p}^{(t)}\right)
        -
        R^{(t+1)}\left(\boldsymbol{p}^{(t)}\right)
        \right) \\
        &=
        \frac{1}{2\beta} 
        \sum_{t=1}^T \left\Vert \boldsymbol{g}^{(t)} \right\Vert_{{\boldsymbol{B}^{(t)}}^{-1}}^2.
    \end{align}
    }
    
    Combining all, we have regret upper bound as follows.
    \begin{equation}
    \begin{split}
        &\normalfont\text{Regret}^{(T)}\left(\boldsymbol{p}^{\star}\right)
        = 
        \sum_{t=1}^T \ell^{(t)}\left(\boldsymbol{p}^{(t)}\right) 
        - 
        \sum_{t=1}^T \ell^{(t)}\left(\boldsymbol{p}^{\star}\right) \\
        &\leq
        \sum_{t=1}^T \tilde\ell^{(t)}\left(\boldsymbol{p}^{(t)}\right) 
        - 
        \sum_{t=1}^T \tilde\ell^{(t)}\left(\boldsymbol{p}^{\star}\right) \\
        &\leq
        \frac{\alpha}{2} \Vert \boldsymbol{p} \Vert_2^2 + 
        \frac{\beta}{2} \sum_{t=1}^{T} 
        \left\Vert \boldsymbol{p}^{(t)} - \boldsymbol{p}^\star \right\Vert^2_{\boldsymbol{A}^{(t)}}
        +
        \frac{1}{2\beta} 
        \sum_{t=1}^T \left\Vert \boldsymbol{g}^{(t)} \right\Vert_{{\boldsymbol{B}^{(t)}}^{-1}}^2.
    \end{split}
    \end{equation}
    
    Lastly, from the result of Lemma~\ref{lemma:exp_concave}, we have 
    \begin{equation}
    \begin{split}
    \label{ons_immediate_result}
        &\normalfont\text{Regret}^{(T)}\left(\boldsymbol{p}^{\star}\right)
        = 
        \sum_{t=1}^T \ell^{(t)}\left(\boldsymbol{p}^{(t)}\right) 
        - 
        \sum_{t=1}^T \ell^{(t)}\left(\boldsymbol{p}^{\star}\right) \\
        &\leq
        \sum_{t=1}^T \tilde\ell^{(t)}\left(\boldsymbol{p}^{(t)}\right) 
        - 
        \sum_{t=1}^T \tilde\ell^{(t)}\left(\boldsymbol{p}^{\star}\right) 
        -
        \frac{\beta}{2} \sum_{t=1}^{T} 
        \left\Vert \boldsymbol{p}^{(t)} - \boldsymbol{p}^\star \right\Vert^2_{\boldsymbol{A}^{(t)}} \\
        &\leq
        \frac{\alpha}{2} \Vert \boldsymbol{p} \Vert_2^2 
        +
        \frac{1}{2\beta} 
        \sum_{t=1}^T \left\Vert \boldsymbol{g}^{(t)} \right\Vert_{{\boldsymbol{B}^{(t)}}^{-1}}^2,
    \end{split}
    \end{equation}
    where we need $\left\vert\beta\left\langle \boldsymbol{g}^{(t)}, \boldsymbol{p} - \boldsymbol{p}^{(t)}\right\rangle\right\vert\leq\frac{1}{2}$ due to the assumption.
    
    To meet the assumption, we have
    \begin{equation}
    \begin{gathered}
        \left\vert\beta\left\langle \boldsymbol{g}^{(t)}, \boldsymbol{p} - \boldsymbol{p}^{(t)}\right\rangle\right\vert
        \leq
        \beta \Vert \boldsymbol{p} - \boldsymbol{p}^{(t)} \Vert_1 \Vert \boldsymbol{g}^{(t)} \Vert_\infty
        \leq
        2 \beta L_\infty, 
    \end{gathered}
    \end{equation}
    where the first inequality is due to H\"{o}lder's inequality and the second inequality is due to Lemma~\ref{lemma:lipschitz} and the fact that a diameter of the simplex is 2.
    Thus, we can set $\beta=\frac{1}{4 L_\infty}$ to satisfy the assumption.

    For the first term, 
    using the fact that $\Vert \cdot \Vert_2 \leq \Vert \cdot \Vert_1$, we have
    \begin{equation}
    \begin{gathered}
        \frac{\alpha}{2} \Vert \boldsymbol{p} \Vert_2^2 
        \leq
        \frac{\alpha}{2} \Vert \boldsymbol{p} \Vert_1^2 
        \leq
        \frac{\alpha}{2},
    \end{gathered}
    \end{equation}
    where the last equality is due to $\boldsymbol{p}\in\Delta_{K-1}$.
    
    For the second term, 
    denote $\lambda_1, ..., \lambda_K$ as the eigenvalues of $\boldsymbol{B}^{(T)} - \alpha \boldsymbol{I}_K$, 
    then we have:
    \begin{equation}
    \begin{gathered}
        \sum_{t=1}^T \left\Vert \boldsymbol{g}^{(t)} \right\Vert_{{\boldsymbol{B}^{(t)}}^{-1}}^2
        \leq
        \sum_{i=1}^K \log \left( 1 + \frac{\lambda_i}{\alpha} \right),
    \end{gathered}
    \end{equation}   
    which is the direct result of Lemma 11.11 and Theorem 11.7 of \cite{logdetineq}.
    
    This can be further bounded by AM-GM inequality as follows.
    \begin{equation}
    \begin{gathered}
        \sum_{i=1}^K \log \left( 1 + \frac{\lambda_i}{\alpha} \right)
        \leq
        K \log \left( 1 + \frac{1}{K\alpha}\sum_{i=1}^K{\lambda_i} \right).
    \end{gathered}
    \end{equation}    
    
    Since we have
    \begin{equation}
    \begin{gathered}
        \sum_{i=1}^K \lambda_i 
        = \operatorname{trace}\left( \boldsymbol{B}^{(T)} - \alpha \boldsymbol{I}_K \right)
        = \operatorname{trace}\left( \beta\sum_{t=1}^T \boldsymbol{g}^{(t)} {\boldsymbol{g}^{(t)}}^\top \right)
        =
        \beta\sum_{t=1}^T \Vert \boldsymbol{g}^{(t)} \Vert_2^2 \\
        \leq 
        \beta K T \Vert \boldsymbol{g}^{(t)} \Vert_\infty^2 
        \leq 
        \beta K T L_\infty^2
        = \frac{K T L_\infty}{4},
    \end{gathered}
    \end{equation}
    where $L_\infty$ is a Lipschitz constant w.r.t. $\Vert \cdot \Vert_\infty$ from Lemma~\ref{lemma:lipschitz},
    thus the inequality is due to $\Vert \boldsymbol{p} \Vert_2 \leq \sqrt{K} \Vert \boldsymbol{p} \Vert_\infty, \forall \boldsymbol{p}\in\Delta_{K-1}$.
    
    Followingly, we can upper-bound the second term as
    \begin{equation}
    \begin{gathered}
        \sum_{t=1}^T \left\Vert \boldsymbol{g}^{(t)} \right\Vert_{{\boldsymbol{B}^{(t)}}^{-1}}^2
        \leq
        K \log \left( 1 + \frac{TL_\infty}{4\alpha} \right).
    \end{gathered}
    \end{equation}        
    
    Putting them all together, we have:
    \begin{equation}
    \begin{split}
        &\normalfont\text{Regret}^{(T)}\left(\boldsymbol{p}^{\star}\right)
        = 
        \sum_{t=1}^T \ell^{(t)}\left(\boldsymbol{p}^{(t)}\right) 
        - 
        \sum_{t=1}^T \ell^{(t)}\left(\boldsymbol{p}^{\star}\right) \\
        &\leq
        \sum_{t=1}^T \tilde\ell^{(t)}\left(\boldsymbol{p}^{(t)}\right) 
        - 
        \sum_{t=1}^T \tilde\ell^{(t)}\left(\boldsymbol{p}^{\star}\right) 
        -
        \frac{1}{2} \sum_{t=1}^{T} 
        \left\Vert \boldsymbol{p}^{(t)} - \boldsymbol{p}^\star \right\Vert^2_{\boldsymbol{A}^{(t)}} \\
        &\leq
        \frac{\alpha}{2} \Vert \boldsymbol{p} \Vert_2^2 
        +
        \frac{1}{2\beta} 
        \sum_{t=1}^T \left\Vert \boldsymbol{g}^{(t)} \right\Vert_{{\boldsymbol{B}^{(t)}}^{-1}}^2 \\
        &
        \leq
        \frac{\alpha}{2} + \frac{K}{2\beta} \log \left( 1 + \frac{TL_\infty}{4\alpha} \right)\\
        &=
        \frac{\alpha}{2} + 2 L_\infty K \log \left( 1 + \frac{TL_\infty}{4\alpha} \right).
    \end{split}
    \end{equation}     
    If we further set $\alpha = 4 L_\infty K$,
    we finally have    
    \begin{equation}
    \begin{gathered}
        \normalfont\text{Regret}^{(T)}\left(\boldsymbol{p}^{\star}\right)
        \leq
        2 L_\infty K \left( 
        1 + \log \left( 1 + \frac{T}{16K} \right) 
        \right).
    \end{gathered}
    \end{equation}
    \end{proof}

\newpage
\subsection{Regret Bound of \texttt{AAggFF-D} with Full Client Participation: Proof of Theorem~\ref{thm:crossdevice_full}}
\label{sec:device_proof}
\begin{proof}
    Again, since Lemma~\ref{lemma:regret_bound_default_form} holds for arbitrary loss function,
    let us set ${L}^{(t)}\left(\boldsymbol{p}\right)\triangleq\sum_{\tau=1}^{t-1} \tilde\ell^{(\tau)} (\boldsymbol{p}) + \zeta^{(t+1)}\sum_{i=1}^K p_i \log p_i$ with a slight abuse of notation.
    Note that we set $R^{(t)}\left(\boldsymbol{p}\right) = \zeta^{(t)}\sum_{i=1}^K p_i \log p_i$ is a negative entropy regularizer with non-decreasing time-varying step size $\zeta^{(t)}$,
    thus ${L}^{(t)}\left(\boldsymbol{p}\right)$ is $\zeta^{(t)}$-strongly convex w.r.t. $\Vert\cdot\Vert_1$. (Proposition 5.1 from \cite{bregmanmd})
    Then, we have an upper bound of the regret of \texttt{AAggFF-D} (with full-client participation setting) as follows.
    {\allowdisplaybreaks
    \begin{align}
        &\normalfont\text{Regret}^{(T)}\left(\boldsymbol{p}^{\star}\right) \nonumber \\
        &= 
        \sum_{t=1}^T \ell^{(t)}\left(\boldsymbol{p}^{(t)}\right) 
        - 
        \sum_{t=1}^T \ell^{(t)}\left(\boldsymbol{p}^{\star}\right) \nonumber \\
        &\leq
        \sum_{t=1}^T \tilde\ell^{(t)}\left(\boldsymbol{p}^{(t)}\right) 
        - 
        \sum_{t=1}^T \tilde\ell^{(t)}\left(\boldsymbol{p}^{\star}\right) \nonumber \\
        &=
        \underbrace{
        R^{(T+1)}\left(\boldsymbol{p}^{\star}\right)
        -R^{(1)}\left(\boldsymbol{p}^{(1)}\right)}_\text{(i)}
        +
        \underbrace{
        {L}^{(T+1)}\left(\boldsymbol{p}^{(T+1)}\right)
        -{L}^{(T+1)}\left({\boldsymbol{p}^{\star}}\right)}_\text{(ii)} \\
        &+
        \underbrace{
        \sum_{t=1}^T
        \left[
        {L}^{(t)}\left(\boldsymbol{p}^{(t)}\right)
        -
        {L}^{(t+1)}\left(\boldsymbol{p}^{(t+1)}\right)
        +
        \tilde\ell^{(t)}\left(\boldsymbol{p}^{(t)}\right) \right]}_\text{(iii)} \nonumber,
    \end{align}
    }
    where the inequality is due to Corollary~\ref{corollary:lin_loss}.

    For (i), recall from Lemma~\ref{lemma:regret_bound_default_form} that the regret does not depend on the regularizer, 
    we can bound it after changing from $R^{(T+1)}\left(\boldsymbol{p}^{\star}\right)$ 
    to $R^{(T)}\left(\boldsymbol{p}^{\star}\right)$.
    \begin{equation}
    \begin{gathered}
        R^{(T)}\left(\boldsymbol{p}^{\star}\right)
        -
        R^{(1)}\left(\boldsymbol{p}\right)
        \leq
        \zeta^{(T)}
        \sum_{i=1}^K
        {p_i^\star \log p_i^\star}
        +
        \zeta^{(1)}\log{K}
        \leq
        \zeta^{(T)}
        \sum_{i=1}^K
        {p_i^\star \log p_i^\star}
        +
        \zeta^{(T)}\log{K}
        \leq
        \zeta^{(T)}\log{K}.
    \end{gathered}
    \end{equation}

    For (ii), we use the assumption in Lemma~\ref{lemma:one_step_bound}, 
    where $\boldsymbol{p}^{(t)} = \argmin_{\boldsymbol{p}\in\Delta_{K-1}} L^{(t)}\left(\boldsymbol{p}\right)$.
    From the assumption, since $\boldsymbol{p}^{(T+1)}$ is the minimizer of $L^{(T+1)}$, (ii) becomes negative.
    Thus, we can exclude it from the upper bound.

    For (iii), we directly use the result of Lemma~\ref{lemma:one_step_bound} as follows.
    \begin{equation}
    \begin{gathered}
        \sum_{t=1}^T
        \left[
        {L}^{(t)}\left(\boldsymbol{p}^{(t)}\right)
        -
        {L}^{(t+1)}\left(\boldsymbol{p}^{(t+1)}\right)
        +
        \tilde\ell^{(t)}\left(\boldsymbol{p}^{(t)}\right) \right]
        \leq
        \sum_{t=1}^T
        \frac{ \left\Vert \boldsymbol{g}^{(t)} \right\Vert_\infty^2 }
        { 2\zeta^{(t)} }
        \leq
        \sum_{t=1}^T
        \frac{ L_\infty^2 }
        { 2\zeta^{(t)} },
    \end{gathered}
    \end{equation}
    where the additional terms are removed due to the non-decreasing property of regularizer thanks to the assumption of $\zeta^{(t)}$, and the last inequality is due to Lemma~\ref{lemma:lipschitz}.
    Note that $\boldsymbol{g}^{(t)}=\nabla \ell^{(t)}(\boldsymbol{p}^{(t)})$.

    Combining all, we have regret upper bound as follows.
    \begin{equation}
    \begin{gathered}
        \normalfont\text{Regret}^{(T)}\left(\boldsymbol{p}^{\star}\right)
        = 
        \sum_{t=1}^T \ell^{(t)}\left(\boldsymbol{p}^{(t)}\right) 
        - 
        \sum_{t=1}^T \ell^{(t)}\left(\boldsymbol{p}^{\star}\right)
        \leq
        \sum_{t=1}^T \tilde\ell^{(t)}\left(\boldsymbol{p}^{(t)}\right) 
        - 
        \sum_{t=1}^T \tilde\ell^{(t)}\left(\boldsymbol{p}^{\star}\right)
        \leq
        \zeta^{(T)}\log{K}
        +
        \sum_{t=1}^T
        \frac{ L_\infty^2 }
        { 2\zeta^{(t)} }.
    \end{gathered}
    \end{equation}
    
    Finally, by setting $\zeta^{(t)}=\frac{L_\infty\sqrt{t}}{\sqrt{\log{K}}}$, we have
    \begin{equation}
    \begin{gathered}
        \leq
        L_\infty\sqrt{T\log{K}}+\frac{L_\infty\sqrt{\log{K}}}{2}\sum_{t=1}^T\frac{1}{\sqrt{t}}
        \leq
        2L_\infty\sqrt{T\log{K}},
    \end{gathered}
    \end{equation}
    where the inequality is due to $\sum_{t=1}^{T} \frac{1}{\sqrt{t}} \leq 
    \int_{0}^{T} \frac{\mathrm{d}x}{\sqrt{x}} = 2\sqrt{T}$.
    See also equation (7.3) of \cite{oco3}.
\end{proof}

\newpage
\subsection{Regret Bound of \texttt{AAggFF-D} with Partial Client Participation: Proof of Corollary~\ref{cor:crossdevice_partial}}
\begin{proof}
    Denote $\breve{\ell}^{(t)}$ as a linearized loss constructed from $\breve{\boldsymbol{r}}^{(t)}$ and $\breve{\boldsymbol{g}}^{(t)}$.
    i.e., 
    \begin{equation}
    \begin{gathered}
        \breve{\ell}^{(t)}(\boldsymbol{p}) = \left\langle \boldsymbol{p}, \breve{\boldsymbol{g}}^{(t)} \right\rangle
        = 
        \left\langle
        \boldsymbol{p},
        \frac{ \breve{\boldsymbol{r}}^{(t)} }
        {1 +\left\langle \boldsymbol{p}^{(t)}, \bar{\boldsymbol{r}}\boldsymbol{1}_K \right\rangle}
        +
        \frac{\bar{\boldsymbol{r}}\boldsymbol{1}_K {\boldsymbol{p}^{(t)}}^\top (\breve{\boldsymbol{r}}^{(t)} - \bar{\boldsymbol{r}}\boldsymbol{1}_K)}
        {(1 +\left\langle \boldsymbol{p}^{(t)}, \bar{\boldsymbol{r}}\boldsymbol{1}_K \right\rangle)^2}
        \right\rangle
    \end{gathered}
    \end{equation}
    
    The expected regret is
    {\allowdisplaybreaks
    \begin{align}
        &\mathbb{E}\left[ \normalfont\text{Regret}^{(T)}\left(\boldsymbol{p}^{\star}\right) \right]
        = 
        \mathbb{E}\left[ 
        \sum_{t=1}^T 
        \left( \ell^{(t)}\left(\boldsymbol{p}^{(t)}\right) 
        - 
        \ell^{(t)}\left(\boldsymbol{p}^{\star}\right) \right) 
        \right] \nonumber \\
        &\leq
        \mathbb{E}\left[ 
        \sum_{t=1}^T 
        \left( \breve\ell^{(t)}\left(\boldsymbol{p}^{(t)}\right) 
        - 
        \breve\ell^{(t)}\left(\boldsymbol{p}^{\star}\right) \right) 
        \right]
        =
        \mathbb{E}\left[ 
        \sum_{t=1}^T 
        \left\langle \breve{\boldsymbol{g}}^{(t)},
        \boldsymbol{p}^{(t)}
        - 
        \boldsymbol{p}^{\star} \right\rangle 
        \right] \nonumber \\
        &=
        \mathbb{E}\left[ 
        \sum_{t=1}^T 
        \mathbb{E}_{i\in S^{(t)}}\left[ 
        \left\langle \breve{\boldsymbol{g}}^{(t)},
        \boldsymbol{p}^{(t)}
        - 
        \boldsymbol{p}^{\star} \right\rangle 
        \right]
        \right] (\because \text{Law of Total Expectation}) \\
        &\approx
        \mathbb{E}\left[ 
        \sum_{t=1}^T  
        \left\langle {\boldsymbol{g}}^{(t)},
        \boldsymbol{p}^{(t)}
        - 
        \boldsymbol{p}^{\star} \right\rangle 
        \right] (\because \text{Lemma~\ref{lemma:linearized_grad}}) \nonumber \\
        &=
        \sum_{t=1}^T 
        \left( \tilde\ell^{(t)}\left(\boldsymbol{p}^{(t)}\right) 
        - 
        \tilde\ell^{(t)}\left(\boldsymbol{p}^{\star}\right) \right) \nonumber
        \leq
        \mathcal{O}\left(L_\infty K \sqrt{T\log{K}} \right)
    \end{align}
    }
    \end{proof}

    \begin{remark}
    \label{remark:inflated_lip_const}
        Even though we can enjoy the same regret upper bound \textit{in expectation} from Corollary~\ref{cor:crossdevice_partial},
        it should be noted that the raw regret (i.e., regret without expectation) from $\breve{\boldsymbol{g}}^{(t)}$ may inflate the regret upper bound from
        $\mathcal{O}\left(L_\infty K \sqrt{T\log{K}} \right)$ to $\mathcal{O}\left(\breve{L}_\infty K \sqrt{T\log{K}} \right)$,
        where $\breve{L}_\infty$ is a Lipschitz constant from Lemma~\ref{lemma:lipschitz_lin_grad}, upper bounding $\left\Vert \breve{\boldsymbol{g}}^{(t)} \right\Vert_\infty \leq \breve{L}_\infty$.
        It is because $\breve{L}_\infty$ is dominated by ${1}/{C}\approx\mathcal{O}\left(K\right)$, 
        which can be a \textit{huge number} when if $C$ is a tiny constant.
        Although this inflation hinders proper update of \texttt{AAggFF-D} empirically,
        this can be easily eliminated in \texttt{AAggFF-D} through an appropriate choice of a range ($C_1$ and $C_2$) of the response vector,
        which ensures practicality of \texttt{AAggFF-D}. See Appendix~\ref{app:range} for a detail.
    \end{remark}

\newpage
\subsection{Derivation of Closed-From Update of \texttt{AAggFF-D}}
\label{app:deriv_closed}
    The objective of \texttt{AAggFF-D} in (\ref{eq:lin_ftrl}) can be written in the following form.
    \begin{equation}
    \begin{split}
        &\min_{\boldsymbol{p}\in\Delta_{K-1}} \sum_{\tau=1}^{t} \tilde\ell^{(\tau)}(\boldsymbol{p}) + \zeta^{(t+1)}\sum_{i=1}^K p_i \log p_i
        =
        \min_{\boldsymbol{p}\in\Delta_{K-1}} \left\langle \boldsymbol{p}, \sum_{\tau=1}^{t-1} \breve{\boldsymbol{g}}^{(\tau)} \right\rangle + \zeta^{(t+1)}\sum_{i=1}^K p_i \log p_i \\
        &=
        \min_{\boldsymbol{p}\in\Delta_{K-1}} \left\langle \sum_{\tau=1}^{t} \breve{\boldsymbol{g}}^{(\tau)}, \boldsymbol{p} \right\rangle + R^{(t+1)}(\boldsymbol{p})
        =
        \max_{\boldsymbol{p}\in\Delta_{K-1}} \left\langle -\sum_{\tau=1}^{t} \breve{\boldsymbol{g}}^{(\tau)}, \boldsymbol{p} \right\rangle - R^{(t+1)}(\boldsymbol{p}).
    \end{split}
    \end{equation}
    
    It exactly corresponds to the form of the Fenchel conjugate $R^{(t+1)}_*$, which is defined as follows.
    \begin{equation}
    \begin{gathered}
    \label{eq:fenchel_conjugate}
        R^{(t+1)}_*(\boldsymbol{p})
        =
        \max_{\boldsymbol{p}\in\Delta_{K-1}} \left\langle -\sum_{\tau=1}^{t} \breve{\boldsymbol{g}}^{(\tau)}, \boldsymbol{p} \right\rangle - R^{(t+1)}(\boldsymbol{p}).
    \end{gathered}
    \end{equation}
    
    Thus, we can enjoy the property of Fenchel conjugate, which is
    \begin{equation}
    \begin{gathered}
        \boldsymbol{p}^{(t+1)} = \nabla R^{(t+1)}_* \left(-\sum_{\tau=1}^{t} \breve{\boldsymbol{g}}^{(\tau)}\right)
    \end{gathered}
    \end{equation}
    
    Since we can derive an explicit log-sum-exp form by solving (\ref{eq:fenchel_conjugate}) as follows,
    \begin{equation}
    \begin{gathered}
        R^{(t+1)}_*(\boldsymbol{q}) = \log{\left( \sum_{i=1}^K \exp\left( q_i \right) \right)},
    \end{gathered}
    \end{equation}
    we have the closed-form solution for the new decision update.
    \begin{equation}
    \begin{gathered}
        {p}^{(t+1)}_i = \frac
        {\exp \left( -\sum_{\tau=1}^{t} \breve{g}^{(\tau)}_i / \zeta^{(t+1)} \right)}
        {\sum_{j=1}^K \exp \left( -\sum_{\tau=1}^{t} \breve{g}^{(\tau)}_j / \zeta^{(t+1)} \right)}.
    \end{gathered}
    \end{equation}
    
    Note that $\zeta^{(t+1)}$ is already determined in Theorem~\ref{thm:crossdevice_full} and Corollary~\ref{cor:crossdevice_partial} as $\frac{\breve{L}_\infty \sqrt{t + 1}}{\sqrt{\log{K}}}$,
    with the reflection of modified Lipschitz constant from $L_\infty$ to $\breve{L}_\infty$ (see Remark~\ref{remark:inflated_lip_const}).
    See also \cite{eg} and Chapter 6.6 of \cite{oco3}.

\newpage 
\chapter{Perspective on Local Data Distribution, $\mathcal{D}$}
\label{ch:fedevg}
\numberwithin{equation}{chapter}
\numberwithin{figure}{chapter}
\numberwithin{table}{chapter}
\numberwithin{algorithm}{chapter}
\renewcommand{\theequation}{4.\arabic{equation}}
\renewcommand{\thefigure}{4.\arabic{figure}}
\renewcommand{\thetable}{4.\arabic{table}}
\renewcommand{\thealgorithm}{4.\arabic{algorithm}}

\section*{Federated Synthetic Data Generation through Energy-based Composition}
\addcontentsline{toc}{section}{\protect\numberline{}\textbf{Federated Synthetic Data Generation through Energy-based Composition}}
    In this section, we dive into the aspect of \textbf{local distribution}, denoted as $\mathcal{D}$ in eq.~\eqref{eq:fl_obj}.
    It is usually prohibited to access or handle the local distribution \textit{directly} in the federated system, due to sensitive information about local data.
    Nevertheless, it would undoubtedly be beneficial if the federated system could at least be equipped with a proxy of the local distributions, since it can emulate the oracle distribution from which the centralized dataset originates.
    Thereby, the statistical heterogeneity problem can be mitigated through data sharing, data augmentation, or training an unbiased model on the server using the proxy data.
    
    Existing methods have mostly focused on emulating local distributions by training a generative model in a collaborative manner.
    After training the generative model, the central server can utilize it to synthesize plausible datasets for downstream tasks.
    This typically requires the exchange of entire model parameters to obtain a generative model.
    However, the computational load on the clients and the large communication costs of the federated system hinder its practicality and scalability.
    In addition, the synthesized samples are of inferior quality, visually implausible, and sometimes unlabeled.
    
    This leads to the following research question:
    \begin{center} 
    \textit{How can we imitate \textbf{local distributions} in the federated system with efficiency?}
    \end{center}

\newpage
\section{Introduction}
    FL~\cite{fedavg} is a de facto standard method to train a statistical model from decentralized data. 
    It is typically assumed that the central server (e.g., a service provider) orchestrates the whole process of FL, by repeatedly communicating with participating clients with their data.
    When the server broadcasts copies of a global model to clients,
    each client commits a locally updated model instead of its raw data.
    It is then aggregated into a single global model in the server, 
    and this procedure is repeated until the convergence of the global model.
    However, the convergence is often hindered significantly, 
    since the server cannot directly control the statistical heterogeneity naturally derived from disparate local distributions~\cite{fedavg,noniid,noniidconvergence}.     
    Furthermore, the model exchanging scheme of FL inevitably imposes high communication costs within the system, given that the model is considerably larger in size than that of the raw local data.
    Hence, numerous works have been proposed to mitigate this problem, 
    including a modified aggregation scheme~\cite{agg1,agg2,agg3}, personalization~\cite{pfl2,pfl3,pfl4}, 
    use of adaptive mixing coefficients~\cite{mix1,mix2,mix3}, 
    and other approaches including variants of an optimization objective~\cite{opt1,opt2}.

    Though effective, these methods sometimes offer marginal gains compared to data-centric approaches, 
    such as exploiting shared public data~\cite{noniid,pub1,pub2} or synthetic data~\cite{gan1,gan2,gan3,vae1,vae2} 
    (generated from local distributions of clients) provided by the central server.
    While the former is not always viable due to strict regulations of the distinctiveness of the data domain, 
    the latter is a feasible option thanks to the improved techniques in generative modeling.
    For example, most previous studies have proposed training generative models in a collaborative manner, 
    such as generative adversarial networks (GANs~\cite{gan}), or variational auto-encoders (VAEs~\cite{vae}), 
    by exchanging \textit{entire model parameters} as in traditional FL~\cite{fedavg}.
    
    However, these methods have room for improvement since they either 
    i) require communication of whole model parameters, 
    ii) an additional training procedure should be considered for local discriminative modeling, or
    iii) work only under the mild degree of statistical heterogeneity.
    Especially for ii), it may be necessary for the server or the clients to undertake additional training and communication of a discriminative model.
    Therefore, the design of a distributed synthetic data generation method in the federated system has been remained as a challenging and under-explored research direction.
    
    To achieve fundamental improvement in these suboptimal designs, we begin by formulating the federated synthetic data generation as an iterative and collaborative sampling procedure without the direct access of local data distribution. 
    We find our scheme can be realized \textbf{without securing a separate generative model}, 
    through embarrassingly simple modification to the training scheme of a discriminative model in each client, which only adds little computation overhead.
    The trick is closely related to the energy-based models (EBMs~\cite{ebm}), 
    which are flexible in modeling a probability density with \textit{any} deep networks,
    as well as support an easy sample generation using Markov chain Monte Carlo (MCMC) methods.
    
    From this novel perspective, we propose \texttt{FedEvg}, referred to \textit{Federated Energy-based Averaging}.
    Instead of exchanging a model parameter/gradient with clients as in the conventional FL algorithms, \texttt{FedEvg} starts from synthetic data (with random labels) randomly initialized in the server, 
    which is then gradually refined in collaboration with participating clients.
    Therefore, \texttt{FedEvg} is efficient in communication, and does not require direct information from the clients.
    Mirroring current attention of this synthetic data-driven approach in the \textit{cross-silo FL setting}~\cite{syntheticflreview}, we empirically demonstrate the efficacy of our proposed method by conducting experiments suited for FL settings with moderate number of participating clients.

\newpage
\section{Backgrounds}
\subsection{Problem Statement}
\label{subsec:setup}
    Suppose we have unknown centralized data distribution in the server, $p(\mathcal{D})$.
    Then, we can solve the main objective of FL in eq.~\eqref{eq:fl_obj} as follows.
    \begin{equation}
    \begin{gathered}
        \min_{\boldsymbol{\theta}\in\Theta\subseteq\mathbb{R}^d}
        F\left(\boldsymbol{\theta}\right)
        \triangleq
        \mathbb{E}_{\xi \sim p(\mathcal{D})}
        \left[ 
        \mathcal{L}\left(\xi;\boldsymbol{\theta}\right)
        \right].
    \end{gathered}
    \end{equation}
    In other words, the server can obtain a centralized model parameterized by $\boldsymbol{\theta}$ using the centralized data set, $\mathcal{D}$.
    The synthetic data should therefore reflect all local distributions faithfully in order to be useful in the federated system.
    
    Suppose there exist $K$ participating clients in the system indexed by $i\in[K]$.
    We use $a\in[A]$ as an abbreviation of an enumeration notation, $a=1,...,A$.
    Each client owns its local dataset $\mathcal{D}_i=\{(\boldsymbol{x}_j, y_j)\}_{j=1}^{n_i}$,
    where $n_i$ is the number of samples in the $i$-th client, 
    $\boldsymbol{x}\in\mathbb{R}^D$ refers to an input, 
    and $y\in\{1,...,L\}$ is a corresponding class label.
    Denote $\mathcal{D}=\cup_{i=1}^K \mathcal{D}_i$ as the centralized dataset,
    which is not accessible in a typical FL setting. 
    We further denote that $p^\star(\mathcal{D})$ as an unknown true density,
    and $\tilde{\mathcal{D}}$ as a synthetic dataset containing synthetic input and label, $\tilde{\boldsymbol{x}}$ and $\tilde{y}$.
    
    To synthesize a plausible surrogate dataset satisfying $p(\tilde{\mathcal{D}}) \approx p^\star(\mathcal{D})$, 
    one can train an explicit generative model to generate a plausible surrogate of the original dataset.
    The underlying philosophy of such a method is to emulate samples from a distribution close to the true density $p^\star(\mathcal{D})$.
    In order to estimate the density faithfully, we start with the simple assumption:
    \textit{the true density can be approximated by the mixture of local distributions}.
    Given a finite number of parameterized probability density functions $p(\cdot;\boldsymbol{\theta}_i), i\in[K]$ of participating clients,
    let us assume that the true density $p(\mathcal{D})$ can be approximated by \textit{the mixture of each client's probability density functions} as follows.
    \begin{equation}
    \label{eq:main_formula}
    \begin{gathered}
        p^\star(\mathcal{D}) 
        \approx
        \sum_{i=1}^K w_i p(\mathcal{D};\boldsymbol{\theta}_i),
    \end{gathered}
    \end{equation}
    where $\boldsymbol{\theta}_i\in\mathbb{R}^d$ is the parameter of the $i$-th local model,
    and $w_i \geq 0$ and $\sum_{i=1}^K w_i = 1$.    
    Then, the natural follow-up questions should be:
    \begin{itemize}
        \item \textbf{Question 1:} How can we model and estimate each of the parameterized density, $p(\cdot;\boldsymbol{\theta}_i)$?
        \item \textbf{Question 2:} How can we sample from the mixture of them, $\sum_{i=1}^K w_i p(\cdot;\boldsymbol{\theta}_i)$?
    \end{itemize}

\newpage
\subsection{{Answer 1:} Discovery of Hidden Energy-Based Models}
\label{sec:ebm_discovery}
    \paragraph{Energy Based Models (EBMs)}
    An EBM models a probability density as:
    \begin{equation}
    \label{eq:ebm}
    \begin{gathered}
        p(\boldsymbol{x};\boldsymbol{\theta})
        =
        \frac{ \exp\left(-E(\boldsymbol{x};\boldsymbol{\theta})\right) }{ Z(\boldsymbol{\theta}) },
    \end{gathered}
    \end{equation}
    where $E(\cdot;\boldsymbol{\theta}):\mathbb{R}^D\rightarrow\mathbb{R}$ is a parametrized \textit{energy function}
    and $Z(\boldsymbol{\theta})=\int{\exp\left(-E(\boldsymbol{x};\boldsymbol{\theta})\right)}\mathrm{d}\boldsymbol{x}$ is the normalizing constant, which is typically intractable and not directly modeled.
    
    The parameter of EBMs can be estimated by maximizing the log-likelihood, 
    with the gradient-based optimization using the following estimator~\cite{cd,pcd}.
    The derivative of the log-likelihood of a single input $\boldsymbol{x}$ w.r.t. the parameter $\boldsymbol{\theta}$ is:
    \begin{equation}
    \label{eq:mle_ebm}
    \begin{gathered}
        \nabla_{\boldsymbol{\theta}} \log p(\boldsymbol{x};\boldsymbol{\theta})
        =
        \mathbb{E}_{\boldsymbol{x}'\sim p(\boldsymbol{x}';\boldsymbol{\theta})}\left[ \nabla_{\boldsymbol{\theta}} E(\boldsymbol{x}';\boldsymbol{\theta}) \right]
        -
        \nabla_{\boldsymbol{\theta}} E(\boldsymbol{x};\boldsymbol{\theta}).
    \end{gathered}
    \end{equation}
    To use this estimator, we should obtain samples from the EBM, i.e., $\boldsymbol{x}'\sim p(\boldsymbol{x}';\boldsymbol{\theta})$.
    Since it is a non-trivial task, we should resort to MCMC~\cite{cd} to generate approximate samples.
    We defer the introduction of a representative sampling method widely used thus far in the next answer.

    \paragraph{Derivation of the Gradient of EBM's Log-Likelihood}
    Taking a derivative of eq.~\eqref{eq:ebm} w.r.t. $\boldsymbol{\theta}$, we have:
    \begin{align*}
    \begin{gathered}
        \nabla_{\boldsymbol{\theta}} \log p(\boldsymbol{x};\boldsymbol{\theta})
        =
        -\nabla_{\boldsymbol{\theta}} E(\boldsymbol{x};\boldsymbol{\theta})
        -\nabla_{\boldsymbol{\theta}} \log{Z(\boldsymbol{\theta})}.
    \end{gathered}
    \end{align*}
    For the latter term, $\nabla_{\boldsymbol{\theta}} \log{Z(\boldsymbol{\theta})}$,
    we have:
    \begin{align*}
    \begin{split}
        &\nabla_{\boldsymbol{\theta}} \log{Z(\boldsymbol{\theta})}
        = \nabla_{\boldsymbol{\theta}} \log{
        \int{\exp\left(-E(\boldsymbol{x};\boldsymbol{\theta})\right)}\mathrm{d}\boldsymbol{x}
        } \\
        &= \left( \int{\exp\left(-E(\boldsymbol{x};\boldsymbol{\theta})\right)}\mathrm{d}\boldsymbol{x} \right)^{-1}
        \nabla_{\boldsymbol{\theta}} \int{\exp\left(-E(\boldsymbol{x};\boldsymbol{\theta})\right)}\mathrm{d}\boldsymbol{x} \\
        &= \left( \int{\exp\left(-E(\boldsymbol{x};\boldsymbol{\theta})\right)}\mathrm{d}\boldsymbol{x} \right)^{-1}
        \int{\nabla_{\boldsymbol{\theta}}\exp\left(-E(\boldsymbol{x};\boldsymbol{\theta})\right)}\mathrm{d}\boldsymbol{x} \\
        &= \left( \int{\exp\left(-E(\boldsymbol{x};\boldsymbol{\theta})\right)}\mathrm{d}\boldsymbol{x} \right)^{-1}
        \int{
            \exp\left(-E(\boldsymbol{x};\boldsymbol{\theta})\right)
            (-\nabla_{\boldsymbol{\theta}} 
            E(\boldsymbol{x};\boldsymbol{\theta}))
        }\mathrm{d}\boldsymbol{x} \\
        &= \int{
        \left( \int{\exp\left(-E(\boldsymbol{x};\boldsymbol{\theta})\right)}\mathrm{d}\boldsymbol{x} \right)^{-1}
        \exp\left(-E(\boldsymbol{x};\boldsymbol{\theta})\right)
        (-\nabla_{\boldsymbol{\theta}} E(\boldsymbol{x};\boldsymbol{\theta}))
        }\mathrm{d}\boldsymbol{x} \\
        &= \int{ p(\boldsymbol{x};\boldsymbol{\theta})(
            -\nabla_{\boldsymbol{\theta}} E(\boldsymbol{x};\boldsymbol{\theta})
        ) }\mathrm{d}\boldsymbol{x} \\
        &= -\mathbb{E}_{\boldsymbol{x}'\sim p(\boldsymbol{x}';\boldsymbol{\theta})}\left[ \nabla_{\boldsymbol{\theta}} E(\boldsymbol{x}';\boldsymbol{\theta}) \right].
    \end{split}
    \end{align*}
    Plugging this result into the derivative, we have eq.~\eqref{eq:mle_ebm}.
    
    \paragraph{Energy-based Implementation of the Mixture Distribution}
    Now, we regard the parameterized probability densities as EBMs by defining an energy function $E(\cdot;\boldsymbol{\theta}_i)$, which can be \textit{any} statistical model including a deep network, 
    as long as it outputs a scalar energy value.
    Thus, by explicitly adopting eq.~\eqref{eq:ebm}, we can state eq.~\eqref{eq:main_formula} again as follows.  
    \begin{equation}
    \label{eq:moe_formulation}
    \begin{gathered}
        p(\boldsymbol{x}) 
        \approx
        \sum_{i=1}^K \frac{ w_i \exp\left(-E(\boldsymbol{x};\boldsymbol{\theta}_i)\right) }{ Z(\boldsymbol{\theta}_i) },
    \end{gathered}
    \end{equation}
    where $Z(\boldsymbol{\theta}_i)=\int{\exp\left(-E(\boldsymbol{x};\boldsymbol{\theta}_i)\right)}\mathrm{d}\boldsymbol{x}$ is a normalizing constant for each $i\in[K]$.
    
    In FL, however, the common interest is usually learning a \textit{discriminative model} for clients.
    In other owrds, each client trains a classifier to accurately estimate the conditional density, $p(y|\boldsymbol{x};\boldsymbol{\theta}_i)$. 
    Intriguingly, one \textit{can retrieve an EBM from the discriminative model}, following the formulation of a pioneering study~\cite{your}, as follows.

    \paragraph{Hidden EBMs behind Discriminative Models}
    Assume that each client participating in the FL procedure aims to train a classifier for discriminating samples from $L$ classes.
    Suppose the discriminative classifier (e.g., a deep network) $f:\mathbb{R}^D\rightarrow\mathbb{R}^L$ is parameterized by $\boldsymbol{\theta}$.
    We usually estimate categorical distribution using the softmax transformation as follows.
    \begin{equation}
    \begin{gathered}
        p(y|\boldsymbol{x};\boldsymbol{\theta})
        =
        \frac{\exp\left(f(\boldsymbol{x};\boldsymbol{\theta})[y]\right)}
        {\sum_{y'=1}^L \exp\left(f(\boldsymbol{x};\boldsymbol{\theta})[y']\right)},
    \end{gathered}
    \end{equation}
    where $f(\boldsymbol{x};\boldsymbol{\theta})[y]$ refers to the $y$-th element of the $L$-dimensional logit vector, $f(\boldsymbol{x};\boldsymbol{\theta})\in\mathbb{R}^L$.
    
    From this modeling scheme, we can simply approximate the marginal density $p(\boldsymbol{x};\boldsymbol{\theta})$ and the joint density $p(\boldsymbol{x},y;\boldsymbol{\theta})$ using Bayes' theorem: $p(y|\boldsymbol{x};\boldsymbol{\theta})=\frac{p(\boldsymbol{x}, y;\boldsymbol{\theta})}{p(\boldsymbol{x};\boldsymbol{\theta})}$.
    \begin{equation}
    \begin{gathered}
        p(\boldsymbol{x},y;\boldsymbol{\theta})
        \propto
        \exp\left(f(\boldsymbol{x};\boldsymbol{\theta})[y]\right),
        \qquad
        p(\boldsymbol{x};\boldsymbol{\theta})
        \propto
        \sum\nolimits_{y=1}^L \exp\left(f(\boldsymbol{x};\boldsymbol{\theta})[y]\right),
    \end{gathered}
    \end{equation}
    and the corresponding energy functions are as follows:
    \begin{equation}
    \label{eq:marginal_and_joint_energy}
    \begin{gathered}
        E(\boldsymbol{x},y;\boldsymbol{\theta}) 
        = 
        -f({\boldsymbol{x}};\boldsymbol{\theta})[{y}],
        \quad
        E(\boldsymbol{x};\boldsymbol{\theta}) 
        = 
        -\log\left( 
        \sum\nolimits_{y=1}^L \exp\left(
        f({\boldsymbol{x}};\boldsymbol{\theta})[{y}] 
        \right) 
        \right).
    \end{gathered}
    \end{equation}
    
    Thanks to the hidden connection to EBMs, 
    we can easily retrieve (unnomralized) joint \& marginal density from the classifier as long as an {{input}-{label}} pair is provided.
    When training a classifier, the class probability $p(y|\boldsymbol{x};\boldsymbol{\theta})$ is estimated by minimizing cross-entropy loss, 
    and the marginal (or joint) density $p(\boldsymbol{x};\boldsymbol{\theta})$ (or $p(\boldsymbol{x},y;\boldsymbol{\theta})$) can also be estimated \textit{at once} by minimizing eq.~\eqref{eq:mle_ebm} simultaneously.
    In a word, each client can \textit{jointly obtain a discriminative as well as a generative model} from a single classifier~\cite{your}. 

\newpage
\subsection{Answer 2: Composition of Energy-Based Models}
\label{sec:gen_from_mixture}
    \paragraph{Stochastic Gradient Langevin Dynamics (SGLD)}
    To generate a sample following the density $p(\boldsymbol{x};\boldsymbol{\theta})$ modeled by an EBM, one can resort to SGLD as follows~\cite{sgld}.
    \begin{equation}
    \label{eq:sgld}
    \begin{gathered}
        \boldsymbol{x}^{(t+1)}
        =
        \boldsymbol{x}^{(t)} 
        -
        \gamma^{(t)} \nabla_{\boldsymbol{x}^{(t)}} E(\boldsymbol{x}^{(t)};\boldsymbol{\theta}) 
        + 
        \sigma \boldsymbol{\epsilon}, 
        \quad 
        \boldsymbol{\epsilon} \sim \mathcal{N}(\boldsymbol{0}_D, \boldsymbol{I}_D),
    \end{gathered}
    \end{equation}
    where $\gamma^{(t)}$ is a step-size, $\sigma$ is a noise scale, 
    and $\boldsymbol{x}^{(0)} \sim p_\text{init}(\boldsymbol{x})$, which is an initial proposal distribution.
    The sample $\boldsymbol{x}^{(t)}$ generated by SGLD is known to be converged to a sample from $p(\boldsymbol{x};\boldsymbol{\theta})$, if $t\rightarrow\infty, \gamma^{(t)}\rightarrow 0$~\cite{sgld,sgldconv1,sgldconv2}.
    
    With sufficient number of steps $T<\infty$ until convergence, 
    the resulting sample $\boldsymbol{x}^{(T)}$ can be considered as being sampled from the model distribution, $p(\cdot;\boldsymbol{\theta})$.
    Empirically, it is well known that even short-run non-convergent SGLD can generate plausible samples~\cite{cd1,your,ebmcomp1,ebmcomp2,ebmcomp3}.
    Thus, we can set reasonable iteration number $T$ considering a trade-off between the computation overhead and the sample quality.
    Note that the proposal distribution can be chosen as as real data samples (i.e., contrastive divergence; CD~\cite{cd}), 
    or as pure noises (noisy CD; e.g., uniform distribution, or standard normal distribution)~\cite{cd1},
    or as samples generated in the past~\cite{cd2,your} (i.e., persistent contrastive divergence; PCD~\cite{pcd})

    \paragraph{Sampling from Mixture Distribution}
    Now, we learn how to sample from the mixture distribution using SGLD.
    However, we still don't know the explicit energy function of \textit{the mixture distribution} in eq.~\eqref{eq:moe_formulation}.
    Here, we add one more assumption to induce a tractable energy function.

    Suppose $Z(\boldsymbol{\theta}_1)=...=Z(\boldsymbol{\theta}_K)=Z$ following~\cite{ebmcomp2,ebmcomp3}, 
    then, eq.~\eqref{eq:moe_formulation} becomes:
    \begin{equation}
    \label{eq:samenormconst}
    \begin{gathered}
        \frac{1}{Z}
        \sum_{i=1}^K { w_i \exp\left(-E(\boldsymbol{x};\boldsymbol{\theta}_i)\right) }
        =
        \frac{\exp\left( -E(\boldsymbol{x};\boldsymbol{\theta}_1,...,\boldsymbol{\theta}_K) \right)}{Z}.
    \end{gathered}
    \end{equation}
    
    Consequently, we acquire an energy function of the mixture distribution in a simple tractable form, 
    denoted as $E(\boldsymbol{x};\boldsymbol{\theta}_1,...,\boldsymbol{\theta}_K)$, as follows:
    \begin{equation}
    \label{eq:energy_mixture}
    \begin{gathered}
        E(\boldsymbol{x};\boldsymbol{\theta}_1,...,\boldsymbol{\theta}_K) 
        = 
        -\log\left(\sum_{i=1}^K w_i \exp\left( - E(\boldsymbol{x};\boldsymbol{\theta}_i)
        \right)\right).
    \end{gathered}
    \end{equation}
    Now, we have answered all questions raised in section~\ref{subsec:setup}.
    Atop them, we introduce our proposed method in the following section.

\newpage
\section{Proposed Method}
\subsection{\texttt{FedEvg}: Federated Energy-based Averaging}
    In this section, we introduce our proposed method, \texttt{FedEvg}.
    Our method first randomly initializes synthetic data (i.e., random input-label pairs) in the central server,
    and collaboratively refines these synthetic data with participating clients, 
    to ensure that these samples are realizations of the mixture distributions in eq.~\eqref{eq:moe_formulation}. 
    We provide notation table in Table~\ref{tab:fedevg_notation}.
    \begin{table}[H]
\centering
\caption{Notations in Chapter~\ref{ch:fedevg}}
\label{tab:fedevg_notation}
\resizebox{0.7\textwidth}{!}{%
\begin{tabular}{@{}ll@{}}
\toprule
Notation & Description \\ \midrule
$T$ & Total number of communication rounds \\
$B$ & Local batch size \\
$E$ & Number of local iterations \\
$K$ & Total number of participating clients, indexed by $i\in[K]$ \\
$C$ & Fraction of clients selected at each round \\
$\eta$ & Learning rate for the local update \\ \midrule
$D$ & Input dimension \\
$L$ & Number of classes \\
$M$ & Number of synthetic data \\
$R$ & Number of client-side SGLD steps \\ \midrule
$\beta^{(t)}$ & Server-side SGLD step size \\
$\delta$ & Server-side SGLD noise scale \\
$\gamma^{(t)}$ & Client-side SGLD step size \\
$\sigma$ & Client-side SGLD noise scale \\ \bottomrule
\end{tabular}%
}
\end{table}

\subsection{Server-side Optimization}
\label{subsec:server-side}
    \begin{algorithm}[H]
   \caption{\texttt{FedEvg}}
   \label{alg:fedevg}
\begin{algorithmic}[1]
   \STATE {\bfseries Input:} number of clients $K$, client sampling ratio $C\in(0, 1)$, size of synthetic data $M$, total rounds $T$,
   server-side SGLD step size schedule $\{\beta^{(t)}\}_{t=0}^{T-1}$, server-side SGLD noise scale $\delta$.
   \STATE {\bfseries Procedure:}
   \STATE Server randomly initializes synthetic data $\tilde{\mathcal{D}}^{(0)} (\text{s.t. }\lvert \tilde{\mathcal{D}}^{(0)} \rvert=M)$ according to eq.~\eqref{eq:rand_init}.
   \FOR{$t=0$ {\bfseries to} $T-1$}
    \STATE Server randomly selects a subset of clients $S^{(t)}$ of size $\max(1, \lfloor C \cdot K \rfloor)$.
    \FOR{each client $i\in S^{(t)}$ \textbf{in parallel}}
     \STATE $\left\{ 
         \exp\left( \operatorname{sgn}(-E(\tilde{\boldsymbol{x}}_j, \tilde{y}_j;\boldsymbol{\theta}_i)) \right), 
         \nabla_{\tilde{\boldsymbol{x}}_j} E\left( \tilde{\boldsymbol{x}}_j, \tilde{y}_j;\boldsymbol{\theta}_i \right) 
     \right\}_{j=1}^M
     \leftarrow 
     \texttt{FedEvgClientUpdate}(\tilde{\mathcal{D}}^{(t)})$ 
    \ENDFOR
    \STATE $\tilde{\mathcal{D}}^{(t+1)}\leftarrow$ Update $\tilde{\mathcal{D}}^{(t)}$ with $\beta^{(t)}$ and $\delta$ according to eq.~\eqref{eq:server_sgld}.
   \ENDFOR
   \STATE{\bfseries Return:} $\tilde{\mathcal{D}}^{(T)}$.
\end{algorithmic}
\end{algorithm}
    We start by randomly initializing input-label pairs in the server, $\tilde{\mathcal{D}}^{(0)}
    =
    \left\{
        \left( \tilde{\boldsymbol{x}}^{(0)}_j, \tilde{y}^{(0)}_j \right)
    \right\}_{j=1}^M$.
    For example, 
    \begin{equation}
    \label{eq:rand_init}
    \begin{gathered}
        \tilde{\boldsymbol{x}}^{(0)} \sim \mathcal{N}(\boldsymbol{0}_D, \boldsymbol{I}_D), 
        \quad 
        \tilde{y}^{(0)} \sim \operatorname{Cat}\left(\frac{1}{L}\boldsymbol{1}_L\right), 
    \end{gathered}
    \end{equation}
    where $\boldsymbol{0}_D$ is a $D$-dimensional zero-vector,
    $\boldsymbol{I}_D$ is an identity matrix of size $D \times D$,
    and $\boldsymbol{1}_L$ is a $L$-dimensional one-vector.
    This simple initialization scheme is beneficial since the server does not need to prepare a distinct labeling mechanism,
    since we simply assign random labels to randomly initialized inputs in the beginning.
    To alleviate the statistical heterogeneity in the FL procedure, we set the number of samples from each class to be equal, i.e., $\lfloor M / L \rfloor$ samples per class. 
    
    After the random initialization, the server begins to refine these samples in collaboration with the clients over $t\in[T]$ rounds,
    to approximate sample generation from the mixture distribution in eq.~\eqref{eq:main_formula}.
    To realize this, i) the server performs a single step of SGLD in each round ii) after aggregating the signals from participating clients.
    As long the server can calculate to the gradient of an energy function of mixture distributions, 
    $\nabla_{\tilde{\boldsymbol{x}}} E(\tilde{\boldsymbol{x}};\boldsymbol{\theta}_1,...,\boldsymbol{\theta}_K)$,
    the server-side synthetic data can be refined using the SGLD (eq.~\eqref{eq:sgld}).

    Since we know the explicit form of the energy function of the mixture distribution, viz. eq.~\eqref{eq:energy_mixture},
    we can derive the gradient w.r.t. the server-side synthetic sample $\tilde{\boldsymbol{x}}$ as follows.
    \begin{equation}
    \label{eq:moe_grad}
    \begin{gathered}
        \nabla_{\tilde{\boldsymbol{x}}} E(\tilde{\boldsymbol{x}}; \boldsymbol{\theta}_1,...,\boldsymbol{\theta}_K) 
        =
        \frac{
        \sum_{i=1}^K w_i 
        {\exp(-E(\tilde{\boldsymbol{x}}; \boldsymbol{\theta}_i))} 
        {\nabla_{\tilde{\boldsymbol{x}}} E(\tilde{\boldsymbol{x}}; \boldsymbol{\theta}_i)}
        }{
        \sum_{j=1}^K w_j 
        {\exp(-E(\tilde{\boldsymbol{x}}; \boldsymbol{\theta}_j))}
        }.
    \end{gathered}
    \end{equation}

    Intuitively, eq.~\eqref{eq:moe_grad} can be decomposed into two components:
    i) an exponentiated \textit{energy value}, $\exp(-E(\tilde{\boldsymbol{x}}; \boldsymbol{\theta}_i))$ 
    and ii) an \textit{energy gradient}, $\nabla_{\tilde{\boldsymbol{x}}} E(\tilde{\boldsymbol{x}}; \boldsymbol{\theta}_i)$.
    For $w_i$, we simply set it to $n_i/n$, which is a rate of sample size of the $i$-th client (where $n=\sum_{i=1}^K n_i$.
    This decomposition scheme is undoubtedly favorable to the federated setting, 
    since each signal can be constructed in parallel from participating clients.
    Provided that each client transmits these \textit{energy signals}, the server-side synthetic data will be close to the sample from the mixture of local distributions even without explicitly accessing each of them.
    The server-side SGLD update is given as follows:
    \begin{equation}
    \label{eq:server_sgld}
    \begin{gathered}
        \tilde{\boldsymbol{x}}^{(t+1)}
        =
        \tilde{\boldsymbol{x}}^{(t)} 
        -
        \beta^{(t)} \nabla_{\tilde{\boldsymbol{x}}^{(t)}} E(\tilde{\boldsymbol{x}}^{(t)};\boldsymbol{\theta}_1,...,\boldsymbol{\theta}_K) 
        + 
        \delta \boldsymbol{\epsilon}, 
        \quad 
        \boldsymbol{\epsilon} \sim \mathcal{N}(\boldsymbol{0}_D, \boldsymbol{I}_D),
    \end{gathered}
    \end{equation}
    where $\{\beta^{(t)}\}_{t=0}^{T-1}$ is a server-side SGLD step size schedule and $\delta$ is a server-side SGLD noise scale.
    
    One can find that this procedure bears resemblance to the federated averaging algorithm (FedAvg~\cite{fedavg}),
    where the server takes a step of a global model update in each round by aggregating locally updated parameters from participating clients.
    However, it should be noted that our method does not require the communication of \textit{model parameters} from clients.
    In addition, the \textit{energy value} is a scalar, and the energy gradient is $D$-dimensional vector of the same as the number of inputs of synthetic data $\tilde{\boldsymbol{x}}_j,j\in[M]$.
    Both are usually smaller in size compared to $d$-dimensional model parameters.
    Therefore, communicating with these signals is undoubtedly more cost-effective in terms of communication than exchanging an entire model parameter, as is typically done in traditional FL settings.

\newpage
\subsection{Client-side Optimization}
\label{subsec:client-side}
    \begin{algorithm}[H]
   \caption{\texttt{FedEvgClientUpdate}}
   \label{alg:fedvgclientupdate}
\begin{algorithmic}[1]
   \STATE {\bfseries Input:} 
   downloaded synthetic data $\tilde{\mathcal{D}}$, 
   local model $f(\cdot;\boldsymbol{\theta})$, 
   mini-batch size $B$, 
   local iterations $E$, 
   learning rate $\eta$, 
   constant $\lambda$,
   SGLD steps $R$, 
   SGLD step sizes $\{\gamma^{(r)}\}_{r=0}^{R-1}$, 
   SGLD noise $\sigma$.
   \STATE {\bfseries Procedure:}
   \STATE Initialize a local model and set $\mathcal{D}_{\text{init}}\leftarrow\tilde{\mathcal{D}}$.
   \STATE Set $e=0$.
   \WHILE{$e < E$}
    \STATE Sample local batch $\mathcal{B}$ of size $B$ from local training set.
    \STATE Sample synthetic batch $\tilde{\mathcal{B}}$ of size $B$ from $\tilde{\mathcal{D}}$.
    \FOR{mini-batches $(\tilde{\boldsymbol{x}}, \tilde{y})\in\tilde{\mathcal{B}}$ and $(\boldsymbol{x}, y)\in\mathcal{B}$}
     \STATE Set $(\breve{\boldsymbol{x}}^{(0)},\breve{y})
        \leftarrow(\tilde{\boldsymbol{x}}, \tilde{y})$.
     \FOR{$r=0,...,R-1$}
      \STATE $\breve{\boldsymbol{x}}^{(r+1)} \leftarrow \breve{\boldsymbol{x}}^{(r)} + \gamma^{(r)} \nabla_{\breve{\boldsymbol{x}}^{(r)}} f(\breve{\boldsymbol{x}}^{(r)};\boldsymbol{\theta})[\breve{y}] + \sigma \boldsymbol{\epsilon}, 
        \quad 
        \boldsymbol{\epsilon} \sim \mathcal{N}(\boldsymbol{0}_D, \boldsymbol{I}_D)$, viz. eq.~\eqref{eq:marginal_and_joint_energy} and~\eqref{eq:sgld}.
      \ENDFOR
      \STATE $\mathcal{L}_{\text{PCD}}(\boldsymbol{\theta})\leftarrow f(\boldsymbol{x};\boldsymbol{\theta})[y] - f(\breve{\boldsymbol{x}}^{(R)};\boldsymbol{\theta})[\breve{y}]$ (\textit{i.e., minimization proxy of eq.~\eqref{eq:mle_ebm}}).
      \STATE $\mathcal{L}_{\text{CE}}(\boldsymbol{\theta})\leftarrow
      \texttt{CrossEntropy}(f(\boldsymbol{x};\boldsymbol{\theta}),y)
      +
      \lambda\cdot\texttt{CrossEntropy}(f(\tilde{\boldsymbol{x}};\boldsymbol{\theta}),\tilde{y})$. 
      \STATE Set total loss $\mathcal{L}(\boldsymbol{\theta})\leftarrow\mathcal{L}_{\text{PCD}}(\boldsymbol{\theta})+\mathcal{L}_{\text{CE}}(\boldsymbol{\theta})$.
      \STATE Update the parameter $\boldsymbol{\theta} \leftarrow \boldsymbol{\theta} - \eta \nabla_{\boldsymbol{\theta}} \mathcal{L}(\boldsymbol{\theta})$.
      \STATE Replace $\tilde{\boldsymbol{x}}\leftarrow\breve{\boldsymbol{x}}^{(R)}$ for persistent update.
      \STATE $e \leftarrow e+1$
     \ENDFOR
   \ENDWHILE
   \STATE{\bfseries Return:} $\left\{ 
         \exp\left(\operatorname{sgn}(f({\boldsymbol{x}}_j;\boldsymbol{\theta})[{y}_j]) \right), 
         -\nabla_{{\boldsymbol{x}}_j} f\left( {\boldsymbol{x}}_j;\boldsymbol{\theta} \right)[{y}_j] 
     \right\}_{j=1}^M$ for $\left\{(\boldsymbol{x}_j,y_j)\right\}_{j=1}^M$ in $\mathcal{D}_{\text{init}}$.
\end{algorithmic}
\end{algorithm}
    The client trains a local model with its own dataset, where the required signals from the server are to be obtained from the local model by recovering EBMs, $E(\cdot;\boldsymbol{\theta}_i), i\in[K]$.
    Therefore, each client should train a local EBM using its own dataset, which is equivalent to training a local parameterized model $f(\cdot;\boldsymbol{\theta}_i)$ as discussed in section~\ref{sec:ebm_discovery}.
    
    In detail, the client-side optimization process is a joint training of both generative and discriminative perspectives.
    First, the \textit{generative modeling} is equivalent to maximizing the log-likelihood of the local EBM, as in eq.~\eqref{eq:mle_ebm}.
    This also requires MCMC sampling, e.g., SGLD (eq.~\eqref{eq:sgld}) due to the maximum likelihood objective incorporates an estimate from a sample generated from a model distribution, $p(\cdot;\boldsymbol{\theta})$.
    In this local MCMC sampling procedure, the initial sample $\breve{\boldsymbol{x}}^{(0)}$ is assigned from the synthetic data downloaded from the server.
    This is corresponded to the persistent initialization, where the last sample of the past MCMC chain is directly used as an initial sample of a new MCMC chain. 
    (i.e., PCD~\cite{pcd,cd1,cd2})
    Note that the updated samples from the $R$ steps of client-side MCMC, $\breve{\boldsymbol{x}}^{(R)}$ replaces its original state, as in the line 16 of Algorithm~\ref{alg:fedvgclientupdate}. 
    Second, the \textit{discriminative modeling} is simply corresponded to maximizing class-conditional log-likelihoods, typically realized through the minimization of the cross-entropy loss.
    At the end of the client-side optimization, each client calculates energy values and energy gradients w.r.t. initial synthetic data downloaded from the central server and uploads them to the central server.

\paragraph{Note on the Energy Value}
    We empirically found that replacing the exponentiated energy value $\exp\left(f(\tilde{\boldsymbol{x}};\boldsymbol{\theta}_i)[\tilde{y}] \right)$ into exponentiated \textit{signed} energy value, 
    i.e., $\exp\left(\operatorname{sgn}(f(\tilde{\boldsymbol{x}};\boldsymbol{\theta}_i)[\tilde{y}]) \right)$, is beneficial in that it stabilizes the estimation of the energy gradient of the mixture distribution in eq.~\eqref{eq:moe_grad}.
    
    This is mainly due to the fact that the energy value $E(\tilde{\boldsymbol{x}},\tilde{y};\boldsymbol{\theta})=-f(\tilde{\boldsymbol{x}};\boldsymbol{\theta}_i)[\tilde{y}]\in\mathbb{R}$ is unbounded.
    We empirically found that the unbounded energy value can significantly influence the magnitude of the resulting energy gradient of the mixture distribution in eq.~\eqref{eq:moe_grad}.
    As the sign of the logit can only yields one of three values, $\exp\left(\operatorname{sgn}(f(\tilde{\boldsymbol{x}};\boldsymbol{\theta}_i)[\tilde{y}]) \right)\in\{\exp(-1), \exp(0), \exp(1)\}$, it prevents exploding or vanishing of the energy gradient,
    thereby the server can stably estimate the aggregated energy gradient in eq.~\eqref{eq:moe_grad}.

\newpage
\section{Experimental Results}
\subsection{Setup}
\paragraph{Datasets}
    We conduct experiments on four vision classification benchmarks, 
    that are i) widely used in FL studies: MNIST~\cite{mnist}, CIFAR-10~\cite{cifar},
    and ii) reflect practical cross-silo FL settings: DermaMNIST, OrganCMNIST~\cite{medmnist}.
    The number of classes of MNIST and CIFAR-10 is $L=10$, DermaMNIST is $L=7$, and OrganCMNIST is $L=11$.
    The spatial dimension of all inputs of i) is resized to be $32\times32$, and ii) is resized to be $64\times64$.
    
    To simulate the federated setting, we split the pre-defined training set of each dataset into $K$ subset using Dirichlet distribution for $L$ classes following~\cite{dirichlet}.
    From the Dirichlet distribution, we can sample a probability vector that can be used for designating a label distribution of each client.
    The degree of the statistical heterogeneity in the federated system can be adjusted by selecting an appropriate concentration parameter $\alpha$, where $\alpha=0$ represents complete heterogeneity in class labels across clients and $\alpha\rightarrow\infty$ represents homogeneous label distributions.
    In all experiments, we use $\alpha=\left\{0.01, 1.00\right\}$ for dataset partitioning.
    Note that when $\alpha$ is small, it is possible that some clients may have samples from partial classes, e.g., two out of $L$ classes.

\paragraph{Federated Settings}
    We use $K=\left\{ 10, 100 \right\}$ clients across all experiments with $T=500$ FL rounds.
    To secure communication efficiency, we set the client sampling probability as $C=\left\{ 1.0, 0.1 \right\}$, to guarantee that there exist at least 10 participating clients in each communication round.
    Unless otherwise stated, we use ResNet-10~\cite{resnet} as a classifier backbone in all experiments.
    All results are reported with averaged accuracy and a standard deviation across three different random seeds (seeds=1, 2, 3).
    Following~\cite{noniidconvergence}, we applied learning rate decay for all runs to stabilize the convergence behavior of FL algorithms.

\paragraph{Configurations of \texttt{FedEvg}}
    For our proposed method, we use Swish activation function~\cite{swish1,swish2,swish3} instead of rectified linear unit (ReLU) in the model architecture, 
    and apply the spectral normalization~\cite{spectralnorm} only for the last layer of a model. 
    Both of them empirically stabilize the EBM training with minimal additional computation.
    Pertaining to the server-side SGLD, we choose the step size schedule, 
    $\{\beta^{(t)}\}_{t=0}^{T-1}$, as $\beta^{(0)}=10.0$ and $\beta^{(T-1)}=0.01$ for all settings.
    The noise scale is also adjusted with the step size as $\delta=0.01$, and 
    The step size schedule is a decreasing sequence following the cosine scheduling proposed in~\cite{cosine}. 
    For the client-side SGLD, we find that only a single step ($R=1$) is sufficient for all experiments, which significantly boosts the speed of a local training. 
    The noise scale is selected as a fixed value of $\sigma=0.01$ and the step size is also fixed as $\gamma^{(0)}=...=\gamma^{(R-1)}=1$, following~\cite{your}.
    Lastly, the constant $\lambda$ is set to $0.1$, since we want the effect of synthetic data to be less influential than that of real data in local discriminative modeling. 
    Still, we expect the nonnegative $\lambda$ may mitigate the statistical heterogeneity in the local training.

\newpage
\subsection{Quantitative \& Qualitative Utilities of Synthetic Data}
\label{subsec:gen}
\paragraph{Baselines}
    Most of the existing work in federated synthetic data generation uses GANs~\cite{gan}, which consists of two sub-networks, i.e., a discriminator and a generator.
    We select one of the basic baselines, \texttt{FedCGAN}~\cite{gan3}, and let each client train both sub-networks in each round.
    For both benchmarks, we use the best hyperparameters reported in~\cite{gan3}.
    In addition, we select \texttt{FedCVAE} as an additional baseline, 
    which utilizes conditional VAEs~\cite{cvae} for generating synthetic samples in federated settings.
    We also search the best hyperparameters for \texttt{FedCVAE}.
    While there are another lines of works~\cite{dense,vae2} that also aims to achieve one-shot FL through synthetic data, our work is not restricted to a single-round communication setting. 
    Therefore, we exclude these methods from our baselines, since the scope of our work is to synthesise realistic samples under acceptable communication rounds.
    
\paragraph{Synthetic Data Generation}
    The number of synthetic data is designated as $M = L\times\text{spc}$, where $\text{spc}$ refers to `samples per class'.
    We set $\text{spc}=10$ for all experiments.
    Note that the parameter count of ResNet-10 is $\approx5\times10^6$, and total size of synthetic data for single-channel datasets is $M=10\times10\times1\times(32\times32)\approx1\times10^5$ for MNIST
    and $M=11\times10\times1\times(64\times64)\approx4.5\times10^5$ for OrganCMNIST.
    For the RGB-colored datasets, $M=10\times10\times3\times(32\times32)\approx3\times10^5$ for CIFAR-10
    and $M=7\times10\times3\times(64\times64)\approx8.6\times10^5$ for DermaMNIST.
    Therefore, the size of synthetic data $M$ is an order of magnitude smaller than the parameter size in all datasets. 
    
\paragraph{Evaluation of Synthetic Data}
    We evaluate the effectiveness of the generated synthetic data in two aspects:
    i) the efficacy of the synthetic data for mitigating heterogeneity
    and ii) the quality of the synthetic data as a proxy of local distributions.
    For the former, the server trains a classifier at the end of the FL training round using the generated synthetic data in the server.
    The classifier is then broadcast to all clients, and each client reports the performance of the global classifier using own local holdout set.
    We report the results in Table~\ref{tab:disc} and Table~\ref{tab:disc2}.
    For the latter, we compute the Fréchet Inception Distance (FID)~\cite{fid} of the server-side synthetic data using the pre-defined test set of each benchmark.
    Note that smaller FID is preferred, since it is the distance of activations between the synthetic and the real data using the large scale pre-trained deep networks.
    Corresponding results are reported in Table~\ref{tab:gen} and Table~\ref{tab:gen2}.

    From both quantitative and qualitative perspectives, the synthetic data generated by the proposed method (i.e., \texttt{FedEvg}) show superior performances,
    regardless of the number of participating clients (10 or 100) and the degree of statistical heterogeneity ($\alpha=0.01,1.00$).
    \begin{table}[H]
\centering
\caption{Discriminative performance (Accuracy) of the server-side classifier trained on generated synthetic data (higher is better) --- MNIST and CIFAR-10}
\label{tab:disc}
\resizebox{\textwidth}{!}{%
\begin{tabular}{!{}lcccc|cccc!{}}
\toprule
\textbf{Dataset} & \multicolumn{4}{c|}{\textbf{MNIST}} & \multicolumn{4}{c}{\textbf{CIFAR-10}} \\
\textbf{Method} & \multicolumn{4}{c|}{\small(Acc. 1)} & \multicolumn{4}{c}{\small(Acc. 1)} \\ \midrule
\multicolumn{1}{c}{\# clients} & \multicolumn{2}{c}{10} & \multicolumn{2}{c|}{100} & \multicolumn{2}{c}{10} & \multicolumn{2}{c}{100} \\
\multicolumn{1}{c}{concentration ($\alpha$)} & 0.01 & \multicolumn{1}{c}{1.00} & 0.01 & 1.00 & 0.01 & \multicolumn{1}{c}{1.00} & 0.01 & 1.00 \\ \midrule
\texttt{FedCGAN} & \begin{tabular}[c]{@{}c@{}}44.18\\ \footnotesize \color[HTML]{9B9B9B}(16.09)\end{tabular} & \multicolumn{1}{c}{\begin{tabular}[c]{@{}c@{}}60.81\\ \footnotesize \color[HTML]{9B9B9B}(9.12)\end{tabular}} & \begin{tabular}[c]{@{}c@{}}56.39\\ \footnotesize \color[HTML]{9B9B9B}(12.68)\end{tabular} & \begin{tabular}[c]{@{}c@{}}68.06\\ \footnotesize \color[HTML]{9B9B9B}(11.78)\end{tabular} & \begin{tabular}[c]{@{}c@{}}9.96\\ \footnotesize \color[HTML]{9B9B9B}(1.42)\end{tabular} & \multicolumn{1}{c}{\begin{tabular}[c]{@{}c@{}}15.95\\ \footnotesize \color[HTML]{9B9B9B}(4.06)\end{tabular}} & \begin{tabular}[c]{@{}c@{}}14.88\\ \footnotesize \color[HTML]{9B9B9B}(1.78)\end{tabular} & \begin{tabular}[c]{@{}c@{}}21.31\\ \footnotesize \color[HTML]{9B9B9B}(2.42)\end{tabular} \\
\texttt{FedCVAE} & \begin{tabular}[c]{@{}c@{}}\textbf{66.93}\\ \footnotesize \color[HTML]{9B9B9B}(4.92)\end{tabular} & \multicolumn{1}{c}{\begin{tabular}[c]{@{}c@{}}64.73\\ \footnotesize \color[HTML]{9B9B9B}(2.53)\end{tabular}} & \begin{tabular}[c]{@{}c@{}}63.13\\ \footnotesize \color[HTML]{9B9B9B}(3.91)\end{tabular} & \begin{tabular}[c]{@{}c@{}}63.80\\ \footnotesize \color[HTML]{9B9B9B}(4.30)\end{tabular} & \begin{tabular}[c]{@{}c@{}}13.98\\ \footnotesize \color[HTML]{9B9B9B}(5.46)\end{tabular} & \multicolumn{1}{c}{\begin{tabular}[c]{@{}c@{}}7.66\\ \footnotesize \color[HTML]{9B9B9B}(2.01)\end{tabular}} & \begin{tabular}[c]{@{}c@{}}12.86\\ \footnotesize \color[HTML]{9B9B9B}(1.81)\end{tabular} & \begin{tabular}[c]{@{}c@{}}12.04\\ \footnotesize \color[HTML]{9B9B9B}(0.57)\end{tabular} \\
\rowcolor[HTML]{FFF5E6} 
\texttt{FedEvg} & \begin{tabular}[c]{@{}c@{}}62.95\\ \footnotesize \color[HTML]{9B9B9B}(5.43)\end{tabular} & \multicolumn{1}{c}{\cellcolor[HTML]{FFF5E6}\begin{tabular}[c]{@{}c@{}}\textbf{70.94}\\ \footnotesize \color[HTML]{9B9B9B}(3.20)\end{tabular}} & \cellcolor[HTML]{FFF5E6}\begin{tabular}[c]{@{}c@{}}\textbf{72.25}\\ \footnotesize \color[HTML]{9B9B9B}(3.56)\end{tabular} & \begin{tabular}[c]{@{}c@{}}\textbf{86.49}\\ \footnotesize \color[HTML]{9B9B9B}(2.38)\end{tabular} & \begin{tabular}[c]{@{}c@{}}\textbf{17.25}\\ \footnotesize \color[HTML]{9B9B9B}(2.21)\end{tabular} & \multicolumn{1}{c}{\cellcolor[HTML]{FFF5E6}\begin{tabular}[c]{@{}c@{}}\textbf{29.08}\\ \footnotesize \color[HTML]{9B9B9B}(3.44)\end{tabular}} & \cellcolor[HTML]{FFF5E6}\begin{tabular}[c]{@{}c@{}}\textbf{42.68}\\ \footnotesize \color[HTML]{9B9B9B}(1.82)\end{tabular} & \begin{tabular}[c]{@{}c@{}}\textbf{58.57}\\ \footnotesize \color[HTML]{9B9B9B}(2.05)\end{tabular} \\ \bottomrule
\end{tabular}%
}
\end{table}
    \begin{table}[H]
\centering
\caption{Discriminative performance (Accuracy) of the server-side classifier trained on generated synthetic data (higher is better) --- DermaMNIST and OrganCMNIST}
\label{tab:disc2}
\resizebox{\textwidth}{!}{%
\begin{tabular}{!{}lcccccccc!{}}
\toprule
\textbf{Dataset} & \multicolumn{4}{c}{\textbf{DermaMNIST}} & \multicolumn{4}{c}{\textbf{OrganCMNIST}} \\
\textbf{Method} & \multicolumn{4}{c|}{\small (Acc. 1)} & \multicolumn{4}{c}{\small (Acc. 1)} \\ \midrule
\multicolumn{1}{c}{\# clients} & \multicolumn{2}{c}{10} & \multicolumn{2}{c|}{100} & \multicolumn{2}{c}{10} & \multicolumn{2}{c}{100} \\
\multicolumn{1}{c}{concentration ($\alpha$)} & 0.01 & 1.00 & 0.01 & \multicolumn{1}{c|}{1.00} & 0.01 & 1.00 & 0.01 & 1.00 \\ \midrule
\texttt{FedCGAN} & \begin{tabular}[c]{@{}c@{}}3.74\\ \footnotesize \color[HTML]{9B9B9B}(20.42)\end{tabular} & \begin{tabular}[c]{@{}c@{}}10.15\\ \footnotesize \color[HTML]{9B9B9B}(6.73)\end{tabular} & \begin{tabular}[c]{@{}c@{}}11.39\\ \footnotesize \color[HTML]{9B9B9B}(19.20)\end{tabular} & \multicolumn{1}{c|}{\begin{tabular}[c]{@{}c@{}}21.95\\ \footnotesize \color[HTML]{9B9B9B}(15.39)\end{tabular}} & \begin{tabular}[c]{@{}c@{}}2.46\\ \footnotesize \color[HTML]{9B9B9B}(5.45)\end{tabular} & \begin{tabular}[c]{@{}c@{}}19.24\\ \footnotesize \color[HTML]{9B9B9B}(4.75)\end{tabular} & \begin{tabular}[c]{@{}c@{}}4.38\\ \footnotesize \color[HTML]{9B9B9B}(15.02)\end{tabular} & \begin{tabular}[c]{@{}c@{}}26.11\\ \footnotesize \color[HTML]{9B9B9B}(11.50)\end{tabular} \\
\texttt{FedCVAE} & \begin{tabular}[c]{@{}c@{}}23.56\\ \footnotesize \color[HTML]{9B9B9B}(32.48)\end{tabular} & \begin{tabular}[c]{@{}c@{}}16.18\\ \footnotesize \color[HTML]{9B9B9B}(12.34)\end{tabular} & \begin{tabular}[c]{@{}c@{}}15.84\\ \footnotesize \color[HTML]{9B9B9B}(22.92)\end{tabular} & \multicolumn{1}{c|}{\begin{tabular}[c]{@{}c@{}}22.18\\ \footnotesize \color[HTML]{9B9B9B}(13.04)\end{tabular}} & \begin{tabular}[c]{@{}c@{}}8.55\\ \footnotesize \color[HTML]{9B9B9B}(16.86)\end{tabular} & \begin{tabular}[c]{@{}c@{}}14.32\\ \footnotesize \color[HTML]{9B9B9B}(9.03)\end{tabular} & \begin{tabular}[c]{@{}c@{}}11.12\\ \footnotesize \color[HTML]{9B9B9B}(22.14)\end{tabular} & \begin{tabular}[c]{@{}c@{}}15.04\\ \footnotesize \color[HTML]{9B9B9B}(10.98)\end{tabular} \\
\rowcolor[HTML]{FFF5E6} 
\texttt{FedEvg} & \begin{tabular}[c]{@{}c@{}}\textbf{32.40}\\ \footnotesize \color[HTML]{9B9B9B}(23.58)\end{tabular} & \begin{tabular}[c]{@{}c@{}}\textbf{51.31}\\ \footnotesize \color[HTML]{9B9B9B}(10.12)\end{tabular} & \begin{tabular}[c]{@{}c@{}}\textbf{38.89}\\ \footnotesize \color[HTML]{9B9B9B}(31.04)\end{tabular} & \multicolumn{1}{c|}{\cellcolor[HTML]{FFF5E6}\begin{tabular}[c]{@{}c@{}}\textbf{63.23}\\ \footnotesize \color[HTML]{9B9B9B}(17.54)\end{tabular}} & \begin{tabular}[c]{@{}c@{}}\textbf{37.16}\\ \footnotesize \color[HTML]{9B9B9B}(18.51)\end{tabular} & \begin{tabular}[c]{@{}c@{}}\textbf{44.87}\\ \footnotesize \color[HTML]{9B9B9B}(10.31)\end{tabular} & \begin{tabular}[c]{@{}c@{}}\textbf{54.58}\\ \footnotesize \color[HTML]{9B9B9B}(26.78)\end{tabular} & \begin{tabular}[c]{@{}c@{}}\textbf{77.59}\\ \footnotesize \color[HTML]{9B9B9B}(12.26)\end{tabular} \\ \bottomrule
\end{tabular}%
}
\end{table}
    \begin{table}[H]
\centering
\caption{Generative performance (FID) of the server-side synthetic data (lower is better) --- MNIST and CIFAR-10}
\label{tab:gen}
\resizebox{\textwidth}{!}{%
\begin{tabular}{!{}lcccc|cccc!{}}
\toprule
\textbf{Dataset} & \multicolumn{4}{c|}{\textbf{MNIST}} & \multicolumn{4}{c}{\textbf{CIFAR-10}} \\
\textbf{Method} & \multicolumn{4}{c|}{\small(FID)} & \multicolumn{4}{c}{\small(FID)} \\ \midrule
\multicolumn{1}{c}{\# clients} & \multicolumn{2}{c}{10} & \multicolumn{2}{c|}{100} & \multicolumn{2}{c}{10} & \multicolumn{2}{c}{100} \\
\multicolumn{1}{c}{concentration ($\alpha$)} & 0.01 & \multicolumn{1}{c}{1.00} & 0.01 & 1.00 & 0.01 & \multicolumn{1}{c}{1.00} & 0.01 & 1.00 \\ \midrule
\texttt{FedCGAN} & \begin{tabular}[c]{@{}c@{}}331.48\\ \footnotesize \color[HTML]{9B9B9B}(0.23)\end{tabular} & \multicolumn{1}{c}{\begin{tabular}[c]{@{}c@{}}300.67\\ \footnotesize \color[HTML]{9B9B9B}(0.59)\end{tabular}} & \begin{tabular}[c]{@{}c@{}}353.86\\ \footnotesize \color[HTML]{9B9B9B}(1.78)\end{tabular} & \begin{tabular}[c]{@{}c@{}}325.66\\ \footnotesize \color[HTML]{9B9B9B}(1.78)\end{tabular} & \begin{tabular}[c]{@{}c@{}}388.13\\ \footnotesize \color[HTML]{9B9B9B}(0.67)\end{tabular} & \multicolumn{1}{c}{\begin{tabular}[c]{@{}c@{}}291.59\\ \footnotesize \color[HTML]{9B9B9B}(1.21)\end{tabular}} & \begin{tabular}[c]{@{}c@{}}363.78\\ \footnotesize \color[HTML]{9B9B9B}(0.17)\end{tabular} & \begin{tabular}[c]{@{}c@{}}326.10\\ \footnotesize \color[HTML]{9B9B9B}(0.78)\end{tabular} \\
\texttt{FedCVAE} & \begin{tabular}[c]{@{}c@{}}214.35\\ \footnotesize \color[HTML]{9B9B9B}(0.15)\end{tabular} & \multicolumn{1}{c}{\begin{tabular}[c]{@{}c@{}}181.30\\ \footnotesize \color[HTML]{9B9B9B}(0.08)\end{tabular}} & \begin{tabular}[c]{@{}c@{}}199.71\\ \footnotesize \color[HTML]{9B9B9B}(0.94)\end{tabular} & \begin{tabular}[c]{@{}c@{}}171.36\\ \footnotesize \color[HTML]{9B9B9B}(2.01)\end{tabular} & \begin{tabular}[c]{@{}c@{}}399.27\\ \footnotesize \color[HTML]{9B9B9B}(1.03)\end{tabular} & \multicolumn{1}{c}{\begin{tabular}[c]{@{}c@{}}396.34\\ \footnotesize \color[HTML]{9B9B9B}(0.59)\end{tabular}} & \begin{tabular}[c]{@{}c@{}}385.54\\ \footnotesize \color[HTML]{9B9B9B}(0.33)\end{tabular} & \begin{tabular}[c]{@{}c@{}}376.58\\ \footnotesize \color[HTML]{9B9B9B}(1.91)\end{tabular} \\
\rowcolor[HTML]{FFF5E6} 
\texttt{FedEvg} & \begin{tabular}[c]{@{}c@{}}\textbf{171.53}\\ \footnotesize \color[HTML]{9B9B9B}(0.24)\end{tabular} & \multicolumn{1}{c}{\cellcolor[HTML]{FFF5E6}\begin{tabular}[c]{@{}c@{}}\textbf{124.72}\\ \footnotesize \color[HTML]{9B9B9B}(0.38)\end{tabular}} & \cellcolor[HTML]{FFF5E6}\begin{tabular}[c]{@{}c@{}}\textbf{175.90}\\ \footnotesize \color[HTML]{9B9B9B}(1.42)\end{tabular} & \begin{tabular}[c]{@{}c@{}}\textbf{154.79}\\ \footnotesize \color[HTML]{9B9B9B}(1.77)\end{tabular} & \begin{tabular}[c]{@{}c@{}}\textbf{359.18}\\ \footnotesize \color[HTML]{9B9B9B}(1.50)\end{tabular} & \multicolumn{1}{c}{\cellcolor[HTML]{FFF5E6}\begin{tabular}[c]{@{}c@{}}\textbf{286.26}\\ \footnotesize \color[HTML]{9B9B9B}(0.10)\end{tabular}} & \cellcolor[HTML]{FFF5E6}\begin{tabular}[c]{@{}c@{}}\textbf{338.42}\\ \footnotesize \color[HTML]{9B9B9B}(0.64)\end{tabular} & \begin{tabular}[c]{@{}c@{}}\textbf{316.49}\\ \footnotesize \color[HTML]{9B9B9B}(0.89)\end{tabular} \\ \bottomrule
\end{tabular}%
}
\end{table}
    \begin{table}[H]
\centering
\caption{Generative performance (FID) of the server-side synthetic data (lower is better) --- DermaMNIST and OrganCMNIST}
\label{tab:gen2}
\resizebox{\textwidth}{!}{%
\begin{tabular}{!{}lcccc|cccc!{}}
\toprule
\textbf{Dataset} & \multicolumn{4}{c|}{\textbf{DermaMNIST}} & \multicolumn{4}{c}{\textbf{OrganCMNIST}} \\
\textbf{Method} & \multicolumn{4}{c|}{\small (FID)} & \multicolumn{4}{c}{\small (FID)} \\ \midrule
\multicolumn{1}{c}{\# clients} & \multicolumn{2}{c}{10} & \multicolumn{2}{c|}{100} & \multicolumn{2}{c}{10} & \multicolumn{2}{c}{100} \\
\multicolumn{1}{c}{concentration ($\alpha$)} & 0.01 & 1.00 & 0.01 & 1.00 & 0.01 & 1.00 & 0.01 & 1.00 \\ \midrule
\texttt{FedCGAN} & \begin{tabular}[c]{@{}c@{}}372.89\\ \footnotesize \color[HTML]{9B9B9B}(11.12)\end{tabular} & \begin{tabular}[c]{@{}c@{}}\textbf{296.88}\\ \footnotesize \color[HTML]{9B9B9B}(15.91)\end{tabular} & \begin{tabular}[c]{@{}c@{}}364.40\\ \footnotesize \color[HTML]{9B9B9B}(7.18)\end{tabular} & \begin{tabular}[c]{@{}c@{}}326.60\\ \footnotesize \color[HTML]{9B9B9B}(12.64)\end{tabular} & \begin{tabular}[c]{@{}c@{}}389.88\\ \footnotesize \color[HTML]{9B9B9B}(16.77)\end{tabular} & \begin{tabular}[c]{@{}c@{}}333.50\\ \footnotesize \color[HTML]{9B9B9B}(12.98)\end{tabular} & \begin{tabular}[c]{@{}c@{}}461.91\\ \footnotesize \color[HTML]{9B9B9B}(17.06)\end{tabular} & \begin{tabular}[c]{@{}c@{}}386.53\\ \footnotesize \color[HTML]{9B9B9B}(16.21)\end{tabular} \\
\texttt{FedCVAE} & \begin{tabular}[c]{@{}c@{}}355.76\\ \footnotesize \color[HTML]{9B9B9B}(15.02)\end{tabular} & \begin{tabular}[c]{@{}c@{}}309.81\\ \footnotesize \color[HTML]{9B9B9B}(8.89)\end{tabular} & \begin{tabular}[c]{@{}c@{}}351.55\\ \footnotesize \color[HTML]{9B9B9B}(21.94)\end{tabular} & \begin{tabular}[c]{@{}c@{}}316.43\\ \footnotesize \color[HTML]{9B9B9B}(13.51)\end{tabular} & \begin{tabular}[c]{@{}c@{}}442.12\\ \footnotesize \color[HTML]{9B9B9B}(15.08)\end{tabular} & \begin{tabular}[c]{@{}c@{}}379.25\\ \footnotesize \color[HTML]{9B9B9B}(17.64)\end{tabular} & \begin{tabular}[c]{@{}c@{}}350.21\\ \footnotesize \color[HTML]{9B9B9B}(16.41)\end{tabular} & \begin{tabular}[c]{@{}c@{}}438.63\\ \footnotesize \color[HTML]{9B9B9B}(19.01)\end{tabular} \\
\rowcolor[HTML]{FFF5E6} 
\texttt{FedEvg} & \begin{tabular}[c]{@{}c@{}}\textbf{317.52}\\ \footnotesize \color[HTML]{9B9B9B}(10.73)\end{tabular} & \begin{tabular}[c]{@{}c@{}}319.34\\ \footnotesize \color[HTML]{9B9B9B}(13.11)\end{tabular} & \begin{tabular}[c]{@{}c@{}}\textbf{272.89}\\ \footnotesize \color[HTML]{9B9B9B}(11.42)\end{tabular} & \begin{tabular}[c]{@{}c@{}}\textbf{235.84}\\ \footnotesize \color[HTML]{9B9B9B}(12.77)\end{tabular} & \begin{tabular}[c]{@{}c@{}}\textbf{298.27}\\ \footnotesize \color[HTML]{9B9B9B}(12.14)\end{tabular} & \begin{tabular}[c]{@{}c@{}}\textbf{303.14}\\ \footnotesize \color[HTML]{9B9B9B}(13.01)\end{tabular} & \begin{tabular}[c]{@{}c@{}}\textbf{209.96}\\ \footnotesize \color[HTML]{9B9B9B}(12.28)\end{tabular} & \begin{tabular}[c]{@{}c@{}}\textbf{188.01}\\ \footnotesize \color[HTML]{9B9B9B}(13.05)\end{tabular} \\ \bottomrule
\end{tabular}%
}
\end{table}

\newpage
\paragraph{Visualization of Synthetic Data}
    The visualizations of synthetic data of each dataset (\texttt{MNIST}, \texttt{CIFAR-10}, \texttt{OrganCMNIST}, and \texttt{DermaMNIST}) 
    from each method (\texttt{FedCGAN}, \texttt{FedCVAE}, and \texttt{FedEvg}) are provided in Figure~\ref{fig:mnist},~\ref{fig:cifar},~\ref{fig:organcmnist} and Figure~\ref{fig:dermamnist}.
    All results are obtained in the setting of $K=10$ clients with $\alpha=0.01$ and $R=500$ rounds.
    Although the synthetic data generated are not as diverse as the original data, they are all sufficiently representative of each class. 

    Compared to existing methods, \texttt{FedEvg} shows far better generation quality compared to baselines.
    For example, pertaining to the result of \texttt{MNIST} dataset (Figure~\ref{fig:mnist}), \texttt{FedCVAE} fails to generate samples from classes of digit two and digit six.
    Rather, they resemble synthetic data samples from the class of digit four, which implies that \texttt{FedCVAE} does not fully understand the structural difference between these classes.
    
    In case of more complicated dataset (in terms of in-class diversity and data dimension, i.e., colored dataset), such as \texttt{CIFAR-10}, 
    the excellence of the proposed method is even more emphasized.
    While \texttt{FedCGAN} generates samples that appeared to be separated by class, but samples are too noisy and of poor quality.
    Worse, in \texttt{FedCVAE}, the generated results are so blurry that they are indiscernible by visual perception.
    
    In contrast, \texttt{FedEvg} clearly produces plausible and diverse samples whose class can be easily inferred from their appearance.
    When if the size of data is increased from $32\times32$ to $64\times64$ (e.g., \texttt{OrganCMNIST} and \texttt{DermaMNIST}), 
    the aforementioned trends are even accentuated, proving the superiority of the proposed method compared over existing baselines.
    This may promote the effectiveness of synthetic data generated by \texttt{FedEvg} for the downstream tasks after federated training, such as providing synthetic data for sample-deficient clients, or for data homogenization as a public shared dataset.     
    \begin{sidewaysfigure}
        \centering
        \includegraphics[width=\textwidth,keepaspectratio]{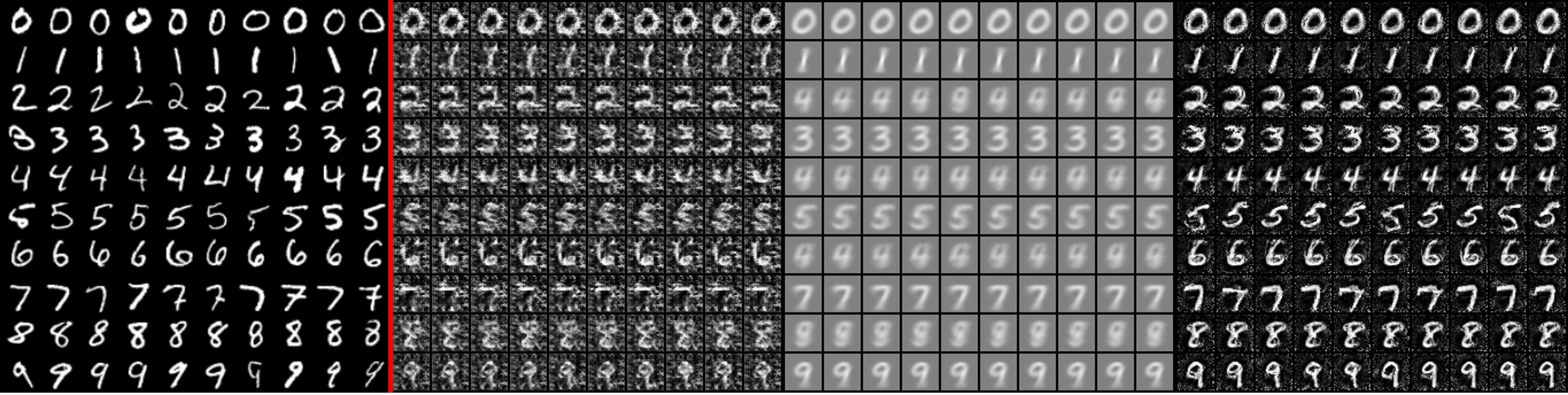}
        \caption[Generation results of MNIST dataset]
        {The first image to the left of the red vertical bar is a set of random samples from the original \texttt{MNIST} dataset.
        The sets of images to the right of the red vertical bar are the results generated by \texttt{FedCGAN}, the second by \texttt{FedCVAE}, and the third by \texttt{FedEvg}.
        Each row in each set of images consists of samples from the same class, and there are 10 samples (i.e., $\text{spc}=10$) in each row.}
        \label{fig:mnist}
    \end{sidewaysfigure}
    \begin{sidewaysfigure}
        \centering
        \includegraphics[width=\textwidth,keepaspectratio]{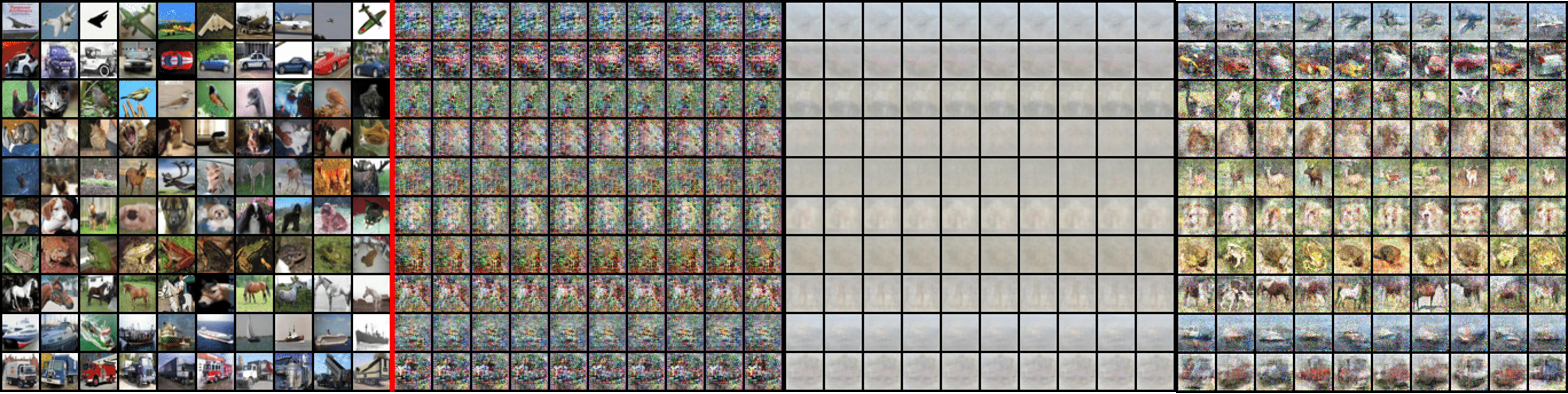}
        \caption[Generation results of CIFAR-10 dataset]
        {The first image to the left of the red vertical bar is a set of random samples from the original \texttt{CIFAR-10} dataset.
        The sets of images to the right of the red vertical bar are the results generated by \texttt{FedCGAN}, the second by \texttt{FedCVAE}, and the third by \texttt{FedEvg}.
        Each row in each set of images consists of samples from the same class, and there are 10 samples (i.e., $\text{spc}=10$) in each row.}
        \label{fig:cifar}
    \end{sidewaysfigure}
    \begin{sidewaysfigure}
        \centering
        \includegraphics[width=\textwidth,keepaspectratio]{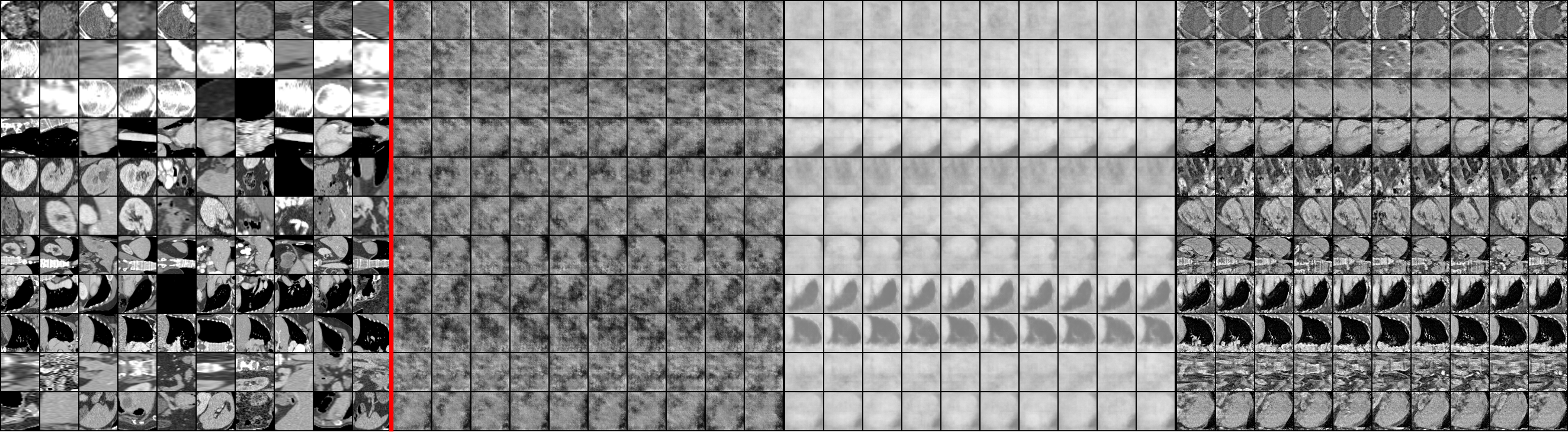}
        \caption[Generation results of OrganCMNIST dataset]
        {The first image to the left of the red vertical bar is a set of random samples from the original \texttt{OrganCMNIST} dataset.
        The sets of images to the right of the red vertical bar are the results generated by \texttt{FedCGAN}, the second by \texttt{FedCVAE}, and the third by \texttt{FedEvg}.
        Each row in each set of images consists of samples from the same class, and there are 10 samples (i.e., $\text{spc}=10$) in each row.}
        \label{fig:organcmnist}
    \end{sidewaysfigure}
    \begin{sidewaysfigure}
        \centering
        \includegraphics[width=\textwidth,keepaspectratio]{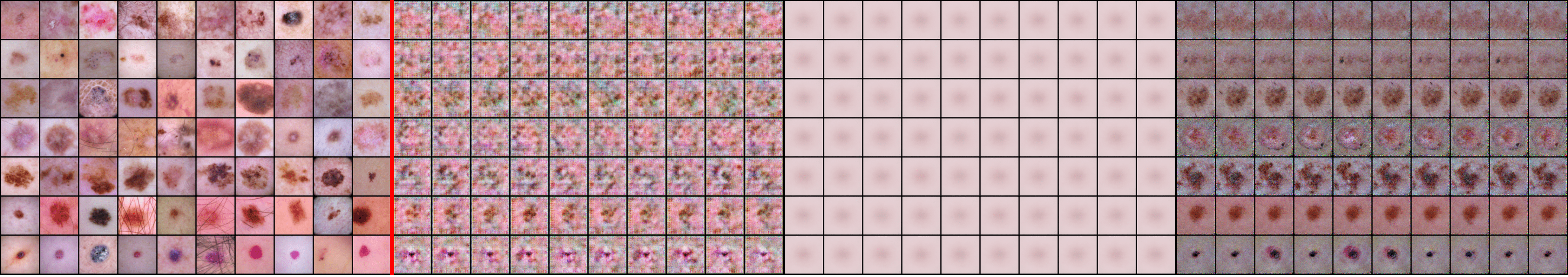}
        \caption[Generation results of DermaMNIST dataset]
        {The first image to the left of the red vertical bar is a set of random samples from the original \texttt{DermaMNIST} dataset.
        The sets of images to the right of the red vertical bar are the results generated by \texttt{FedCGAN}, the second by \texttt{FedCVAE}, and the third by \texttt{FedEvg}.
        Each row in each set of images consists of samples from the same class, and there are 10 samples (i.e., $\text{spc}=10$) in each row.}
        \label{fig:dermamnist}
    \end{sidewaysfigure}

\newpage
\section{Related Works}
\subsection{Synthetic Data in Federated Learning}
    FL is a simple and scalable distributed optimization method ~\cite{fedavg}.
    However, since the orchestrator (i.e., the central server) does not have direct access to the local samples of the participating clients, 
    the curse of statistical heterogeneity always threatens the federated system.
    Since it affects different local optimization trajectories, 
    a global model constructed from these disparate local updates may not achieve satisfactory performance or may not be even converged. 
    Most previous studies are \textit{model-centric} approaches, focusing on rectifying heterogeneous local updates to be aligned with each other.
    This includes a regularization-based methods~\cite{opt1,opt2}, usage of control variates~\cite{scaffold}, server-side momentum~\cite{dirichlet},
    or replacing aggregation schemes~\cite{agg1,agg2,agg3}, and even equipping with additional models for personalization~\cite{pfl2,pfl3,pfl4}.
    Orthogonal to these methods, \textit{data-centric} approaches have been relatively less explored.
    This encompasses the sharing of pre-defined server-side public datasets~\cite{noniid,pub1,pub2} and the use of virtual data initialized from an untrained generative model (e.g.,~\cite{vhl,ccvr}) or resorting to the low-dimensional data, which is latent representations of local models~\cite{fedfed,fedproto,datafreekd}. 
    However, the public data is not always available, and other surrogates cannot be used separately when FL training is completed. 
    Therefore, the usage of synthetic data has recently received considerable attention~\cite{syntheticflreview}.
    While promising, training a generative model for the synthetic data in federated settings is undoubtedly challenging. 
    Few studies have been conducted on this topic, with the majority of research concentrating on GANs~\cite{gan1,gan2,gan3,mdgan,priflgan}, but it requires heavy computation and communication cost, and is vulnerable to the mode collapse, a common failure of generative training.
    Other lines of work~\cite{fedsynth,feddm,fedaf} are built upon the dataset condensation method~\cite{datasetdistill,datasetcondens}, 
    but do not produce plausible synthetic data, since the generated samples are merely a compressed version of the raw counterparts.
    Our method can not only contribute to the generation of plausible synthetic data without the need to prepare a generative model,
    but can also provide local discriminative models with a slight modification in the training procedure.
    Thus, \texttt{FedEvg} can achieve the best of both worlds: server-side benefits of acquiring synthetic data samples and client-side benefits of obtaining an improved local model.

\newpage
\subsection{Energy-Based Model}
    EBMs~\cite{ebm} are a classic method of modeling an unnormalized probability density.
    The omission of the normalizing constant estimation gives EBMs much flexibility,
    such as no restriction in the model family and easy compositional operations between EBMs.
    Our work relies heavily on~\cite{your}, which revealed that the EBM can be externalized from any classifiers.
    We modify and extend this method to federated settings, where multiple classifiers are trained separately in parallel.
    To harmonize these EBM-driven classifiers, we choose to treat the sampling from the composition of EBMs as the sampling from implicitly aggregated local density functions.
    Although the composition of EBMs, also known as a mixture-of-expert, has been explored in previous works~\cite{ebmcomp1,ebmcomp2,ebmcomp3} in the data centralized setting, 
    \texttt{FedEvg} is, to the best of our knowledge, the first to bring this intriguing property into federated settings.

\newpage
\section{Discussion}
    While EBMs are flexible by not modeling a normalizing constant, their training is non-trivial in return.
    Moreover, as we do not have access to the (normalized) log-likelihood during training, this obfuscates monitoring the status of EBM training.
    Therefore, stabilizing the training of local EBMs in our work with improved techniques can be a promising research direction.
    For example, MCMC teaching~\cite{coopflow,coopdiff,adaptiveflow,mcmcteaching} and connection to diffusion modeling~\cite{ebmscore,ebmcomp3} may be good candidates.
    
    As an alternative to EBMs, score-based generative modeling~\cite{ncsn,ncsn2,ncsnpp} shows a strong generation quality by estimating the score (i.e., energy gradient) directly from the model, rather than the energy value.
    However, one can hesitate to adopt score-based modeling for the following reasons:
    First, to realize \textit{the sampling from the mixture distribution}, it is essential to know the local energy value to estimate the energy gradient as introduced in eq.~\eqref{eq:moe_grad}.
    While some tricks are introduced to construct a proxy for energy value from the score-based approach~\cite{ebmcomp3,ebmscore}, they are only supported by empirical evidence.
    Second, when using score-based modeling, each client in the federated system should consume much more local computation and sample size than the EBM to improve the generation quality of server-side synthetic data.  
    Nevertheless, it is worthwhile to use the score-based approach in future work, as more advanced techniques have been proposed for efficient training and compositional generation of score-based generative models.  
    
    Another strong candidate that can replace the current EBM-based scheme is Probabilistic Circuits (PCs)~\cite{pc1,pc2}.
    It is fundamentally designed to model probabilistic distributions based on generating polynomials between random variables, and thus it naturally supports compositional operations as EBMs.
    While this work aggressively assumes that the normalizing constants of different local EBMs are equal to each other (as in eq.~\eqref{eq:samenormconst}), this is hardly acceptable in practice.
    It is an unavoidable assumption when using EBMs as the core engine, however, the intractable modeling of the probability density is certainly a part to be improved, especially when the effects of statistical heterogeneity are also present.
    From this perspective, PCs can solve the key challenge of modeling the normalization constant because they essentially model a tractable probability density.
    Thus, replacing EBMs with PCs may address major problems with EBMs without modifying the proposed main assumption (i.e., sampling from a mixture of local distributions) for federated synthetic data generation.
    
    Apart from the modeling scheme, the empirical evaluation of current work should also be extended outside the vision domain, e.g., text and tabular benchmarks, to evaluate the practicality of the proposed method.
    In a similar context, an explicit introduction of privacy-preserving mechanisms, e.g., differential privacy~\cite{dp}, should also be considered, mirroring the deployment scenario of the federated system equipped with the proposed method.
    Although we avoid explicit risks by initializing synthetic data on the server and not using the raw energy value directly, the potential risk of specific privacy leakage from various threat scenarios requires further investigation.
    Last but not least, our method can be roughly viewed as a cooperative MCMC sampling method.
    Hence, we can borrow theoretical analyses from the rich MCMC literature on the convergence behavior of \texttt{FedEvg} to the true unknown density.

\newpage
\section{Conclusion}
    In this work, we propose \texttt{FedEvg}, a collaborative synthetic data generation method in a data-decentralized setting.
    Our method not only supports the generative ability in terms of server-side synthetic data generation, but also guarantees the strong discriminative performance required for the good of the participating clients.
    The main motivation is the hidden connection to energy-based models, which supports flexible modeling of an unknown probability density, as well as easy compositions between them.
    By viewing federated synthetic data generation as a collaborative sampling process from a mixture of inaccessible local densities, we can easily secure synthetic data using locally constructed energy signals.
    As there is a growing need for the use of synthetic data in practical federated settings, we hope that our method can serve as a stepping stone.

\newpage 
\chapter{Concluding Remarks} 
    Starting from the main objective of FL, we have presented three distinct perspectives for salvaging the collaborative learning method from the statistical heterogeneity problem: a parameter, a mixing coefficient, and local distributions.
    Contrary to dealing with a single global model as in traditional FL, a model-mixture based personalization method, \texttt{SuPerFed}, is first explored (Chapter~\ref{ch:superfed}).
    Instead of simply mixing different models, a special regularization-based method based on mode connectivity is proposed to promote explicit synergies between models.
    Furthermore, a principled online decision making framework, \texttt{AAggFF}, is established, which is implemented by finding an optimal mixing coefficient in each communication round (Chapter~\ref{ch:aaggff}).
    This is especially suited for server-side sequential decision making in that only a few bits are provided in updating the mixing coefficient in an adaptive manner.
    Lastly, a collaborative synthetic data generation method, \texttt{FedEvg}, is proposed (Chapter~\ref{ch:fedevg}), which does not depend upon the exchange of parameters.
    By disclosing hidden connections to energy-based modeling, the main objective is transformed into sampling from a mixture of inaccessible local distributions.
    Instead of communicating a model parameter, the federated system can exploit energy-based signals that contain information about local densities.

    Nevertheless, there is much room for improvement in each approach, which is a limitation of the current dissertation as well as a direction for future research.
    First of all, \texttt{SuPerFed} assumes stateful clients, where each client keeps its own local model until the end of FL training.
    While this can improve personalization performance in the cross-silo FL setting, it may not be appropriate for the cross-device setting since it assumes that each client is stateless.
    Thus, improving the method to be performed in a cross-device setting can be a promising direction.
    Second, the theoretical analysis of \texttt{AAggFF} is only presented for the mixing coefficient in terms of vanishing regret, but the convergence of finding a parameter still remains unanswered.
    This may be connected to min-max stochastic optimization or bi-level optimization, both of which are relatively underexplored in FL settings. 
    Moreover, a combination of online convex optimization (for a mixing coefficient) and stochastic optimization (for a parameter) is itself novel and has sufficient value to be studied.
    Third but not least, advanced methods for the stable training of an EBM certainly improve the utility of \texttt{FedEvg}.
    Since \texttt{FedEvg} communicates signals other than a model parameter, a new theoretical analysis should be explored for analyzing the convergence behavior.
    Therefore, an independent analysis should be provided to explain the behavior of this algorithm, 
    as well as what kind of disclosure risks exist and how we can prevent them using proper add-ons, e.g., Differential Privacy (DP).
    Last but not least, the practicality and scalability of all approaches can be greatly enhanced if each of them can deal with the setting close to reality, such as a classification task with an extreme number of classes, recognition of high-fidelity inputs, and long-sequence modeling.
    
    While many studies are proliferating in each direction, i.e., parameters, mixing coefficients, and local distributions, 
    this dissertation is the first to address all three perspectives in order.
    I hope that this work will be a minimal step towards improving the current FL framework into a more practical and scalable technology, as well as broadening the perspectives of the collaborative machine learning paradigm for machine learning as a whole.

\clearpage

\phantomsection
\addcontentsline{toc}{chapter}{References}
\bibliographystyle{IEEEtranS}
\bibliography{reference.bib}
\clearpage

\phantomsection
\addcontentsline{toc}{chapter}{Acknowledgements}
\chapter*{\centerline{\Large Acknowledgements}}
    A long journey has come to an end. The word `philosophy' in `Ph.D' means the love (philia) of learning (sophia).
    As the one who really enjoys learning, it is certainly a privilege to pursue it as a career.
    From the time I entered UNIST as an undergraduate in 2015 to the time I graduate UNIST with my Ph.D. now in 2024, almost a decade has been passed.
    Without the support of UNIST, Ulsan \& Ulju-gun, and the Republic of Korea, my country, I could not successfully complete this journey.
    
    The discussions about Zen and Buddhism that I had around the dining table at my parents' home, 
    as well as the experience of meditating taught by mom and dad to reach the deep and peaceful state of mind,
    made me strong and resilient even in the extremely stressful situations.
    I deeply appreciate my parents and my one and only brother, Iiju, for their unconditional love and wholehearted belief in and support of me.  
    
    I am deeply thankful to my advisor, Prof. Junghye Lee, for her guidance and support throughout my Ph.D. studies.
    She introduced me to federated learning research and guided me with great enthusiasm, patience, and encouragement.
    I believe that what I have learned from my advisor will certainly be helpful in my future career and life.
    I am also grateful to my co-advisor, Prof. Gi-Soo Kim, for sincerely guiding and supporting my research, as well as accommodating a seat for me in her lab.
    Without this consideration, I would be lost and struggle at the end of my Ph.D. period.
    Furthermore, thanks to committee members, Prof. Sung Whan Yoon, Prof. Dongyoung Lim, Prof. Saerom Park, and Prof. Sungbin Lim, for their keen, insightful, and warm comments on my works. 

    I am sincerely happy to have shared my time with nice and remarkable colleagues in the Data Mining Lab. --- Prof. Suhyeon Kim, Myeonghoon Lee, Wonho Sohn, Dongcheol Lim, and Hyerim Park. 
    I can never forget their warm words, emotional support, and joyful memories with them.
    I would also like to sincerely thank for all the people in the Statistical Decision Making Lab. for respecting me and getting along with me in short time.
    Thanks to mates from other labs and other major --- Gyeongho Kim, Jihyun Kim, Jae-Jun Lee, Jun Hwi min, Minsub Jeong, and Hoichang Jeong, I could often refresh my Ph.D. life by chewing the fat, eating tasty food, and hanging out with them.
    
    In particular, I have been extremely fortunate to know and to have a chance to work with Yoontae Hwang and Jaeho Kim, precious and wonderful colleagues who have always inspired me with their great passion, minds, and grit. They have also brought so many laughs into my Ph.D. life, and I genuinely wish them all the best, from the bottom of my heart. I hope to meet these guys in the future at Kyunghee restaurant (or Ilpoom eels restaurant if we all get success), as we did in the past. 
    
    Last but not least, I am incredibly a lucky guy to have an amazing and lovely soul mate, JaeGyeong Choi.
    She is always there to pat me on the back when I get dispirited and to congratulate me before anyone else.
    We have shared countless experiences together, from an unplanned picnic in Seonbawi, swimming in the middle of the ocean in Okinawa, watching Netflix with incessant giggling, and going around every corner of Ulsan.
    Without her unlimited love and unwavering support, I would definitely not be able to complete my Ph.D. successfully.
\clearpage

\hbox{ }
\thispagestyle{empty}
\clearpage
\end{document}